\documentclass[twoside]{article}

\usepackage[accepted]{aistats2021}
\usepackage{adjustbox}

\usepackage{hyperref}       % hyperlinks
\usepackage{url}            % simple URL typesetting
\usepackage{booktabs}       % professional-quality tables
\usepackage{amsfonts}       % blackboard math symbols
\usepackage{nicefrac}       % compact symbols for 1/2, etc.
\usepackage{microtype}      % microtypography

\usepackage{amssymb, amsmath, amsthm, latexsym}
\usepackage{color}

\usepackage{colortbl}
\definecolor{bgcolor}{rgb}{0.8,1,1}
\definecolor{bgcolor2}{rgb}{0.8,1,0.8}

\usepackage{url}
\usepackage{algorithm}
\usepackage{algpseudocode}
\usepackage{tabularx}
\usepackage{paralist}
\usepackage{mathtools}
\usepackage{makecell}

\newcommand{\mytextstyle}{\textstyle}

\newcommand{\myred}[1]{{\color{red}#1}}
\newcommand{\myblue}[1]{{\color{blue}#1}}

\usepackage[numbers]{natbib}
\usepackage{sidecap}

\usepackage[dvipsnames]{xcolor}
\usepackage{pifont}

\usepackage{subfigure}

\newcommand{\R}{\mathbb{R}}
\newcommand{\eqdef}{\stackrel{\text{def}}{=}}

\def\<#1,#2>{\left\langle #1,#2\right\rangle}

\usepackage{thmtools}

\newtheorem{lemma}{Lemma}[section]
\newtheorem{theorem}{Theorem}[section]
\newtheorem{definition}{Definition}[section]
\newtheorem{proposition}{Proposition}[section]
\newtheorem{assumption}{Assumption}[section]
\newtheorem{corollary}{Corollary}[section]
\theoremstyle{plain}

\newtheorem{remark}{Remark}[section]

\declaretheorem[sibling=definition]{example}

\usepackage[colorinlistoftodos,bordercolor=orange,backgroundcolor=orange!20,linecolor=orange,textsize=scriptsize]{todonotes}

\newcommand{\cD}{{\cal D}}

\newcommand{\cL}{{\cal L}}

\newcommand{\cN}{{\cal N}}
\newcommand{\cO}{{\cal O}}

\newcommand{\Var}{\mathrm{Var}}

\newcommand{\mI}{{\bf I}}

\newcommand{\mM}{{\bf M}}

\newcommand{\EE}{\mathbf{E}}

\newcommand{\psvrg}{q}

\newcommand{\tA}{\widetilde{A}}
\newcommand{\hA}{\widehat{A}}
\newcommand{\tB}{\widetilde{B}}
\newcommand{\hB}{\widehat{B}}
\newcommand{\tF}{\widetilde{F}}
\newcommand{\hF}{\widehat{F}}
\newcommand{\tD}{\widetilde{D}}
\newcommand{\hD}{\widehat{D}}

\newcommand{\txi}{\widetilde{\xi}}
\newcommand{\oxi}{\overline{\xi}}

\usepackage{hyperref}
\graphicspath{{plots/}}

\usepackage{accents}
\newlength{\dhatheight}

\begin{document}



\twocolumn[

\aistatstitle{Local SGD: Unified Theory and New Efficient Methods}

\aistatsauthor{ Eduard Gorbunov \And Filip Hanzely \And  Peter Richt\'{a}rik }

\aistatsaddress{\makecell{MIPT, Yandex, Sirius, Russia\\ KAUST, Saudi Arabia} \And  KAUST, Saudi Arabia \And KAUST, Saudi Arabia} ]

\begin{abstract}
We present a unified framework for analyzing local {\tt SGD} methods in the convex and strongly convex regimes for distributed/federated training of supervised machine learning models. We recover several known methods as a special case of our general framework, including {\tt Local-SGD}/{\tt FedAvg}, {\tt SCAFFOLD}, and several variants of {\tt SGD} not originally designed for federated learning. Our framework covers both the identical and heterogeneous data settings, supports both random and deterministic number of local steps, and can work with a wide array of local stochastic gradient estimators, including shifted estimators which are able to adjust the fixed points of local iterations for faster convergence. As an application of our framework, we develop multiple novel FL optimizers which are superior to existing methods. In particular, we develop the first linearly converging local {\tt SGD} method which does not require any data homogeneity or other strong assumptions.
\end{abstract}

\section{Introduction}\label{sec:intro}
In this paper we are interested in a centralized distributed optimization problem of the form
\begin{equation}
\mytextstyle
	\min\limits_{x\in\R^d} f(x) = \frac{1}{n}\sum\limits_{i=1}^n f_i(x), \label{eq:main_problem}
\end{equation}
where $n$ is the number of devices/clients/nodes/workers. We assume that $f_i$ can be represented either as a) an expectation, i.e.,
\begin{equation}
\mytextstyle
	f_i(x) = \EE_{\xi_i\sim \cD_i}\left[f_{\xi_i}(x)\right], \label{eq:f_i_expectation}
\end{equation}
where $\cD_i$ describes the distribution of data on device $i$,  or b) as a finite sum, i.e.,
\begin{equation}
\mytextstyle
	f_i(x) = \frac{1}{m}\sum\limits_{j=1}^m f_{ij}(x). \label{eq:f_i_sum}
\end{equation} 
While our theory allows the number of functions $m$ to vary across the devices, for simplicity of exposition, we restrict the narrative to this simpler case.

Federated learning (FL)---an emerging subfield of machine learning \cite{mcmahan2016federated, konevcny2016federated, mcmahan2017communication}---is traditionally cast as an instance of problem~\eqref{eq:main_problem} with several idiosyncrasies. First, the number of devices $n$ is very large: tens of thousands to millions. Second,  the devices (e.g., mobile phones) are often very heterogeneous in their compute, connectivity, and storage capabilities. The data defining each function $f_i$ reflects the usage patterns of the device owner, and as such, it is either unrelated or at best related only weakly. Moreover, device owners desire to protect their local private data, and for that reason, training needs to take place with the data remaining on the devices.  Finally, and this is of key importance for the development in this work, communication among the workers, typically conducted via a trusted aggregation server, is very expensive. 

% At the other end of the spectrum lies the distributed optimization within a datacenter where one wishes to take advantage of the abundant computational power in the form of spare machines. In such a case, all workers create identical copies of the full dataset locally and jointly solve the resulting instance of~\eqref{eq:main_problem}.

\paragraph{Communication bottleneck.} There are two main directions in the literature for tackling the communication cost issue in FL. The first approach consists of algorithms that aim to reduce the number of transmitted bits by applying a  carefully chosen gradient compression scheme, such as quantization~\cite{alistarh2016qsgd, pmlr-v80-bernstein18a, mishchenko2019distributed, horvath2019stochastic,ramezani2019nuqsgd, reisizadeh2020fedpaq}, sparsification~\cite{aji2017sparse, lin2017deep, alistarh2018convergence, wangni2018gradient, wang2018atomo, mishchenko202099}, or other more sophisticated strategies~\cite{karimireddy2019error, stich2019error, pmlr-v80-wu18d, vogels2019powersgd, beznosikov2020biased,gorbunov2020linearly}. The second approach---one that we investigate in this paper---instead focuses on increasing the total amount of local computation in between the communication rounds in the hope that this will reduce the total number of communication rounds needed to build a model of sufficient quality~\cite{shamir2014communication,zhang2015disco,  reddi2016aide, li2018federated, pathak2020fedsplit}. These two approaches, {\em communication compression} and {\em local computation}, can be combined for a better practical performance~\cite{basu2019qsparse}.

\paragraph{Local first-order algorithms.} Motivated by recent development in the field~\cite{zinkevich2010parallelized, mcmahan2016federated, stich2018local, lin2018don,liang2019variance, wu2019federated, karimireddy2019scaffold, khaled2020tighter, woodworth2020local}, in this paper we perform an in-depth and general study of {\em local first-order algorithms}. Contrasted with zero or higher order local methods, local first order methods perform several gradient-type steps in between the communication rounds.  In particular, we consider the following family of methods:
\begin{equation}
	x_i^{k+1} = \begin{cases} x_i^k - \gamma g_i^k,& \text{if } c_{k+1} = 0,\\ \frac{1}{n}\sum\limits_{i=1}^n \left(x_i^k - \gamma g_i^k\right),& \text{if } c_{k+1} = 1, \end{cases} \label{eq:local_sgd_def}
\end{equation}
where $x_i^k$ represents the local variable maintained by the $i$-th device, $g_i^k$ represents local first order direction\footnote{Vector $g_i^k$ can be a simple unbiased estimator of $\nabla f_i(x_i^k)$, but  can also involve a local ``shift'' designed to correct the (inherently wrong) fixed point of local methods. We elaborate on this point later.} and (possibly random) sequence $\{c_{k}\}_{k\ge 1}$ with $c_k \in\{0,1\}$ encoding the times when communication takes place. 

Both the classical {\tt Local-SGD/FedAvg}~\cite{mcmahan2016federated, stich2018local, khaled2020tighter, woodworth2020local} and shifted local {\tt SGD}~\cite{liang2019variance, karimireddy2019scaffold} methods fall into this category of algorithms. However, most of the existing methods have been analyzed with limited flexibility only, leaving many potentially fruitful directions unexplored. The most important unexplored questions include i) better understanding of the local shift that aims to correct the fixed point of local methods, ii) support for more sophisticated local gradient estimators that allow for importance sampling, variance reduction, or coordinate descent, iii) variable number of local steps, and iv) general theory supporting multiple data similarity types, including identical, heterogeneous and partially heterogeneous ($\zeta$-heterogeneous - defined later). 

Consequently, there is a need for a single framework unifying the theory of local stochastic first order methods, ideally one capable of pointing to new and more efficient variants. This is what we do in this work.

\paragraph{Unification of stochastic algorithms.} There have been multiple recent papers aiming to unify the theory of first-order optimization algorithms. The closest to our work is the unification of (non-local) stochastic algorithms in~\cite{gorbunov2019unified} that proposes a relatively simple yet powerful framework for analyzing variants of {\tt SGD} that allow for minibatching, arbitrary sampling,\footnote{A tight convergence rate given any sampling strategy and any smoothness structure of the objective.} variance reduction, subspace gradient oracle, and quantization. We recover this framework as a special case in a non-local regime. Next, a framework for analyzing error compensated or delayed SGD methods was recently proposed in~\cite{gorbunov2020linearly}. Another relevant approach covers the unification of decentralized {\tt SGD} algorithms~\cite{koloskova2020unified}, which is able to recover the basic variant of {\tt Local-SGD} as well. While our framework matches their rate for basic {\tt Local-SGD}, we cover a broader range of local methods in this work as we focus on the centralized setting.

\subsection{Our Contributions}

In this paper, we propose a general framework for analyzing a broad family of local stochastic gradient methods of the form~\eqref{eq:local_sgd_def}. Given that a particular local algorithm satisfies a specific parametric assumption (Assumption~\ref{ass:key_assumption}) in a certain scenario, we provide a tight convergence rate of such a method. 

Let us give a glimpse of our results and their generality. A local algorithm of the  form~\eqref{eq:local_sgd_def} is allowed to consist of an {\em arbitrary} local stochastic gradient estimator (see Section~\ref{sec:local_solver} for details), a possible {\em drift/shift} to correct for the non-stationarity of local methods\footnote{Basic local algorithms such as {\tt FedAvg}/{\tt Local-SGD} or {\tt FedProx}~\cite{li2018federated} have incorrect fixed points~\cite{pathak2020fedsplit}. To eliminate this issue, a strategy of adding an extra ``drift'' or ``shift'' to the local gradient has been proposed recently~\cite{liang2019variance, karimireddy2019scaffold}.} and a fixed or random local loop size. Further, we provide a tight convergence rate in both the identical and heterogeneous data regimes for strongly (quasi) convex and convex objectives. Consequently, our framework is capable of:

$\bullet$  {\bf Recovering known optimizers along with their tight rates.} We recover multiple known local optimizers as a special case of our general framework, along with their convergence rates (up to small constant factors). This includes {\tt FedAvg/Local-SGD}~\cite{mcmahan2016federated, stich2018local} with currently the best-known convergence rate~\cite{khaled2020tighter, woodworth2020local, koloskova2020unified, woodworth2020minibatch} and {\tt SCAFFOLD}~\cite{karimireddy2019scaffold}. Moreover, in a special case we recover a general framework for analyzing non-local {\tt SGD} method developed in ~\cite{gorbunov2019unified}, and consequently we recover multiple variants of {\tt SGD} with and without variance reduction, including {\tt SAGA}~\cite{defazio2014saga}, {\tt L-SVRG}~\cite{kovalev2019don},  {\tt SEGA}~\cite{hanzely2018sega}, gradient compression methods~\cite{mishchenko2019distributed, horvath2019stochastic} and many more. 

$\bullet$  {\bf Filling missing gaps for known methods.} Many of the recovered optimizers have only been analyzed under specific and often limiting circumstances and regimes. Our framework allows us to extend known methods into multiple hitherto unexplored settings. For instance, for each (local) method our framework encodes, we allow for a random/fixed local loop size, identical/heterogeneous/$\zeta$-heterogeneous data (introduced soon), and convex/strongly convex objective.

$\bullet$  {\bf Extending the established optimizers.} To the best of our knowledge, none of the known local methods have been analyzed under arbitrary smoothness structure of the local objectives\footnote{By this we mean that function $f_{i,j}$ from~\eqref{eq:f_i_sum} is $\mM_{i,j}$-smooth with $\mM_{i,j}\in \R^{d\times d}, \mM_{i,j}\succeq 0$, i.e., for all $x,y\in  \R^d$ we have $f_{i,j}(x)\leq f_{i,j}(y) + \langle\nabla  f_{i,j}(y),x-y \rangle + \frac{1}{2} (x-y)^\top \mM_{i,j} (x-y)$. As an example, logistic regression possesses naturally such a structure with matrices $\mM_{i,j}$ of rank 1.} and consequently, our framework is the first to allow for the local stochastic gradient to be constructed via importance (possibly minibatch) sampling. Next, we allow for a local loop with a random length, which is a new development contrasting with the classical fixed-length regime. We discuss advantages of of the random loop in Section~\ref{sec:data_and_loop}.

$\bullet$  {\bf New efficient algorithms.} Perhaps most importantly, our framework is powerful enough to point to a range of novel methods. A notable example is {\tt S-Local-SVRG}, which is a local variance reduced {\tt SGD} method able to learn the optimal drift. This is the first time that local variance reduction is successfully combined with an on-the-fly learning of the local drift. Consequently, this is the first method which enjoys a linear convergence rate to the exact optimum (as opposed to a neighborhood of the solution only) without any restrictive assumptions and is thus superior in theory to the convergence of all existing local first order methods. We also develop another linearly converging method: {\tt S*-Local-SGD*}. Albeit not of practical significance as it depends on the a-priori knowledge of the optimal solution $x^*$, it is of theoretical interest as it enabled us to discover {\tt S-Local-SVRG}. See Table~\ref{tbl:special_cases} which summarizes all our complexity results.

\paragraph{Notation.} Due to its generality, our paper is heavy in notation. For the reader's convenience, we present a notation table in Sec.~\ref{sec:notation_table} of the appendix.

\section{Our Framework}\label{sec:main_res}
In this section we present the main result of the paper. Let us first introduce the key assumptions that we impose on our objective~\eqref{eq:main_problem}. We start with a relaxation of $\mu$-strong convexity.
\begin{assumption}[$(\mu,x^*)$-strong quasi-convexity]\label{ass:quasi_strong_convexity}
Let $x^*$ be a minimizer of $f$. We assume that  $f_i$ is $(\mu,x^*)$-strongly quasi-convex for all $i\in[n]$ with $\mu\geq 0$, i.e.\ for all $x\in\R^d$:
	\begin{equation}
	\mytextstyle	f_i(x^*) \ge f_i(x) + \langle\nabla f_i(x), x^* - x\rangle + \frac{\mu}{2}\|x - x^*\|^2. \label{eq:str_quasi_cvx}
	\end{equation}
\end{assumption}

Next, we require classical $L$-smoothness\footnote{While we require $L$-smoothness of $f_i$ to establish the main convergence theorem, some of the parameters of As.~\ref{ass:key_assumption} can be tightened considering a more complex smoothness structure of the local objective.} of local objectives, or equivalently, $L$-Lipschitzness of their gradients.
\begin{assumption}[$L$-smoothness]\label{ass:L_smoothness}
	Functions $f_i$ are $L$-smooth for all $i\in[n]$ with $L\geq 0$, i.e., 
	\begin{equation}
		\|\nabla f_i(x) - \nabla f_i(y)\| \le L\|x-y\|, \quad \forall x,y\in\R^d. \label{eq:L_smoothness}
	\end{equation}
\end{assumption}

In order to simplify our notation, it will be convenient to introduce the notion of virtual iterates $x^k$ defined as a mean of the local iterates~\cite{stich2019error}:
$
	\mytextstyle x^k \eqdef \frac{1}{n}\sum_{i=1}^n x_i^k. 
$
Despite the fact that $x^k$ is being physically computed only for $k$ for which  $c_k = 1$, virtual iterates are a very useful tool facilitating the convergence analysis. Next, we shall measure the discrepancy between the local and virtual iterates via the quantity $V_k$ defined as $
	\mytextstyle 	V_k \eqdef \frac{1}{n}\sum_{i=1}^n\|x_i^k - x^k\|^2. $

We are now ready to introduce the parametric assumption on both stochastic gradients $g_i^k$ and function $f$. This is a non-trivial generalization of the assumption from~\cite{gorbunov2019unified} to the class of local stochastic methods of the form~\eqref{eq:local_sgd_def}, and forms the heart of this work.\footnote{Recently, the assumption from~\cite{gorbunov2019unified} was generalized in a different way to cover  the class of the methods with error compensation and delayed updates~\cite{gorbunov2020linearly}.}

\begin{assumption}[Key parametric assumption]\label{ass:key_assumption}
	Assume that for all $k\ge 0$ and $i\in[n]$, local stochastic directions $g_i^k$ satisfy
	\begin{equation}
	\mytextstyle
		\frac{1}{n}\sum\limits_{i=1}^n\EE_k\left[g_i^k\right] = \frac{1}{n}\sum\limits_{i=1}^n\nabla f_i(x_i^k), \label{eq:unbiasedness}
	\end{equation}
	where $\EE_k[\cdot]$ defines the expectation w.r.t.\ randomness coming from the $k$-th iteration only. Further, assume that there exist  non-negative constants $A, A', B, B', C, C', F, F',G, H, D_1, D_1',D_2, D_3 \ge 0, \rho \in (0,1]$ and a sequence of (possibly random) variables $\{\sigma_k^2\}_{k\ge 0}$ such that
	\begin{align}
	\mytextstyle
		\frac{1}{n}\sum\limits_{i=1}^n\EE\left[\|g_i^k\|^2\right] \le & 2A\EE\left[f(x^k) - f(x^*)\right] + B\EE\left[\sigma_k^2\right] \nonumber \\
		& \quad + F\EE\left[V_k\right] + D_1, \label{eq:second_moment_bound}\\
		\mytextstyle
		\EE\left[\left\|\frac{1}{n}\sum\limits_{i=1}^ng_i^k\right\|^2\right] \le& 2A'\EE\left[f(x^k) - f(x^*)\right] + B'\EE\left[\sigma_k^2\right]  \nonumber \\
		& \quad + F'\EE\left[V_k\right] + D_1', \label{eq:second_moment_bound_2}\\
		\EE\left[\sigma_{k+1}^2\right] \le& (1-\rho)\EE\left[\sigma_k^2\right] + 2C\EE\left[f(x^k) - f(x^*)\right]  \nonumber \\
		& \quad + G\EE\left[V_k\right] + D_2,\label{eq:sigma_k+1_bound}\\
		\mytextstyle
		2L\sum\limits_{k=0}^K w_k\EE[V_k] \le& \mytextstyle \frac{1}{2}\sum\limits_{k=0}^Kw_k\EE\left[f(x^k) - f(x^*)\right]  \label{eq:sum_V_k_bounds}  \\
		&\mytextstyle \quad + 2LH\EE\sigma_0^2+ 2LD_3\gamma^2 W_K ,\nonumber
	\end{align}
	where sequences $\{W_K\}_{K\ge 0}$, $\{w_k\}_{k\ge 0}$ are defined as
	\begin{equation}
	\mytextstyle
		W_K \eqdef \sum\limits_{k=0}^K w_k,\quad w_k \eqdef \frac{1}{\left( 1 - \min\left\{\gamma\mu,\frac{\rho}{4}\right\}  \right)^{k+1}}, \label{eq:w_k_definition}
	\end{equation}

\end{assumption}

Admittedly, with its many parameters (whose meaning will become clear from the rest of the paper), As.~\ref{ass:key_assumption} is not easy to parse on first reading. Several comments are due at this point. First, while the complexity of this assumption may be misunderstood as being problematic, the opposite is true. This assumption enables us to prove a single theorem (Thm.~\ref{thm:main_result}) capturing the convergence behavior, in a tight manner, of all local first-order methods described by our framework \eqref{eq:local_sgd_def}. So, the parametric and structural complexity of this assumption is paid for by the unification aspect it provides. Second, for each specific method we consider in this work, we {\em prove} that As.~\ref{ass:key_assumption} is satisfied, and each such proof is based on much simpler and generally accepted assumptions. So, As.~\ref{ass:key_assumption} should be seen as a ``meta-assumption'' forming an intermediary and abstract step in the analysis, one revealing the structure of the inequalities needed to obtain a general and tight convergence result for local first-order methods. We dedicate the rest of the paper to explaining these parameters and to describing the algorithms and the associate rates their combination encodes. We are now ready to present our main convergence result.

% For now, it is enough to know that our goal is to give a tight convergence guarantee for scheme~\eqref{eq:local_sgd_def} given an arbitrary set of parameters of Assumption~\ref{ass:key_assumption}. Such a result is provided in Theorem~\ref{thm:main_result}. 

%		\gamma &\le& \min\left\{\frac{1}{2(\tau-1)\mu}, \frac{1}{2(\tau-1)\sqrt{\left(F+\frac{2BG}{\rho(1-\rho)}\right)}}, \frac{1}{4(\tau-1)\sqrt{2L\left(A + \frac{2BC}{\rho(1-\rho)}\right)}}\right\}.

\begin{theorem}\label{thm:main_result}
	Let As.~\ref{ass:quasi_strong_convexity},~\ref{ass:L_smoothness} and~\ref{ass:key_assumption} be satisfied and assume the stepsize satisfies $0<
		\gamma \le \min\left\{\frac{1}{2(A'+ \frac{4CB'}{3\rho})}, \frac{L}{F'+\frac{4GB'}{3\rho}}\right\}$. Define $\overline{x}^K \eqdef \frac{1}{W_K}\sum_{k=0}^K w_k x^k$, $\Phi^0 \eqdef \frac{2\|x^0 - x^*\|^2 +   \frac{8B'}{3\rho}\gamma^2 \EE\sigma_0^2 + 4LH\gamma\EE\sigma_0^2}{\gamma}$ and $\Psi^0 \eqdef 2\left(D_1' + \frac{4B'}{3\rho}D_2 + 2L\gamma D_3\right)$. Let $\theta \eqdef 1 - \min\left\{\gamma\mu,\frac{\rho}{4}\right\}$. Then if $\mu>0$, we have
	\begin{align}
		\EE\left[f(\overline{x}^K) \right] - f(x^*) \le& \theta^K\Phi^0 + \gamma \Psi^0, \label{eq:main_result_1}
	\end{align}
	and in the case when $\mu = 0$, we have
	\begin{align}
	\mytextstyle
		\EE\left[f(\overline{x}^K) \right] - f(x^*)\le& \mytextstyle\frac{\Phi^0}{ K}  + \gamma \Psi^0. \label{eq:main_result_2} 
	\end{align}
\end{theorem}

As already mentioned, Thm.~\ref{thm:main_result} serves as a general, unified theory for local stochastic gradient algorithms.  The strongly convex case provides a linear convergence rate up to a specific neighborhood of the optimum. On the other hand, the weakly convex case yields an $\cO(K^{-1})$ convergence rate up to a particular neighborhood. One might easily derive  $\cO(K^{-1})$ and $\cO(K^{-2})$ convergence rates to the exact optimum in the strongly and weakly convex case, respectively, by using a particular decreasing stepsize rule. The next corollary gives an example of such a result in the strongly convex scenario, where the estimate of $D_3$ does not depend on the stepsize $\gamma$. A detailed result that covers all cases is provided in Section~\ref{sec:corollaries} of the appendix.

\begin{corollary}\label{cor:main_complexity_cor_str_cvx}
Consider the setup from Thm.~\ref{thm:main_result} and by  $\frac1\nu$ denote  the resulting upper bound on $\gamma$.\footnote{In order to get tight estimate of $D_3$ and $H$, we will impose further bounds on $\gamma$ (see Tbl.~\ref{tbl:data_loop}). Assume that these extra bounds are included in parameter $h$.} Suppose that $\mu> 0$ and $D_3$ does not depend on $\gamma$. Let
\[
\mytextstyle
		\gamma = \min\left\{\frac{1}{\nu},  \frac{\ln\left(\max\left\{2,\min\left\{\frac{\Upsilon_1\mu^2K^2}{\Upsilon_2},\frac{\Upsilon_1\mu^3K^3}{\Upsilon_3}\right\}\right\}\right)}{\mu K}\right\},
\]
where $\Upsilon_1 = 2\|x^0 - x^*\|^2 +   \frac{8B'\EE\sigma_0^2}{3\nu^2\rho} + \frac{4LH\EE\sigma_0^2}{\nu}$, $\Upsilon_2 = 2D_1' + \frac{4B'D_2}{3\rho}$, $\Upsilon_3 = 4LD_3$. Then, the procedure~\eqref{eq:local_sgd_def} achieves $$\EE\left[f(\overline{x}^K) \right] - f(x^*)\le \varepsilon$$ as long as 
	\begin{equation*}
	\mytextstyle
		K\geq \widetilde\cO\left(\left(\frac{1}{\rho} + \frac{\nu}{\mu}\right)\log\left(\frac{\nu \Upsilon_1}{\varepsilon}\right) + \frac{\Upsilon_2}{\mu \varepsilon} + \sqrt{\frac{\Upsilon_3}{\mu^2 \varepsilon}}\right).
	\end{equation*}

\end{corollary}

\begin{remark}\label{rem:2nd_moment_decomposition}
Admittedly, Thm.~\ref{thm:main_result} does not yield the tightest known convergence rate in the heterogeneous setup under As.~\ref{thm:main_result}. Specifically, the neighborhood to which {\tt Local-SGD} converges can be slightly smaller~\cite{koloskova2020unified}. While we provide a tighter theory that matches the best-known results, we have deferred it to the appendix for the sake of clarity. In particular, to get the tightest rate, one shall replace the bound on the second moment of the stochastic direction~\eqref{eq:second_moment_bound} with two analogous bounds -- first one for the variance and the second one for the squared expectation. See As.~\ref{ass:hetero_second_moment} for details. Fortunately, Thm.~\ref{thm:main_result} does not need to change as it does not require parameters from~\eqref{eq:second_moment_bound}; these are only used later to derive $D_3, H, \gamma$ based on the data type. Therefore, only a few extra parameters should be determined in the specific scenario to get the tightest rate. 
\end{remark}

The parameters that drive both the convergence speed and the neighborhood size are determined by As.~\ref{ass:key_assumption}. In order to see through the provided rates, we shall discuss the value of these parameters in various scenarios. In general, we would like to have $\rho \in (0,1]$ as large as possible, while all other parameters are desired to be small so as to make the inequalities as tight as possible.

Let us start with studying data similarity and inner loop type as these can be decoupled from the type of the local direction that the method~\eqref{eq:local_sgd_def} takes.

\section{Data Similarity and Local Loop~\label{sec:data_and_loop}}
We now explain how our framework supports fixed and random local loop, and several data similarity regimes. 

\paragraph{Local loop.} Our framework supports \emph{local loop of a fixed length} $\tau \geq 1$ (i.e., we support local methods performing $\tau$ local iterations in between communications). This option, which is the de facto standard for local methods in theory and practice~\cite{mcmahan2016federated}, is recovered by setting $c_{a\tau } = 1$ for all non-negative integers $a$ and $c_k = 0$ for $k$ that are not divisible by $\tau$ in~\eqref{eq:local_sgd_def}. However, our framework also captures the very rarely considered \emph{local loop with a random length}. We recover this when $c_k$ are random samples from the Bernoulli distribution $\text{Be}(p)$ with parameter $p\in(0,1]$.

\paragraph{Data similarity.} We look at various possible data similarity regimes. The first option we consider is the fully heterogeneous setting where we do not assume any similarity between the local objectives whatsoever. Secondly, we consider the identical data regime with $f_1=\ldots=f_n$. Lastly, we consider the $\zeta$-heterogeneous data setting, which bounds the dissimilarity between the full and the local gradients~\cite{woodworth2020minibatch} (see Def.~\ref{def:zeta_hetero}).

\begin{definition}[$\zeta$-heterogeneous functions]\label{def:zeta_hetero}
	We say that functions $f_1,\ldots,f_n$ are $\zeta$-heterogeneous for some $\zeta\geq 0$ if the following inequality holds for all $x\in\R^d$:
	\begin{equation}
\mytextstyle
		\frac{1}{n}\sum\limits_{i=1}^n\|\nabla f_i(x) - \nabla f(x)\|^2 \le \zeta^2. \label{eq:bounded_data_dissimilarity}
	\end{equation}
\end{definition}

The $\zeta$-heterogeneous data regime recovers the heterogeneous data for $\zeta =\infty$ and identical data for $\zeta = 0$.

In Sec.~\ref{sec:a_data_and_loop} of the appendix, we show that  the local loop type and the data similarity type affect parameters $H$ and $D_3$ from As.~\ref{ass:key_assumption} only. However, in order to obtain an efficient bound on these parameters, we impose additional constraints on the stepsize $\gamma$. While we do not have space to formally state our results in the main body, we provide a comprehensive summary in Tbl.~\ref{tbl:data_loop}.

\begin{table*}[!t]
\caption{The effect of data similarity and local loop on As.~\ref{ass:key_assumption}. Constant factors are ignored. Homogeneous data are recovered as a special case of  $\zeta$-heterogeneous data with $\zeta=0$. Heterogeneous case is slightly loose in light of Remark~\ref{rem:2nd_moment_decomposition}. If one replaces the bound on the second moments~\eqref{eq:second_moment_bound} with a analogous bound on variance squared expectation (see As.~\ref{ass:hetero_second_moment}), the bounds on $\gamma$, $D_3$ and $H$ will have $(\tau-1)$ times better dependence on the variance parameters (or $\frac{1-p}{p}$ times for the random loop). See Sec.~\ref{sec:const_loop_hetero} and~\ref{sec:random_local_loop_hetero} of appendix for more details.}
\label{tbl:data_loop}
\begin{center}
\footnotesize
\begin{tabular}{|c|c|c|c|c|}
\hline
 Data &   Loop & Extra upper bounds on $\gamma$ & $D_3$ & $H$   \\
\hline
\hline
het  & fixed & {   $\frac{1}{\tau\mu}, \frac{1}{\tau\sqrt{\left(F+\frac{BG}{\rho(1-\rho)}\right)}}, \frac{1}{\tau\sqrt{2L\left(A + \frac{BC}{\rho(1-\rho)}\right)}}$} & {  $ (\tau-1)^2\left(D_1 + \frac{BD_2}{\rho}\right)$} & {  $\frac{B(\tau-1)^2\gamma^2}{\rho}$}  \\
 \hline
 $\zeta$-het  & fixed & {   $\frac{1}{\tau\mu}, \frac{1}{\sqrt{\tau\left(F+\frac{BG}{\rho(1-\rho)}\right)}}, \frac{1}{\sqrt{L\tau\left(A + \frac{BC}{\rho(1-\rho)}\right)}}$} & {  $  (\tau-1)\left(D_1 + \frac{\zeta^2}{\gamma\mu} + \frac{BD_2}{\rho}\right)$} & {  $\frac{B(\tau-1)\gamma^2}{\rho}$}  \\
 \hline
 het  & random & { $\frac{p}{\mu}$,   $\frac{p}{\sqrt{(1-p)F}}, \frac{p\sqrt{\rho(1-\rho)}}{\sqrt{BG(1-p)}}, \frac{p}{\sqrt{L(1-p)\left(A + \frac{BC}{\rho(1-\rho)}\right)}}$} & {  $ \frac{(1-p)\left(D_1 + \frac{BD_2}{\rho}\right) }{p^2}$} & {  $\frac{B(1-p)\gamma^2}{p^2\rho}$}  \\
 \hline
 $\zeta$-het  & radnom & {  $\frac{p}{\mu}$,  $\sqrt{\frac{p}{F(1-p)}}, \sqrt{\frac{p\rho(1-\rho)}{BG(1-p)}}, \sqrt{\frac{p}{L(1-p)\left(A + \frac{BC}{\rho(1-\rho)}\right)}}$} & {  $ \frac{(1-p)}{p}\left(D_1 + \frac{\zeta^2}{\gamma\mu} + \frac{BD_2}{\rho}\right)$} & {  $\frac{B(1-p)\gamma^2}{p\rho}$}  \\
 \hline
\end{tabular}
\end{center}
\end{table*}

Methods with a random loop communicate once per $p^{-1}$ iterations on average, while the fixed loop variant communicates once every $\tau$ iterations. Consequently, we shall compare the two loop types for $\tau= p^{-1}$. In such a case, parameters $D_3$ and $H$ and the extra conditions on stepsize $\gamma$ match exactly, meaning that the loop type does not influence the convergence rate. 
Having said that, random loop choice provides more flexibility compared to the fixed loop. Indeed, one might want the local direction $g_i^k$ to be synchronized with the communication time-stamps in some special cases. However, our framework does not allow such synchronization for a fixed loop since we assume that the local direction $g_i^k$ follows some stationary distribution over stochastic gradients. The random local loop comes in handy here; the random variable that determines the communication follows a stationary distribution, thus possibly synchronized with the local computations.

\section{Local Stochastic Direction \label{sec:local_solver}}

This section discusses how the choice of $g_i^k$ allows us to obtain the remaining parameters from As.~\ref{ass:key_assumption} that were not covered in the previous section. To cover the most practical scenarios, we set $g_i^k$ to be a difference of two components $a_i^k,b_i^k\in \R^d$, which we explain next. We stress that the construction of $g_i^k$ is very general: we recover various state-of-the-art methods along with their rates while covering many new interesting algorithms. We will discuss this in more detail in Sec.~\ref{sec:special_cases}.

\subsection{Unbiased local gradient estimator $a_i^k$}
The first component of the local direction that the method~\eqref{eq:local_sgd_def} takes is $a_i^k$ -- an unbiased, possibly variance reduced, estimator of the local gradient, i.e., $\EE_k[a_i^k] = \nabla f_i(x^k_i)$. Besides the unbiasedness, $a_i^k$ is allowed to be anything that satisfies the parametric recursive relation from~\cite{gorbunov2019unified}, which tightly covers many variants of {\tt SGD} including non-uniform, minibatch, and variance reduced stochastic gradient. The parameters of such a relation are capable of encoding both the general smoothness structure of the objective and the gradient estimator's properties that include a diminishing variance, for example. We state the adapted version of this recursive relation as As.~\ref{ass:sigma_k_original}. 

\begin{assumption} \label{ass:sigma_k_original}
Let the unbiased local gradient estimator $a_i^k$ be such that
\begin{align*}
&		\EE_k\left[\|a_i^k- \nabla f_i(x^*)\|^2 \right] \le 2A_iD_{f_i} (x^k_i, x^*)+ B_i\sigma_{i,k}^2 + D_{1,i}, 
		\\
&		\EE_k\left[\sigma_{i,k+1}^2\right] \le (1-\rho_i)\sigma_{ik}^2 + 2C_iD_{f_i} (x^k_i, x^*)  +  D_{2,i}
	\end{align*}
for $A_i\geq0, B_i\geq0, D_{1,i} \geq0, 0\leq \rho_i \leq 1,C_i \geq0, D_{2,i} \geq0$ and a non-negative sequence $ \{\sigma^2_{i,k}\}_{k=0}^\infty$.\footnote{By $D_{f_i}(x_i^k, x^k)$ we mean Bregman distance between $x_i^k, x^k$ defined as $D_{f_i}(x_i^k, x^k) \eqdef f_i(x_i^k) - f_i(x^k) - \langle \nabla f_i(x^k), x_i^k - x^k\rangle$.}
\end{assumption}

Note that the parameters of As.~\ref{ass:sigma_k_original} can be taken directly from~\cite{gorbunov2019unified} and offer a broad range of unbiased local gradient estimators $a_i^k$ in different scenarios. The most interesting setups covered include minibatching, importance sampling, variance reduction, all either under the classical smoothness assumption or under a uniform bound on the stochastic gradient variance. 

Our next goal is to derive the parameters of As.~\ref{ass:key_assumption} from the parameters of As.~\ref{ass:sigma_k_original}. However, let us first discuss the second component of the local direction -- the local shift $b_i^k$.

\subsection{Local shift $b_i^k$} 

The local update rule~\eqref{eq:local_sgd_def} can include the local shift/drift $b_i^k$ allowing us to eliminate the infamous non-stationarity of the local methods. The general requirement for the choice of $b_i^k$ is so that it sums up to zero ($\sum_{i=1}^n b_i^k = 0$) to avoid unnecessary extra bias. For the sake of simplicity (while maintaining generality), we will consider three choices of $b_i^k$ -- zero, ideal shift ($=\nabla f_i(x^*)$) and on-the-fly shift via a possibly outdated local stochastic non-variance reduced gradient estimator that satisfies a similar bound as As.~\ref{ass:sigma_k_original}.

\begin{assumption} \label{ass:bik}
Consider the following choices:\\
Case I: $b_i^k = 0 $, \\
Case II: $b_i^k = \nabla f_i(x^*)  $, \\
Case III: $b_i^k = h_i^k  - \frac1n \sum_{i=1}^n h_i^k$ where $h_i^k\in \R^d$ is a delayed local gradient estimator defined recursively as
$$
h_i^{k+1} = \begin{cases}
h_i^k & \text{with probability } 1-\rho_i' \\
l_i^k &  \text{with probability } \rho_i' \
\end{cases},
$$
 where $0\leq \rho'_i \leq 1$ and $l_i^k\in \R^d$ is an unbiased non-variance reduced possibly stochastic gradient estimator of $\nabla f_i(x^k)$ such that for some $A'_i, D_{3,i}\geq 0$ we have
\begin{equation}\label{eq:bdef}
\EE_k\left[\|l_i^k- \nabla f_i(x^*)\|^2 \right] \le 2A'_iD_{f_i} (x^k_i, x^*)+ D_{3,i}.
\end{equation}

\end{assumption}

Let us look closer at Case III as this one is the most interesting. Note that what we assume about $l_i^k$ (i.e., ~\eqref{eq:bdef}) is essentially a variant of As.~\ref{ass:bik} with $\sigma^2_{i,k}$ parameters set to zero. This is achievable for a broad range of non-variance reduced gradient estimators that includes minibatching and importance sampling~\cite{gower2019sgd}. An intuitive choice of $l_i^k$ is to set it to $a_i^k$ given that $a_i^k$ is not variance reduced. In such a case, the scheme~\eqref{eq:local_sgd_def} reduces to {\tt SCAFFOLD}~\cite{karimireddy2019scaffold} along with its rate.  

However, our framework can do much more beyond this example. First, we cover the local variance reduced gradient $a_i^k$ with $l_i^k$ constructed as its non-variance reduced part. In such a case, the neighborhood of the optimum from Thm.~\ref{thm:main_result} to which the method~\eqref{eq:local_sgd_def} converges shrinks. There is a way to get rid of this neighborhood, noticing that $l_i^k$ is used only once in a while. Indeed, the combination of the full local gradient $l_i^k$ together with the variance reduced $a_i^k$ leads to a linear rate in the strongly (quasi) convex case or $\cO(K^{-1})$ rate in the weakly convex case. We shall remark that the variance reduced gradient might require a sporadic computation of the full local gradient -- it makes  sense to synchronize it with the update rule for $h_i^k$. In such a case, the computation of $l_i^k$ is for free. We have just described the {\tt S-Local-SVRG} method (Algorithm~\ref{alg:l_local_svrg_fs}). 

\iffalse
\begin{remark} 
Suppose that the optimal solution $x^*$ minimizes all functions $f_i$ simultaneously, i.e., $\nabla f_i(x^*) = 0$ for all $1\leq i \leq n$. Then, Cases I and II coincide along with the corresponding parameters and rates. We can expect this to happen either in the homogeneous case, or in the interpolation regime.
\end{remark}
\fi 

\subsection{Parameters of Assumption~\ref{ass:key_assumption}}

We proceed with a key lemma that provides us with the remaining parameters of As.~\ref{ass:key_assumption} that were not covered in Sec.~\ref{sec:data_and_loop}. These parameters will be chosen purely based on the selection of $a_i^k$ and $b_i^k$ discussed earlier.

\begin{lemma}\label{lem:local_solver}
For all $i \in [n]$ suppose that $a_i^k$ satisfies As.~\ref{ass:sigma_k_original}, while $b_i^k$ was chosen as per As.~\ref{ass:bik}. Then,~\eqref{eq:second_moment_bound}, \eqref{eq:second_moment_bound_2} and \eqref{eq:sigma_k+1_bound} hold with
	\begin{align*}
	A &= 4\max_i A_{i}, B = 2, F = 4 L \max_i A_{i},\\ 
	D_1 &= \begin{cases}
	\frac{2}{n} \sum_{i=1}^n \left( D_{1,i}  + \| \nabla f_i(x^*)\|^2\right) & \text{Case I}, \\
	\frac{2}{n} \sum_{i=1}^n  D_{1,i} & \text{Case II, III},
	\end{cases}
	\\
 B' &\mytextstyle= \frac1n, F' = \frac{2 L \max_i A_{i}}{n}+2L^2, D_1' = \frac{1}{n^2} \sum_{i=1}^n D_{1,i} \\
  A' &\mytextstyle = \frac{2\max_i A_{i}}{n} + L , G = \nicefrac{CL}{2}, \\
	\rho &=\begin{cases} \min_i \rho_i  & \text{Case I, II}, \\
	  \min_i \min \left\{ \rho_i, \rho'_i\right\} & \text{Case III,}
	 \end{cases}    \\
	 D_2  &=\begin{cases}\frac2n\sum\limits_{i=1}^n B_iD_{2,i} , & \text{Case I, II}, \\
  \frac1n\sum\limits_{i=1}^n\left( 2B_iD_{2,i} + \rho_i' D_{3,i} \right)& \text{Case III,}
	 \end{cases}    \\
	 C &=\begin{cases}
	4 \max_{i}\{B_iC_i\}   & \text{Case I, II}, \\
		4 \max_{i}\{B_iC_i \} +4\max_i\{\rho_i' A'_i\}& \text{Case III}.
	\end{cases}
	\end{align*}

	\end{lemma}

We have just broken down the parameters of As.~\ref{ass:key_assumption} based on the optimization objective and the particular instance of~\eqref{eq:local_sgd_def}. However, it might still be hard to understand particular rates based on these choices. In the appendix, we state a range of methods and decouple their convergence rates. A summary of the key parameters from As.~\ref{ass:key_assumption} is provided in Tbl.~\ref{tbl:special_cases-parameters}.

\iffalse
We have just broken down the parameters of As.~\ref{ass:key_assumption} based on the optimization objective and the particular instance of local algorithm~\eqref{eq:local_sgd_def}.

 However, it might still be hard to see particular rates based on these choices. Let us start with the simplest example and show how we recover the rate of gradient descent. 

\begin{example} \label{ex:gd} Let $n=1$, $g_1^k = \nabla f_1(x^k)$ for all $k$ and consider a fixed local loop with length 1 (i.e., $\tau=2$). Then we have $a_i^k =  \nabla f_i(x^k)$, $b_i^k = 0$ and thus As.~\ref{ass:sigma_k_original} holds with $A_i =L, B_i=0, \sigma^2_{i,k} = 0, \rho_i=1, C_i=0, D_{2,i}=0$. Using Lem.~\ref{lem:local_solver}, As.~\ref{ass:sigma_k_original} holds with $A=6L, B=1, F=12 L^2, D_1 = 0, B' = \frac1n, F' = \frac{6L^2}{n}, D'_1 = 0, A' = \frac{3L}{n}, G = 0, \rho = 1, D_2 = 0$ and $C = 0$. Next, Tbl.~\ref{tbl:data_loop} yields $D_3 = H=0$ and thus $\gamma  = \cO(L^{-1})$ is supported both by Thm.~\ref{thm:main_result} and Tbl.~\ref{tbl:data_loop}. Therefore we recover $\cO\left( \frac{L}{\mu} \log \frac1\epsilon\right)$ complexity in the strongly quasi-convex case and $\cO\left( \frac{L}\epsilon\right)$ rate in the weakly convex regime.
\end{example} 

Example~\ref{ex:gd} serves as a basic demonstration of how to use our framework; we state a range of more complex methods in the appendix and decouple their convergence rates. We also provide a summary of the key parameters from As.~\ref{ass:key_assumption} in Table~\ref{tbl:special_cases-parameters}.

\fi

\begin{table*}[!t]
\caption{A selection of methods that can be analyzed using our framework, which we detail in the appendix. A choice of $a_i^k, b_i^k$ and $l_i^k$ is presented along with the established complexity bounds (= number of iterations to find such $\hat x$ that $\EE[f(\hat{x}) - f(x^*)]\le \varepsilon$) and a specific setup under which the methods are analyzed. For Algorithms 1-4 we suppress constants and $\log \tfrac{1}{\varepsilon}$ factors. Since Algorithms 5 and 6 converge linearly, we suppress constants only while keeping $\log \tfrac{1}{\varepsilon}$ factors. All rates are provided in the \textbf{strongly convex} setting. UBV stands for the ``Uniform Bound on the Variance'' of local stochastic gradient, which is often assumed when $f_i$ is of the form~\eqref{eq:f_i_expectation}. ES stands for the ``Expected Smoothness''~\cite{gower2019sgd}, which does not impose any extra assumption on the objective/noise, but rather can be derived given the sampling strategy and the smoothness structure of $f_i$. Consequently, such a setup allows us to obtain local methods with importance sampling. Next, the simple setting is a special case of ES when we uniformly sample a single index on each node each iteration. $^\clubsuit$: {\tt Local-SGD} methods have never been analyzed under ES assumption. Notation: $\sigma^2$ -- averaged (within nodes) uniform upper bound for the variance of local stochastic gradient, $\sigma_*^2$ -- averaged variance of local stochastic gradients at the solution, $\zeta_*^2 \eqdef \frac1n\sum_{i=1}^n \| \nabla f_i(x^*)\|^2$, $\max L_{ij}$ -- the worst smoothness of $f_{i,j}, i\in[n],j\in[m]$, $\cL$ -- the worst ES constant for all nodes.
}
\label{tbl:special_cases}
\begin{center}
\vskip -0.2cm
\footnotesize
\begin{tabular}{|c|c|c|c|c|c|c|c|}
\hline
{\bf Method} & {\bf  \# } &  {\bf Ref} &   $\myred{a_i^k}, \myblue{b_i^k}, l_i^k$   & {\bf Complexity}  &   {\bf Setting}  & {\bf Sec}  \\
\hline
\hline
 {\tt Local-SGD}  &   \ref{alg:local_sgd} & {\cite{woodworth2020minibatch}}  & $\myred{f_{\xi_i}(x_i^k)}, \myblue{0}, - $& $\frac{L}{\mu}+\frac{\sigma^2}{n\mu\varepsilon}+\sqrt{\frac{L\tau(\sigma^2 + \tau\zeta^2)}{\mu^2\varepsilon}}$ &  \begin{tabular}{c}
	UBV,\\
	$\zeta$-Het
\end{tabular} & \ref{sec:sgd_bounded_var}  \\
%%%%%%%%%%%%%%%%%%%%
%%%%%%%%%%%%%%%%%%%%
\hline
 {\tt Local-SGD}  &   \ref{alg:local_sgd} & {\cite{koloskova2020unified}}  & $\myred{f_{\xi_i}(x_i^k)},\myblue{0}, - $& $\frac{\tau L}{\mu}+\frac{\sigma^2}{n\mu\varepsilon}+\sqrt{\frac{L(\tau-1)(\sigma^2 + (\tau-1)\zeta_*^2)}{\mu^2\varepsilon}}$ &  \begin{tabular}{c}
	UBV,\\
	Het
\end{tabular} & \ref{sec:sgd_bounded_var}  \\
%%%%%%%%%%%%%%%%%%%%
%%%%%%%%%%%%%%%%%%%%
\hline
 {\tt Local-SGD}  &  \ref{alg:local_sgd} & {\cite{khaled2020tighter}}$^{\clubsuit}$  & $\myred{f_{\xi_i}(x_i^k)}, \myblue{0}, - $  & \begin{tabular}{c}
 $\frac{L+\nicefrac{\cL}{n}+\sqrt{(\tau-1) L\cL}}{\mu}+\frac{\sigma_*^2}{n\mu\varepsilon}\quad\quad$\\
 $+ \frac{L\zeta^2(\tau-1)}{\mu^2\varepsilon}+\sqrt{\frac{L(\tau-1)(\sigma_*^2 + \zeta_*^2)}{\mu^2\varepsilon}}$
\end{tabular}  &   \begin{tabular}{c}
	ES,\\
	$\zeta$-Het
\end{tabular} & \ref{sec:sgd_es}  \\
%%%%%%%%%%%%%%%%%%%%
%%%%%%%%%%%%%%%%%%%%
\hline
 {\tt Local-SGD}  &  \ref{alg:local_sgd} & {\cite{khaled2020tighter}}$^{\clubsuit}$  & $\myred{f_{\xi_i}(x_i^k)}, \myblue{0}, - $  & \begin{tabular}{c}
	$\frac{L\tau+\nicefrac{\cL}{n}+\sqrt{(\tau-1)L\cL}}{\mu}+\frac{\sigma_*^2}{n\mu\varepsilon}\quad\quad$\\
	$\quad\quad+\sqrt{\frac{L(\tau-1)(\sigma_*^2 + (\tau-1)\zeta_*^2)}{\mu^2\varepsilon}}$ 
\end{tabular}  &   \begin{tabular}{c}
	ES,\\
	Het
\end{tabular} & \ref{sec:sgd_es}  \\
%%%%%%%%%%%%%%%%%%%%
%%%%%%%%%%%%%%%%%%%%
\hline
\rowcolor{bgcolor} {\tt Local-SVRG}  &   \ref{alg:local_svrg} & NEW  & \begin{tabular}{c}
	\myred{$\nabla f_{i,j_i}(x^k_i) -  \nabla f_{i,j_i}(y_i^k)$}\\\myred{$ +  \nabla f_{i}(y_i^k)$},\\ \myblue{$0$}, $- $ 
\end{tabular}  & \begin{tabular}{c}
	 $m+\frac{L+\nicefrac{\max L_{ij}}{n}+\sqrt{(\tau-1) L\max L_{ij}}}{\mu}$\\
	 $\quad+ \frac{L\zeta^2(\tau-1)}{\mu^2\varepsilon}+\sqrt{\frac{L(\tau-1)\zeta_*^2}{\mu^2\varepsilon}}$
\end{tabular} & \begin{tabular}{c}
	simple,\\
	$\zeta$-Het
\end{tabular} & \ref{sec:llsvrg}   \\
%%%%%%%%%%%%%%%%%%%%
%%%%%%%%%%%%%%%%%%%%
\hline
\rowcolor{bgcolor} {\tt Local-SVRG}  &   \ref{alg:local_svrg} & NEW  & \begin{tabular}{c}
	\myred{$\nabla f_{i,j_i}(x^k_i) -  \nabla f_{i,j_i}(y_i^k)$}\\\myred{$ +  \nabla f_{i}(y_i^k)$},\\ \myblue{$0$}, $- $ 
\end{tabular}  & \begin{tabular}{c}
	 $m+\frac{L\tau+\nicefrac{\max L_{ij}}{n}+\sqrt{(\tau-1)L\max L_{ij}}}{\mu}$\\
	 $+\sqrt{\frac{L(\tau-1)^2\zeta_*^2}{\mu^2\varepsilon}}$
\end{tabular} & \begin{tabular}{c}
	simple,\\
	Het
\end{tabular} & \ref{sec:llsvrg}   \\
%%%%%%%%%%%%%%%%%%%%
%%%%%%%%%%%%%%%%%%%%
\hline
\rowcolor{bgcolor} {\tt S*-Local-SGD}  &  \ref{alg:local_sgd_star} & NEW  & $\myred{f_{\xi_i}(x_i^k)}, \myblue{\nabla f_i(x^*)}, - $ & $\frac{\tau L}{\mu}+\frac{\sigma^2}{n\mu\varepsilon}+\sqrt{\frac{L(\tau-1)\sigma^2}{\mu^2\varepsilon}}$  &  \begin{tabular}{c}
	UBV,\\
	Het
\end{tabular} & \ref{sec:sgd_star_bounded_var}  \\
%%%%%%%%%%%%%%%%%%%%
%%%%%%%%%%%%%%%%%%%%
\hline
 {\tt SS-Local-SGD}  &  \ref{alg:l_local_svrg} & \cite{karimireddy2019scaffold}  & \begin{tabular}{c}
$\myred{f_{\xi_i}(x_i^k)}, \myblue{h_i^k  - \frac1n \sum_{i=1}^n h_i^k},$\\
$ \nabla f_{\tilde{\xi}_i^k} (y_i^k)$ 
\end{tabular}  & $\frac{L}{p\mu}+\frac{\sigma^2}{n\mu\varepsilon}+\sqrt{\frac{L(1-p)\sigma^2}{p\mu^2\varepsilon}}$ &  \begin{tabular}{c}
	UBV,\\
	Het
\end{tabular} & \ref{sec:loopless_local_svrg}  \\
%%%%%%%%%%%%%%%%%%%%
%%%%%%%%%%%%%%%%%%%%
\hline
\rowcolor{bgcolor} {\tt SS-Local-SGD}  &  \ref{alg:l_local_svrg} & NEW  & \begin{tabular}{c}
$\myred{f_{\xi_i}(x_i^k)}, \myblue{h_i^k  - \frac1n \sum_{i=1}^n h_i^k},$\\
$ \nabla f_{\tilde{\xi}_i^k} (y_i^k)$ 
\end{tabular}  &  \begin{tabular}{c}
$\frac{L}{p\mu} + \frac{\cL}{n\mu} + \frac{\sqrt{L\cL(1-p)}}{p\mu}$\\ $+\frac{\sigma_*^2}{n\mu\varepsilon}+\sqrt{\frac{L(1-p)\sigma_*^2}{p\mu^2\varepsilon}}$
\end{tabular} &  \begin{tabular}{c}
	ES,\\
	Het
\end{tabular} & \ref{sec:loopless_local_svrg_es}  \\
%%%%%%%%%%%%%%%%%%%%
%%%%%%%%%%%%%%%%%%%%
\hline
\rowcolor{bgcolor2} {\tt S*-Local-SGD*}  &  \ref{alg:local_sgd_star_star} & NEW  & \begin{tabular}{c}
 	\myred{$\nabla f_{i,j_i}(x^k_i) -  \nabla f_{i,j_i}(x^*)$}\\
 	\myred{$ +  \nabla f_{i}(x^*)$}, $\myblue{\nabla f_i(x^*)}, - $\\ 	
 \end{tabular} &  \begin{tabular}{c}
$\Big(\frac{\tau L}{\mu}+\frac{\max L_{ij}}{n\mu}\quad\quad\quad\quad\quad\quad\quad\quad$\\ $\quad\quad\quad\quad+ \frac{\sqrt{(\tau-1)L\max L_{ij}}}{\mu}\Big)\log\frac{1}{\varepsilon}$ 
\end{tabular}   &  \begin{tabular}{c}
	simple,\\
	Het
\end{tabular} & \ref{sec:S*-Local-SGD*}  \\
%%%%%%%%%%%%%%%%%%%%
%%%%%%%%%%%%%%%%%%%%
\hline
\rowcolor{bgcolor2} {\tt S-Local-SVRG}  &   \ref{alg:l_local_svrg_fs} & NEW  &\begin{tabular}{c}
	$\myred{ \nabla f_{i,j_i}(x^k_i) -  \nabla f_{i,j_i}(y^k)}$\\
	\myred{$ +  \nabla f_{i}(y^k) $},\\
	$\myblue{h_i^k  - \frac1n \sum_{i=1}^n h_i^k}, \nabla f_{i} (y^k)$ 
\end{tabular} & \begin{tabular}{c}
$\Big(m + \frac{L}{p\mu}+\frac{\max L_{ij}}{n\mu}\quad\quad\quad\quad\quad\quad$\\ $\quad\quad\quad\quad+ \frac{\sqrt{L\max L_{ij}(1-p)}}{p\mu}\Big)\log\frac{1}{\varepsilon}$ 
\end{tabular} & \begin{tabular}{c}
	simple,\\
	Het
\end{tabular} &   \ref{sec:loopless_local_svrg_fs} \\
%%%%%%%%%%%%%%%%%%%%
%%%%%%%%%%%%%%%%%%%%
\hline
\end{tabular}
\end{center}
\vskip -0.2cm
\end{table*}

\section{Special Cases}\label{sec:special_cases}

Our theory covers a broad range of local stochastic gradient algorithms. While we are able to recover multiple known methods along with their rates, we also introduce several new methods along with extending the analysis of known algorithms. As already mentioned, our theory covers convex and strongly convex cases, identical and heterogeneous data regimes. From the algorithmic point of view, we cover the fixed and random loop, various shift types, and arbitrary local stochastic gradient estimator. We stress that our framework gives a tight convergence rate under any circumstances.

While we might not cover all of these combinations in a deserved detail, we thoroughly study a subset of them in Sec.~\ref{sec:special_appendix} of the appendix. An overview of these methods is presented in Tbl.~\ref{tbl:special_cases} together with their convergence rates in the strongly convex case (see Tbl.~\ref{tbl:special_cases_weakly_convex} in the appendix for the rates in the weakly convex setting). Next, we describe a selected number of special cases of our framework.

$\bullet$  {\bf Non-local stochastic methods.} Our theory recovers a broad range of non-local stochastic methods. In particular, if $n=1$, we have $V_k = 0$, and consequently we can choose $A=A', B=B', D_1=  D_1', F=F'=G= H=D_3=0$. With such a choice, our theory matches\footnote{Up to the non-smooth regularization/proximal steps and small constant factors.} the general analysis of stochastic gradient methods from~\cite{gorbunov2019unified} for $\tau=1$. Consequently, we recover a broad range of algorithms as a special case along with their convergence guarantees, namely {\tt SGD}~\cite{RobbinsMonro:1951} with its best-known rate on smooth objectives~\cite{nguyen2018sgd, gower2019sgd}, variance reduced finite sum algorithms such as {\tt SAGA}~\cite{defazio2014saga}, {\tt SVRG}~\cite{johnson2013accelerating},  {\tt L-SVRG}~\cite{hofmann2015variance, kovalev2019don}, variance reduced subspace descent methods such as {\tt SEGA/SVRCD}~\cite{hanzely2018sega, hanzely2019one}, quantized methods~\cite{mishchenko2019distributed, horvath2019stochastic} and others. 

$\bullet$  {\bf ``Star''-shifted local methods.} As already mentioned, local methods have inherently incorrect fixed points~\cite{pathak2020fedsplit}; and one can fix these by shifting the local gradients. Star-shifted local methods employ the ideal stationary shift using the local gradients at the optimum $b_i^k = \nabla f_i(x^*)$ (i.e.,  Case II from As.~\ref{ass:bik}) and serve as a transition from the plain local methods (Case I from As.~\ref{ass:bik}) to the local methods that shift using past gradients such as {\tt SCAFFOLD} (Case III from As.~\ref{ass:bik}). In the appendix, we present two such methods: {\tt S*-Local-SGD} (Algorithm~\ref{alg:local_sgd_star}) and  {\tt S*-Local-SGD*} (Algorithm~\ref{alg:local_sgd_star_star}).  While being impractical in most cases since $\nabla f_i(x^*)$ is not known, star-shifted local methods give new insights into the role and effect of the shift for local algorithms. Specifically, these methods enjoy superior convergence rate when compared to methods without local shift (Case I) and methods with a shift constructed from observed gradients (Case III), while their rate serves as an aspiring goal for local methods in general. Fortunately, in several practical scenarios, one can match the rate of star methods using an approach from Case III, as we shall see in the next point. 

$\bullet$  {\bf Shifted Local {\tt SVRG} ({\tt S-Local-SVRG}).} As already mentioned, local {\tt SGD} suffers from convergence to a neighborhood of the optimum only, which is credited to i) inherent variance of the local stochastic gradient, and ii) incorrect fixed point of local {\tt GD}. We propose a way to correct both issues. To the best of our knowledge, this is the first time that on-device variance reduction was combined with the trick for reducing the non-stationarity of local methods. Specifically, the latter is achieved by selecting $b_i^k$ as a particular instance of Case III from As.~\ref{ass:bik} such that $l_i^k$ is the full local gradient, which in turns yields $D'_{1,i}=0, A'_i = L$. In order to not waste local computation, we synchronize the evaluation of $l_i^k$ with the computation of the full local gradient for the {\tt L-SVRG}~\cite{hofmann2015variance, kovalev2019don} estimator, which we use to construct $a_i^k$. Consequently, some terms cancel out, and we obtain a simple, fast, linearly converging local {\tt SGD} method, which we present as Algorithm~\ref{alg:l_local_svrg_fs} in the appendix. We believe that this is remarkable since only a very few local methods converge linearly to the exact optimum.\footnote{A linearly converging local {\tt SGD} variant can be recovered from stochastic decoupling~\cite{mishchenko2019stochastic}, although this was not considered therein. Besides that, FedSplit~\cite{pathak2020fedsplit} achieves a linear rate too, however, with a much stronger local oracle.}

\section{Experiments} \label{sec:exp}
We perform multiple experiments to verify the theoretical claims of this paper. Due to space limitations, we only present a single experiment in the main body; the rest can be found in Section~\ref{sec:extra_exp} of the appendix.

We demonstrate the benefit of on-device variance reduction, which we introduce in this paper. For that purpose, we compare standard {\tt Local-SGD} (Algorithm~\ref{alg:local_sgd}) with our {\tt Local-SVRG} (Algorithm~\ref{alg:local_svrg}) on a regularized logistic regression problem with LibSVM data~\cite{chang2011libsvm}. For each problem instance, we compare the two algorithms with the stepsize $\gamma\in \{1, 0.1, 0.01 \}$ (we have normalized the data so that $L=1$).  The remaining details for the setup are presented in Section~\ref{sec:extralog} of the appendix.

Our theory predicts that both {\tt Local-SGD} and {\tt Local-SVRG} have identical convergence rate early on. However, the neighborhood of the optimum to which {\tt Local-SVRG} converges is smaller comparing to {\tt Local-SGD}. For both methods, the neighborhood is controlled by the stepsize: the smaller the stepsize is, the smaller the optimum neighborhood is. The price to pay is a slower rate at the beginning.

The results are presented in Fig.~\ref{fig:sgd_svrg_hom}. As predicted, {\tt Local-SVRG} always outperforms {\tt Local-SGD} as it converges to a better neighborhood. Fig.~\ref{fig:sgd_svrg_hom} also demonstrates that one can trade the smaller neighborhood for the slower convergence by modifying the stepsize.

\begin{figure}[!h]
\centering
\begin{minipage}{0.5\textwidth}
\begin{minipage}{0.49\textwidth}
  \centering
\includegraphics[width =  \textwidth ]{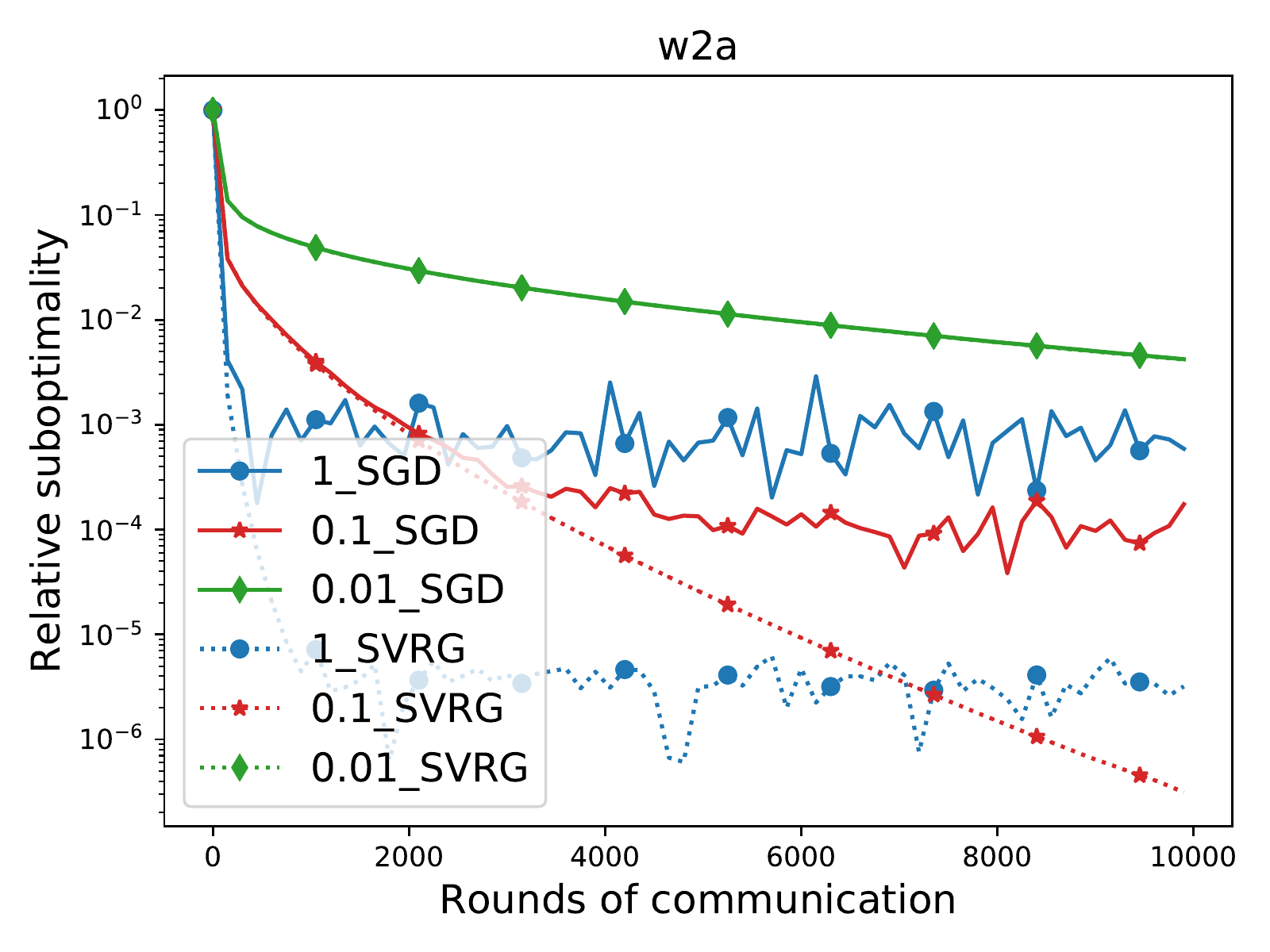}
        %\caption{ Residual vs. iteration  }\label{fig:bl_ex_flops}
\end{minipage}
\begin{minipage}{0.49\textwidth}
  \centering
\includegraphics[width =  \textwidth ]{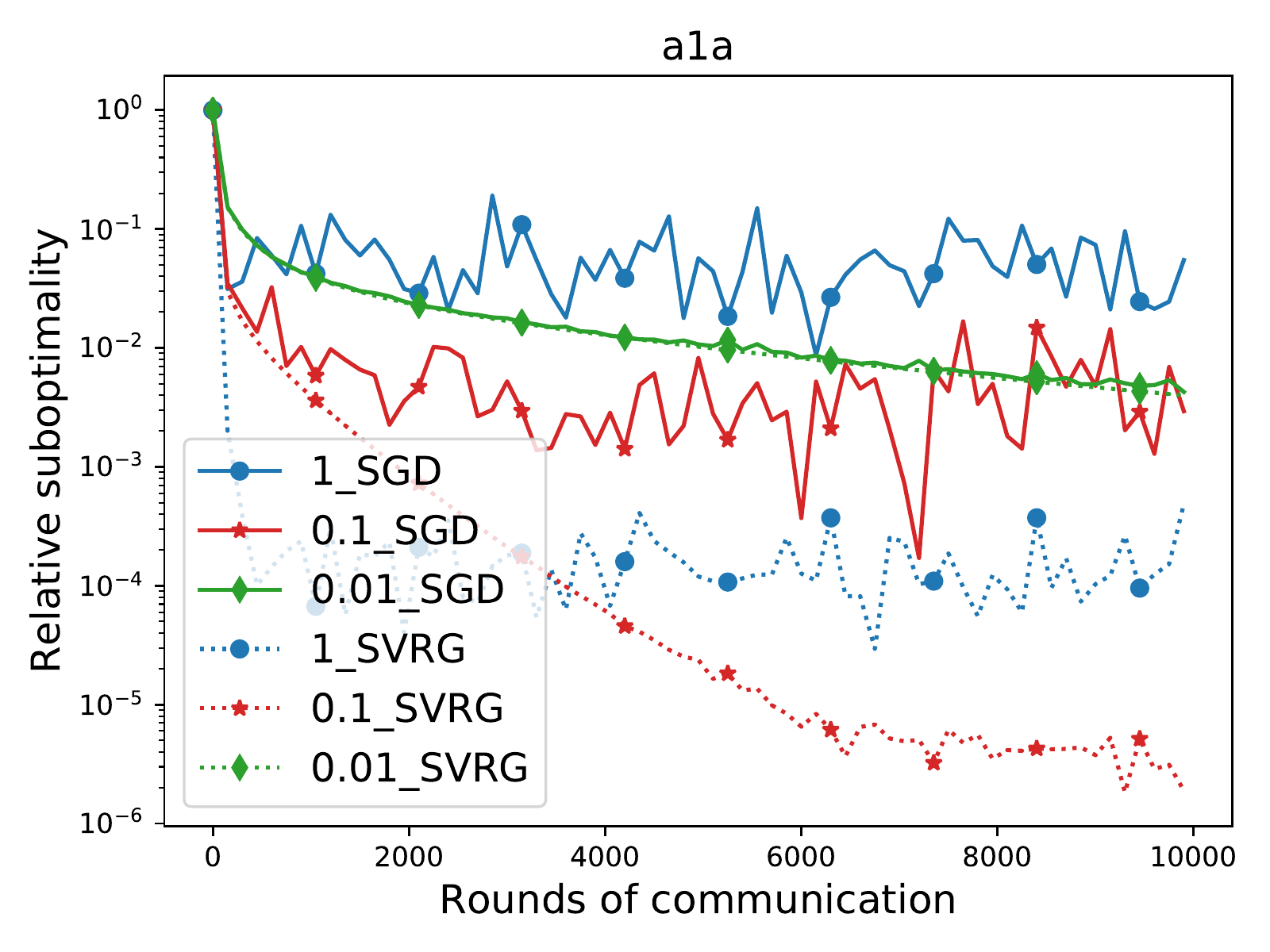}
        %\caption{ Residual vs. iteration  }\label{fig:bl_ex_flops}
\end{minipage}
\end{minipage}
\\
\begin{minipage}{0.5\textwidth}
\begin{minipage}{0.49\textwidth}
  \centering
\includegraphics[width =  \textwidth ]{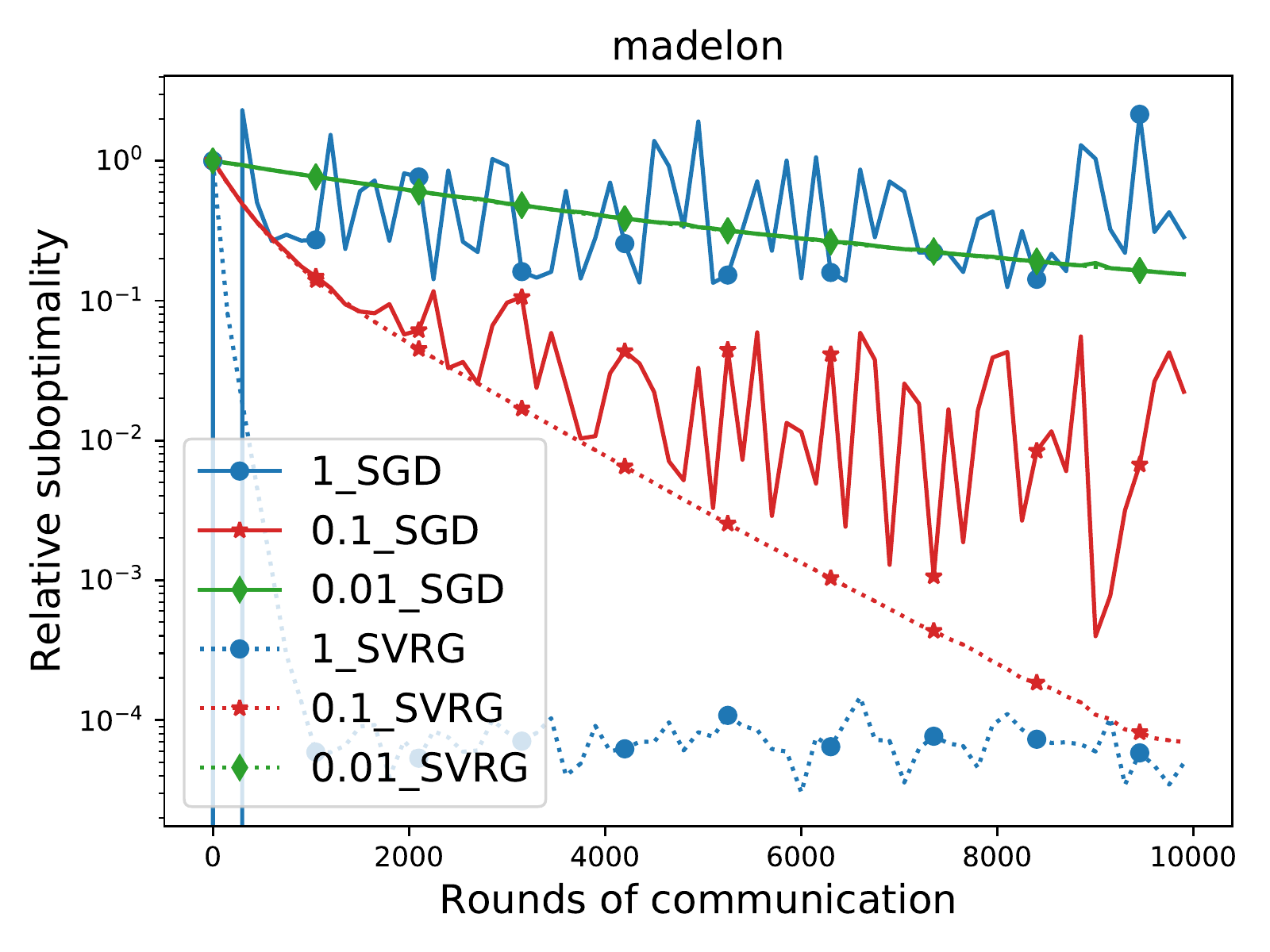}
        %\caption{ Residual vs. iteration  }\label{fig:bl_ex_flops}
\end{minipage}
\begin{minipage}{0.49\textwidth}
  \centering
\includegraphics[width =  \textwidth ]{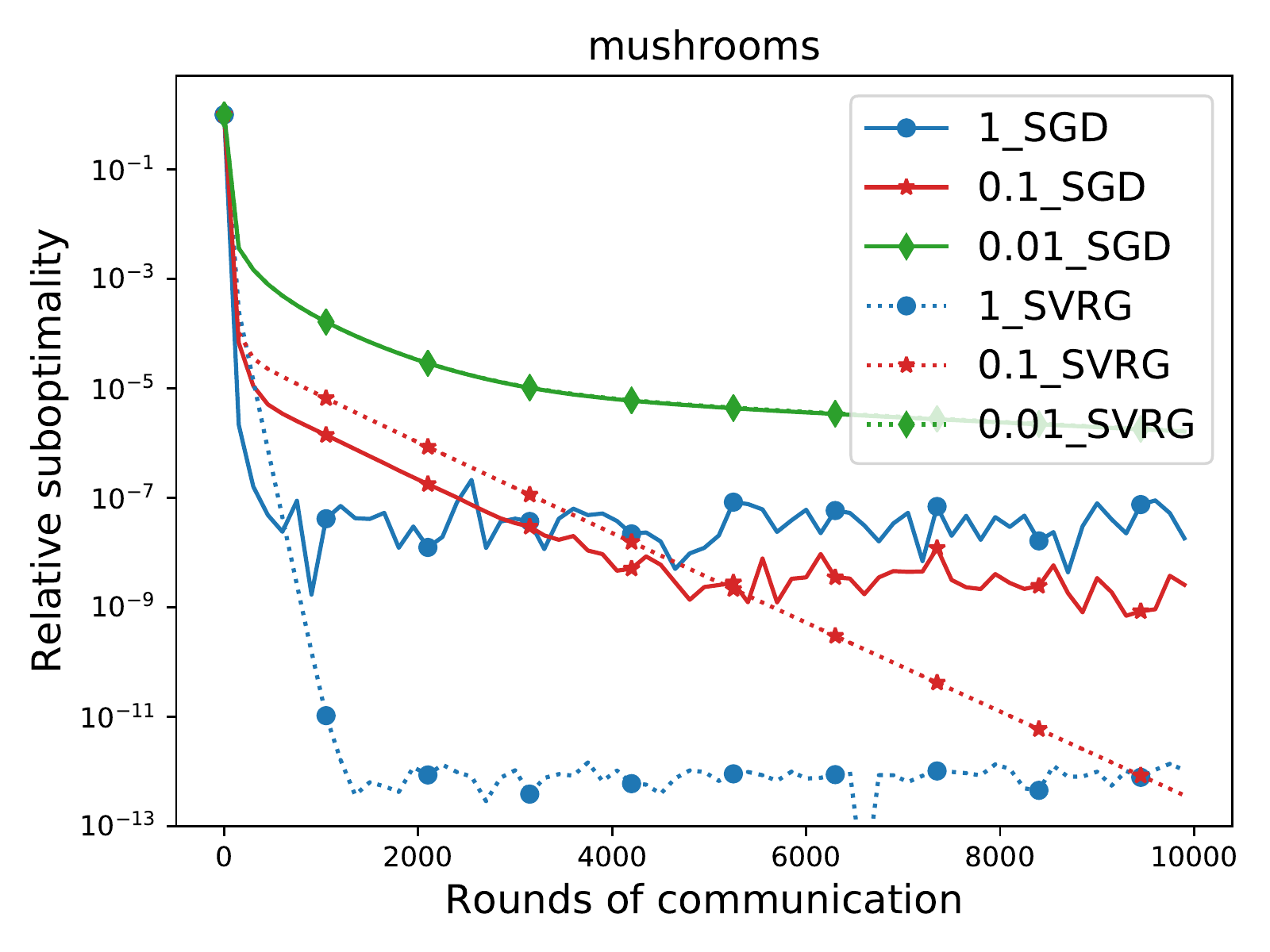}
        %\caption{ Residual vs. iteration  }\label{fig:bl_ex_flops}
\end{minipage}
\end{minipage}
\\
\begin{minipage}{0.5\textwidth}
\begin{minipage}{0.49\textwidth}
  \centering
\includegraphics[width =  \textwidth ]{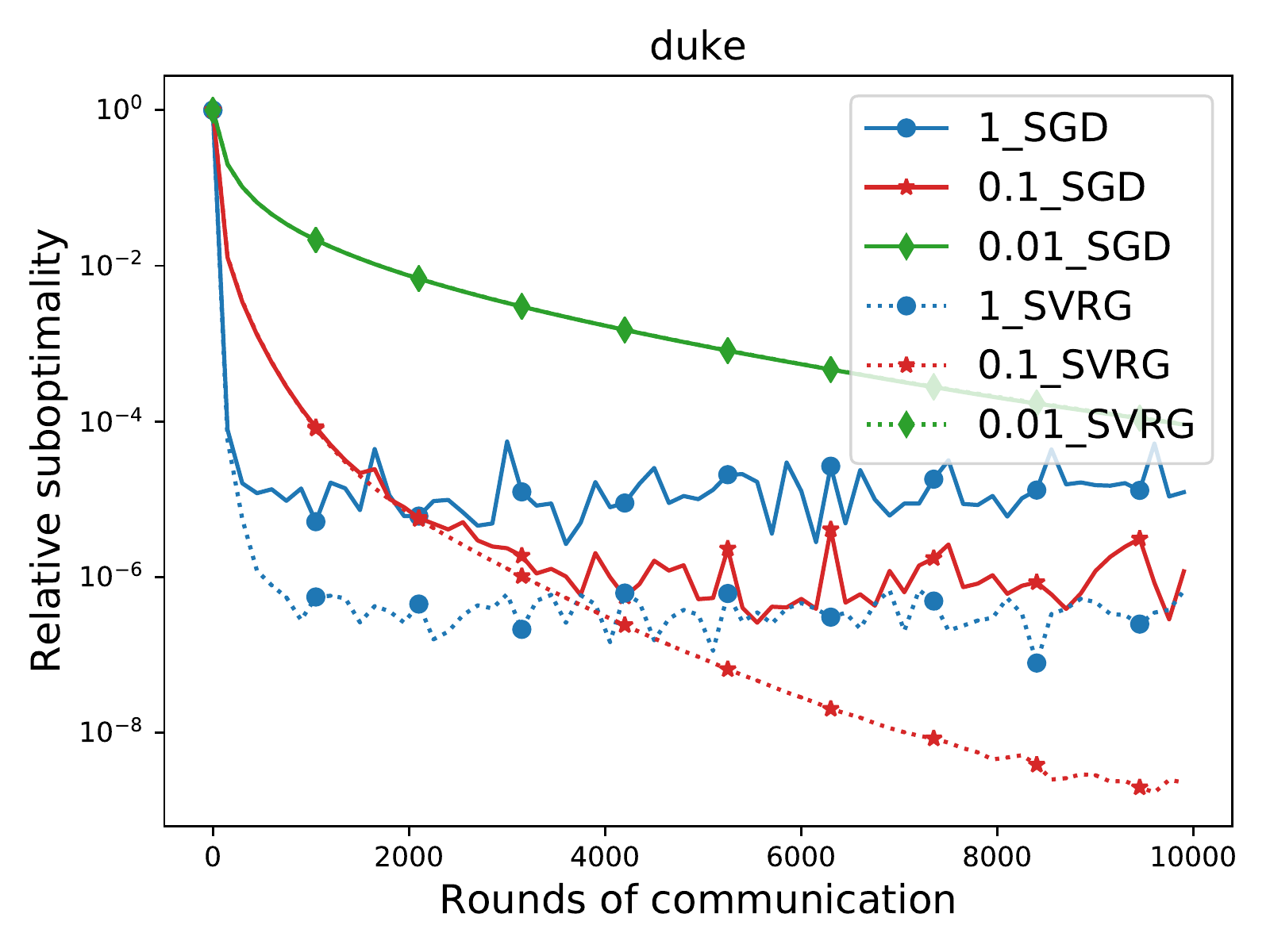}
        %\caption{ Residual vs. iteration  }\label{fig:bl_ex_flops}
\end{minipage}
\begin{minipage}{0.49\textwidth}
  \centering
\includegraphics[width =  \textwidth ]{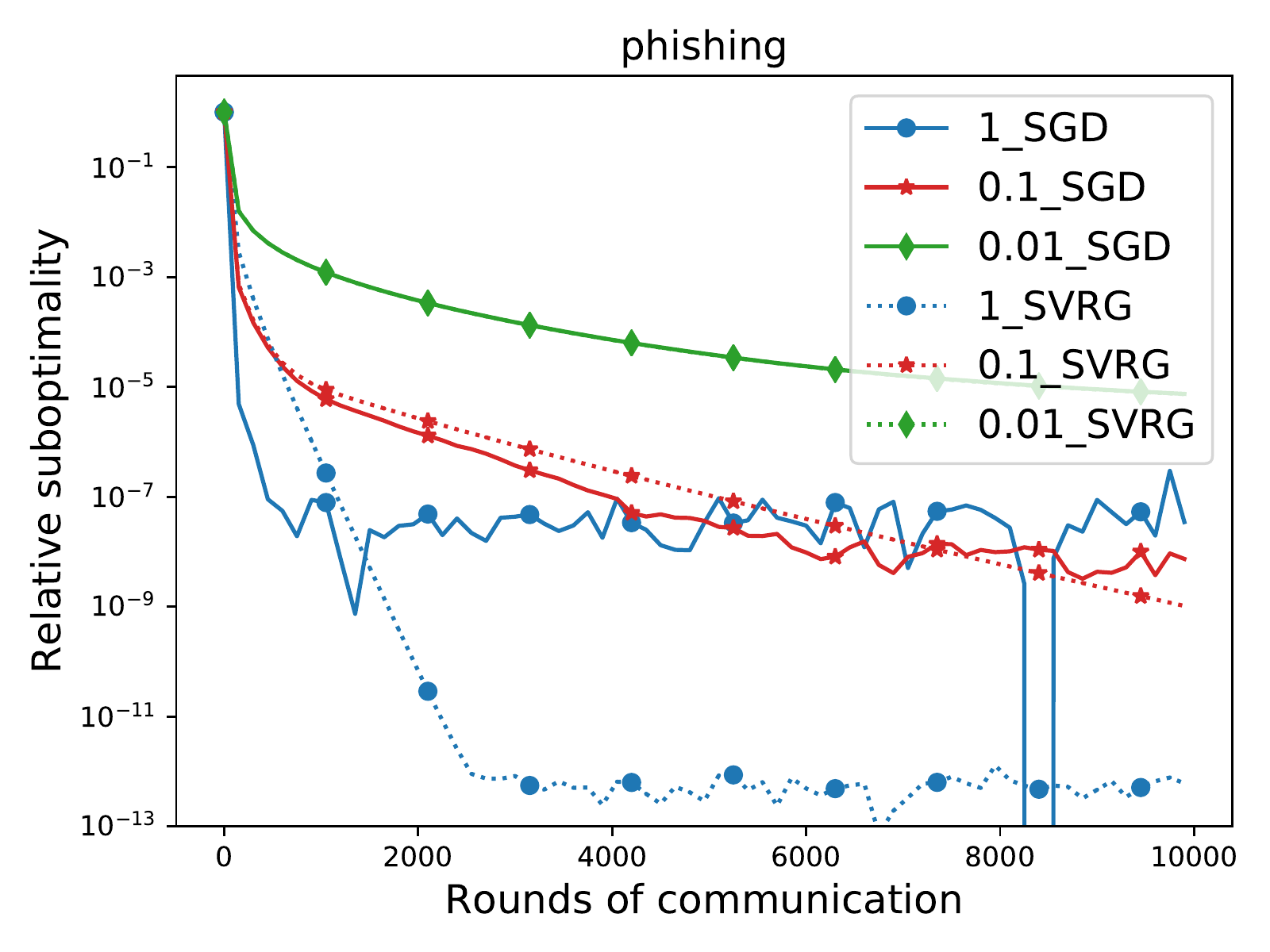}
        %\caption{ Residual vs. iteration  }\label{fig:bl_ex_flops}
\end{minipage}
\end{minipage}
\caption{Comparison of standard {\tt Local-SGD} (Alg.~\ref{alg:local_sgd}) and our {\tt Local-SVRG} (Alg.~\ref{alg:local_svrg}) for varying $\gamma$. Logistic regression applied on LibSVM~\cite{chang2011libsvm}.  Other parameters: $L=1, \mu=10^{-4}, \tau = 40$. Parameter $n$ chosen as per Tbl.~\ref{tbl:ns} in the appendix. }
\label{fig:sgd_svrg_hom}
\end{figure}

\section{Conclusions and Future Work}
This paper develops a unified approach to analyzing and designing a wide class of local stochastic first order algorithms. While our framework covers a broad range of methods, there are still some types of algorithms that we did not include but desire attention in future work. First, it would be interesting to study algorithms with  {\em biased} local stochastic gradients; these are popular for minimizing finite sums; see {\tt SAG}~\cite{schmidt2017minimizing} or {\tt SARAH}~\cite{nguyen2017sarah}. The second hitherto unexplored direction is including Nesterov's acceleration~\cite{nesterov1983method} in our framework. This idea is gaining traction in the area of  local methods already~\cite{pathak2020fedsplit, yuan2020federated}. However, it is not at all clear how this should be done and several attempts at achieving this unification goal failed. The third direction is allowing for a regularized local objective, which has been under-explored in the FL community so far. Other compelling directions that we do not cover are the local higher-order or proximal methods~\cite{li2018federated, pathak2020fedsplit} and methods supporting partial participation~\cite{mcmahan2016federated}.

{\subsubsection*{Acknowledgements}
This work was supported by the KAUST baseline research grant of P.~Richt\'{a}rik. Part of this work was done while E.~Gorbunov was a research intern at KAUST. The research of E.~Gorbunov was also partially supported by the Ministry of Science and Higher Education of the Russian Federation (Goszadaniye) 075-00337-20-03 and RFBR, project number 19-31-51001.}

\bibliography{local_sigma_k}

\clearpage

\onecolumn
\part*{Appendix}

\appendix

Since the appendix contains substantial amount of material, we have decided to include a table of contents.

\tableofcontents

\clearpage

\section{Table of Frequently Used Notation} \label{sec:notation_table}

\begin{table}[!h]
\caption{Summary of frequently used notation.}
\label{tbl:notation}
\begin{center}

\begin{tabular}{|c|l|c|}
\hline
\multicolumn{3}{|c|}{{\bf Main notation} }\\
\hline
\hline
%$\R^n$ & Set of $n$ dimensional real vectors &\\ \hline
  $f: \R^d \rightarrow \R$ & Objective to be minimized&  \eqref{eq:main_problem} \\ \hline
  $f_i: \R^d \rightarrow \R$ & Local objective  owned by device/worker $i$ & \eqref{eq:f_i_expectation} or \eqref{eq:f_i_sum} \\ \hline
 $x^*$ & Global optimum of \eqref{eq:main_problem};  $x^* \in \R^d$  & \\ \hline
  $d$ & Dimensionality of the problem space  & \eqref{eq:main_problem} \\ \hline
    $n$ & Number of clients/devices/nodes/workers & \eqref{eq:main_problem} \\ \hline
            $x_i^k$ & Local iterate;  $x_i^k \in \R^d$  & \eqref{eq:local_sgd_def} \\ \hline
        $g_i^k$ & Local stochastic direction;  $g_i^k \in \R^d$  & \eqref{eq:local_sgd_def} \\ \hline
           $\gamma$ & Stepsize/learning rate; $\gamma \geq 0$ & \eqref{eq:local_sgd_def} \\ \hline
         $c_k$ & Indicator of the communication; $c_k \in \{0,1\}$& \eqref{eq:local_sgd_def} \\ \hline
          $\mu$ &Strong quasi-convexity of the local objective; $\mu \geq 0$ & ~\eqref{eq:str_quasi_cvx} \\ \hline
    $L$ &  Smoothness of the local objective; $L\geq \mu$ & ~\eqref{eq:L_smoothness} \\  \hline
          $x^k$ & Virtual iterate;  $x^k \in \R^d$  & Sec~\ref{sec:main_res} \\ \hline
      $V^k$ & Discrepancy between local and virtual iterates;  $V^k \geq 0$  & Sec~\ref{sec:main_res}\\ \hline
         $    \overline{x}^K$ & Weighted average of historical iterates;  $\overline{x}^K \in \R^d$  & Thm~\ref{thm:main_result} \\ \hline
         $\zeta$ &Heterogeneity parameter; $\zeta \geq 0$ & ~\eqref{eq:bounded_data_dissimilarity} \\ \hline
       $\tau$ &  Size of the fixed local loop $\tau \geq 0$ & Sec~\ref{sec:data_and_loop}  \\ \hline
        $p$ &  Probability of aggregation fixed for the random local loop $p \in [0,1] $ & Sec~\ref{sec:data_and_loop}  \\ \hline
  $a_i^k$ & Unbiased local gradient; $a_i^k \in \R^d $ & Sec~\ref{sec:local_solver}  \\ \hline
    $b_i^k$ & Local shift; $b_i^k \in \R^d $ & Sec~\ref{sec:local_solver}  \\ \hline
        $h_i^k$ & Delayed local gradient estimator used to construct $b_i^k$; $h_i^k \in \R^d $ & Sec~\ref{sec:local_solver}  \\ \hline
       $l_i^k$ & Unbiased local gradient estimator used to construct $b_i^k$; $l_i^k \in \R^d $ & Sec~\ref{sec:local_solver}  \\ \hline
       $\cL$ & Expected smoothness of local objectives; $\cL \geq 0 $ & \eqref{eq:expected_smoothness_1}  \\ \hline
      $\max L_{ij}$ & Smoothness constant of local summands; $\max L_{ij}\geq 0$ & Sec~\eqref{sec:llsvrg}  \\ \hline
     $\sigma^2$ & Averaged upper bound for the variance of local stochastic gradient  & Tab~\eqref{tbl:instances}  \\ \hline
          $\sigma^2_*$ & Averaged variance of local stochastic gradients at the solution  & Tab~\eqref{tbl:instances}  \\ \hline
           $\zeta_*^2$ &  $\eqdef \frac1n\sum_{i=1}^n \| \nabla f_i(x^*)\|^2$  & Tab~\eqref{tbl:instances}  \\         
\hline
\hline
\multicolumn{3}{|c|}{{\bf Parametric Assumptions} }\\
\hline
\hline
\makecell{ $A, A', B, B', C, C', F, F',$ \\ $G, H, D_1, D_1',D_2, D_3, \rho$} & \multicolumn{2}{|c|}{ Parameters of Assumption~\ref{ass:key_assumption} } \\
\hline
 $A_i, B_i\, D_{1,i} , \rho_i , C_i ,  D_{2,i}$ & \multicolumn{2}{|c|}{ Parameters of Assumption~\ref{ass:sigma_k_original} } \\
 \hline
  $A'_i, D_{3,i}$ & \multicolumn{2}{|c|}{ Parameters of Assumption~\ref{ass:bik} } \\
  \hline
   $\sigma^2_k, \sigma^2_{i,k}$ & \multicolumn{2}{|c|}{ Possibly random non-negative sequences from Assumptions~\ref{ass:key_assumption},~\ref{ass:sigma_k_original},~\ref{ass:hetero_second_moment} } \\
 \hline
\hline
\multicolumn{3}{|c|}{{\bf Standard} }\\
\hline
\hline
$\EE[\cdot]$ & Expectation & \\
\hline
$\EE\left[\cdot\mid x^k\right]$  & $\eqdef \EE\left[\cdot\mid x_1^k,\ldots,x_n^k\right]$; expectation conditioned on $k$-th local iterates  & \\
\hline
$D_h(x,y)$ & $\eqdef h(x)-h(y)- \langle \nabla h(y), x-y \rangle$; Bregman distance of $x,y$ w.r.t. $h$  & As~\ref{ass:sigma_k_original}\\
 \hline
\end{tabular}
 \end{center}
\end{table}

\clearpage

\section{Table with Complexity Bounds in the Weakly Convex Case}
\begin{table*}[!h]
\caption{A selection of methods that can be analyzed using our framework. A choice of $a_i^k, b_i^k$ and $l_i^k$ is presented along with the established complexity bounds (= number of iterations to find such $\hat x$ that $\EE[f(\hat{x}) - f(x^*)]\le \varepsilon$) and a specific setup under which the methods are analyzed. For all algorithms we suppress constants factors. All rates are provided in the \textbf{weakly convex} setting. UBV stands for the ``Uniform Bound on the Variance'' of local stochastic gradient, which is often assumed when $f_i$ is of the form~\eqref{eq:f_i_expectation}. ES stands for the ``Expected Smoothness''~\cite{gower2019sgd}, which does not impose any extra assumption on the objective/noise, but rather can be derived given the sampling strategy and the smoothness structure of $f_i$. Consequently, such a setup allows us to obtain local methods with importance sampling. Next, the simple setting is a special case of ES when we uniformly sample a single index on each node each iteration. $^\clubsuit$: {\tt Local-SGD} methods have never been analyzed under ES assumption. Notation: $\sigma^2$ -- averaged (within nodes) uniform upper bound for the variance of local stochastic gradient, $\sigma_*^2$ -- averaged variance of local stochastic gradients at the solution, $\zeta_*^2 \eqdef \frac1n\sum_{i=1}^n \| \nabla f_i(x^*)\|^2$, $\max L_{ij}$ -- the worst smoothness of $f_{i,j}, i\in[n],j\in[m]$, $\cL$ -- the worst ES constant for all nodes, $R_0\eqdef \|x^0 - x^*\|$ -- distance of the starting point $x^0$ from the closest solution $x^*$, $\Delta_0 \eqdef f(x^0)-f(x^*)$.
}
\label{tbl:special_cases_weakly_convex}
\begin{center}
\vskip -0.2cm
\footnotesize
\begin{tabular}{|c|c|c|c|c|c|c|c|}
\hline
{\bf Method} & {\bf  \# } &  {\bf Ref} &   $\myred{a_i^k}, \myblue{b_i^k}, l_i^k$   & {\bf Complexity}  &   {\bf Setting}  & {\bf Sec}  \\
\hline
\hline
 {\tt Local-SGD}  &   \ref{alg:local_sgd} & {\cite{woodworth2020minibatch}}  & $\myred{f_{\xi_i}(x_i^k)}, \myblue{0}, - $& \begin{tabular}{c}
$\frac{L R_0^2}{\varepsilon}+\frac{\sigma^2R_0^2}{n\varepsilon^2}\quad\quad\quad\quad\quad\quad\quad\quad$\\
$\quad\quad\quad+\frac{R_0^2\sqrt{L\tau(\sigma^2 + \tau\zeta^2)}}{\varepsilon^{\nicefrac{3}{2}}}$ 
\end{tabular} &  \begin{tabular}{c}
	UBV,\\
	$\zeta$-Het
\end{tabular} & \ref{sec:sgd_bounded_var}  \\
%%%%%%%%%%%%%%%%%%%%
%%%%%%%%%%%%%%%%%%%%
\hline
 {\tt Local-SGD}  &   \ref{alg:local_sgd} & {\cite{koloskova2020unified}}  & $\myred{f_{\xi_i}(x_i^k)},\myblue{0}, - $& \begin{tabular}{c}
$\frac{\tau L R_0^2}{\varepsilon}+\frac{\sigma^2R_0^2}{n\varepsilon^2}\quad\quad\quad\quad\quad\quad\quad\quad$\\
$\quad\quad\quad+\frac{R_0^2\sqrt{L(\tau-1)(\sigma^2 + (\tau-1)\zeta_*^2)}}{\varepsilon^{\nicefrac{3}{2}}}$ 
\end{tabular}   &  \begin{tabular}{c}
	UBV,\\
	Het
\end{tabular} & \ref{sec:sgd_bounded_var}  \\
%%%%%%%%%%%%%%%%%%%%
%%%%%%%%%%%%%%%%%%%%
\hline
 {\tt Local-SGD}  &  \ref{alg:local_sgd} & {\cite{khaled2020tighter}}$^{\clubsuit}$  & $\myred{f_{\xi_i}(x_i^k)}, \myblue{0}, - $  & \begin{tabular}{c}
	$\frac{\left(L+\nicefrac{\cL}{n}+\sqrt{(\tau-1)L\cL}\right)R_0^2}{\varepsilon}+\frac{\sigma_*^2R_0^2}{n\varepsilon^2}\quad$\\
	$+\frac{L\zeta^2(\tau-1)R_0^2}{\mu\varepsilon^2}+\frac{R_0^2\sqrt{L(\tau-1)(\sigma_*^2 + \zeta_*^2)}}{\varepsilon^{\nicefrac{3}{2}}}$ 
\end{tabular}  &   \begin{tabular}{c}
	ES,\\
	$\zeta$-Het
\end{tabular} & \ref{sec:sgd_es}  \\
%%%%%%%%%%%%%%%%%%%%
%%%%%%%%%%%%%%%%%%%%
\hline
 {\tt Local-SGD}  &  \ref{alg:local_sgd} & {\cite{khaled2020tighter}}$^{\clubsuit}$  & $\myred{f_{\xi_i}(x_i^k)}, \myblue{0}, - $  & \begin{tabular}{c}
	$\frac{\left(L\tau+\nicefrac{\cL}{n}+\sqrt{(\tau-1)L\cL}\right)R_0^2}{\varepsilon}+\frac{\sigma_*^2R_0^2}{n\varepsilon^2}\quad$\\
	$\quad\quad+\frac{R_0^2\sqrt{L(\tau-1)(\sigma_*^2 + (\tau-1)\zeta_*^2)}}{\varepsilon^{\nicefrac{3}{2}}}$ 
\end{tabular}  &   \begin{tabular}{c}
	ES,\\
	Het
\end{tabular} & \ref{sec:sgd_es}  \\
%%%%%%%%%%%%%%%%%%%%
%%%%%%%%%%%%%%%%%%%%
\hline
\rowcolor{bgcolor} {\tt Local-SVRG}  &   \ref{alg:local_svrg} & NEW  & \begin{tabular}{c}
	\myred{$\nabla f_{i,j_i}(x^k_i) -  \nabla f_{i,j_i}(y_i^k)$}\\\myred{$ +  \nabla f_{i}(y_i^k)$},\\ \myblue{$0$}, $- $ 
\end{tabular}  & \begin{tabular}{c}
	 $\frac{\left(L + \max L_{ij}\sqrt{\nicefrac{m}{n}} + \sqrt{(\tau-1)L\max L_{ij}}\right)R_0^2}{\varepsilon}$\\
	 $\frac{\sqrt[3]{(\tau-1)mL\max L_{ij}}R_0^2}{\varepsilon}+ \frac{L\zeta^2(\tau-1)R_0^2}{\mu\varepsilon^2}$\\
	 $+\frac{R_0^2\sqrt{L(\tau-1)\zeta_*^2}}{\varepsilon^{\nicefrac{3}{2}}}$
\end{tabular} & \begin{tabular}{c}
	simple,\\
	$\zeta$-Het
\end{tabular} & \ref{sec:llsvrg}   \\
%%%%%%%%%%%%%%%%%%%%
%%%%%%%%%%%%%%%%%%%%
\hline
\rowcolor{bgcolor} {\tt Local-SVRG}  &   \ref{alg:local_svrg} & NEW  & \begin{tabular}{c}
	\myred{$\nabla f_{i,j_i}(x^k_i) -  \nabla f_{i,j_i}(y_i^k)$}\\\myred{$ +  \nabla f_{i}(y_i^k)$},\\ \myblue{$0$}, $- $ 
\end{tabular}  & \begin{tabular}{c}
	 $\frac{\left(L\tau + \max L_{ij}\sqrt{\nicefrac{m}{n}} + \sqrt{(\tau-1)L\max L_{ij}}\right)R_0^2}{\varepsilon}$\\
	 $\frac{\sqrt[3]{(\tau-1)mL\max L_{ij}}R_0^2}{\varepsilon}+\frac{R_0^2\sqrt{L(\tau-1)^2\zeta_*^2}}{\varepsilon^{\nicefrac{3}{2}}}$
\end{tabular} & \begin{tabular}{c}
	simple,\\
	Het
\end{tabular} & \ref{sec:llsvrg}   \\
%%%%%%%%%%%%%%%%%%%%
%%%%%%%%%%%%%%%%%%%%
\hline
\rowcolor{bgcolor} {\tt S*-Local-SGD}  &  \ref{alg:local_sgd_star} & NEW  & $\myred{f_{\xi_i}(x_i^k)}, \myblue{\nabla f_i(x^*)}, - $ & $\frac{\tau LR_0^2}{\varepsilon}+\frac{\sigma^2R_0^2}{n\varepsilon^2}+\frac{R_0^2\sqrt{L(\tau-1)\sigma^2}}{\varepsilon^{\nicefrac{3}{2}}}$  &  \begin{tabular}{c}
	UBV,\\
	Het
\end{tabular} & \ref{sec:sgd_star_bounded_var}  \\
%%%%%%%%%%%%%%%%%%%%
%%%%%%%%%%%%%%%%%%%%
\hline
 {\tt SS-Local-SGD}  &  \ref{alg:l_local_svrg} & \cite{karimireddy2019scaffold}  & \begin{tabular}{c}
$\myred{f_{\xi_i}(x_i^k)}, \myblue{h_i^k  - \frac1n \sum_{i=1}^n h_i^k},$\\
$ \nabla f_{\tilde{\xi}_i^k} (y_i^k)$ 
\end{tabular}  & $\frac{LR_0^2}{p\varepsilon}+\frac{\sigma^2R_0^2}{n\varepsilon^2}+\frac{R_0^2\sqrt{L(1-p)\sigma^2}}{p^{\nicefrac{1}{2}}\varepsilon^{\nicefrac{3}{2}}}$ &  \begin{tabular}{c}
	UBV,\\
	Het
\end{tabular} & \ref{sec:loopless_local_svrg}  \\
%%%%%%%%%%%%%%%%%%%%
%%%%%%%%%%%%%%%%%%%%
\hline
\rowcolor{bgcolor} {\tt SS-Local-SGD}  &  \ref{alg:l_local_svrg} & NEW  & \begin{tabular}{c}
$\myred{f_{\xi_i}(x_i^k)}, \myblue{h_i^k  - \frac1n \sum_{i=1}^n h_i^k},$\\
$ \nabla f_{\tilde{\xi}_i^k} (y_i^k)$ 
\end{tabular}  &  \begin{tabular}{c}
$\frac{\left(L + \nicefrac{p\cL}{n} + \sqrt{p(1-p)L\cL}\right)R_0^2}{p\varepsilon}$\\$+ \frac{\sqrt[3]{(1-p)L(L+p\cL)R_0^4\Delta_0}}{p\varepsilon}$\\ $+\frac{\sqrt[3]{(1-p)L\sigma_*^2R_0^4}}{p^{\nicefrac{2}{3}}\varepsilon}+\frac{\sigma_*^2R_0^2}{n\varepsilon^2}$\\$+\frac{R_0^2\sqrt{L(1-p)\sigma_*^2}}{p^{\nicefrac{1}{2}}\varepsilon^{\nicefrac{3}{2}}}$
\end{tabular} &  \begin{tabular}{c}
	ES,\\
	Het
\end{tabular} & \ref{sec:loopless_local_svrg_es}  \\
%%%%%%%%%%%%%%%%%%%%
%%%%%%%%%%%%%%%%%%%%
\hline
\rowcolor{bgcolor2} {\tt S*-Local-SGD*}  &  \ref{alg:local_sgd_star_star} & NEW  & \begin{tabular}{c}
 	\myred{$\nabla f_{i,j_i}(x^k_i) -  \nabla f_{i,j_i}(x^*)$}\\
 	\myred{$ +  \nabla f_{i}(x^*)$}, $\myblue{\nabla f_i(x^*)}, - $\\ 	
 \end{tabular} &  $\frac{\left(L\tau + \nicefrac{\max L_{ij}}{n} + \sqrt{(\tau-1)L\max L_{ij}}\right)R_0^2}{\varepsilon}$  &  \begin{tabular}{c}
	simple,\\
	Het
\end{tabular} & \ref{sec:S*-Local-SGD*}  \\
%%%%%%%%%%%%%%%%%%%%
%%%%%%%%%%%%%%%%%%%%
\hline
\rowcolor{bgcolor2} {\tt S-Local-SVRG}  &   \ref{alg:l_local_svrg_fs} & NEW  &\begin{tabular}{c}
	$\myred{ \nabla f_{i,j_i}(x^k_i) -  \nabla f_{i,j_i}(y^k)}$\\
	\myred{$ +  \nabla f_{i}(y^k) $},\\
	$\myblue{h_i^k  - \frac1n \sum_{i=1}^n h_i^k}, \nabla f_{i} (y^k)$ 
\end{tabular} & \begin{tabular}{c}
$\frac{\left(L + pL\sqrt{\nicefrac{m}{n}} + \sqrt{(1-p)L\max L_{ij}}\right)R_0^2}{p\varepsilon}$\\ $+ \frac{R_0^2\sqrt[3]{L\max L_{ij}^2}}{p^{\nicefrac{2}{3}}\varepsilon}$ 
\end{tabular} & \begin{tabular}{c}
	simple,\\
	Het
\end{tabular} &   \ref{sec:loopless_local_svrg_fs} \\
%%%%%%%%%%%%%%%%%%%%
%%%%%%%%%%%%%%%%%%%%
\hline
\end{tabular}
\end{center}
\vskip -0.2cm
\end{table*}

\clearpage

\section{Extra Experiments \label{sec:extra_exp}}

\subsection{Missing details from Section~\ref{sec:exp} and an extra figure\label{sec:extralog}}

In Section~\ref{sec:exp} we study the effect of  local variance reduction on the communication complexity of local methods. We consider the regularized logistic regression objective, i.e., we choose
\[
f_i(x) \eqdef \frac1m \sum_{j=1}^m \log \left(1+\exp\left( \langle a_{(i-1)m+j}, x \rangle\cdot  b_{(i-1)m+j}\right) \right) + \frac{\mu}{2}\| x\|^2,
\]
where $a_{j}\in \R^d, b_j\in \{-1, 1\}$ for $j\leq nm$ are the training data and labels.

\paragraph{Number of the clients.} We select a different number of clients for each dataset in order to capture a variety of scenarios. See Table~\ref{tbl:ns} for details.

 \begin{table}[!h]
 \caption{Number of clients per dataset (Figures~\ref{fig:sgd_svrg_hom} and~\ref{fig:sgd_svrg_het}). }
\label{tbl:ns}
\begin{center}
\small
\begin{tabular}{|c|c|c|c|}
\hline
{\bf Dataset}  & $n$ & {\bf \# datapoints} ($=mn$) & $d$   \\
 \hline
  \hline
\texttt{a1a} & 5 & 1 605	& 123  \\ \hline
\texttt{mushrooms} & 12 & 8 124 & 112   \\ \hline
\texttt{phishing} & 11   & 11 055	&	68 \\ \hline
\texttt{madelon} & 50 & 2 000& 500 \\ \hline
\texttt{duke} & 4  &44 & 7 129  \\ \hline
\texttt{w2a} & 10  &3 470 & 300  \\ \hline
\end{tabular}
\end{center}
\end{table}

\begin{figure}[!h]
\centering
\begin{minipage}{0.3\textwidth}
  \centering
\includegraphics[width =  \textwidth ]{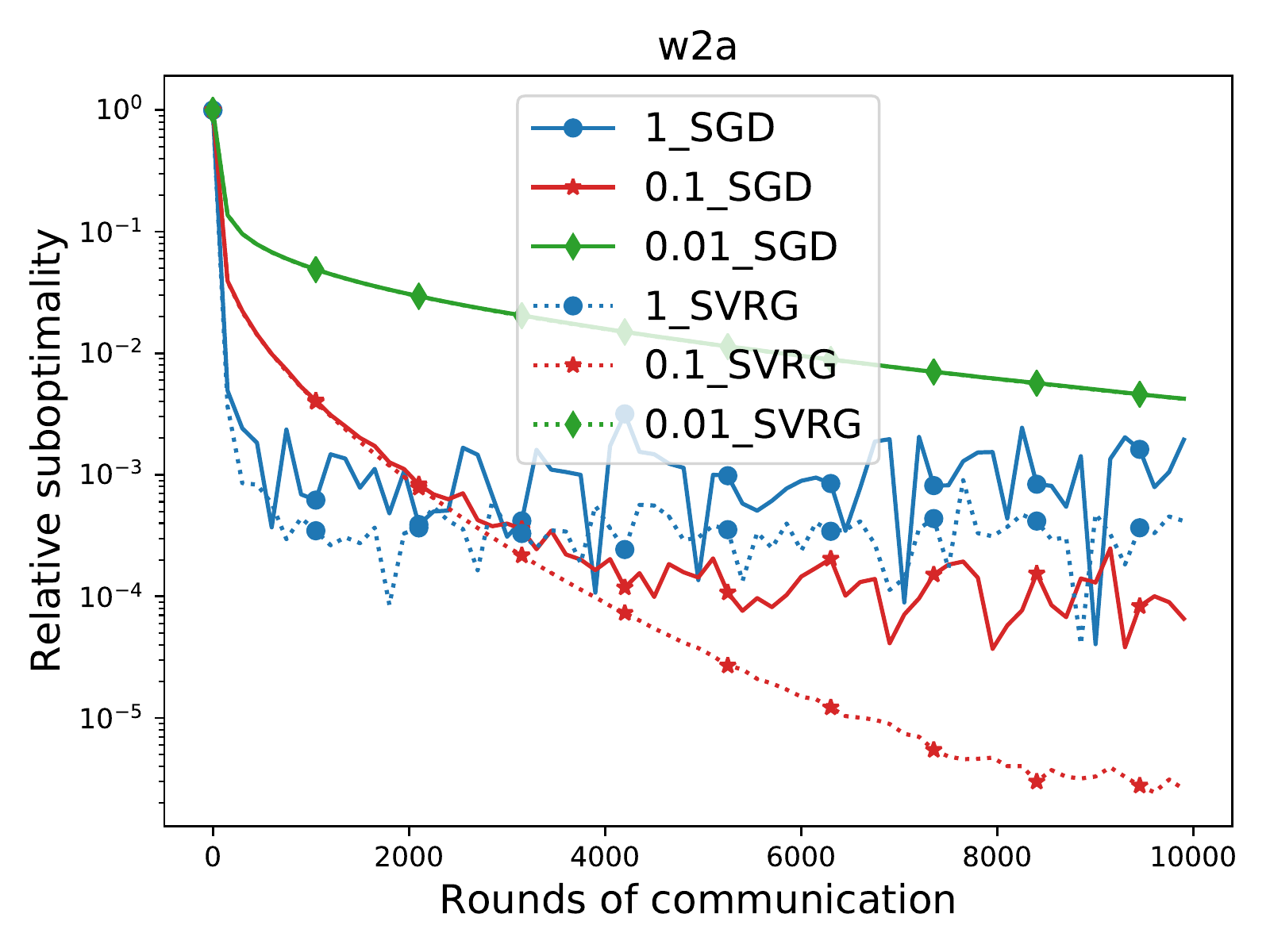}
        %\caption{ Residual vs. iteration  }\label{fig:bl_ex_flops}
\end{minipage}
\begin{minipage}{0.3\textwidth}
  \centering
\includegraphics[width =  \textwidth ]{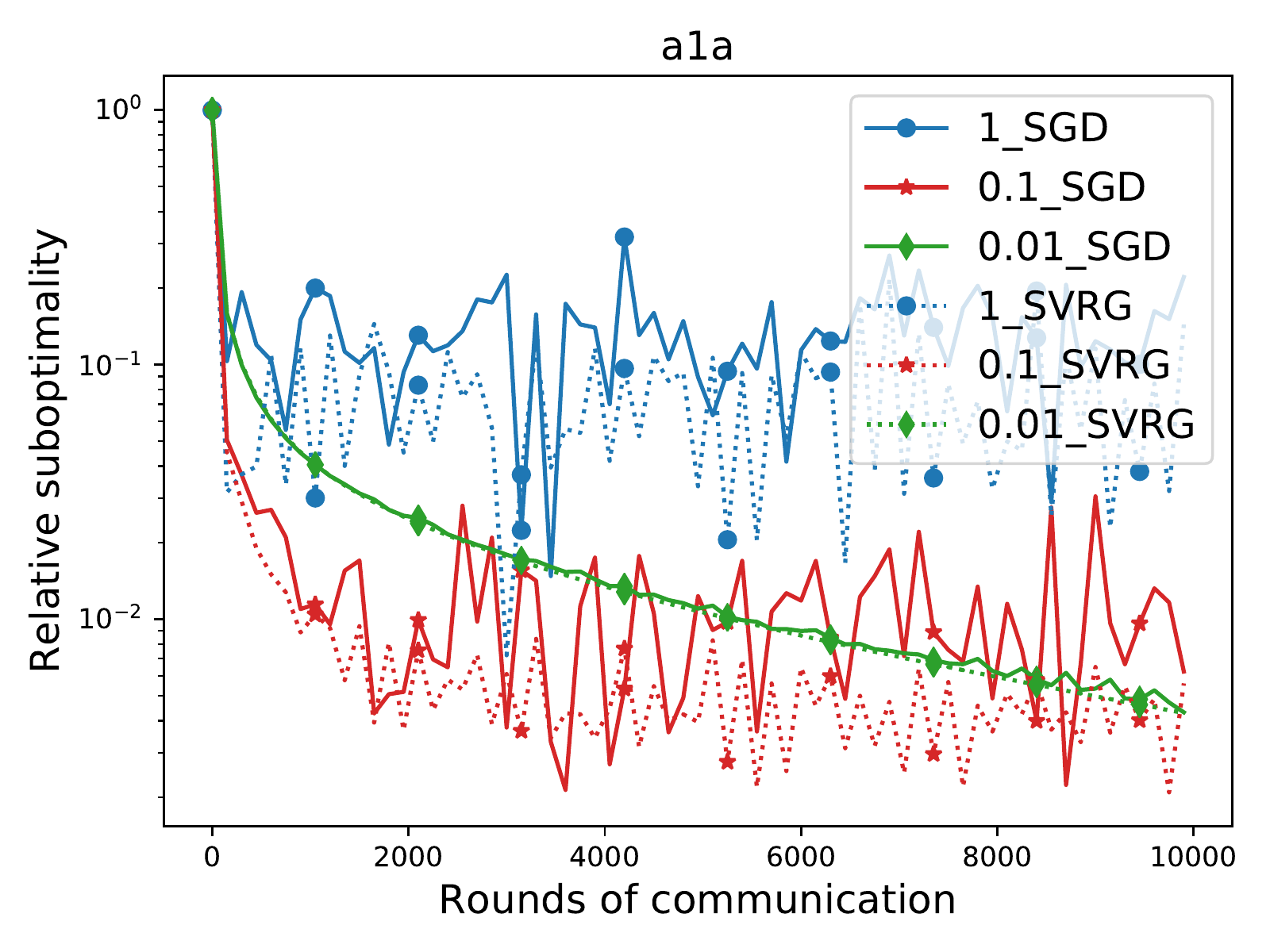}
        %\caption{ Residual vs. iteration  }\label{fig:bl_ex_flops}
\end{minipage}
\begin{minipage}{0.3\textwidth}
  \centering
\includegraphics[width =  \textwidth ]{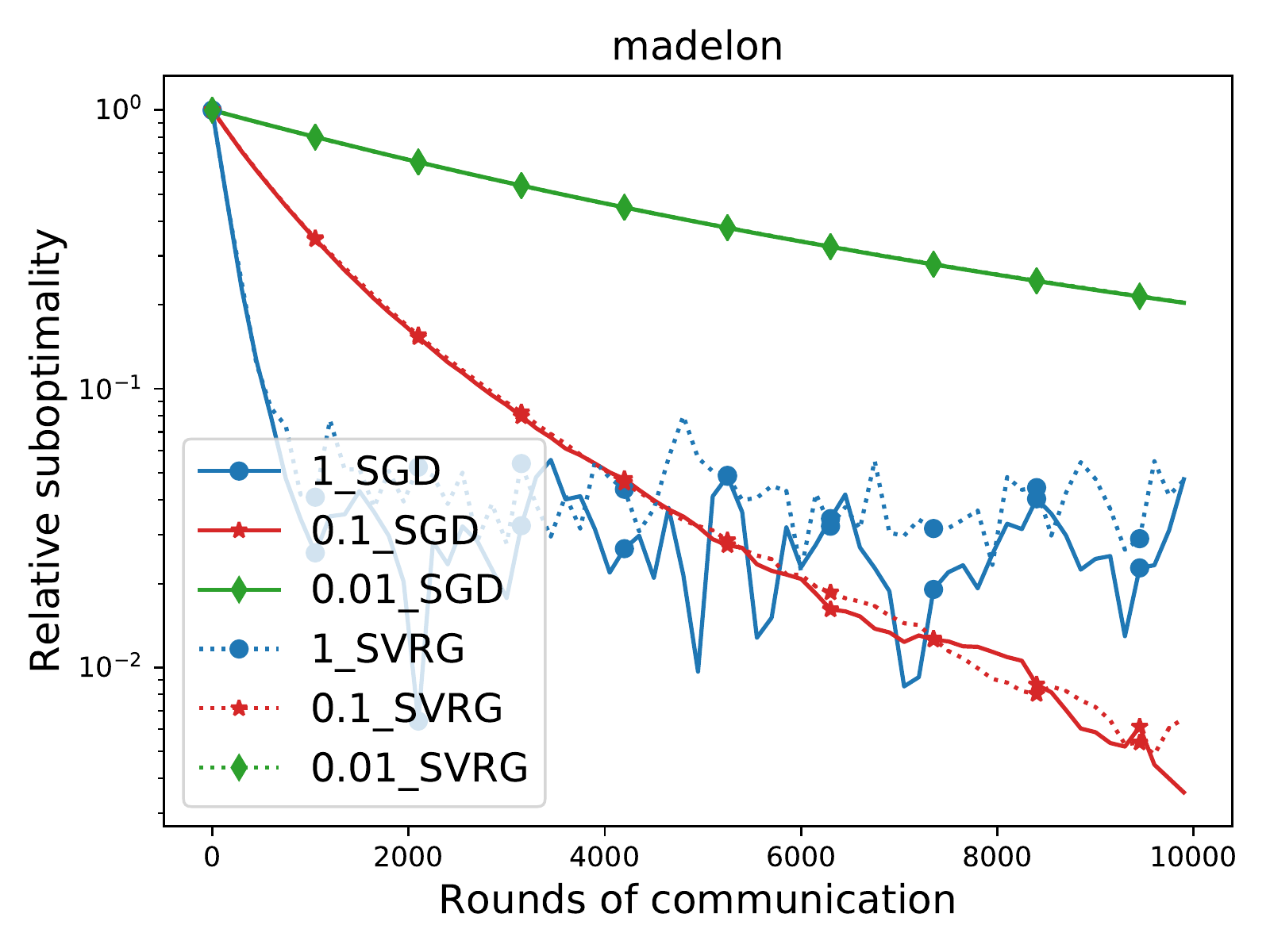}
        %\caption{ Residual vs. iteration  }\label{fig:bl_ex_flops}
\end{minipage}
\\
\begin{minipage}{0.3\textwidth}
  \centering
\includegraphics[width =  \textwidth ]{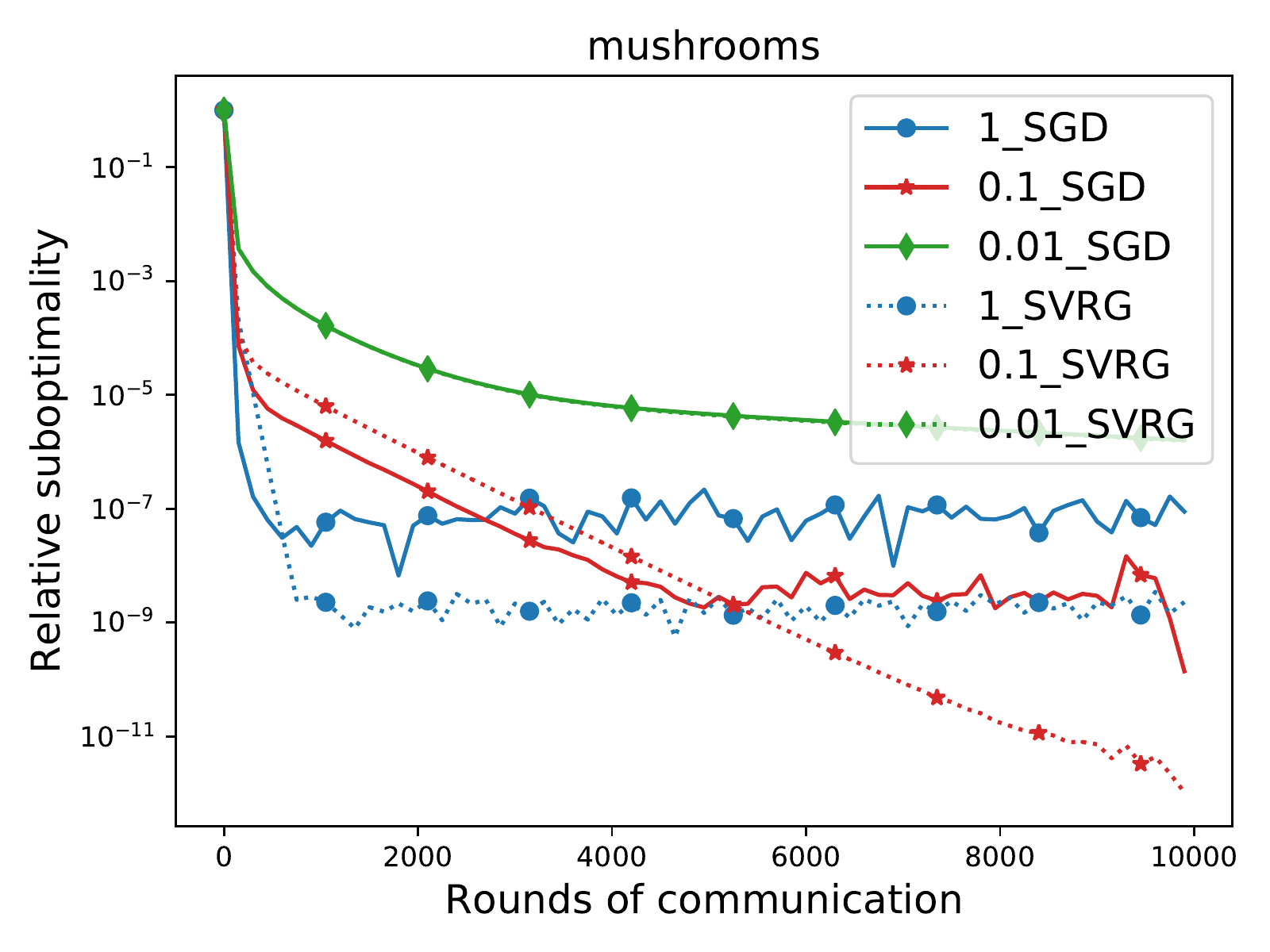}
        %\caption{ Residual vs. iteration  }\label{fig:bl_ex_flops}
\end{minipage}
\begin{minipage}{0.3\textwidth}
  \centering
\includegraphics[width =  \textwidth ]{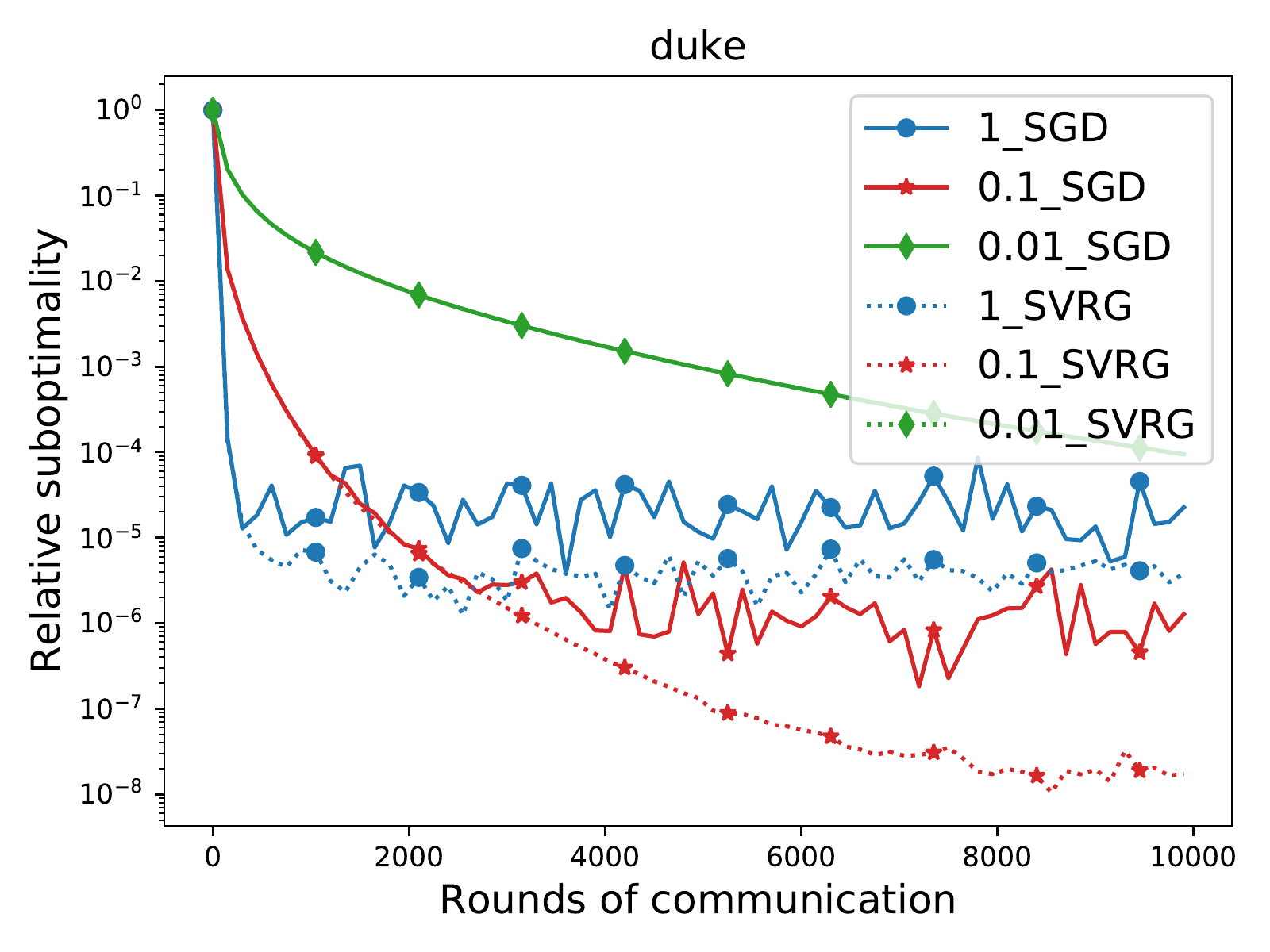}
        %\caption{ Residual vs. iteration  }\label{fig:bl_ex_flops}
\end{minipage}
\begin{minipage}{0.3\textwidth}
  \centering
\includegraphics[width =  \textwidth ]{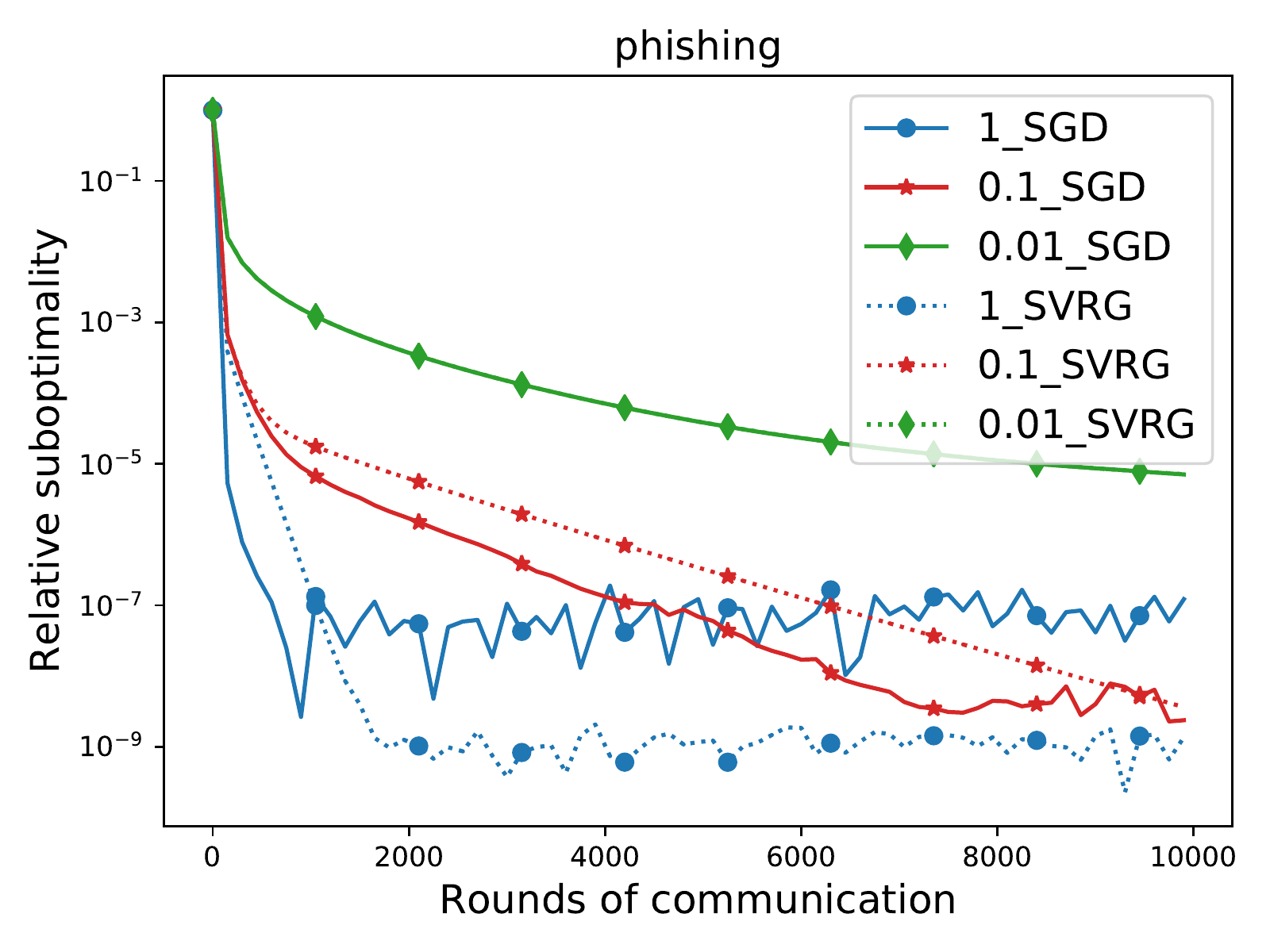}
        %\caption{ Residual vs. iteration  }\label{fig:bl_ex_flops}
\end{minipage}
\caption{Comparison of standard {\tt Local-SGD} (Algorithm~\ref{alg:local_sgd}), and {\tt Local-SVRG} (Algorithm~\ref{alg:local_svrg})  with various stepsizes $\gamma$. Logistic regression applied on LibSVM data~\cite{chang2011libsvm} with heterogenously splitted data.  Other parameters: $L=1, \mu=10^{-4}, \tau = 40$. Parameter $n$ chosen as per Table~\ref{tbl:ns}. (Same as Fig.~\ref{fig:sgd_svrg_hom}, but with the heterogenous data split)}
\label{fig:sgd_svrg_het}
\end{figure}

\paragraph{Data split.} The experiment from Figure~\ref{fig:sgd_svrg_hom} in the main body of the paper splits the data among the clients  uniformly at random (i.e., split according to the the order given by a random permutation). However, in a typical FL scenario, the local data might significantly differ from the population average. For this reason, we also test on a different split of the data: we first sort the data according to the labels, and then split them among the clients. Figure~\ref{fig:sgd_svrg_het} shows the results. We draw a conclusions identical to Figure~\ref{fig:sgd_svrg_hom}. We see that {\tt Local-SVRG} was at least as good as {\tt Local-SGD} for every stepsize choice and every dataset. Further, the prediction that the smaller stepsize yields the smaller of the optimum neighborhood for the price of slower convergence was confirmed.

\paragraph{Environment.} All experiments were performed in a simulated environment on a single machine.

\subsection{The effect of  local shift/drifts}

The experiment presented in Section~\ref{sec:exp} examined the effect of the noise on the performance of local methods and demonstrated that control variates can be efficiently employed to reduce that noise. In this section, we study the second factor that influences the neighborhood to which {\tt Local-SGD} converges: non-stationarity of {\tt Local-GD}.

We have already shown that the mentioned non-stationarity of {\tt Local-GD} can be fixed using a carefully designed idealized/optimal shift that depends on the solution $x^*$ (see Algorithm~\ref{alg:local_sgd_star}). Furthermore, we have shown that this idealized shift can be learned on-the-fly at the small price of slightly slower convergence rate (see Algorithm~\ref{alg:l_local_svrg} -- {\tt SS-Local-SGD}/{\tt SCAFFOLD}).\footnote{In fact, {\tt SCAFFOLD} can be coupled together with {\tt Local-SVRG} given that the local objectives are of a finite-sum structure, resulting in Algorithm~\ref{alg:l_local_svrg_fs}.}

In this experiment, we therefore compare {\tt Local-SGD},  {\tt S*-Local-SGD} and {\tt SCAFFOLD}. In order to decouple the local variance with the non-stationarity of the local methods, we let each algorithm access the full local gradients. Next, in order to have a full control of the setting, we let the local objectives to be artificially generated quadratic problems. Specifically, we set
\begin{equation}\label{eq:quadproblem}
f_i (x)= \frac{\mu}{2} \| x\|^2 + \frac{1-\mu}{2} (x-z_i^*)^\top \left( \sum_{j=1}^m a_i a_i^\top \right)(x-z_i^*),
\end{equation}
where $a_i$ are mutually orthogonal vectors of norm 1 with $m<d$ (generated by orthogonalizing Gaussian vectors), $z_i^*$ are Gaussian vectors and $\mu = 10^{-3}$. We consider four different instances of~\eqref{eq:quadproblem} given by Table~\ref{eq:quadproblem}. Figures~\ref{fig:artif1},~\ref{fig:artif2},~\ref{fig:artif3},~\ref{fig:artif4} show the result.

Through most of the plots across all combinations of type, $\tau$, $n$, we can see that {\tt Local-SGD}  suffers greatly from the fact that it is attracted to an incorrect fixed point and as a result, it never converges to the exact optimum. On the other hand, both {\tt S*-Local-SGD}  and {\tt SCAFFOLD} converge to the exact optimum and therefore outperform {\tt Local-SGD}  in most examples. We shall note that the rate of {\tt SCAFFOLD} involves  slightly worse constants than those in {\tt Local-SGD}  and {\tt S*-Local-SGD}, and therefore it sometimes performs worse in the early stages of the optimization process when compared to the other methods. Furthermore, notice that our method {\tt S*-Local-SGD}  always performed best.

To summarize, our results demonstrate that \begin{itemize}
\item [(i)] the incorrect fixed point of used by standard local methods is an issue not only theory but also in practice, and should be addressed if better performance is required, 
\item [(ii)] the theoretically optimal shift employed by {\tt S*-Local-SGD} is ideal from a performance perspective if it was available (however, this strategy is impractical to implement as the optimal shift presumes the knowledge of the optimal solution), and 
\item [(iii)] {\tt SCAFFOLD}/{\tt SS-Local-SGD} is a practical solution to fixing the incorrect fixed point problem -- it converges to the exact optimum at a price of a slightly worse initial convergence speed. 
\end{itemize}

 \begin{table}[!t]
 \caption{Instances of~\eqref{eq:quadproblem}. }
\label{tbl:instances}
\begin{center}
\small
\begin{tabular}{|c|c|c|}
\hline
 {\bf Type}  & $m$ &$z_i^*$   \\
 \hline
  \hline
0 & 1 & $\sim \cN(0,\mI)$ \\
\hline
1 & 10 & $\sim \cN(0,\mI)$ \\
\hline
2 & 1 & $\sim \cN(0,\mI)$ \\
\hline
3 & 10 & $\sim \cN(0,\mI)$ \\
\hline
\end{tabular}
\end{center}
\end{table}

\begin{figure}[!h]
\centering
\begin{minipage}{0.3\textwidth}
  \centering
\includegraphics[width =  \textwidth ]{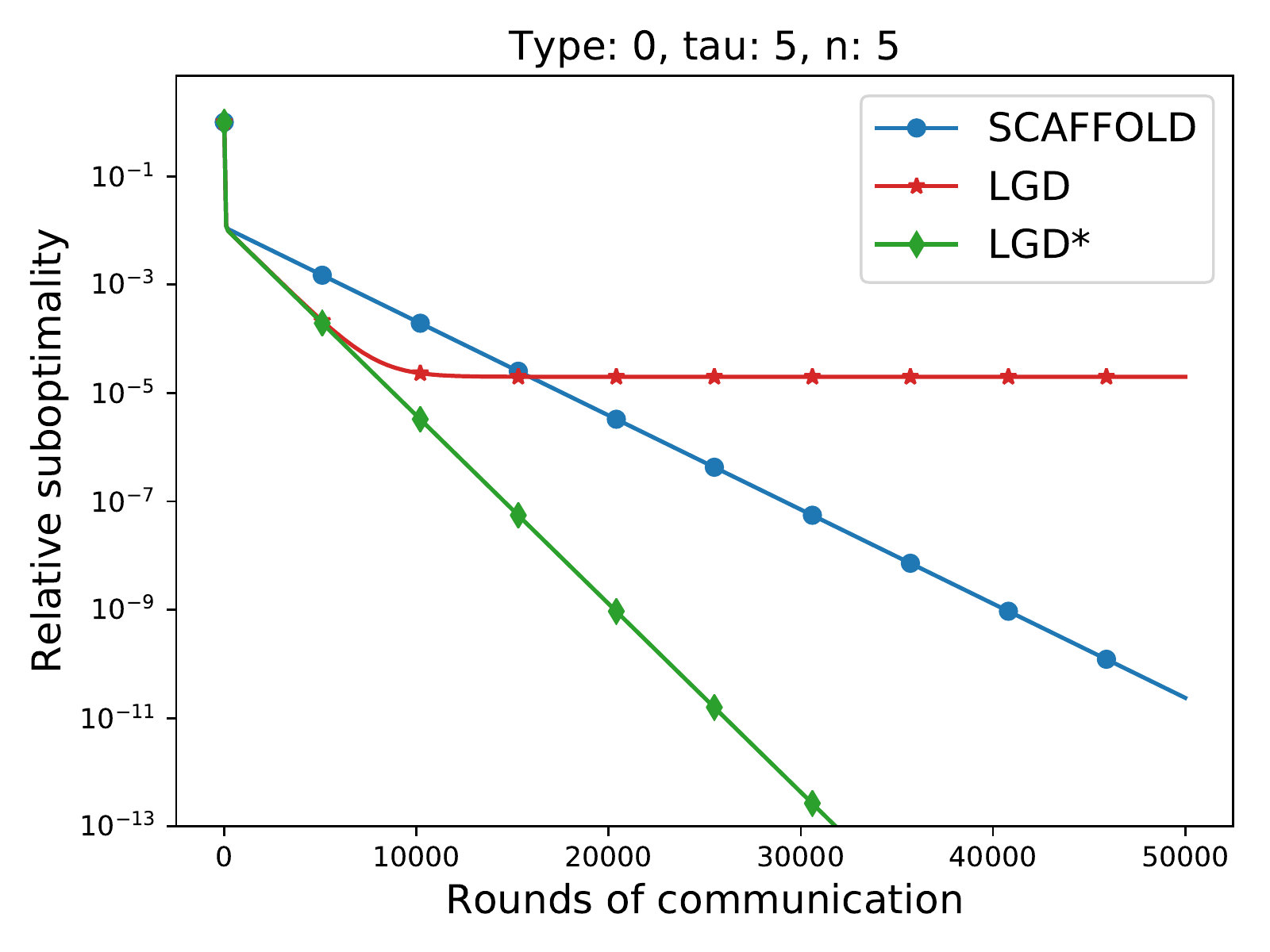}
        %\caption{ Residual vs. iteration  }\label{fig:bl_ex_flops}
\end{minipage}
\begin{minipage}{0.3\textwidth}
  \centering
\includegraphics[width =  \textwidth ]{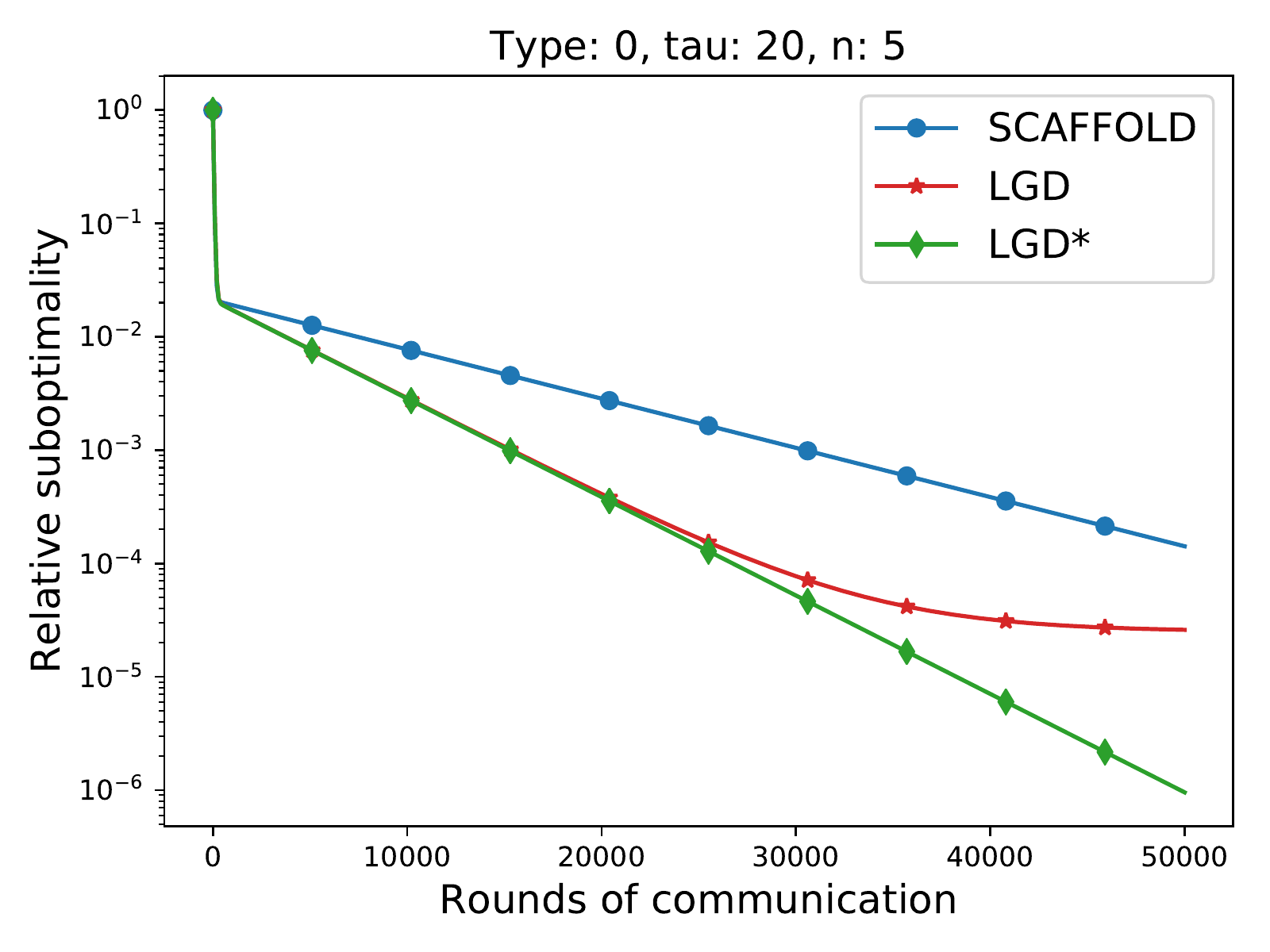}
        %\caption{ Residual vs. iteration  }\label{fig:bl_ex_flops}
\end{minipage}
\begin{minipage}{0.3\textwidth}
  \centering
\includegraphics[width =  \textwidth ]{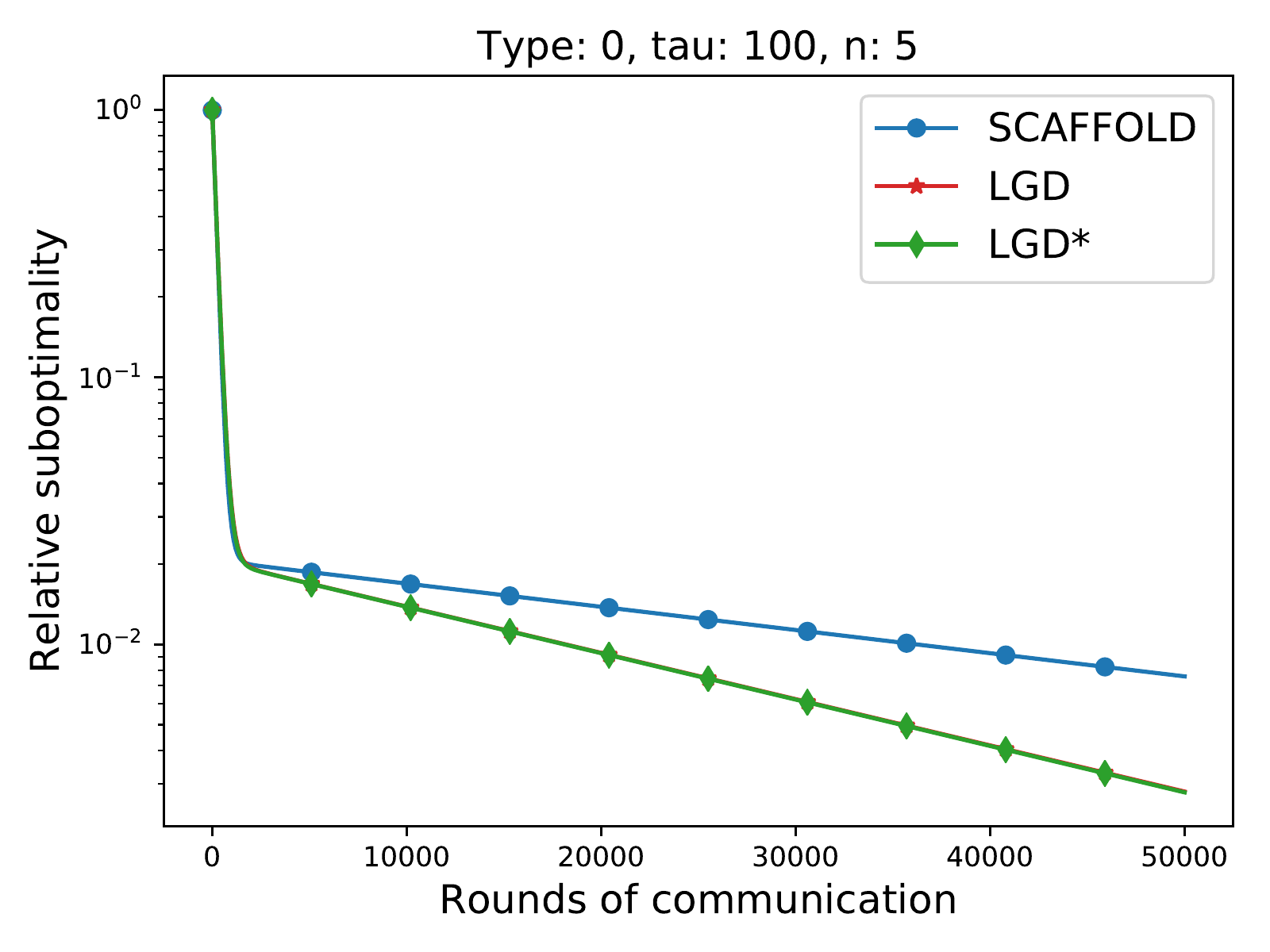}
        %\caption{ Residual vs. iteration  }\label{fig:bl_ex_flops}
\end{minipage}
\\
\begin{minipage}{0.3\textwidth}
  \centering
\includegraphics[width =  \textwidth ]{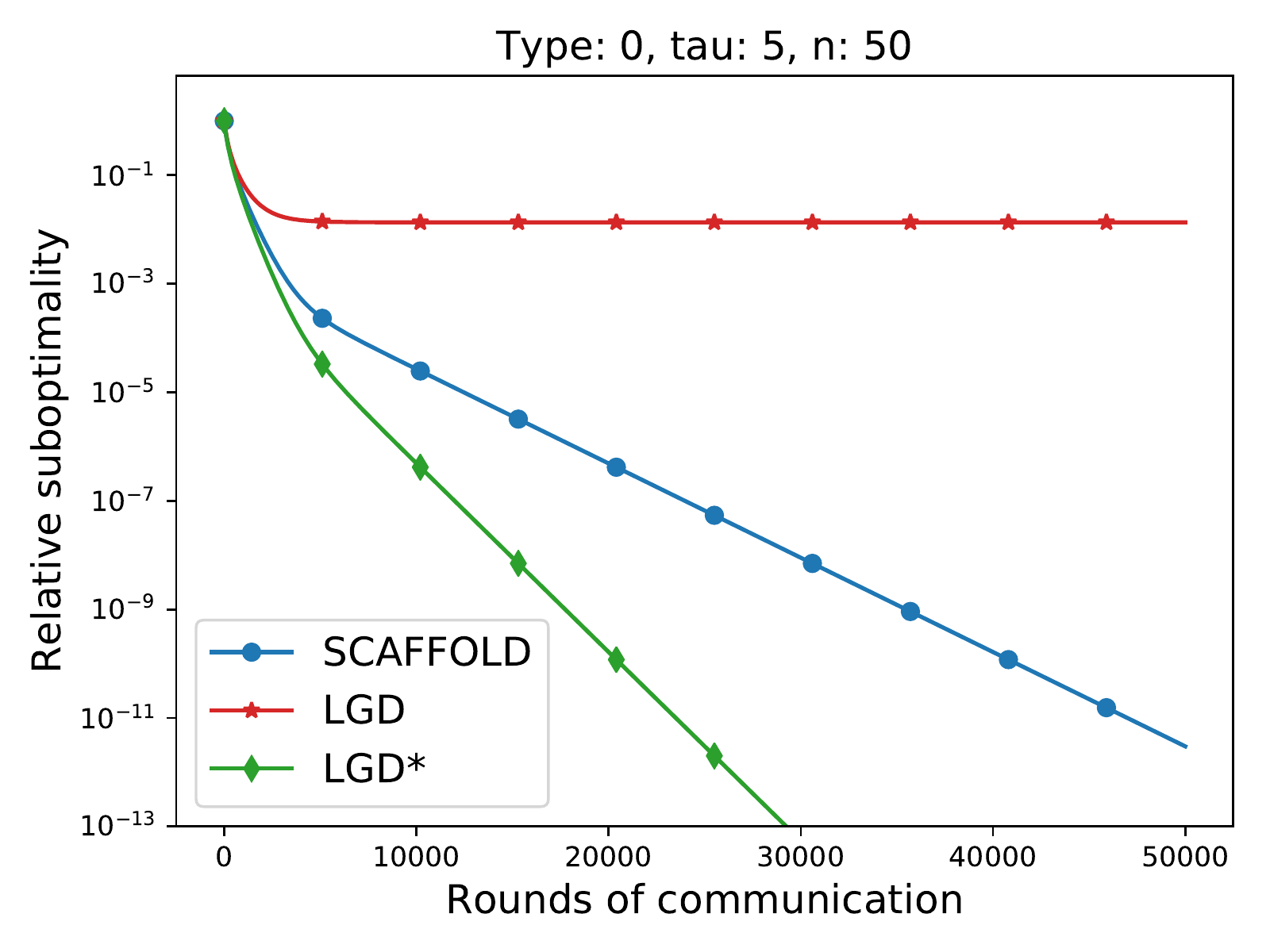}
        %\caption{ Residual vs. iteration  }\label{fig:bl_ex_flops}
\end{minipage}
\begin{minipage}{0.3\textwidth}
  \centering
\includegraphics[width =  \textwidth ]{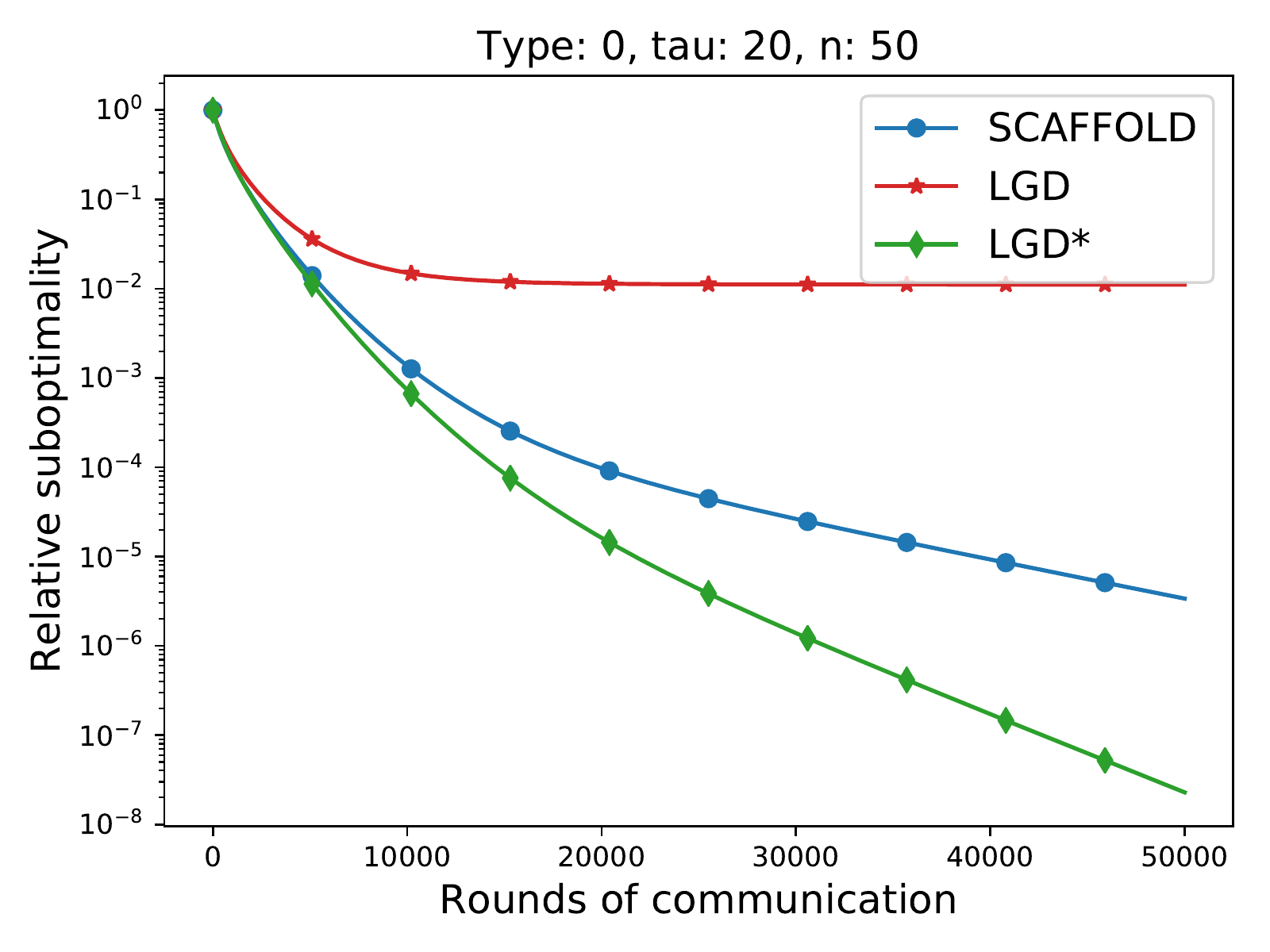}
        %\caption{ Residual vs. iteration  }\label{fig:bl_ex_flops}
\end{minipage}
\begin{minipage}{0.3\textwidth}
  \centering
\includegraphics[width =  \textwidth ]{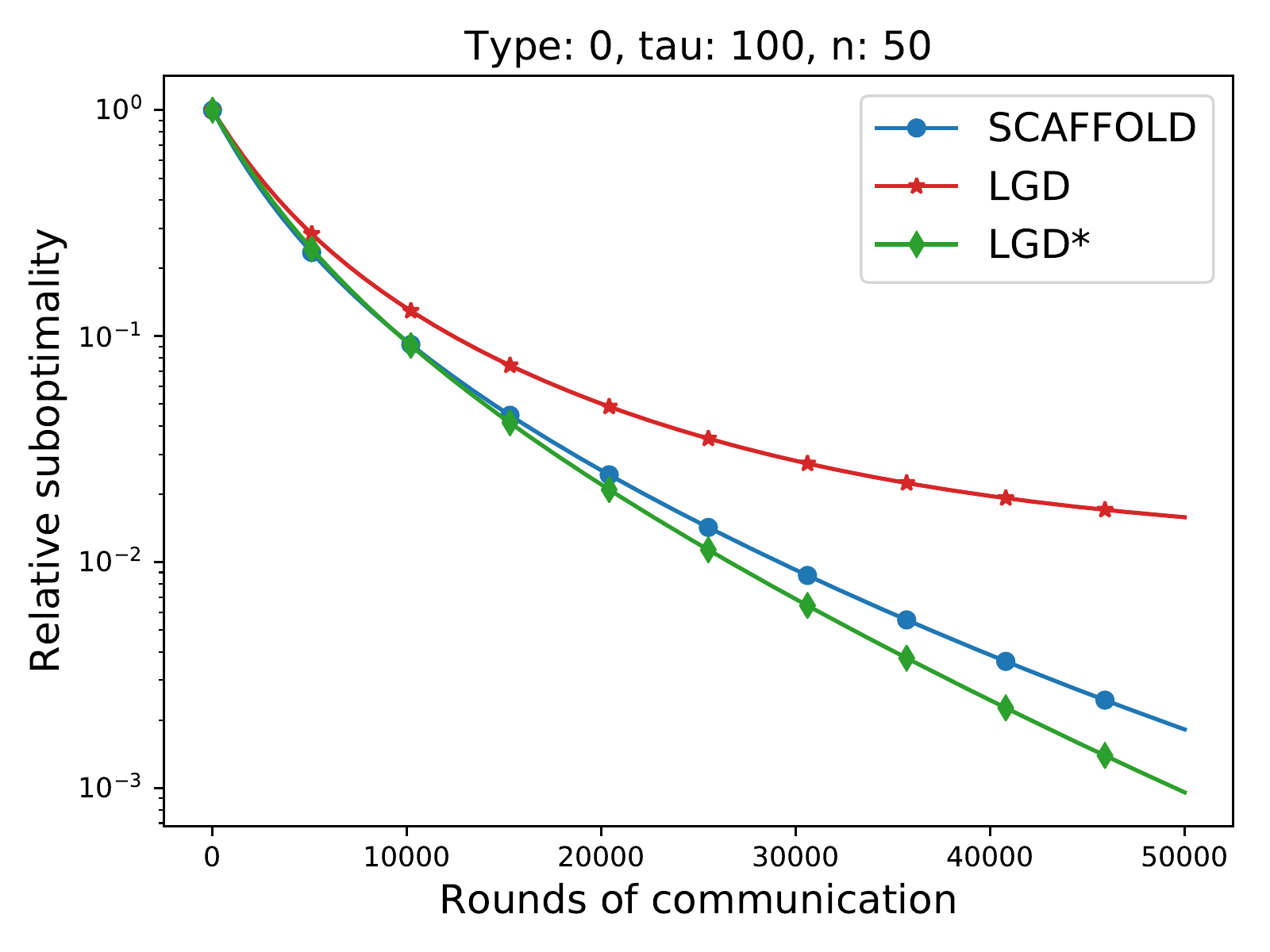}
        %\caption{ Residual vs. iteration  }\label{fig:bl_ex_flops}
\end{minipage}
\\
\begin{minipage}{0.3\textwidth}
  \centering
\includegraphics[width =  \textwidth ]{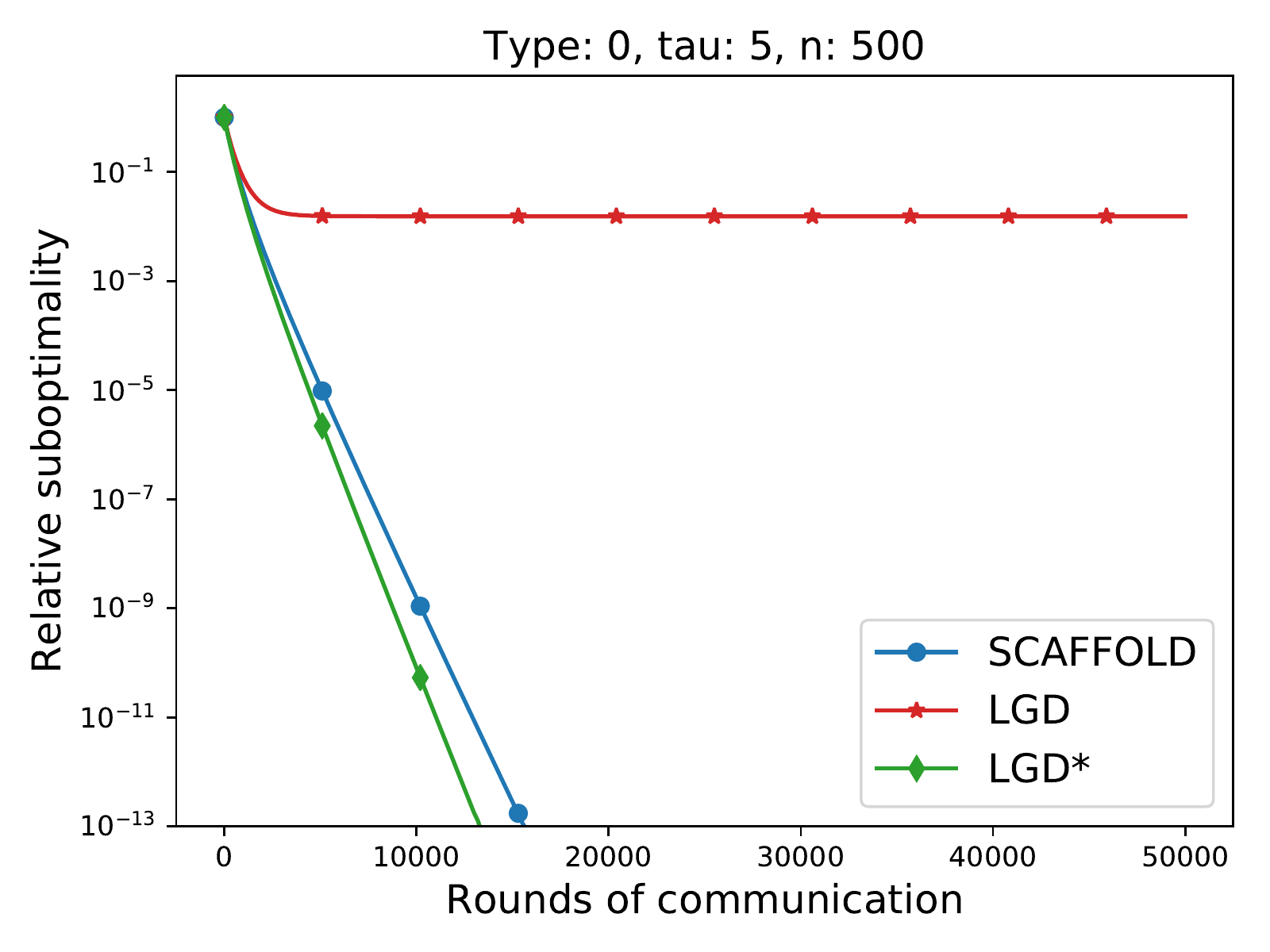}
        %\caption{ Residual vs. iteration  }\label{fig:bl_ex_flops}
\end{minipage}
\begin{minipage}{0.3\textwidth}
  \centering
\includegraphics[width =  \textwidth ]{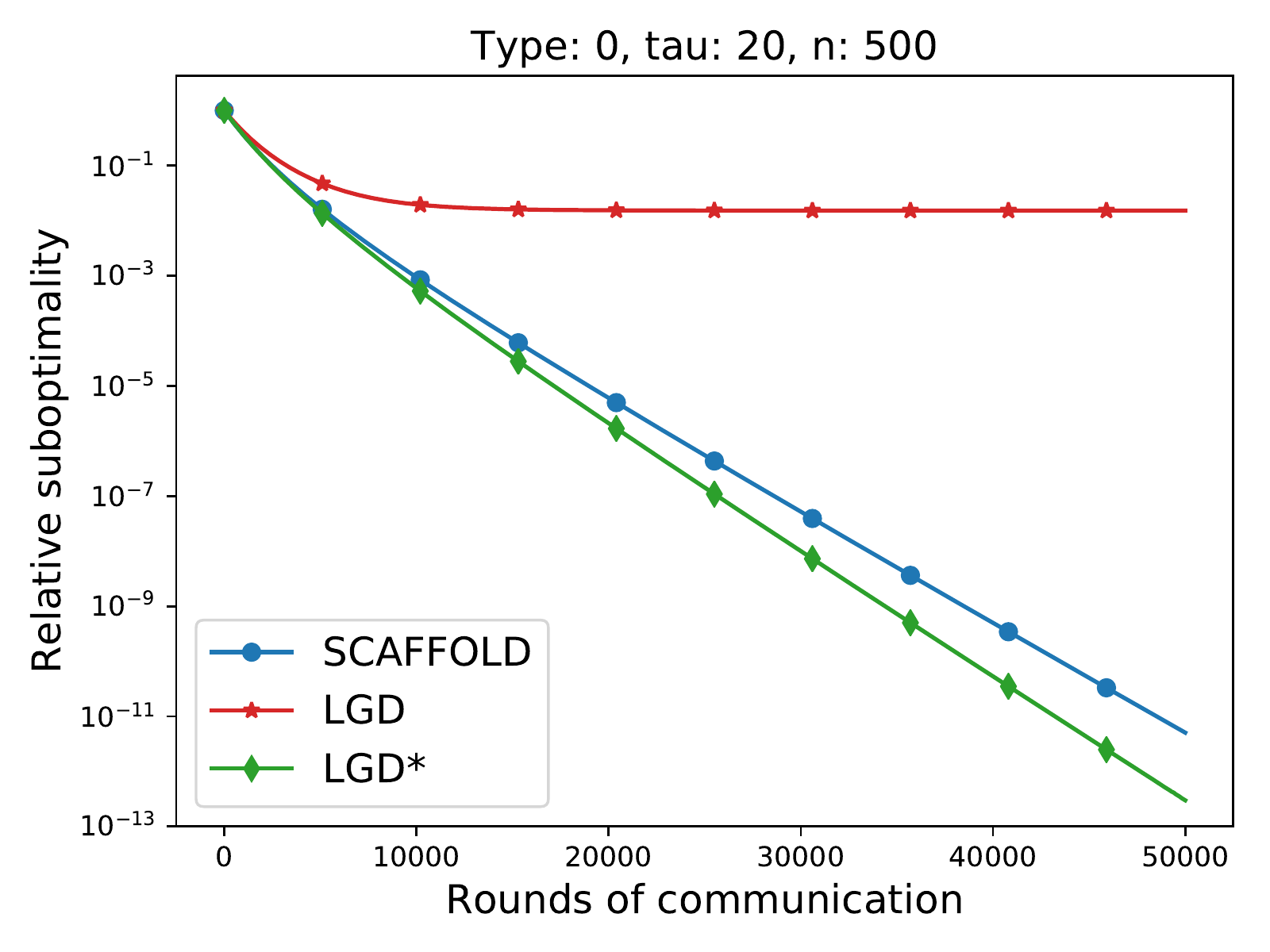}
        %\caption{ Residual vs. iteration  }\label{fig:bl_ex_flops}
\end{minipage}
\begin{minipage}{0.3\textwidth}
  \centering
\includegraphics[width =  \textwidth ]{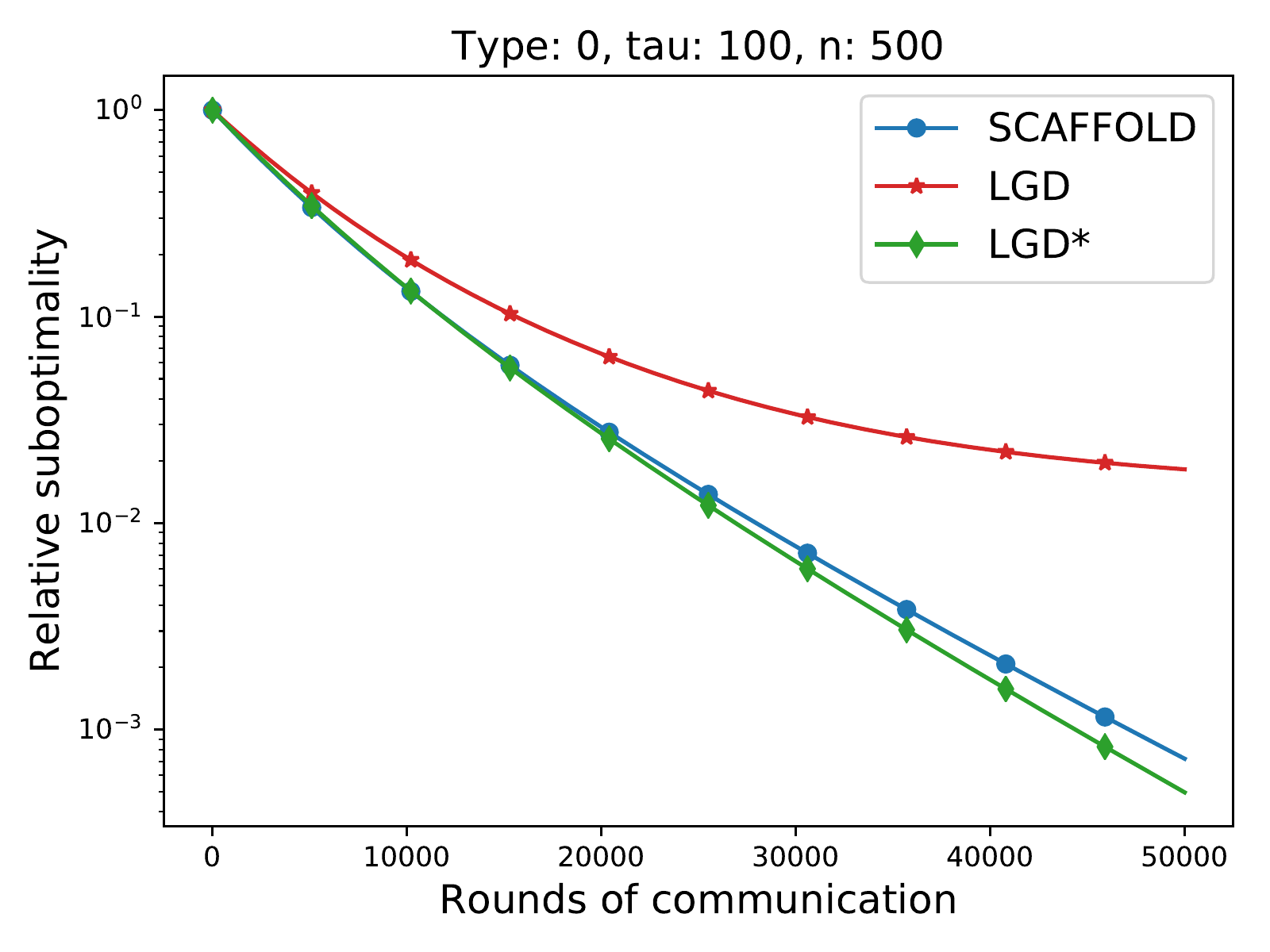}
        %\caption{ Residual vs. iteration  }\label{fig:bl_ex_flops}
\end{minipage}
\caption{Comparison of the following noiseless algorithms  {\tt Local-SGD} ({\tt LGD}, Algorithm~\ref{alg:local_sgd} with no local noise) and {\tt SCAFFOLD}~\cite{karimireddy2019scaffold} (Algorithm~\ref{alg:l_local_svrg} without ``Loopless'') and {\tt S*-Local-SGD} ({\tt LGD*}, Algorithm~\ref{alg:local_sgd_star}). Quadratic minimization, problem type 0 (see Table~\ref{tbl:instances}). }
\label{fig:artif1}
\end{figure}

\begin{figure}[!h]
\centering
\begin{minipage}{0.3\textwidth}
  \centering
\includegraphics[width =  \textwidth ]{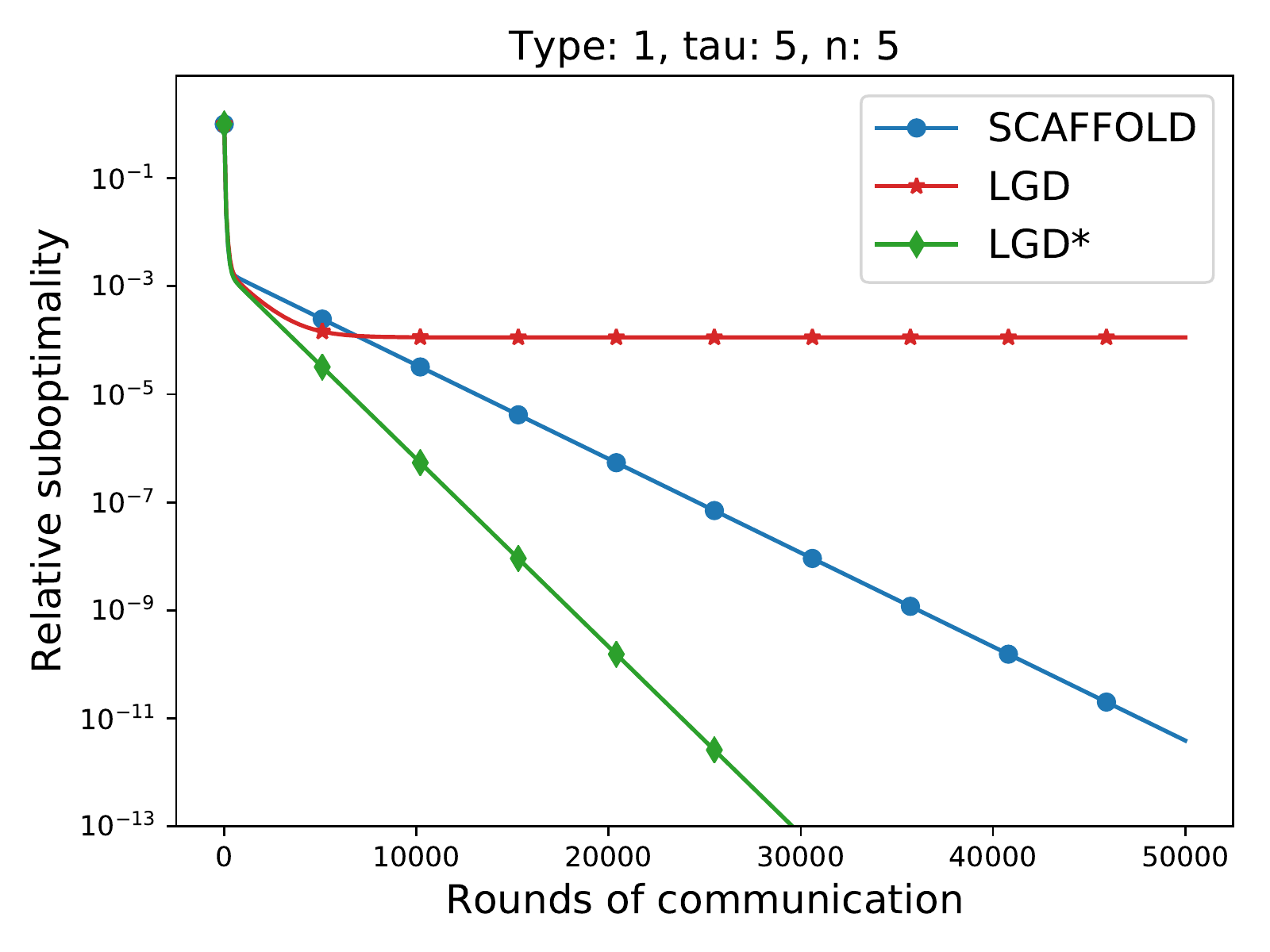}
        %\caption{ Residual vs. iteration  }\label{fig:bl_ex_flops}
\end{minipage}
\begin{minipage}{0.3\textwidth}
  \centering
\includegraphics[width =  \textwidth ]{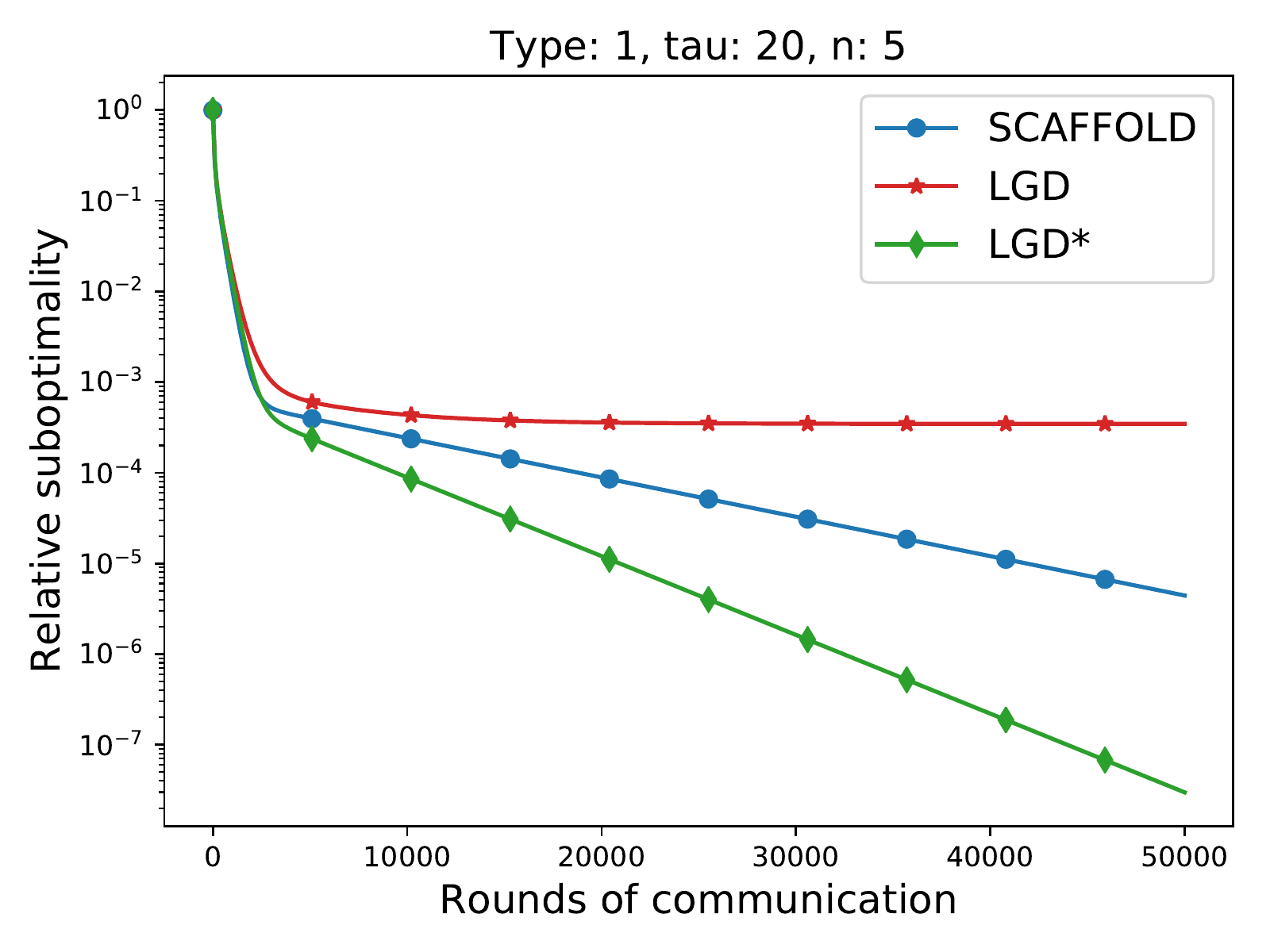}
        %\caption{ Residual vs. iteration  }\label{fig:bl_ex_flops}
\end{minipage}
\begin{minipage}{0.3\textwidth}
  \centering
\includegraphics[width =  \textwidth ]{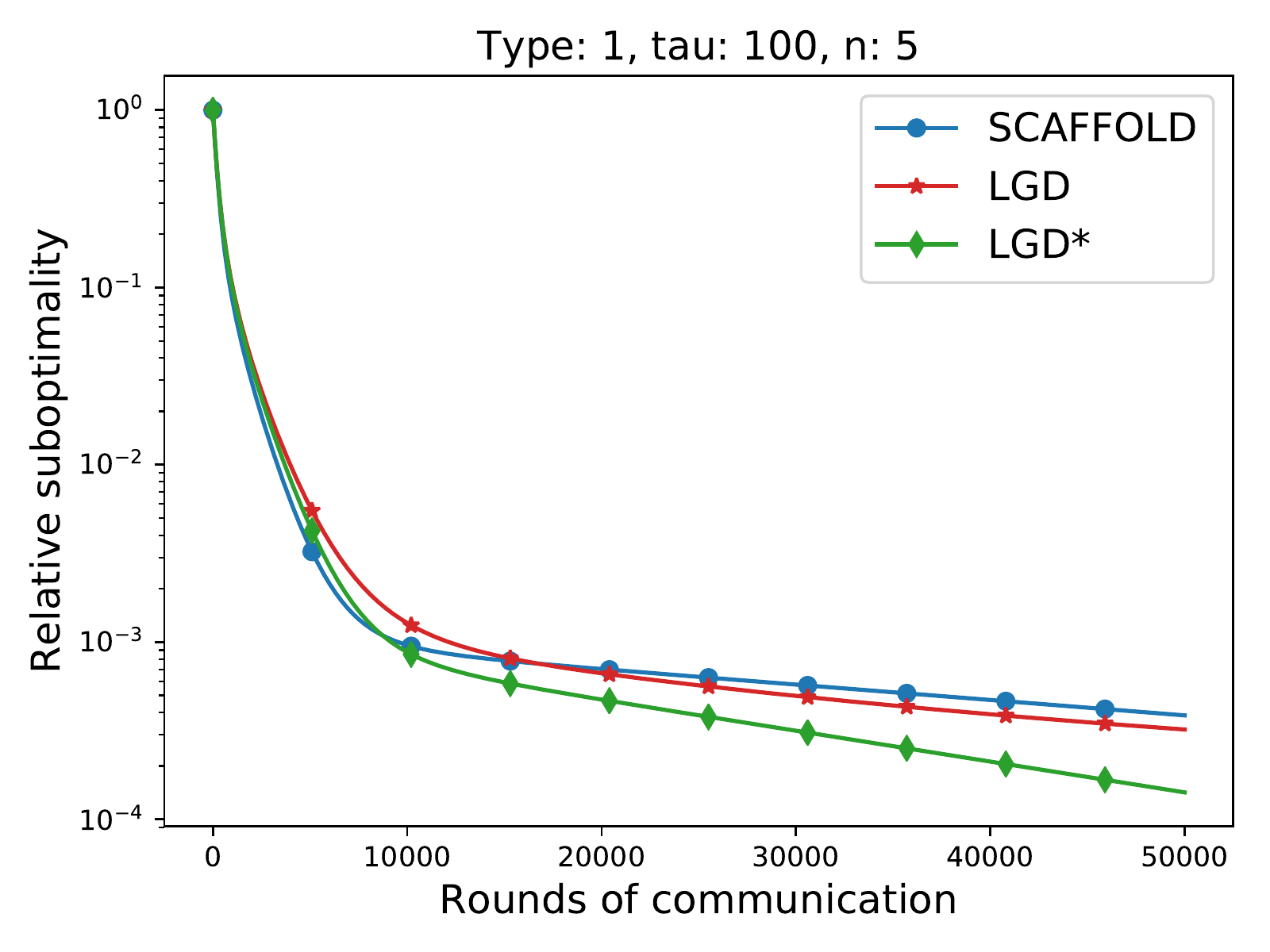}
        %\caption{ Residual vs. iteration  }\label{fig:bl_ex_flops}
\end{minipage}
\\
\begin{minipage}{0.3\textwidth}
  \centering
\includegraphics[width =  \textwidth ]{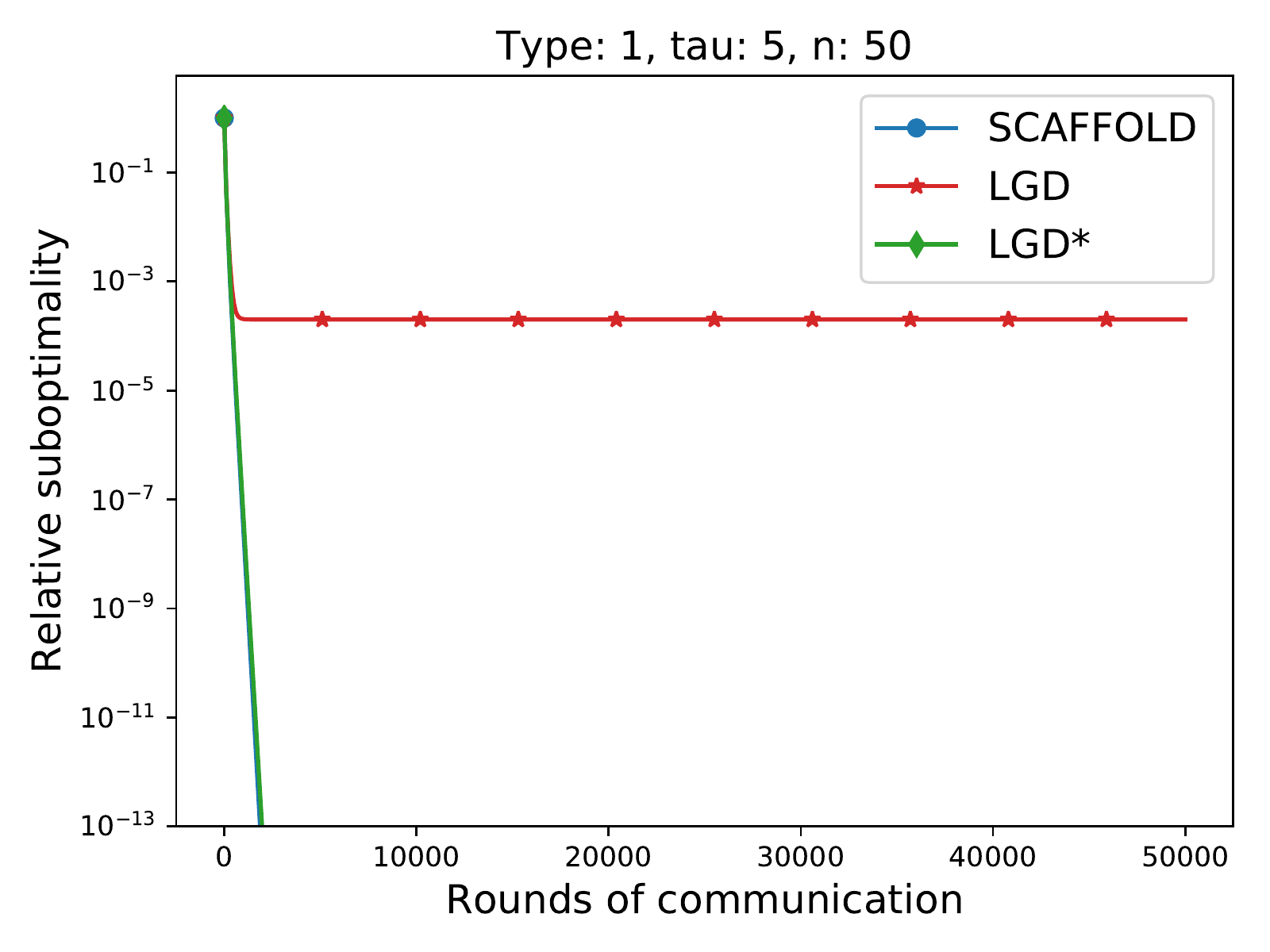}
        %\caption{ Residual vs. iteration  }\label{fig:bl_ex_flops}
\end{minipage}
\begin{minipage}{0.3\textwidth}
  \centering
\includegraphics[width =  \textwidth ]{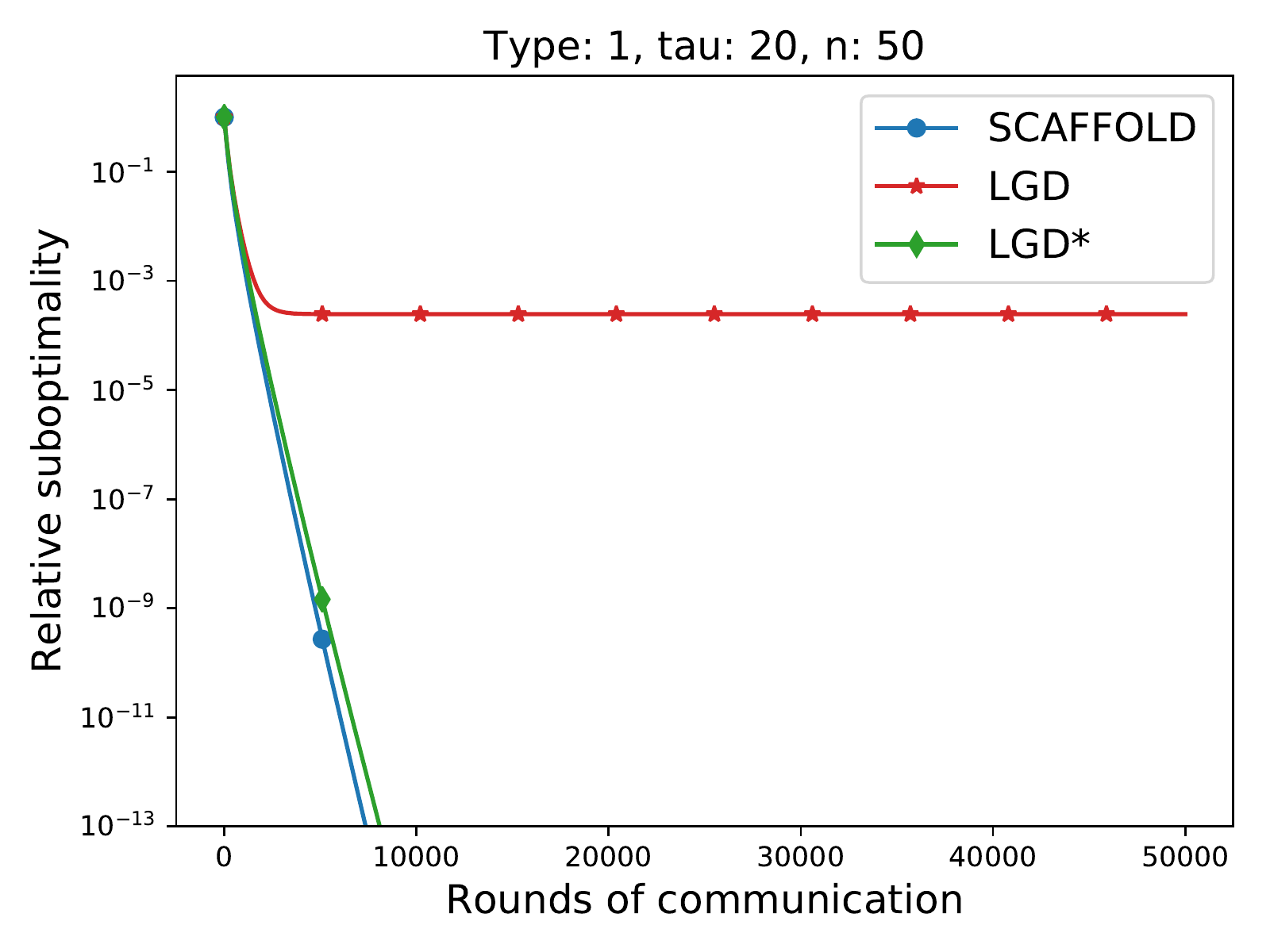}
        %\caption{ Residual vs. iteration  }\label{fig:bl_ex_flops}
\end{minipage}
\begin{minipage}{0.3\textwidth}
  \centering
\includegraphics[width =  \textwidth ]{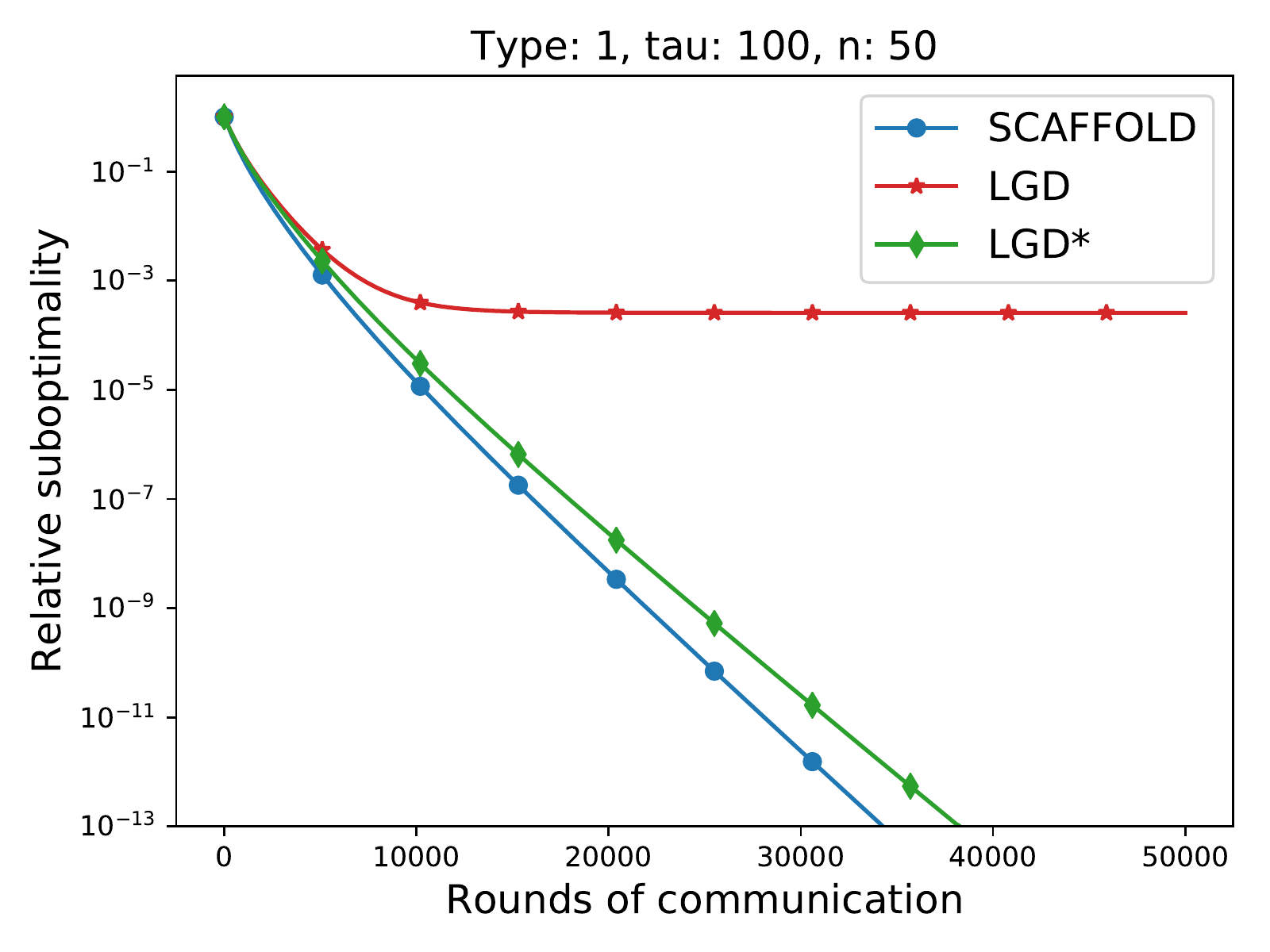}
        %\caption{ Residual vs. iteration  }\label{fig:bl_ex_flops}
\end{minipage}
\\
\begin{minipage}{0.3\textwidth}
  \centering
\includegraphics[width =  \textwidth ]{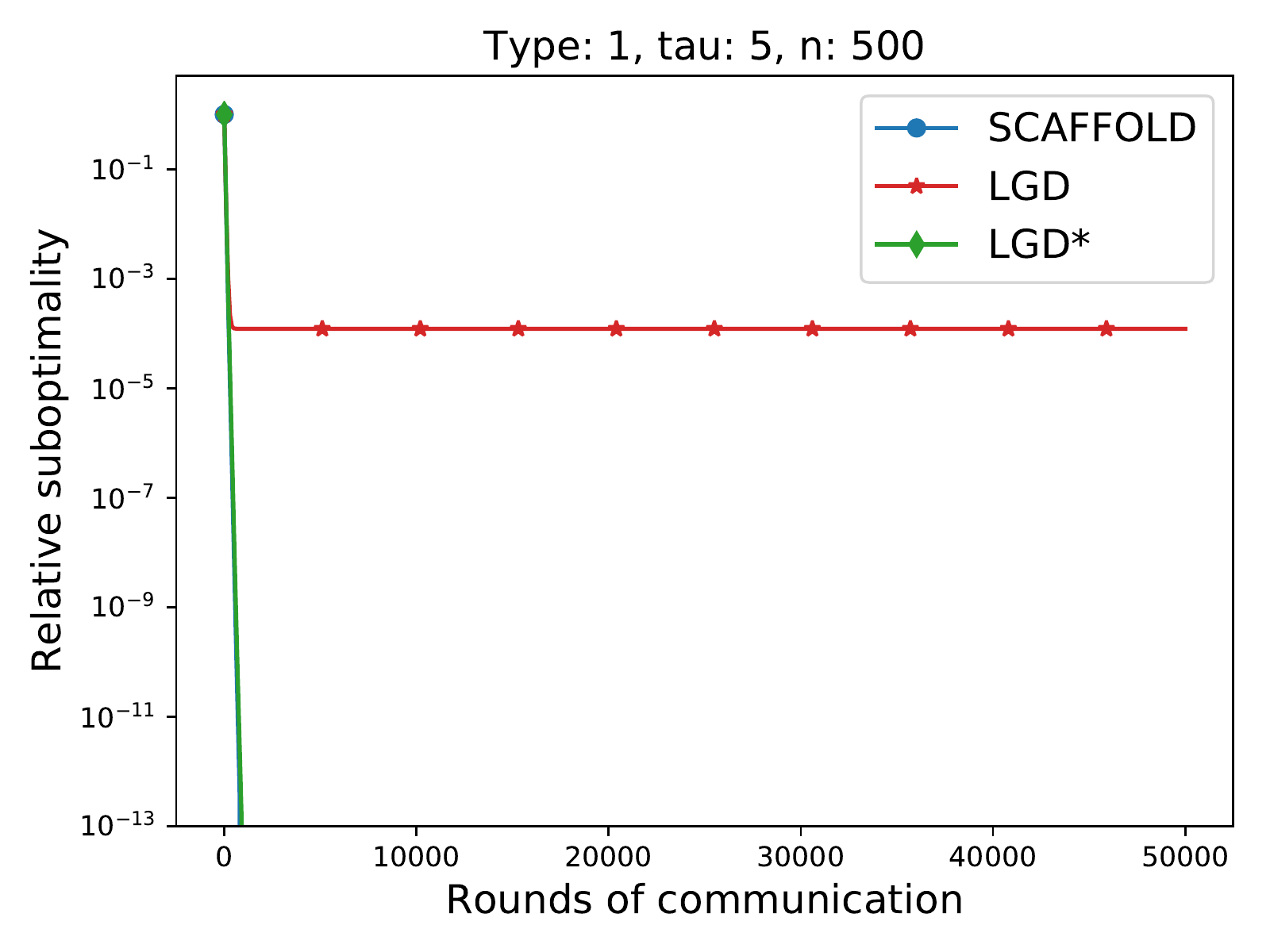}
        %\caption{ Residual vs. iteration  }\label{fig:bl_ex_flops}
\end{minipage}
\begin{minipage}{0.3\textwidth}
  \centering
\includegraphics[width =  \textwidth ]{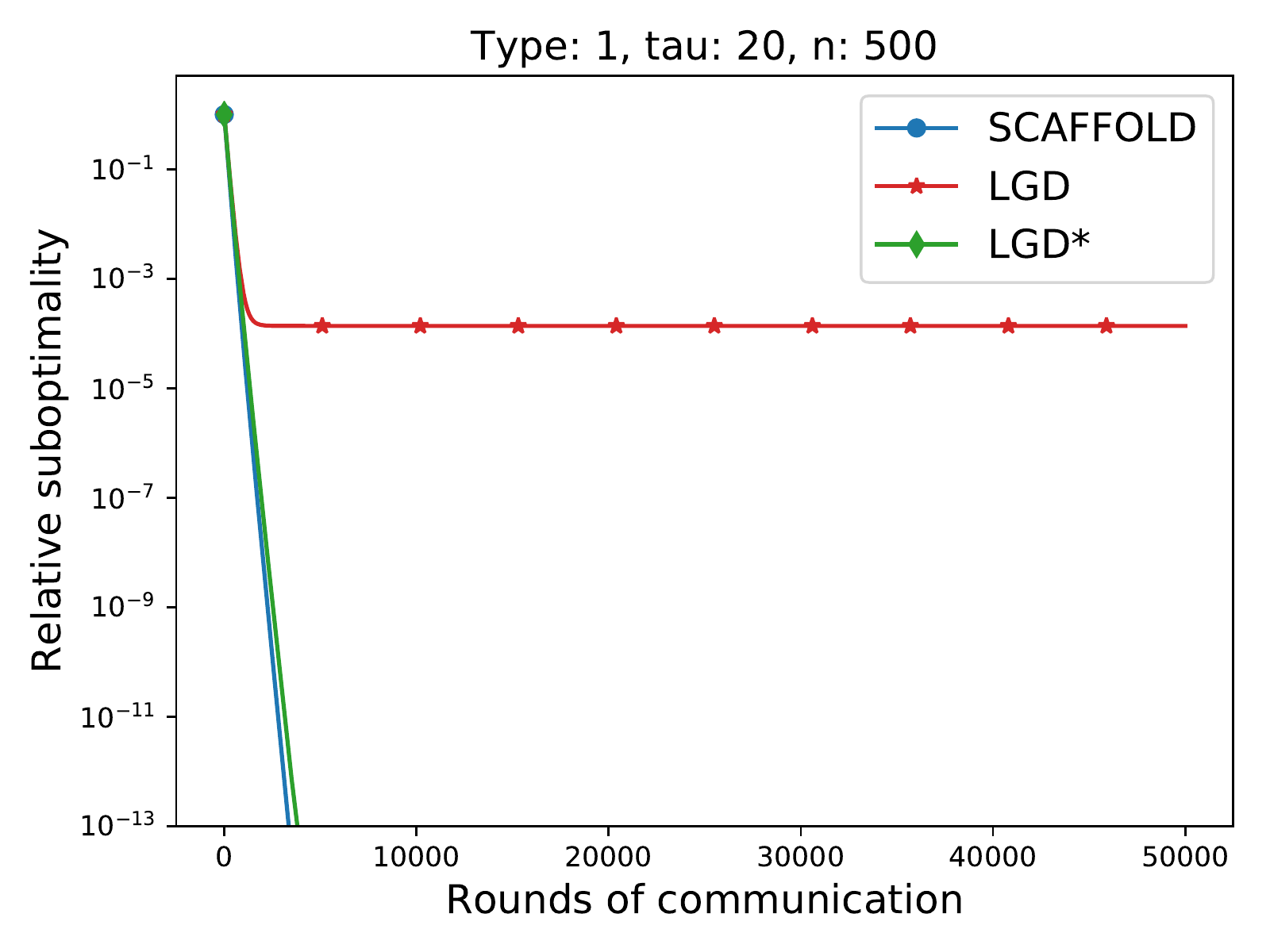}
        %\caption{ Residual vs. iteration  }\label{fig:bl_ex_flops}
\end{minipage}
\begin{minipage}{0.3\textwidth}
  \centering
\includegraphics[width =  \textwidth ]{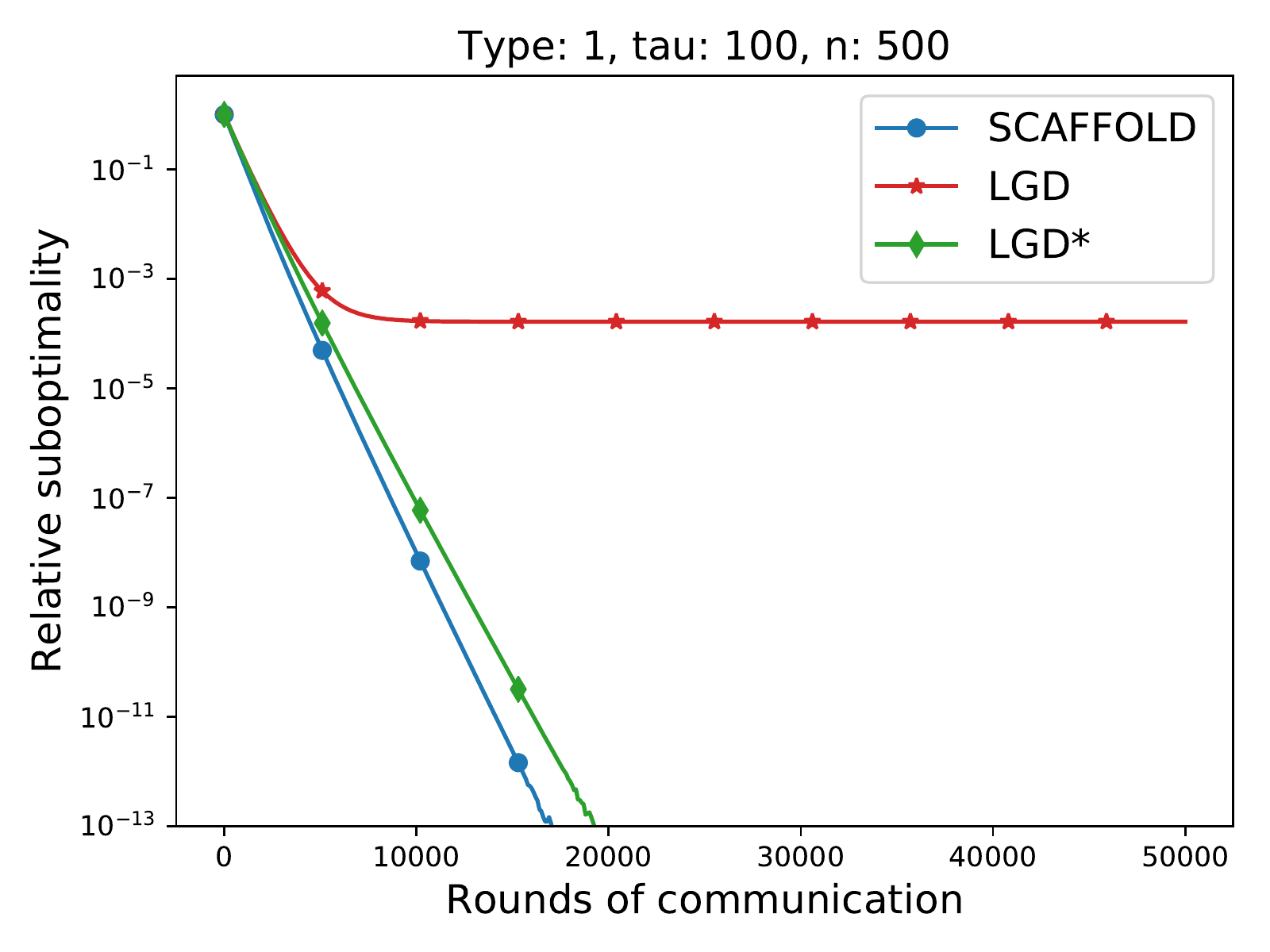}
        %\caption{ Residual vs. iteration  }\label{fig:bl_ex_flops}
\end{minipage}
\caption{Comparison of the following noiseless algorithms  {\tt Local-SGD} ({\tt LGD}, Algorithm~\ref{alg:local_sgd} with no local noise) and {\tt SCAFFOLD}~\cite{karimireddy2019scaffold} (Algorithm~\ref{alg:l_local_svrg} without ``Loopless'') and {\tt S*-Local-SGD} ({\tt LGD*}, Algorithm~\ref{alg:local_sgd_star}). Quadratic minimization, problem type 1 (see Table~\ref{tbl:instances}). }
\label{fig:artif2}
\end{figure}

\begin{figure}[!h]
\centering
\begin{minipage}{0.3\textwidth}
  \centering
\includegraphics[width =  \textwidth ]{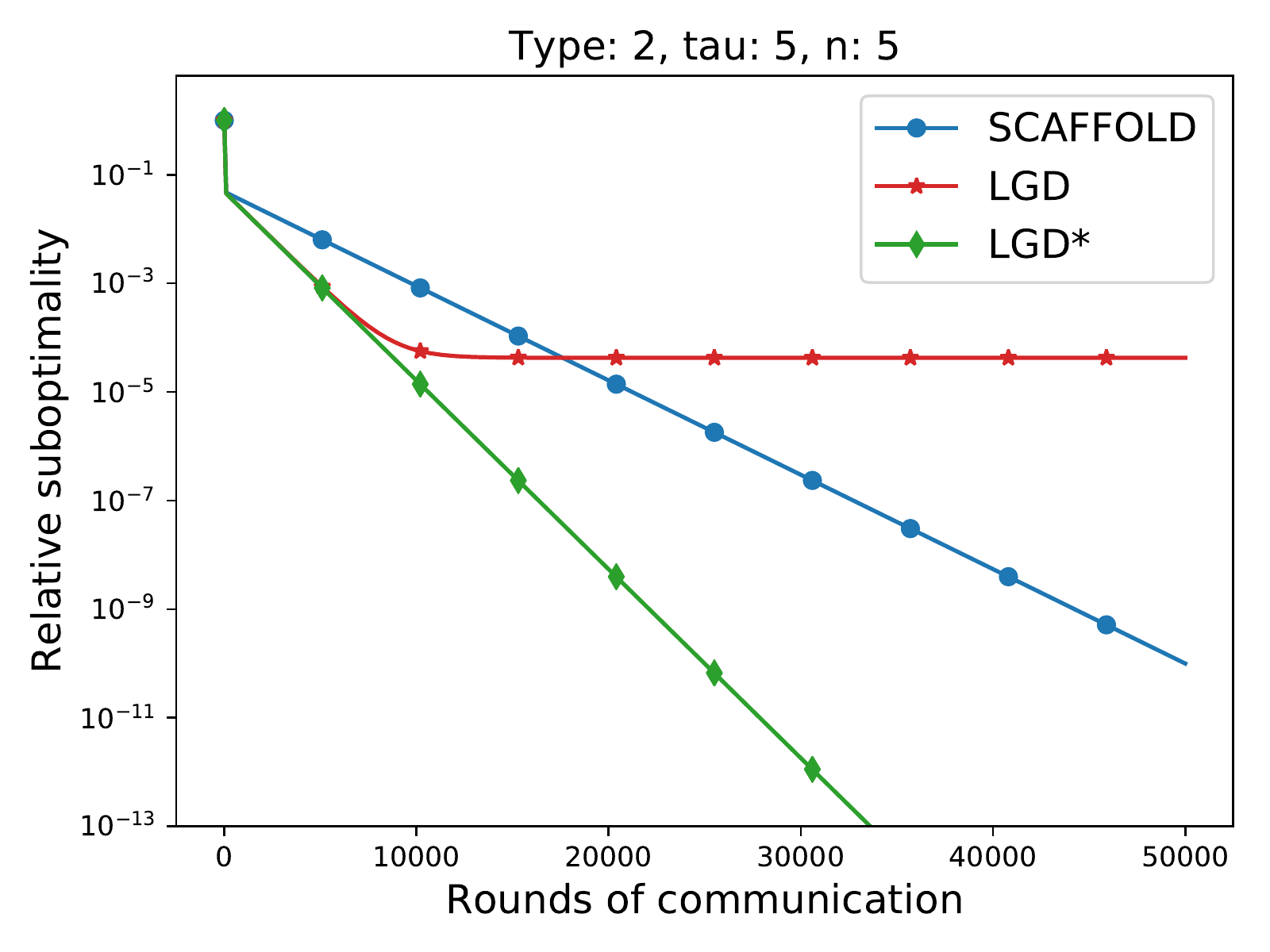}
        %\caption{ Residual vs. iteration  }\label{fig:bl_ex_flops}
\end{minipage}
\begin{minipage}{0.3\textwidth}
  \centering
\includegraphics[width =  \textwidth ]{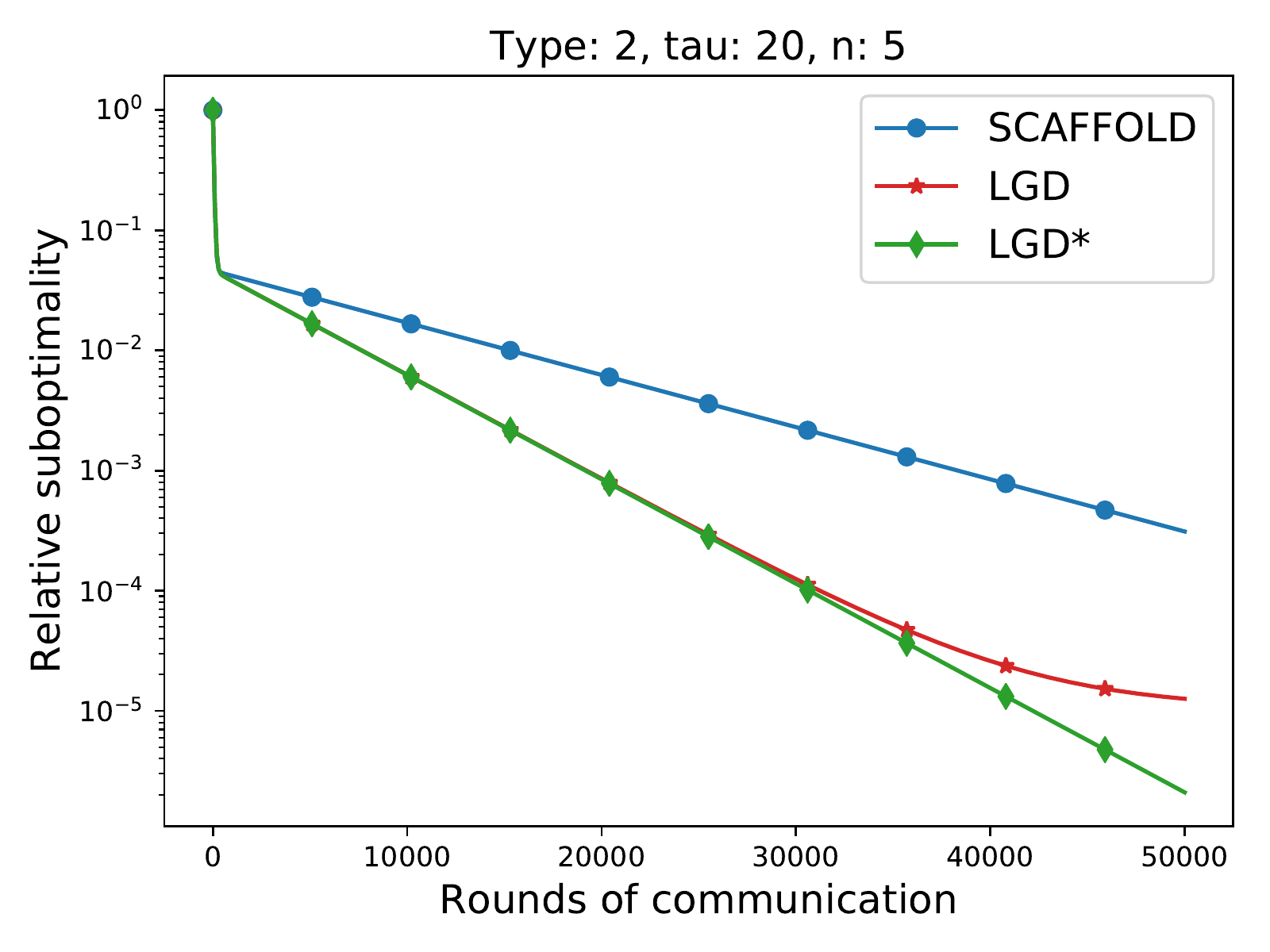}
        %\caption{ Residual vs. iteration  }\label{fig:bl_ex_flops}
\end{minipage}
\begin{minipage}{0.3\textwidth}
  \centering
\includegraphics[width =  \textwidth ]{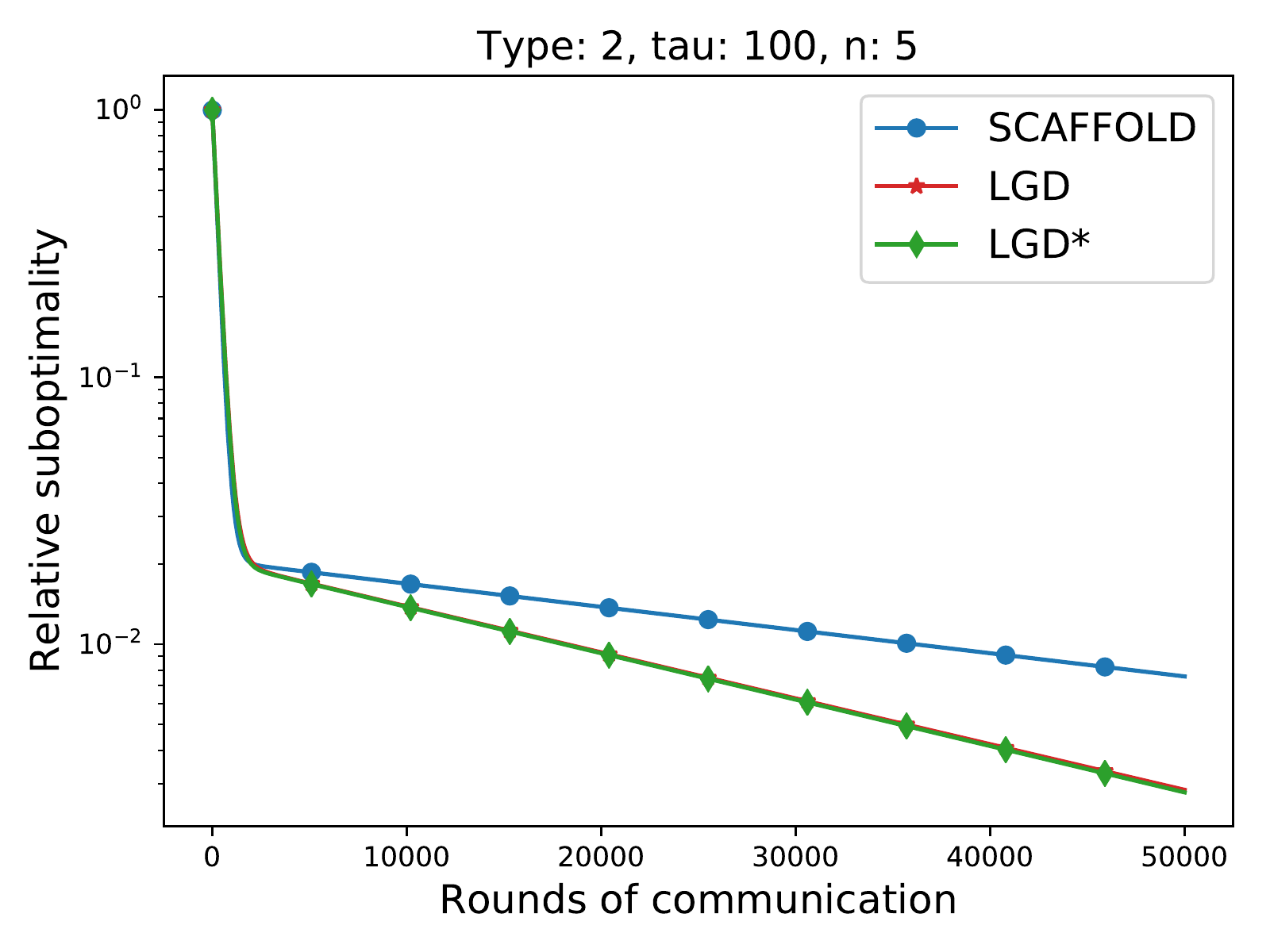}
        %\caption{ Residual vs. iteration  }\label{fig:bl_ex_flops}
\end{minipage}
\\
\begin{minipage}{0.3\textwidth}
  \centering
\includegraphics[width =  \textwidth ]{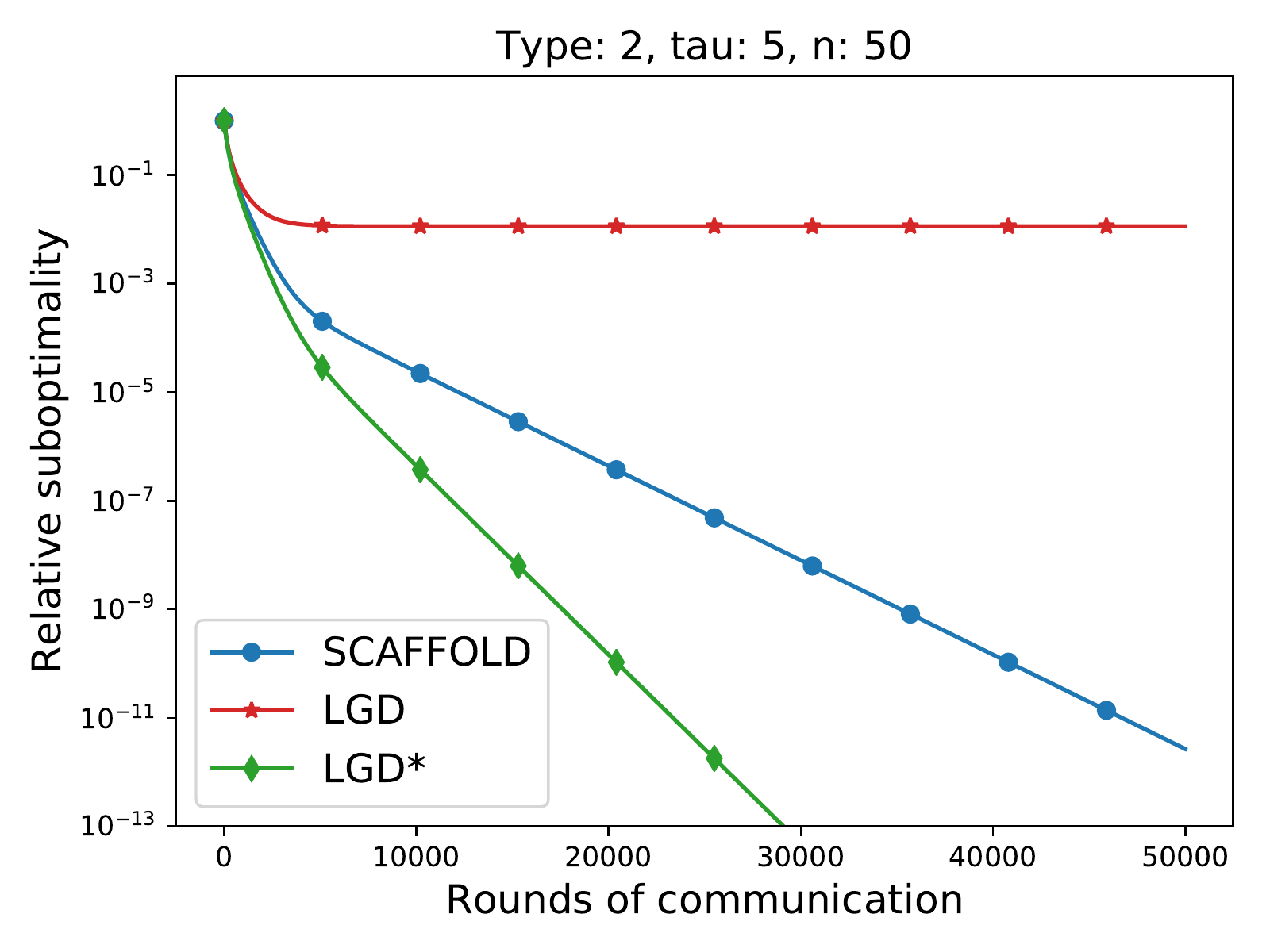}
        %\caption{ Residual vs. iteration  }\label{fig:bl_ex_flops}
\end{minipage}
\begin{minipage}{0.3\textwidth}
  \centering
\includegraphics[width =  \textwidth ]{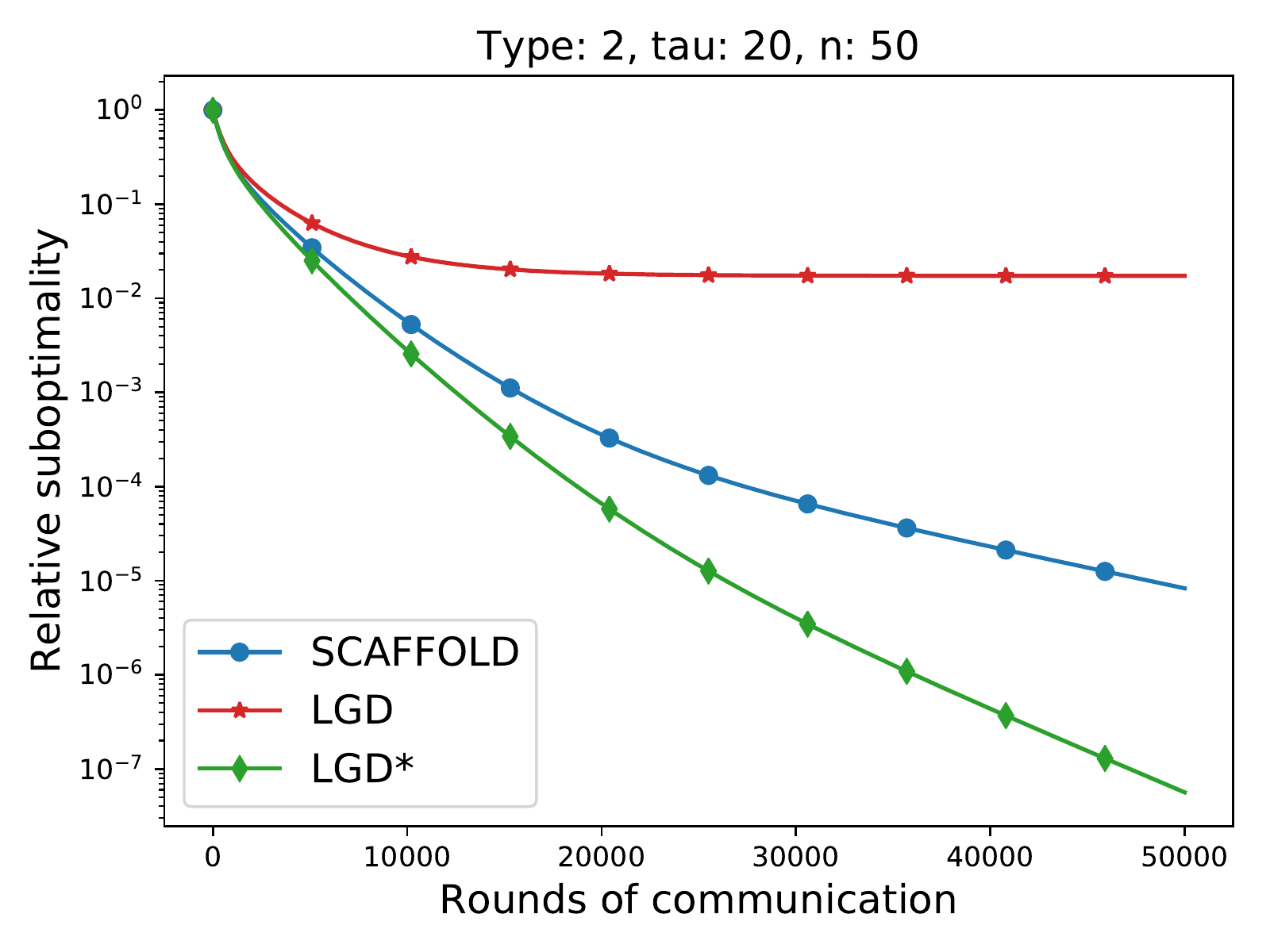}
        %\caption{ Residual vs. iteration  }\label{fig:bl_ex_flops}
\end{minipage}
\begin{minipage}{0.3\textwidth}
  \centering
\includegraphics[width =  \textwidth ]{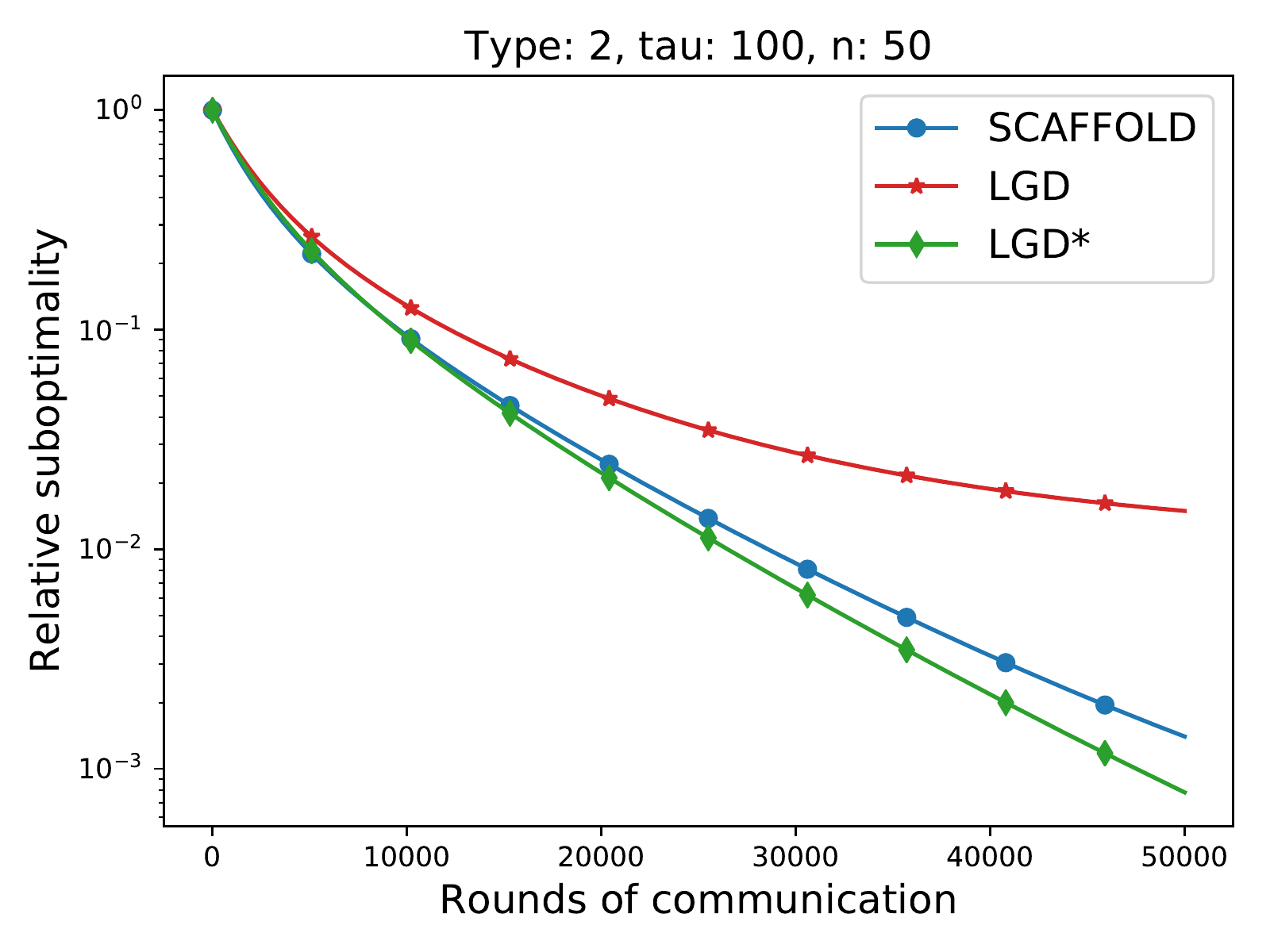}
        %\caption{ Residual vs. iteration  }\label{fig:bl_ex_flops}
\end{minipage}
\\
\begin{minipage}{0.3\textwidth}
  \centering
\includegraphics[width =  \textwidth ]{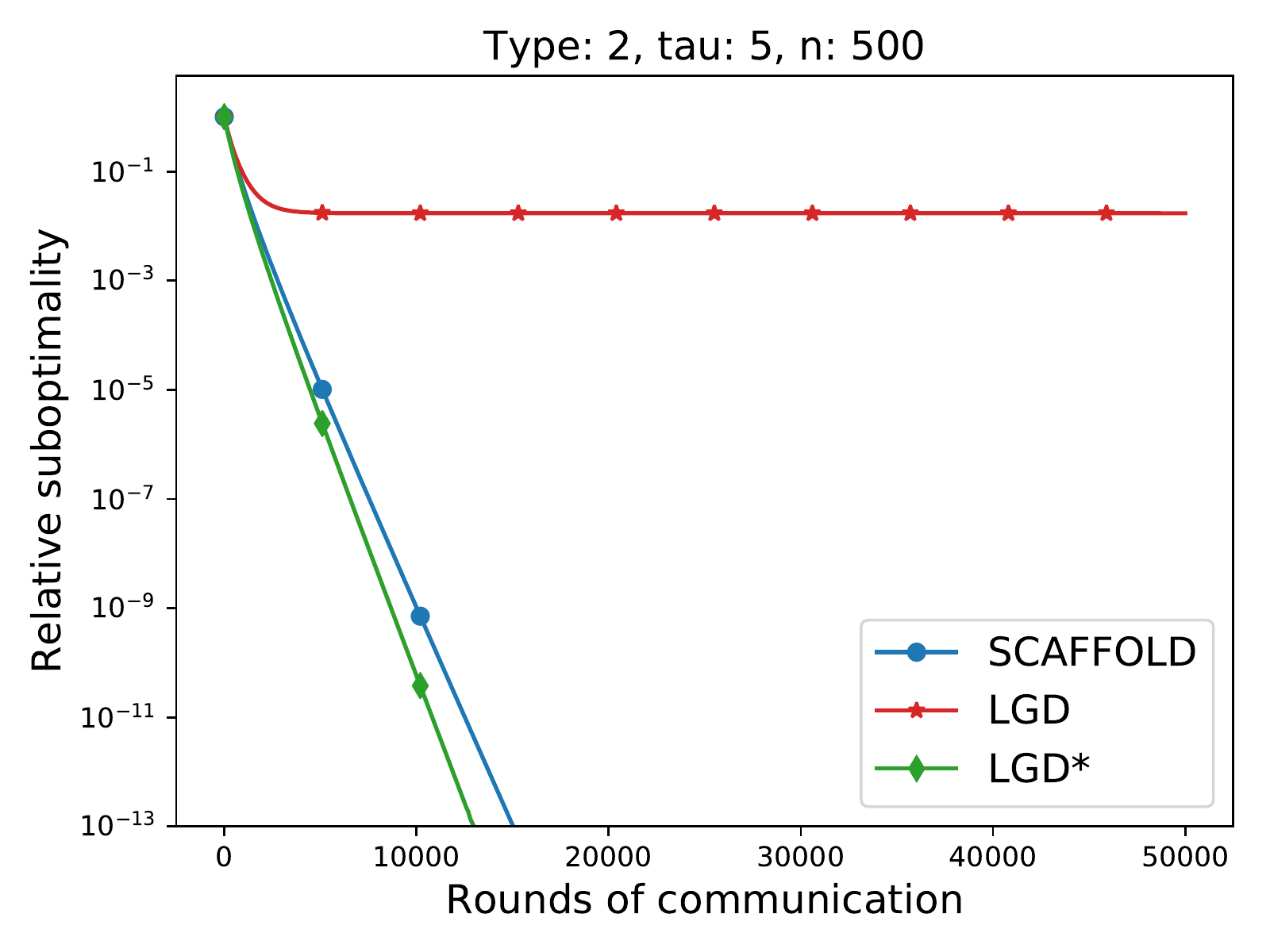}
        %\caption{ Residual vs. iteration  }\label{fig:bl_ex_flops}
\end{minipage}
\begin{minipage}{0.3\textwidth}
  \centering
\includegraphics[width =  \textwidth ]{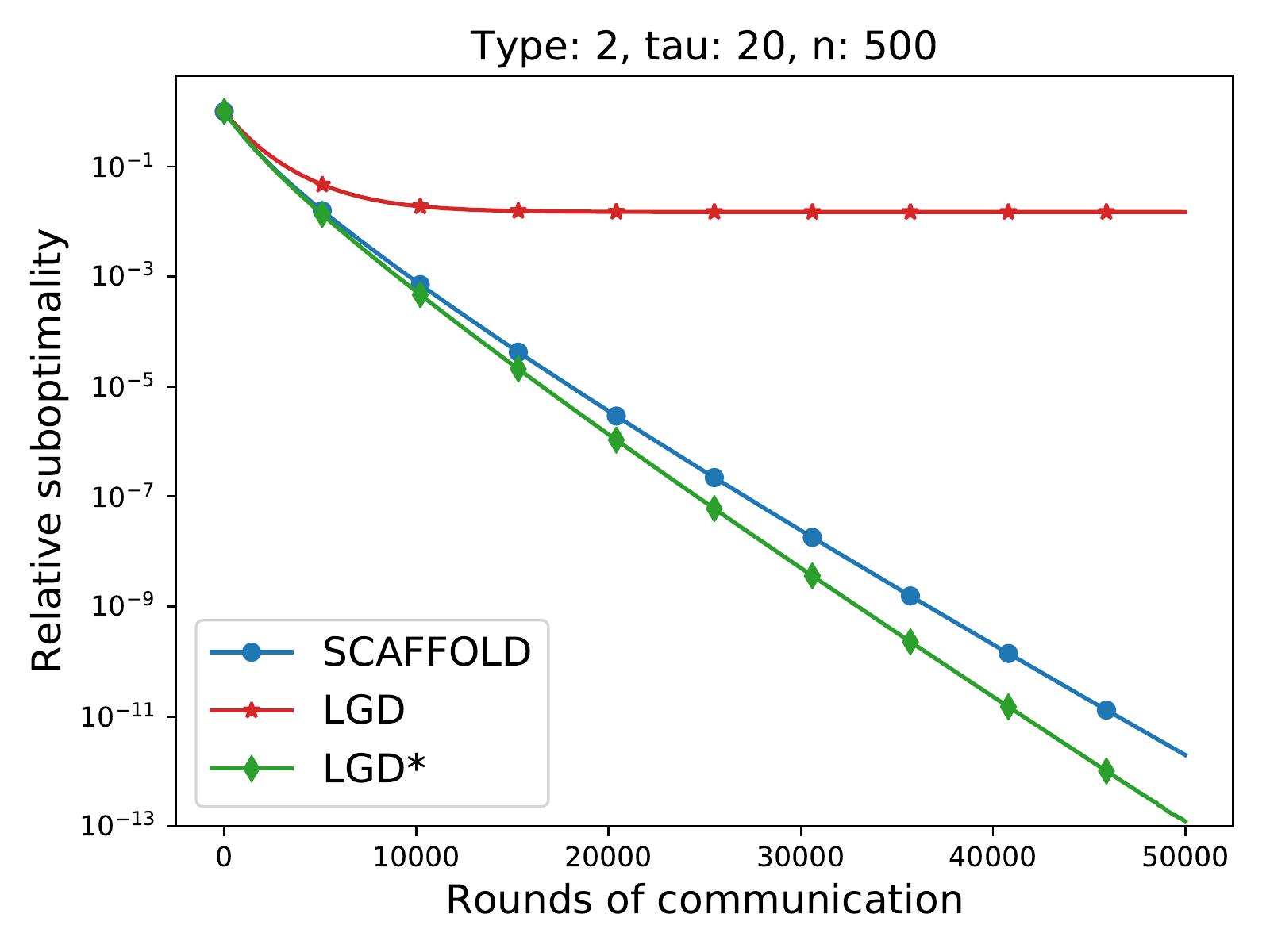}
        %\caption{ Residual vs. iteration  }\label{fig:bl_ex_flops}
\end{minipage}
\begin{minipage}{0.3\textwidth}
  \centering
\includegraphics[width =  \textwidth ]{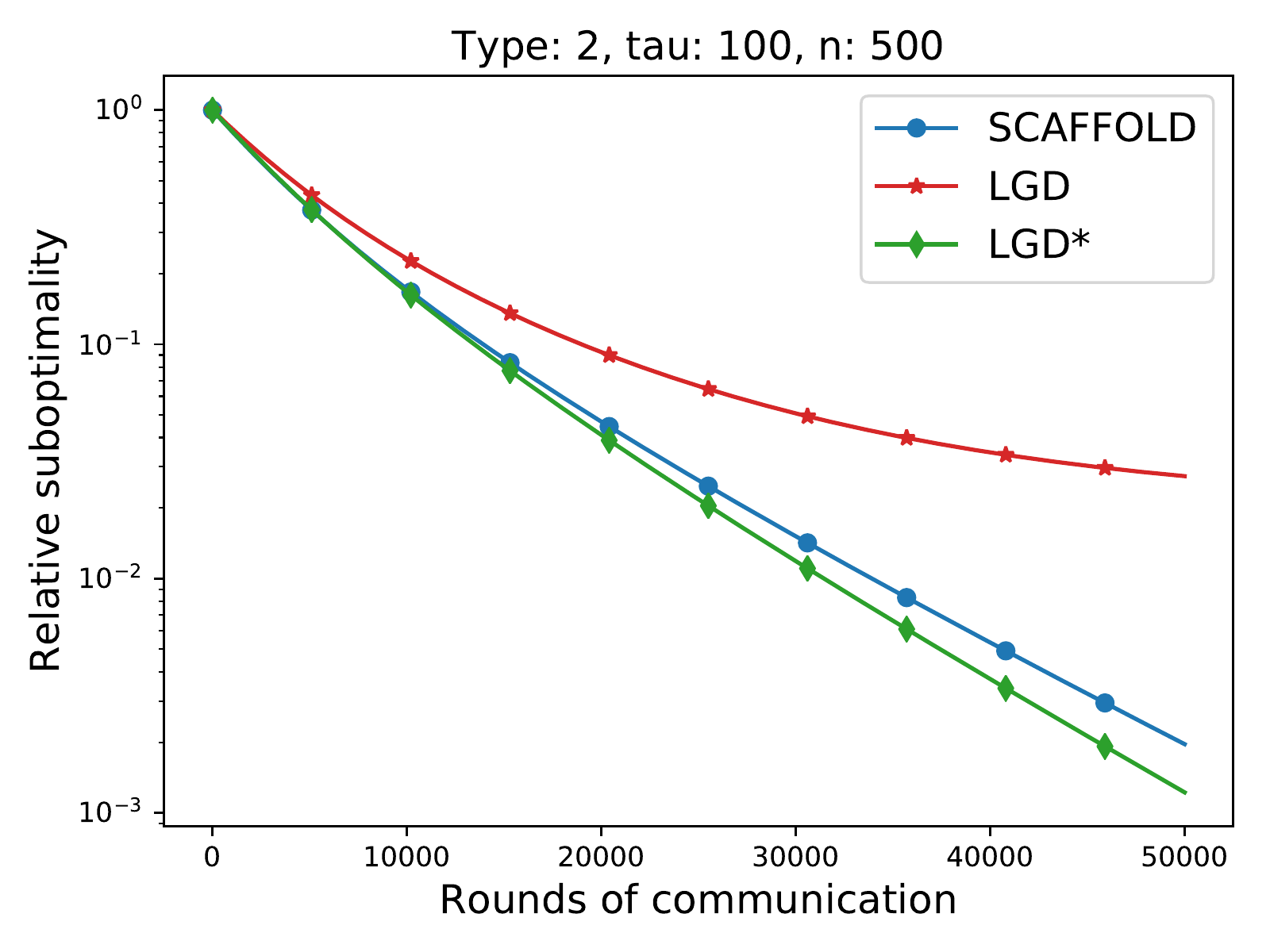}
        %\caption{ Residual vs. iteration  }\label{fig:bl_ex_flops}
\end{minipage}
\caption{Comparison of the following noiseless algorithms  {\tt Local-SGD} ({\tt LGD}, Algorithm~\ref{alg:local_sgd} with no local noise) and {\tt SCAFFOLD}~\cite{karimireddy2019scaffold} (Algorithm~\ref{alg:l_local_svrg} without ``Loopless'') and {\tt S*-Local-SGD} ({\tt LGD*}, Algorithm~\ref{alg:local_sgd_star}). Quadratic minimization, problem type 2 (see Table~\ref{tbl:instances}). }
\label{fig:artif3}
\end{figure}

\begin{figure}[!h]
\centering
\begin{minipage}{0.3\textwidth}
  \centering
\includegraphics[width =  \textwidth ]{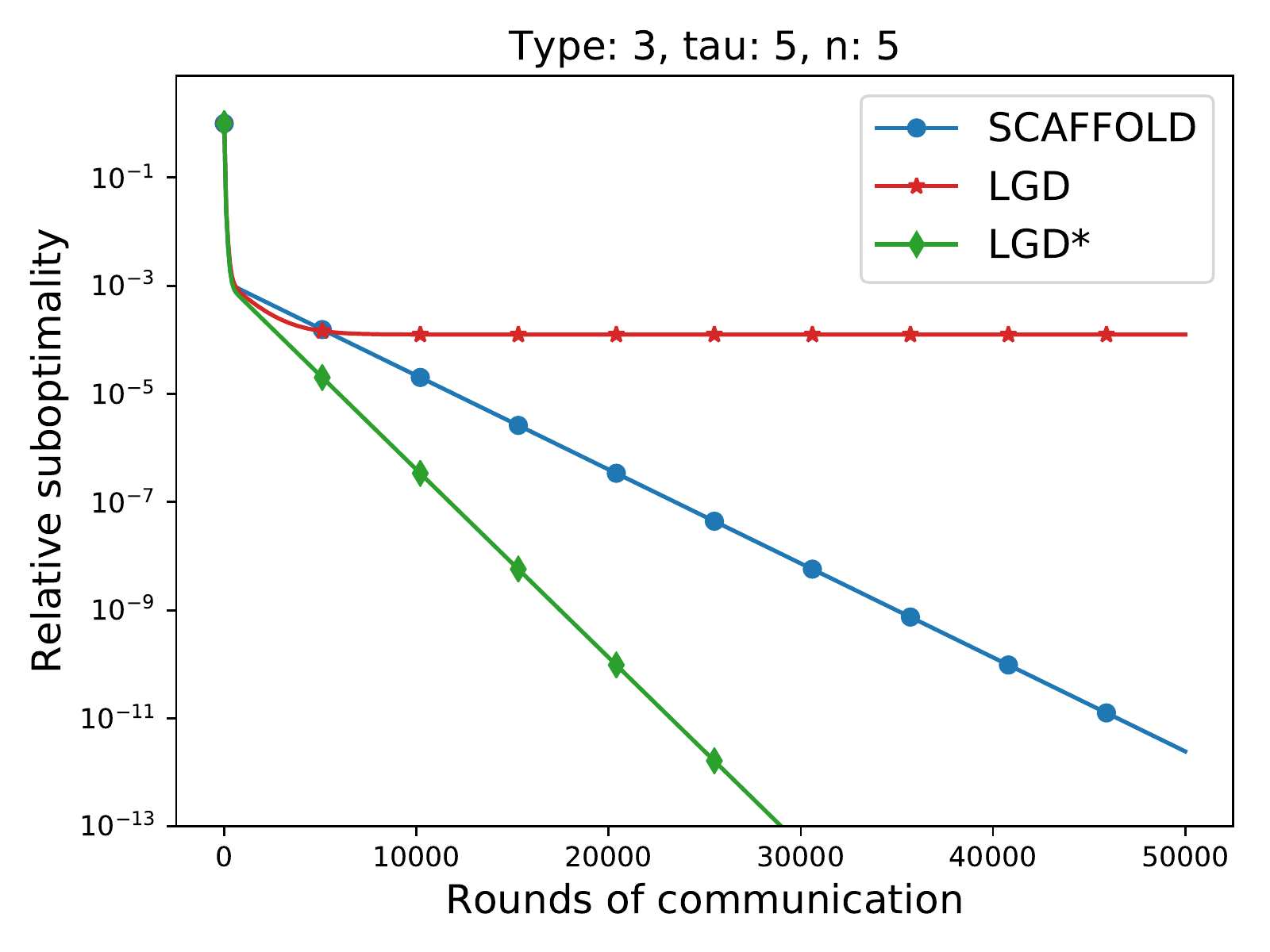}
        %\caption{ Residual vs. iteration  }\label{fig:bl_ex_flops}
\end{minipage}
\begin{minipage}{0.3\textwidth}
  \centering
\includegraphics[width =  \textwidth ]{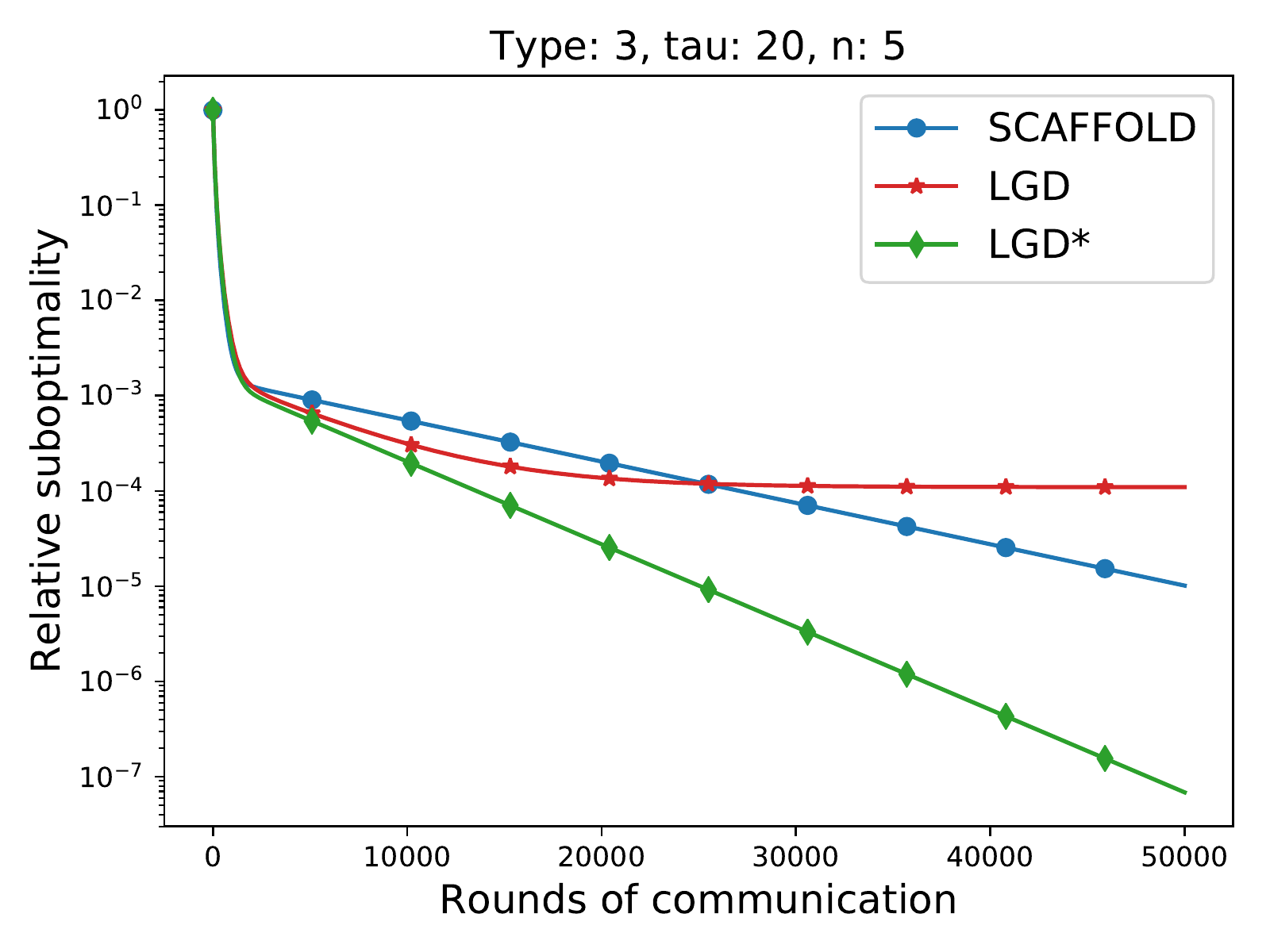}
        %\caption{ Residual vs. iteration  }\label{fig:bl_ex_flops}
\end{minipage}
\begin{minipage}{0.3\textwidth}
  \centering
\includegraphics[width =  \textwidth ]{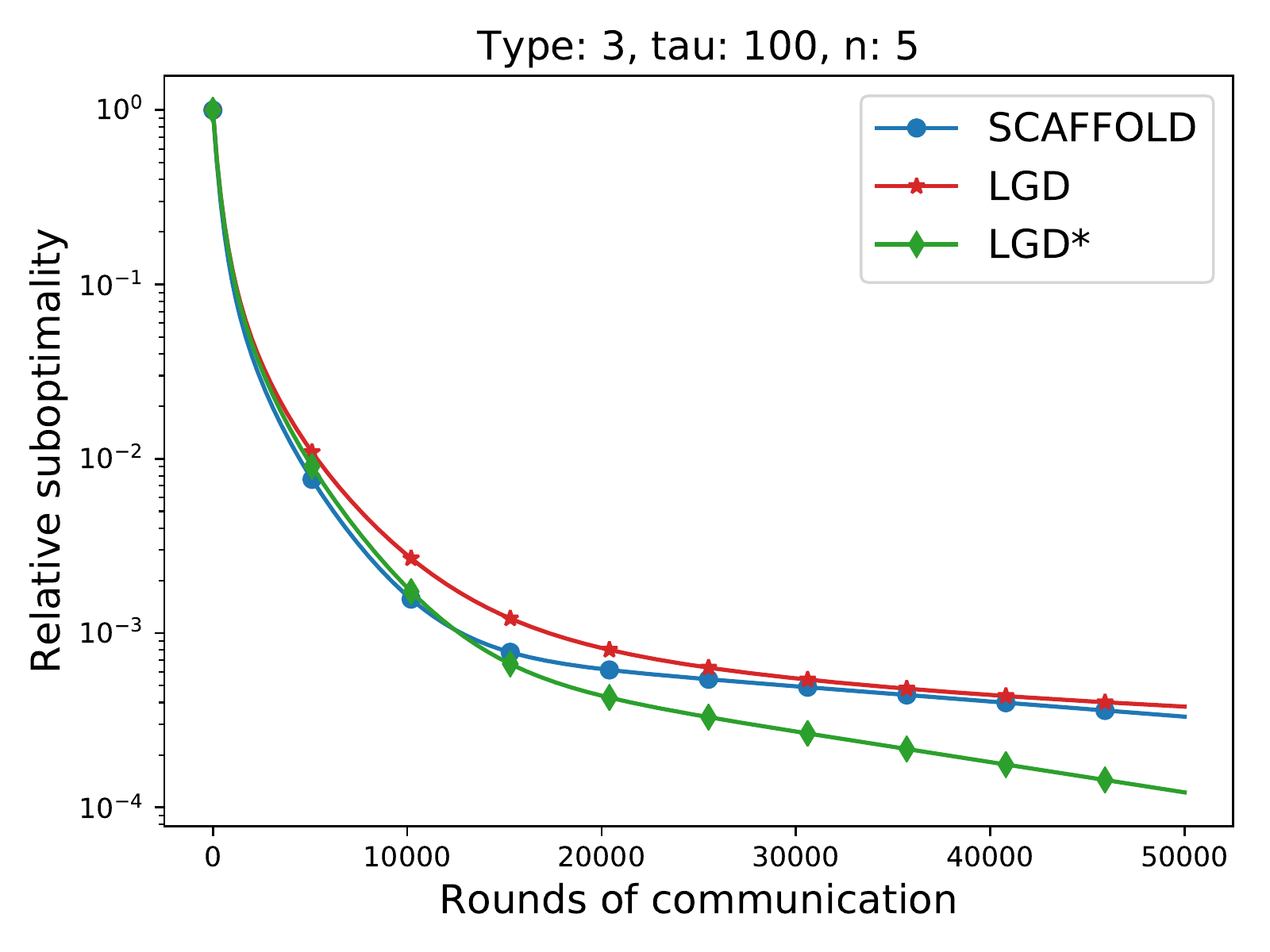}
        %\caption{ Residual vs. iteration  }\label{fig:bl_ex_flops}
\end{minipage}
\\
\begin{minipage}{0.3\textwidth}
  \centering
\includegraphics[width =  \textwidth ]{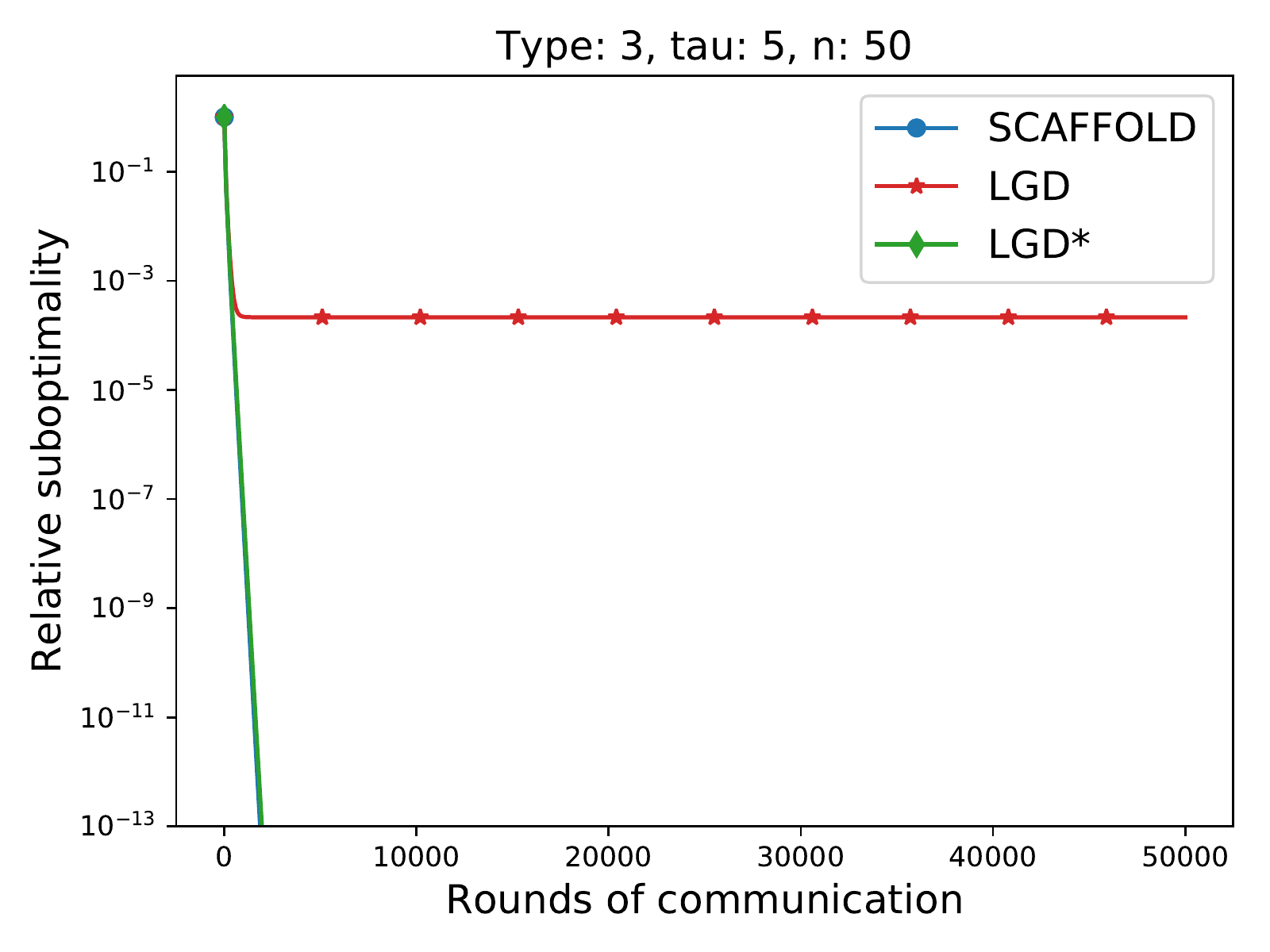}
        %\caption{ Residual vs. iteration  }\label{fig:bl_ex_flops}
\end{minipage}
\begin{minipage}{0.3\textwidth}
  \centering
\includegraphics[width =  \textwidth ]{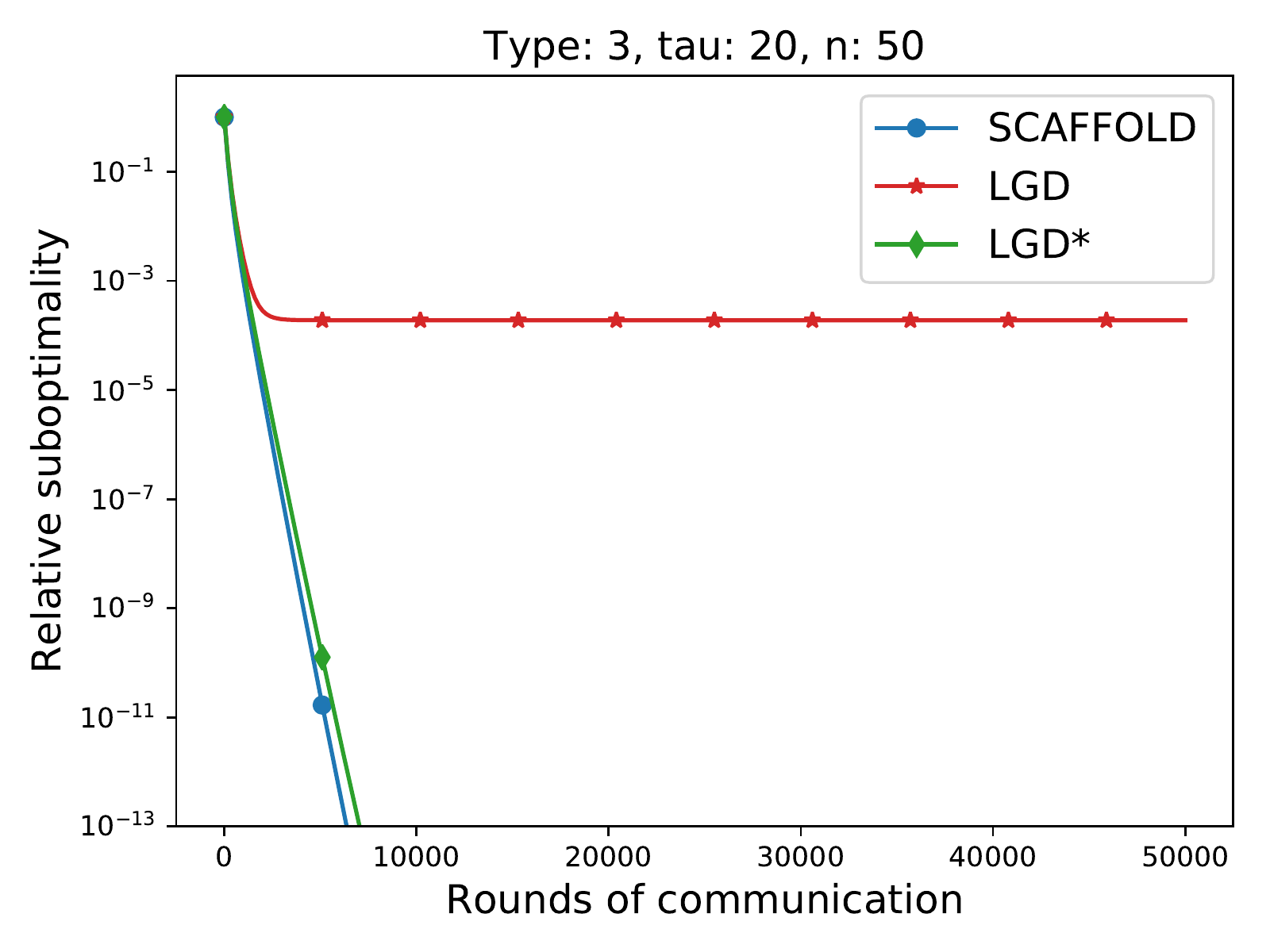}
        %\caption{ Residual vs. iteration  }\label{fig:bl_ex_flops}
\end{minipage}
\begin{minipage}{0.3\textwidth}
  \centering
\includegraphics[width =  \textwidth ]{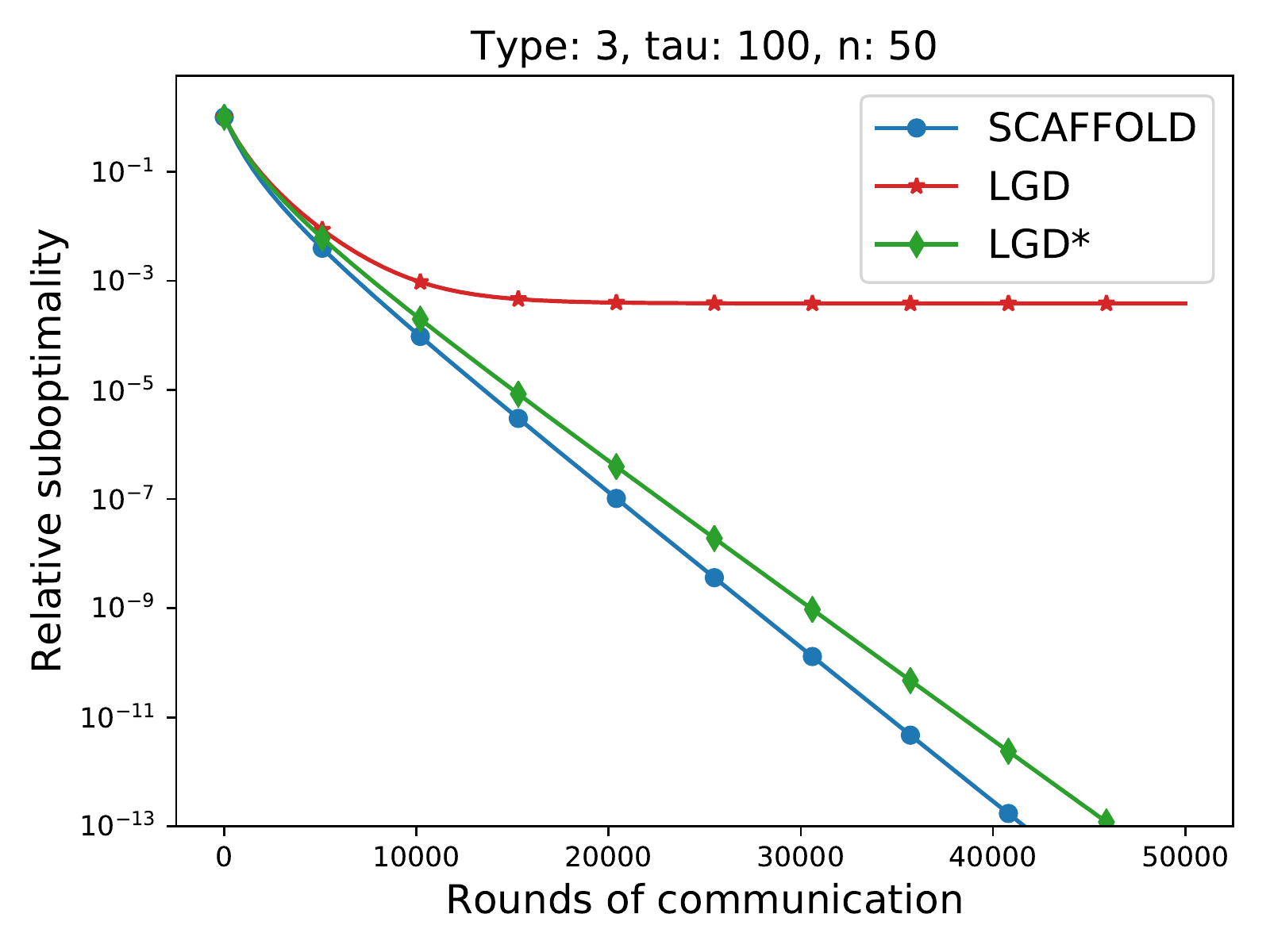}
        %\caption{ Residual vs. iteration  }\label{fig:bl_ex_flops}
\end{minipage}
\\
\begin{minipage}{0.3\textwidth}
  \centering
\includegraphics[width =  \textwidth ]{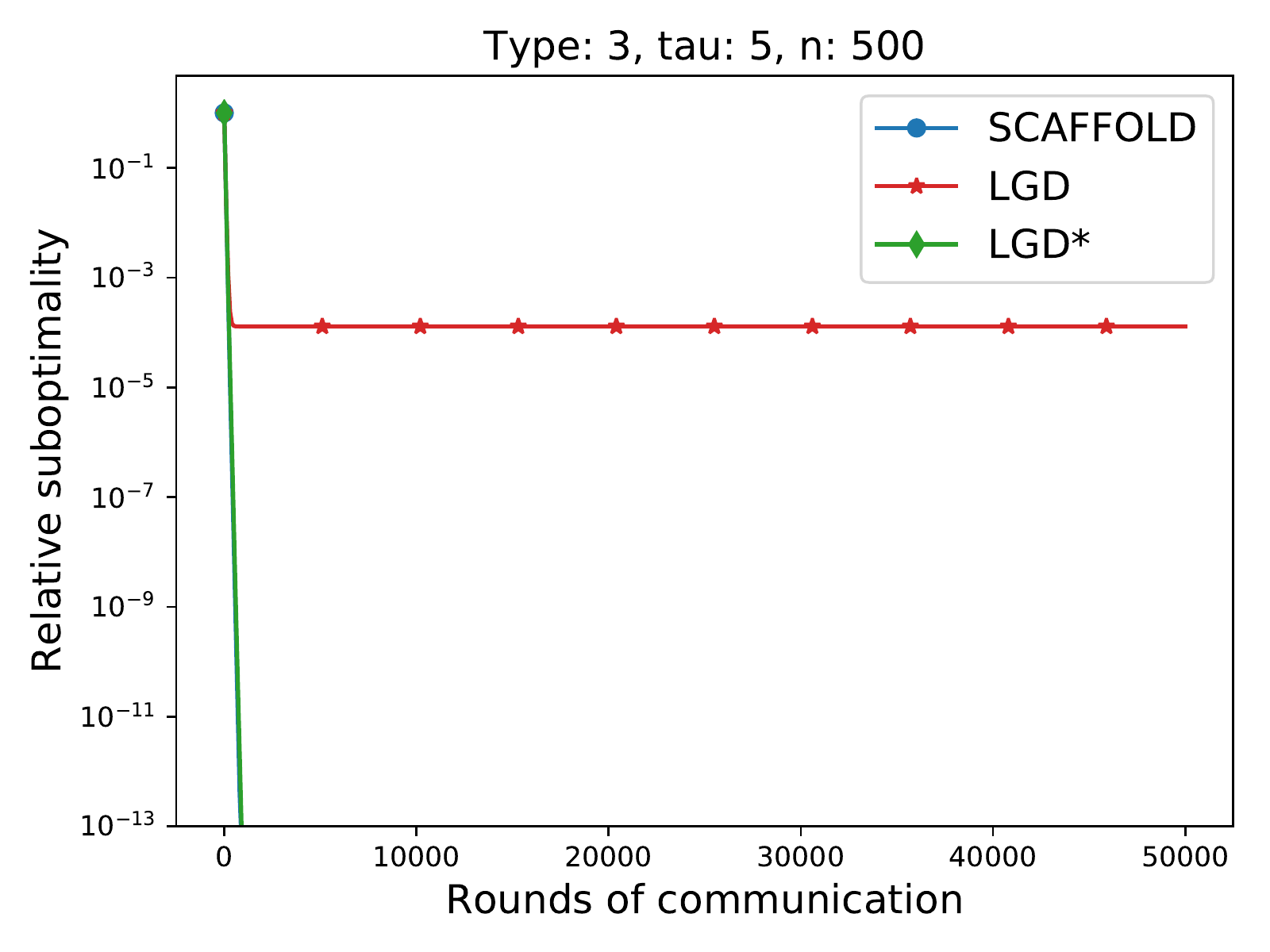}
        %\caption{ Residual vs. iteration  }\label{fig:bl_ex_flops}
\end{minipage}
\begin{minipage}{0.3\textwidth}
  \centering
\includegraphics[width =  \textwidth ]{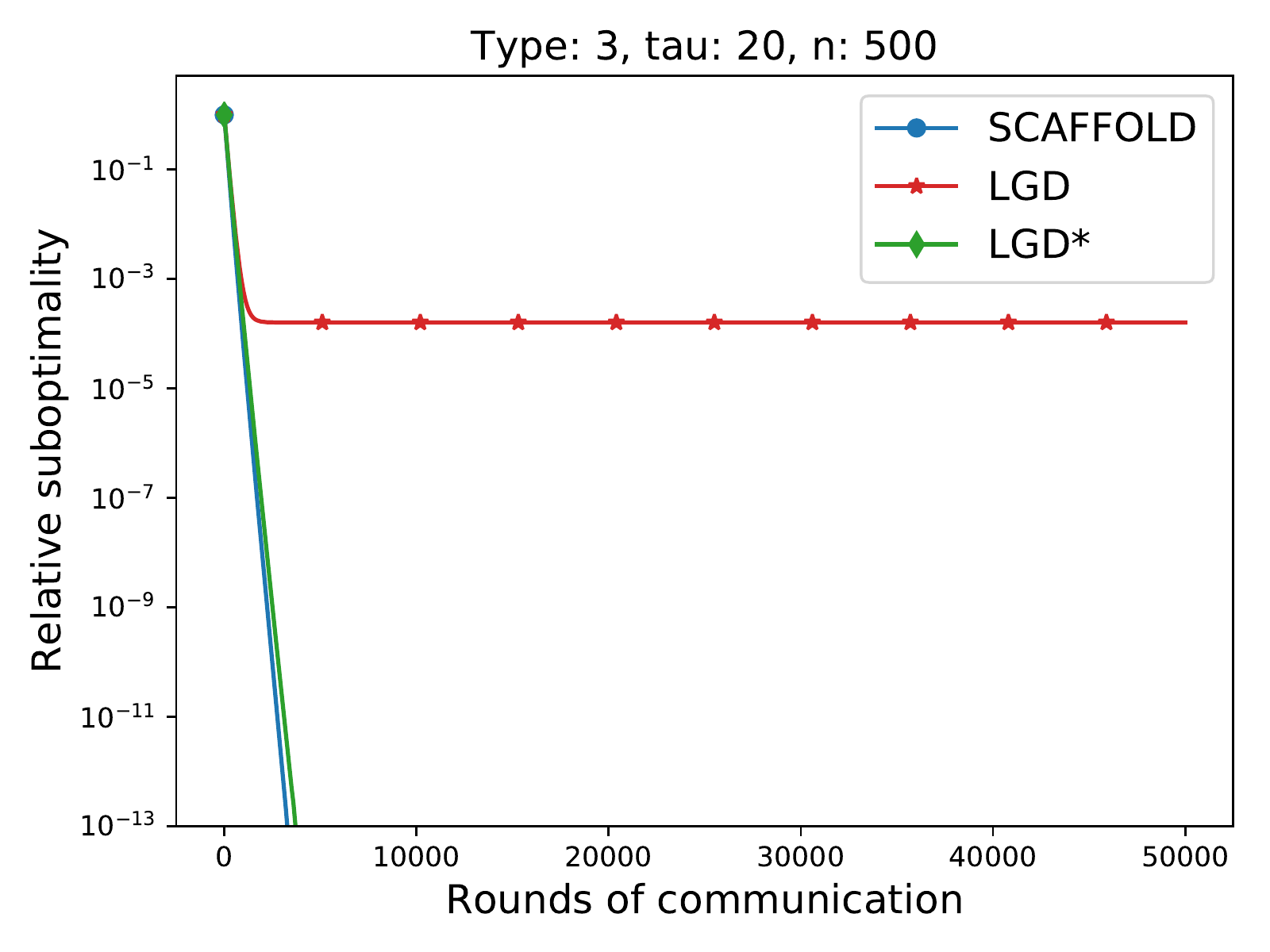}
        %\caption{ Residual vs. iteration  }\label{fig:bl_ex_flops}
\end{minipage}
\begin{minipage}{0.3\textwidth}
  \centering
\includegraphics[width =  \textwidth ]{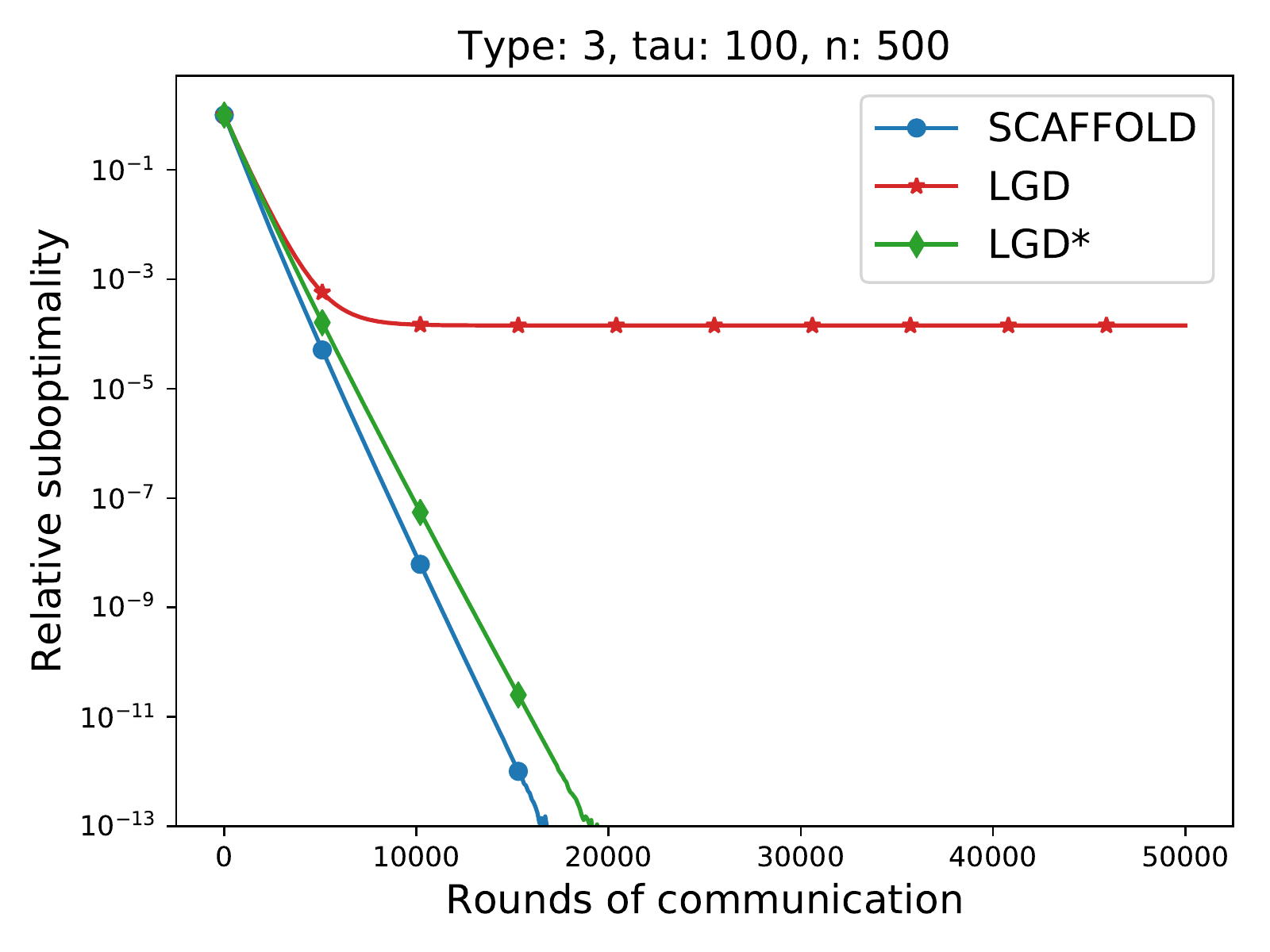}
        %\caption{ Residual vs. iteration  }\label{fig:bl_ex_flops}
\end{minipage}
\caption{Comparison of the following noiseless algorithms:  {\tt Local-SGD} ({\tt LGD}, Algorithm~\ref{alg:local_sgd} with no local noise) and {\tt SCAFFOLD}~\cite{karimireddy2019scaffold} (Algorithm~\ref{alg:l_local_svrg} without ``Loopless'') and {\tt S*-Local-SGD} ({\tt LGD*}, Algorithm~\ref{alg:local_sgd_star}). Quadratic minimization, problem type 3 (see Table~\ref{tbl:instances}). }
\label{fig:artif4}
\end{figure}

\clearpage
\section{Missing Proofs for Section~\ref{sec:main_res}}
Let us first state some well-known consequences of $L$-smoothness. Specifically, if $f_i$ is $L$-smooth, we must have
\begin{equation}
	f_i(y) \le f_i(x) + \langle\nabla f_i(x), y-x \rangle + \frac{L}{2}\|x-y\|^2, \qquad \forall x,y\in \R^d.\label{eq:L_smoothness_cor_1}
\end{equation}
If in addition to this we assume that $f_i$ is  convex, the following bound holds:
\begin{equation}
	\|\nabla f_i(x) - \nabla f_i(y)\|^2 \le 2L(f_i(x) - f_i(y) - \langle\nabla f_i(y), x-y\rangle) \eqdef 2LD_{f_i}(x,y), \qquad \forall x,y\in \R^d \label{eq:L_smoothness_cor}
\end{equation}

We next proceed with the proof of Theorem~\ref{thm:main_result}. Following the technique of virtual iterates from~\cite{stich2019error,khaled2020tighter}, notice that the sequence $\{x^k\}_{k\ge 0}$ satisfies the recursion
\begin{equation}
	x^{k+1} = x^k - \frac{\gamma}{n}\sum\limits_{i=1}^ng_i^k .\label{eq:x^k_recurrsion}
\end{equation}
 This observation forms the backbone of the key lemma of our paper, which we present next. 

\begin{lemma}\label{lem:main_lemma}
	Let As.~\ref{ass:quasi_strong_convexity},~\ref{ass:L_smoothness}~and~\ref{ass:key_assumption} be satisfied and $\gamma \le \min\left\{\nicefrac{1}{2(A'+MC)},\nicefrac{L}{(F'+MG)}\right\}$, where $M = \frac{4B'}{3\rho}$. Let  $\eta \eqdef\min\left\{\gamma\mu, \frac{\rho}{4}\right\} $. Then for all $k\ge 0$ we have
	\begin{equation}
		\gamma\EE\left[f(x^k) - f(x^*)\right] \le (1-\eta)\EE T^{k} - \EE T^{k+1} + \gamma^2(D_1' + MD_2) + 2L\gamma \EE V_k, \label{eq:main_lemma}
	\end{equation}
	where $\eta \eqdef \min\left\{\gamma\mu,\frac{\rho}{4}\right\}$, $T^k \eqdef \|x^k - x^*\|^2 + M\gamma^2 \sigma_k^2$.
\end{lemma}
\begin{proof}
	First of all, to simplify the proofs we introduce new notation: $g^k \eqdef \frac{1}{n}\sum_{i=1}^n g_i^k$. Using this and \eqref{eq:x^k_recurrsion} we get
	\begin{eqnarray*}
		\|x^{k+1}-x^*\|^2 &\overset{\eqref{eq:x^k_recurrsion}}{=}& \left\|x^k - x^* - \gamma g^k\right\|^2\\
		&=& \|x^k - x^*\|^2 -2\gamma\langle x^k - x^*, g^k\rangle + \gamma^2\|g^k\|^2.
	\end{eqnarray*}
	Taking conditional mathematical expectation $\EE_k[\cdot] = \EE[\cdot\mid x^k] \eqdef \EE[\cdot\mid x_1^k,\ldots, x_n^k]$ on both sides of the previous inequality we get
	\begin{eqnarray*}
		\EE\left[\|x^{k+1} - x^*\|^2\mid x^k\right] &\overset{\eqref{eq:unbiasedness}}{=}& \|x^k-x^*\|^2 -\frac{2\gamma}{n}\sum\limits_{i=1}^n\left\langle x^k-x^*,\nabla f_i(x_i^k) \right\rangle + \gamma^2\EE\left[\|g^k\|^2\mid x^k\right],
	\end{eqnarray*}
	hence
	\begin{eqnarray}
		\EE\left[\|x^{k+1}-x^*\|^2\right] &\overset{\eqref{eq:tower_property}}{\le}& \EE\left[\|x^k-x^*\|^2\right] -\frac{2\gamma}{n}\sum\limits_{i=1}^n\EE\left[\left\langle x^k-x^*,\nabla f_i(x_i^k) \right\rangle\right] + \gamma^2\EE\left[\|g^k\|^2\right]\notag\\
		&\overset{\eqref{eq:second_moment_bound}}{\le}& \EE\left[\|x^k-x^*\|^2\right] -\frac{2\gamma}{n}\sum\limits_{i=1}^n\EE\left[\left\langle x^k-x^*,\nabla f_i(x_i^k) \right\rangle\right] + B'\gamma^2\EE\left[\sigma_k^2\right] \notag\\
		&&\quad +2A'\gamma^2\EE\left[f(x^k) - f(x^*)\right] + F'\gamma^2\EE\left[V_k\right] + \gamma^2D_1'. \label{eq:main_lemma_technical_1}
	\end{eqnarray}
	Next, we derive an upper bound for the second term on the right-hand side of the previous inequality:
	\begin{eqnarray}
		-\frac{2\gamma}{n}\sum\limits_{i=1}^n\left\langle x^k-x^*,\nabla f_i(x_i^k) \right\rangle &=& \frac{2\gamma}{n}\sum\limits_{i=1}^n\left(\left\langle x^* - x_i^k, \nabla f_i(x_i^k)\right\rangle + \left\langle x_i^k - x^k, \nabla f_i(x_i^k)\right\rangle\right)\notag\\
		&\overset{\eqref{eq:str_quasi_cvx},\eqref{eq:L_smoothness_cor_1}}{\le}& \frac{2\gamma}{n}\sum\limits_{i=1}^n \left(f_i(x^*) - f_i(x_i^k) - \frac{\mu}{2}\|x_i^k - x^*\|^2\right)\notag\\
		&&\quad + \frac{2\gamma}{n}\sum\limits_{i=1}^n\left(f_i(x_i^k) - f_i(x^k) + \frac{L}{2}\|x^k - x_i^k\|^2\right)\notag\\
		&\overset{\eqref{eq:a_b_norm_squared}}{\le}& -2\gamma\left(f(x^k) - f(x^*)\right) -\mu\gamma\|x^k - x^*\|^2 + L\gamma V_k.\label{eq:main_lemma_technical_2}
	\end{eqnarray}
	Plugging \eqref{eq:main_lemma_technical_2} in \eqref{eq:main_lemma_technical_1}, we obtain
	\begin{eqnarray}
		\EE\left[\|x^{k+1} - x^*\|^2\right] &\overset{\eqref{eq:main_lemma_technical_1},\eqref{eq:main_lemma_technical_2}}{\le}& (1-\gamma\mu)\EE\left[\|x^k - x^*\|^2\right] -2\gamma\left(1 - A'\gamma\right)\EE\left[f(x^k) - f(x^*)\right]\notag\\
		&&\quad + B'\gamma^2\EE\left[\sigma_k^2\right] + \gamma\left(L+F'\gamma\right)\EE\left[V_k\right] + \gamma^2D_1'.\label{eq:main_lemma_technical_3}
	\end{eqnarray}
	It implies that
	\begin{eqnarray*}
		\EE T^{k+1} &=& \EE\left[\|x^{k+1}-x^*\|^2\right] + M\gamma^2\EE\left[\sigma_{k+1}^2\right]\\
		&\overset{\eqref{eq:main_lemma_technical_3},\eqref{eq:sigma_k+1_bound}}{\le}& (1-\gamma\mu)\EE\|x^k - x^*\|^2 + \left(1+\frac{B'}{M}-\rho\right)M\gamma^2\EE\sigma_k^2 \\
		&&\quad -2\gamma\left(1 - \left(A'+MC\right)\gamma\right)\EE\left[f(x^k) - f(x^*)\right]\\
		&&\quad + \gamma\left(L + (F'+MG)\gamma\right)\EE V_k + \gamma^2\left(D_1' + MD_2\right).
	\end{eqnarray*}
	Since $M = \frac{4B'}{3\rho}$, $\eta = \min\left\{\gamma\mu, \frac{\rho}{4}\right\}$ and $\gamma \le \min\left\{\nicefrac{1}{2(A'+MC)},\nicefrac{L}{(F'+MG)}\right\}$, we get
	\begin{eqnarray*}
		\EE T^{k+1} &\le& (1-\gamma\mu)\EE\|x^k - x^*\|^2 + \left(1-\frac{\rho}{4}\right)M\gamma^2\EE\sigma_k^2 -\gamma\EE\left[f(x^k) - f(x^*)\right]\\
		&&\quad + 2L\gamma \EE V_k + \gamma^2\left(D_1' + MD_2\right)\\
		&\le& (1-\eta)\EE T^k  -\gamma\EE\left[f(x^k) - f(x^*)\right] +  2L\gamma \EE V_k + \gamma^2\left(D_1' + MD_2\right).
	\end{eqnarray*}
	Rearranging the terms we get \eqref{eq:main_lemma}.
\end{proof}

Using the above lemma we derive the main complexity result.

\subsection{Proof of Theorem~\ref{thm:main_result}}
	From Lemma~\ref{lem:main_lemma} we have that
	\begin{eqnarray*}
		\gamma\EE\left[f(x^k) - f(x^*)\right] \le (1-\eta)\EE T^{k} - \EE T^{k+1} + \gamma^2(D_1' + MD_2) + 2L\gamma \EE V_k.
	\end{eqnarray*}
	Summing up previous inequalities for $k=0,\ldots,K$ with weights $w_k$ defined in \eqref{eq:w_k_definition} we derive
	\begin{eqnarray*}
		\gamma\sum\limits_{k=0}^K w_k\EE\left[f(x^k) - f(x^*)\right] &\le& \sum\limits_{k=0}^K \left(w_k(1-\eta)\EE T^{k} - w_k\EE T^{k+1}\right) + \gamma^2(D_1' + MD_2)W_K\\
		&&\quad + 2L\gamma\sum\limits_{k=0}^K w_k\EE V_k\\
		&\overset{\eqref{eq:w_k_definition},\eqref{eq:sum_V_k_bounds}}{\le}& \sum\limits_{k=0}^K \left(w_{k-1}\EE T^{k} - w_k\EE T^{k+1}\right) + \gamma^2\left(D_1' + MD_2\right)W_K\\
		&&\quad + \frac{\gamma}{2}\sum\limits_{k=0}^Kw_k\EE\left[f(x^k)-f(x^*)\right] + 2LH\gamma\EE\sigma_0^2 +2L\gamma^3 D_3 W_K.
	\end{eqnarray*}
	Relations $T^k \ge 0$ and $w_{-1} = 1$ imply that
	\begin{eqnarray*}
		\frac{\gamma}{2}\sum\limits_{k=0}^Kw_k\EE\left[f(x^k)-f(x^*)\right] &\le& T^0 + 2LH\gamma\EE\sigma_0^2 + \gamma^2\left(D_1' + MD_2 + 2L\gamma D_3\right)W_K.
	\end{eqnarray*}
	Using the definition of $\overline{x}^K$ and convexity of $f$, we get 
		\begin{eqnarray}
		\EE\left[f(\overline{x}^K) - f(x^*)\right] &\le& \frac{2T^0 + 4LH\gamma\EE\sigma_0^2}{\gamma W_K} + 2\gamma\left(D_1' + MD_2 + 2L\gamma D_3\right). \label{eq:main_result}
	\end{eqnarray}
	 It remains to consider two cases: $\mu > 0$ and $\mu = 0$. If $\mu > 0$ we have $W_K \ge w_K \ge (1-\eta)^{-K}$, where $\eta \eqdef \min\left\{\gamma\mu, \frac{\rho}{4}\right\}$ which implies \eqref{eq:main_result_1}. Finally, when $\mu = 0$, we have $w_k = 1$ for all $k\ge 0$, which implies $W_K = K+1 \ge K$ and \eqref{eq:main_result_2}.

\subsection{Corollaries}\label{sec:corollaries}
We state the full complexity results that can be obtained from Theorem~\ref{thm:main_result}. These results can be obtained as a direct consequence of Lemmas~\ref{lem:lemma2_stich}~and~\ref{lem:lemma_technical_cvx}.

\begin{corollary}\label{cor:app_complexity_cor_str_cvx}
	Consider the setup from Theorem~\ref{thm:main_result} and denote $\frac1h$ to be the resulting upper bound on $\gamma$\footnote{In order to obtain tight estimate of parameters $D_3$ and $H$, we shall impose further bounds on $\gamma$ (see Section~\ref{sec:data_and_loop} and Table~\ref{tbl:data_loop} therein). }
	  and $\mu > 0$.
	\begin{enumerate}
		\item If $D_3$ does not depend on $\gamma$, then for all $K$ such that
		\begin{eqnarray*}
		\text{either} && \frac{\ln\left(\max\{2,\min\{\nicefrac{a\mu^2K^2}{c_1},\nicefrac{a\mu^3K^3}{c_2}\}\}\right)}{K}\le \rho\\
		\text{or} && \frac{1}{h}\le \frac{\ln\left(\max\{2,\min\{\nicefrac{a\mu^2K^2}{c_1},\nicefrac{a\mu^3K^3}{c_2}\}\}\right)}{\mu K},
	\end{eqnarray*}			
	$a = 2\|x^0 - x^*\|^2 +   \frac{8B'\EE\sigma_0^2}{3h^2\rho} + \frac{4LH\EE\sigma_0^2}{h}$, $c_1 = 2D_1' + \frac{4B'D_2}{3\rho}$, $c_2 = 4LD_3$ and
	\begin{eqnarray*}
		\gamma &=& \min\left\{\frac{1}{h}, \gamma_K\right\},\\
		\gamma_K &=& \frac{\ln\left(\max\left\{2,\min\left\{\frac{a\mu^2K^2}{c_1},\frac{a\mu^3K^3}{c_2}\right\}\right\}\right)}{\mu K},
	\end{eqnarray*}
	we have\footnote{$\widetilde{\cO}$ hides numerical constants and logarithmical factors depending on $K$ and parameters of the problem.}
	\begin{equation*}
	\EE\left[f(\overline{x}^K) \right] - f(x^*) =	\widetilde\cO\left(ha\exp\left(-\min\left\{\frac{\mu}{h}, \rho\right\}K\right) + \frac{c_1}{\mu K} + \frac{c_2}{\mu^2 K^2}\right).
	\end{equation*}
	That is, to achieve $\EE\left[f(\overline{x}^K) \right] - f(x^*)\le \varepsilon$, the method requires\footnote{If $c_1 = c_2 = 0$, then one can replace $\widetilde{\cO}$ by $\cO$.}:
	\begin{equation*}
		K = \widetilde\cO\left(\left(\frac{1}{\rho} + \frac{h}{\mu}\right)\log\left(\frac{ha}{\varepsilon}\right) + \frac{c_1}{\mu \varepsilon} + \sqrt{\frac{c_2}{\mu^2 \varepsilon}}\right).
	\end{equation*}
	\item If $D_3 = D_{3,1} + \frac{D_{3,2}}{\gamma}$, then the same bounds hold with $c_1 = 2D_1' + \frac{4B'D_2}{3\rho} + 2LD_{3,2}$ and $c_2 = 4LD_{3,1}$. 
	\end{enumerate}
\end{corollary}

\begin{corollary}\label{cor:app_complexity_cor_cvx}
	Let assumptions of Theorem~\ref{thm:main_result} be satisfied with any $\gamma \le \frac{1}{h}$ and $\mu = 0$.
	\begin{enumerate}
		\item If $D_3$ does not depend on $\gamma$, then for all $K$ and
	\begin{eqnarray*}
		\gamma &=& \min\left\{\frac{1}{h}, \sqrt{\frac{a}{b_1}}, \sqrt[3]{\frac{a}{b_2}}, \sqrt{\frac{a}{c_1 K}}, \sqrt[3]{\frac{a}{c_2 K}}\right\},
	\end{eqnarray*}
	where $a = 2\|x^0 - x^*\|^2$, $b_1 = 4LH\EE\sigma_0^2$, $b_2 = \frac{8B'\EE\sigma_0^2}{3\rho}$, $c_1 = 2D_1' + \frac{4B'D_2}{3\rho}$, $c_2 = 4LD_3$, we have 	\begin{equation*}
\EE\left[f(\overline{x}^K) \right] - f(x^*) =		\cO\left(\frac{ha}{K} + \frac{\sqrt{ab_1}}{K} + \frac{\sqrt[3]{a^2b_2}}{K} + \sqrt{\frac{ac_1}{K}} + \frac{\sqrt[3]{a^2c_2}}{K^{\nicefrac{2}{3}}} \right).
	\end{equation*}
	That is, to achieve $\EE\left[f(\overline{x}^K) \right] - f(x^*)\le \varepsilon$, the method requires 
	\begin{equation*}
		K = \cO\left(\frac{ha}{\varepsilon} + \frac{\sqrt{ab_1}}{\varepsilon} + \frac{\sqrt[3]{a^2b_2}}{\varepsilon} + \frac{ac_1}{\varepsilon^2} + \frac{a\sqrt{c_2}}{\varepsilon^{\nicefrac{3}{2}}} \right).
	\end{equation*}
	\item If $D_3 = D_{3,1} + \frac{D_{3,2}}{\gamma}$, then the same bounds hold with $c_1 = 2D_1' + \frac{4B'D_2}{3\rho} + 2LD_{3,2}$ and $c_2 = 4LD_{3,1}$. 
	\end{enumerate}
\end{corollary}

\clearpage
\section{Missing Proofs and Details for Section~\ref{sec:data_and_loop} \label{sec:a_data_and_loop}}

\subsection{Constant Local Loop}\label{sec:constant_loop}
In this section we show how our results can be applied to analyze \eqref{eq:local_sgd_def} in the case when
\begin{equation*}
	c_{k} = \begin{cases}1,& \text{if } k \mod \tau = 0,\\ 0,& \text{if } k\mod \tau \neq 0, \end{cases}
\end{equation*}
where $\tau$ is number of local steps between two neighboring rounds of communications. This corresponds to the setting in which the local loop size on each device has a fixed length.

\subsubsection{Heterogenous Data}\label{sec:const_loop_hetero}
First of all, we need to assume more about $g_i^k$.
\begin{assumption}\label{ass:hetero_second_moment}
	We assume that inequalities \eqref{eq:second_moment_bound}-\eqref{eq:sigma_k+1_bound} hold and additionally there exist such non-negative constants $\tA, \hA, \tB, \hB, \tF, \hF, \tD_1, \hD_{1}$ that for all $k \ge 0$
	\begin{eqnarray}
		\frac{1}{n}\sum\limits_{i=1}^n\EE\left[\|\bar{g}_i^k\|^2\right] &\le & 2\tA\EE\left[f(x^k) - f(x^*)\right] + \tB\EE\left[\sigma_k^2\right] + \tF\EE\left[V_k\right] + \tD_{1}, \label{eq:hetero_second_moment_bound}\\
		\frac{1}{n}\sum\limits_{i=1}^n\EE\left[\|g_i^k-\bar{g}_i^k\|^2\right] &\le & 2\hA\EE\left[f(x^k) - f(x^*)\right] + \hB\EE\left[\sigma_k^2\right] + \hF\EE\left[V_k\right] + \hD_{1}, \label{eq:hetero_var_bound}
	\end{eqnarray}
	where $\bar{g}_i^k = \EE\left[g_i^k\mid x_1^k,\ldots,x_n^k\right]$.
\end{assumption}
We notice that inequalities \eqref{eq:hetero_second_moment_bound}-\eqref{eq:hetero_var_bound} imply \eqref{eq:second_moment_bound} and vice versa. Indeed, if \eqref{eq:hetero_second_moment_bound}-\eqref{eq:hetero_var_bound} hold then inequality \eqref{eq:second_moment_bound} holds with $A = \tA + \hA$, $B = \tB+\hB$, $F = \tF + \hF$, $D_1 = \tD_{1}+\hD_{1}$ due to variance decomposition formula \eqref{eq:variance_decomposition}, and if \eqref{eq:second_moment_bound} is true then \eqref{eq:hetero_second_moment_bound}-\eqref{eq:hetero_var_bound} also hold with $\tA = \hA = A$, $\tB = \hB = B$, $\tF = \hF = F$, $\tD_{1} = \hD_{1} = D_1$.

We start our analysis without making any assumption on homogeneity of data that workers have an access to. Next lemma provides an upper bound for the weighted sum of $\EE V_k$.
\begin{lemma}\label{lem:V_k_lemma}
	Let As.~\ref{ass:quasi_strong_convexity},~\ref{ass:L_smoothness}~and~\ref{ass:hetero_second_moment} hold and\footnote{When $\rho = 1$ one can always set the parameters in such a way that $\tB = \hB = C = G = 0$, $D_2 = 0$. In this case we assume that $\frac{2\tB C}{\rho(1-\rho)} = \frac{2\hB C}{\rho(1-\rho)} = \frac{2\tB G}{\rho(1-\rho)} = \frac{2\hB G}{\rho(1-\rho)} = 0$.}
	\begin{eqnarray*}
		\gamma &\le& \min\left\{\frac{1}{4(\tau-1)\mu}, \frac{1}{2\sqrt{e(\tau-1)\left(\tF(\tau-1) + \hF + \frac{2G(\tB(\tau-1)+\hB)}{\rho(1-\rho)}\right)}}\right\},\\
		\gamma &\le& \frac{1}{4\sqrt{2eL(\tau-1)\left(\tA(\tau-1)+\hA +\frac{2C(\tB(\tau-1)+\hB)}{\rho(1-\rho)}\right)}}
	\end{eqnarray*}
	Then \eqref{eq:sum_V_k_bounds} holds with
	\begin{equation}
		H = \frac{4e(\tau-1)(\tB(\tau-1)+\hB)(2+\rho)\gamma^2}{\rho},\quad D_3 = 2e(\tau-1)\left(\tD_{1}(\tau-1)+\hD_{1} + \frac{2D_2(\tB(\tau-1)+\hB)}{\rho}\right).\label{eq:V_k_bound}
	\end{equation}
%	\begin{eqnarray}
%		2L\sum\limits_{k=0}^K w_k\EE V_k &\le& \frac{1}{2}\sum\limits_{k=0}^K w_k\EE\left[f(x^k) - f(x^*)\right] + \frac{8LB(\tau-1)^2\gamma^2(2+\rho)\EE\sigma_0^2}{\rho}\notag\\
%		&&\quad + 4L(\tau-1)^2\gamma^2\left(D_1 + \frac{2BD_2}{\rho}\right)W_K.\label{eq:V_k_bound}
%	\end{eqnarray}
\end{lemma}
\begin{proof}
	Consider some integer $k\ge 0$. There exists such integer $t\ge 0$ that $\tau t \le k \le \tau(t+1)-1$. Using this and Lemma~\ref{lem:lemma14_stich} we get
	\begin{eqnarray*}
		\EE[V_k] &\overset{\eqref{eq:local_sgd_def},\eqref{eq:x^k_recurrsion}}{=}& \frac{1}{n}\sum\limits_{i=1}^n \EE\left[\left\|x_i^{\tau t} - \gamma\sum\limits_{l=\tau t}^{k-1} g_i^l - x^{\tau t} + \gamma\sum\limits_{l=\tau t}^{k-1} g^l\right\|^2\right]\\
		&=& \frac{\gamma^2}{n}\sum\limits_{i=1}^n\EE\left[\left\|\sum\limits_{l=\tau t}^{k-1} \left(g_i^l - g^l\right)\right\|^2\right]\\
		&\overset{\eqref{eq:lemma14_stich}}{\le}& \frac{e\gamma^2(k-\tau t)}{n}\sum\limits_{i=1}^n\sum\limits_{l=\tau t}^{k-1}\EE\left[\left\|\bar{g}_i^l - \bar{g}^l\right\|^2\right] + \frac{e\gamma^2}{n}\sum\limits_{i=1}^n\sum\limits_{l=\tau t}^{k-1}\EE\left[\left\|g_i^l - \bar{g}_i^l - \left(g^l - \bar{g}^l\right)\right\|^2\right]\\
		&\overset{\eqref{eq:variance_decomposition}}{\le}& \frac{e\gamma^2(\tau - 1)}{n}\sum\limits_{i=1}^n\sum\limits_{l=\tau t}^{k-1}\EE\left[\left\|\bar{g}_i^l\right\|^2\right] + \frac{e\gamma^2}{n}\sum\limits_{i=1}^n\sum\limits_{l=\tau t}^{k-1}\EE\left[\left\|g_i^l - \bar{g}_i^l\right\|^2\right],
	\end{eqnarray*}
	where $\bar{g}^k = \frac{1}{n}\sum\limits_{i=1}^n\bar{g}_i^k$.
	Applying Assumption~\ref{ass:hetero_second_moment}, we obtain
	\begin{eqnarray*}
		\EE V_k &\overset{\eqref{eq:hetero_second_moment_bound},\eqref{eq:hetero_var_bound}}{\le}& 2e\left(\tA(\tau - 1)+\hA\right)\gamma^2\sum\limits_{l=\tau t}^{k-1}\EE\left[f(x^l) - f(x^*)\right] + e\left(\tB(\tau-1)+\hB\right)\gamma^2\sum\limits_{l=\tau t}^{k-1}\EE\sigma_l^2 \\
		&&\quad + e\left(\tF(\tau-1)+\hF\right)\gamma^2\sum\limits_{l=\tau t}^{k-1} \EE V_l + e(\tau- 1)\left(\tD_{1}(\tau-1)+\hD_{1}\right)\gamma^2,
	\end{eqnarray*}
	hence
	\begin{eqnarray}
		\sum\limits_{j=\tau t}^k w_j \EE V_j &\le& 2e\left(\tA(\tau - 1)+\hA\right)\gamma^2\sum\limits_{j=\tau t}^k\sum\limits_{l=\tau t}^{j-1}w_j\EE\left[f(x^l) - f(x^*)\right] + e\left(\tB(\tau-1)+\hB\right)\gamma^2\sum\limits_{j=\tau t}^k\sum\limits_{l=\tau t}^{j-1}w_j\EE\sigma_l^2 \notag \\
		&&\quad + e\left(\tF(\tau-1)+\hF\right)\gamma^2\sum\limits_{j=\tau t}^k\sum\limits_{l=\tau t}^{j-1} w_j\EE V_l + e(\tau- 1)\left(\tD_{1}(\tau-1)+\hD_{1}\right)\gamma^2\sum\limits_{j=\tau t}^kw_j. \label{eq:V_k_lemma_technical_1}
	\end{eqnarray}
	Recall that $w_k = (1 - \eta)^{-(k+1)}$ and $\eta = \min\left\{\gamma\mu, \frac{\rho}{4}\right\}$. Together with our assumption on $\gamma$ it implies that for all $0 \le i < k$, $0\le j \le \tau-1$ we have
	\begin{eqnarray}
		w_k &=& (1 - \eta)^{-(k-j+1)}\left(1 - \eta\right)^{-j} \overset{\eqref{eq:1-p/2_inequality}}{\le} w_{k-j}\left(1 + 2\eta\right)^{j} \notag\\
		&\le& w_{k-j}\left(1 + 2\gamma\mu\right)^{j} \le w_{k-j}\left(1 + \frac{1}{2(\tau-1)}\right)^j \le w_{k-j}\exp\left(\frac{j}{2(\tau-1)}\right)\notag\\
		&\le& w_{k-j}\exp\left(\frac{1}{2}\right) \le 2w_{k-j}, \label{eq:V_k_lemma_technical_2}\\
		w_k &=& \left(1 - \eta\right)^{-(k-i+1)}\left(1 - \eta\right)^{-i} \overset{\eqref{eq:1-p/2_inequality}}{\le} w_{k-i}\left(1 + 2\eta\right)^i \le w_{k-i}\left(1 + \frac{\rho}{2}\right)^i, \label{eq:V_k_lemma_technical_3}\\
		w_k &\overset{\eqref{eq:1-p/2_inequality}}{\le}& \left(1 + 2\eta\right)^{k+1} \le \left(1 + \frac{\rho}{2}\right)^{k+1}. \label{eq:V_k_lemma_technical_4}
	\end{eqnarray}
	For simplicity, we introduce new notation: $r_k \eqdef \EE\left[f(x^k) - f(x^*)\right]$. Using this we get
	\begin{eqnarray*}
		\sum\limits_{j=\tau t}^k\sum\limits_{l=\tau t}^{j-1}w_jr_l &\overset{\eqref{eq:V_k_lemma_technical_2}}{\le}& \sum\limits_{j=\tau t}^k\sum\limits_{l=\tau t}^{j-1}2w_l r_l \le 2(k-\tau t)\sum\limits_{j=\tau t}^kw_jr_j \le 2(\tau - 1)\sum\limits_{j=\tau t}^k w_jr_j,\\
		\sum\limits_{j=\tau t}^k\sum\limits_{l=\tau t}^{j-1}w_j\EE\sigma_l^2 &\overset{\eqref{eq:V_k_lemma_technical_2}}{\le}& \sum\limits_{j=\tau t}^k\sum\limits_{l=\tau t}^{j-1}2w_l \EE\sigma_l^2 \le 2(k-\tau t)\sum\limits_{j=\tau t}^kw_j\EE\sigma_j^2 \le 2(\tau - 1)\sum\limits_{j=\tau t}^kw_j\EE\sigma_j^2,\\
		\sum\limits_{j=\tau t}^k\sum\limits_{l=\tau t}^{j-1}w_j\EE V_l &\overset{\eqref{eq:V_k_lemma_technical_2}}{\le}& \sum\limits_{j=\tau t}^k\sum\limits_{l=\tau t}^{j-1}2w_l \EE V_l \le 2(k-\tau t)\sum\limits_{j=\tau t}^kw_j\EE V_j \le 2(\tau - 1)\sum\limits_{j=\tau t}^kw_j \EE V_j.
	\end{eqnarray*}
	Plugging these inequalities in \eqref{eq:V_k_lemma_technical_1} we derive
	\begin{eqnarray*}
		\sum\limits_{j=\tau t}^k w_j \EE V_j &\le& 4e(\tau - 1)(\tA(\tau-1)+\hA)\gamma^2\sum\limits_{j=\tau t}^kw_j r_j + 2e(\tau - 1)(\tB(\tau-1)+\hB)\gamma^2\sum\limits_{j=\tau t}^kw_j\EE\sigma_j^2 \notag \\
		&&\quad + 2e(\tau - 1)(\tF(\tau-1)+\hF)\gamma^2\sum\limits_{j=\tau t}^k w_j\EE V_j + e\left(\tD_{1}(\tau - 1)+\hD_{1}\right)\gamma^2\sum\limits_{j=\tau t}^kw_j.
	\end{eqnarray*}
	Since $V_{\tau t} = 0$ for all integer $t \ge 0$ we obtain
	\begin{eqnarray}
		\sum\limits_{k=0}^K w_k \EE V_k &\le& 4e(\tau - 1)(\tA(\tau-1)+\hA)\gamma^2\sum\limits_{k=0}^K w_k r_k + 2e(\tau - 1)(\tB(\tau-1)+\hB)\gamma^2\sum\limits_{k=0}^Kw_k\EE\sigma_k^2 \notag \\
		&&\quad + 2e(\tau - 1)(\tF(\tau-1)+\hF)\gamma^2\sum\limits_{k=0}^K w_k\EE V_k + e\left(\tD_{1}(\tau - 1)+\hD_{1}\right)\gamma^2 \sum\limits_{k=0}^K w_k\label{eq:V_k_lemma_technical_5}
	\end{eqnarray}
	It remains to estimate the second term in the right-hand side of the previous inequality. First of all,
	\begin{eqnarray}
		\EE\sigma_{k+1}^2 &\overset{\eqref{eq:sigma_k+1_bound}}{\le}& (1 - \rho)\EE\sigma_{k}^2 + 2C \underbrace{\EE\left[f(x^k) - f(x^*)\right]}_{r_k} + G\EE V_k + D_2\notag\\
		&\le& (1-\rho)^{k+1}\EE\sigma_0^2 + 2C\sum\limits_{l=0}^{k}(1-\rho)^{k-l}r_l + G\sum\limits_{l=0}^k(1-\rho)^{k-l}\EE V_l + D_2\sum\limits_{l=0}^{k}(1-\rho)^l\notag\\
		&\le& (1-\rho)^{k+1}\EE\sigma_0^2 + 2C\sum\limits_{l=0}^{k}(1-\rho)^{k-l}r_l + G\sum\limits_{l=0}^k(1-\rho)^{k-l}\EE V_l + D_2\sum\limits_{l=0}^{\infty}(1-\rho)^l\notag\\
		&=& (1-\rho)^{k+1}\EE\sigma_0^2 + 2C\sum\limits_{l=0}^{k}(1-\rho)^{k-l}r_l + G\sum\limits_{l=0}^k(1-\rho)^{k-l}\EE V_l + \frac{D_2}{\rho}.\label{eq:sigma_k_useful_recurrence}
	\end{eqnarray}
	It implies that
	\begin{eqnarray}
		\sum\limits_{k=0}^K w_k\EE\sigma_k^2 &\overset{\eqref{eq:sigma_k_useful_recurrence}}{\le}& \EE\sigma_0^2\sum\limits_{k=0}^K w_k(1-\rho)^{k} + \frac{2C}{1-\rho}\sum\limits_{k=0}^K\sum\limits_{l=0}^k w_k(1-\rho)^{k-l}r_l\notag\\
		&&\quad + \frac{G}{1-\rho}\sum\limits_{k=0}^K\sum\limits_{l=0}^k w_k(1-\rho)^{k-l}\EE V_l +  \frac{D_2 W_K}{\rho}\notag\\
		&\overset{\eqref{eq:V_k_lemma_technical_3},\eqref{eq:V_k_lemma_technical_4}}{\le}& \EE\sigma_0^2\left(1+\frac{\rho}{2}\right)\sum\limits_{k=0}^K\left(1+\frac{\rho}{2}\right)^k(1-\rho)^k + \frac{2C}{1-\rho}\sum\limits_{k=0}^K\sum\limits_{l=0}^k w_l\left(1+\frac{\rho}{2}\right)^{k-l}(1-\rho)^{k-l}r_l\notag\\
		&&\quad + \frac{G}{1-\rho}\sum\limits_{k=0}^K\sum\limits_{l=0}^k w_l\left(1+\frac{\rho}{2}\right)^{k-l}(1-\rho)^{k-l}\EE V_l +  \frac{D_2 W_K}{\rho}\notag\\
		&\overset{\eqref{eq:1+p/2_inequality}}{\le}& \EE\sigma_0^2\left(1+\frac{\rho}{2}\right)\sum\limits_{k=0}^K\left(1-\frac{\rho}{2}\right)^k + \frac{2C}{1-\rho}\sum\limits_{k=0}^K\sum\limits_{l=0}^k w_lr_l\left(1-\frac{\rho}{2}\right)^{k-l}\notag\\
		&&\quad + \frac{G}{1-\rho}\sum\limits_{k=0}^K\sum\limits_{l=0}^k w_l\EE V_l\left(1-\frac{\rho}{2}\right)^{k-l} +  \frac{D_2 W_K}{\rho}\notag\\
		&\le& \EE\sigma_0^2\left(1+\frac{\rho}{2}\right)\sum\limits_{k=0}^\infty\left(1-\frac{\rho}{2}\right)^k + \frac{2C}{1-\rho}\left(\sum\limits_{k=0}^Kw_kr_k\right)\left(\sum\limits_{l=0}^\infty \left(1-\frac{\rho}{2}\right)^{l}\right)\notag\\
		&&\quad + \frac{G}{1-\rho}\left(\sum\limits_{k=0}^Kw_k\EE V_k\right)\left(\sum\limits_{l=0}^\infty \left(1-\frac{\rho}{2}\right)^{l}\right) +  \frac{D_2 W_K}{\rho}\notag\\
		&=& \frac{\EE\sigma_0^2(2+\rho)}{\rho} + \frac{4C}{\rho(1-\rho)}\sum\limits_{k=0}^Kw_kr_k + \frac{2G}{\rho(1-\rho)}\sum\limits_{k=0}^Kw_k\EE V_k +  \frac{D_2 W_K}{\rho}.\label{eq:sigma_k_technical_bound}
	\end{eqnarray}
	Plugging this inequality in \eqref{eq:V_k_lemma_technical_5} we get
	\begin{eqnarray*}
		\sum\limits_{k=0}^K w_k \EE V_k &\le& 4e(\tau-1)\gamma^2\left(\tA(\tau-1)+\hA +\frac{2C(\tB(\tau-1)+\hB)}{\rho(1-\rho)}\right)\sum\limits_{k=0}^K w_k r_k\\
		&&\quad + \frac{2e(\tau-1)(\tB(\tau-1)+\hB)\EE\sigma_0^2(2+\rho)\gamma^2}{\rho}\\
		&&\quad + 2e(\tau-1)\gamma^2\left(\tF(\tau-1) + \hF + \frac{2G(\tB(\tau-1)+\hB)}{\rho(1-\rho)}\right)\sum\limits_{k=0}^K w_k\EE V_k \\
		&&\quad + e(\tau-1)\gamma^2\left(\tD_{1}(\tau-1)+\hD_{1} + \frac{2D_2(\tB(\tau-1)+\hB)}{\rho}\right)W_K.
	\end{eqnarray*}
	Our choice of $\gamma$ implies
	\begin{equation*}
		4e(\tau-1)\gamma^2\left(\tA(\tau-1)+\hA +\frac{2C(\tB(\tau-1)+\hB)}{\rho(1-\rho)}\right) \le \frac{1}{8L}
	\end{equation*}
	and
	\begin{equation*}
		2e(\tau-1)\gamma^2\left(\tF(\tau-1) + \hF + \frac{2G(\tB(\tau-1)+\hB)}{\rho(1-\rho)}\right) \le \frac{1}{2}.
	\end{equation*}		
	Using these inequalities we continue our derivations
	\begin{eqnarray*}
		\frac{1}{2}\sum\limits_{k=0}^K w_k\EE V_k &\le& \frac{1}{8L}\sum\limits_{k=0}^K w_k r_k + \frac{2e(\tau-1)(\tB(\tau-1)+\hB)\EE\sigma_0^2(2+\rho)\gamma^2}{\rho} \\
		&&\quad+ e(\tau-1)\gamma^2\left(\tD_{1}(\tau-1)+\hD_{1} + \frac{2D_2(\tB(\tau-1)+\hB)}{\rho}\right)W_K.
	\end{eqnarray*}
	Multiplying both sides by $4L$ we get the result.
\end{proof}

Clearly, this lemma and Theorem~\ref{thm:main_result} imply the following result.
\begin{corollary}\label{cor:const_loop}
	Let the assumptions of Lemma~\ref{lem:V_k_lemma} are satisfied. Then Assumption~\ref{ass:key_assumption} holds and, in particular, if
	\begin{eqnarray*}
		\gamma &\le& \min\left\{\frac{1}{2\left(A'+\frac{4B'C}{3\rho}\right)}, \frac{L}{F'+\frac{4B'G}{3\rho}}\right\},\\
		\gamma &\le& \min\left\{\frac{1}{4(\tau-1)\mu}, \frac{1}{2\sqrt{e(\tau-1)\left(\tF(\tau-1) + \hF + \frac{2G(\tB(\tau-1)+\hB)}{\rho(1-\rho)}\right)}}\right\},\\
		\gamma &\le& \frac{1}{4\sqrt{2eL(\tau-1)\left(\tA(\tau-1)+\hA +\frac{2C(\tB(\tau-1)+\hB)}{\rho(1-\rho)}\right)}},
	\end{eqnarray*}
	then for all $K\ge 0$ we have
	\begin{eqnarray}
		\EE\left[f(\overline{x}^K) - f(x^*)\right] &\le& \frac{2\|x^0 - x^*\|^2 +   \frac{8B'}{3\rho}\gamma^2 \EE\sigma_0^2 + 4LH\gamma\EE\sigma_0^2}{\gamma W_K} + 2\gamma\left(D_1' + \frac{4B'D_2}{3\rho} + 2L\gamma D_3\right), \label{eq:main_result_const_loop}
	\end{eqnarray}
	where $\overline{x}^K \eqdef \frac{1}{W_K}\sum_{k=0}^K w_k x^k$ and
	\begin{equation*}
		H = \frac{4e(\tau-1)(\tB(\tau-1)+\hB)(2+\rho)\gamma^2}{\rho},\quad D_3 = 2e(\tau-1)\left(\tD_{1}(\tau-1)+\hD_{1} + \frac{2D_2(\tB(\tau-1)+\hB)}{\rho}\right).
	\end{equation*}
	 Moreover, if $\mu > 0$, then
	\begin{eqnarray}
		\EE\left[f(\overline{x}^K) - f(x^*)\right] &\le& \left(1 - \min\left\{\gamma\mu,\frac{\rho}{4}\right\}\right)^K\frac{2\|x^0 - x^*\|^2 +   \frac{8B'}{3\rho}\gamma^2 \EE\sigma_0^2 + 4LH\gamma\EE\sigma_0^2}{\gamma}\notag\\
		&&\quad + 2\gamma\left(D_1' + \frac{4B'D_2}{3\rho} + 2L\gamma D_3\right), \label{eq:main_result_1_const_loop}
	\end{eqnarray}
	and in the case when $\mu = 0$, we have
	\begin{eqnarray}
		\EE\left[f(\overline{x}^K) - f(x^*)\right] &\le& \frac{2\|x^0 - x^*\|^2 +   \frac{8B'}{3\rho}\gamma^2 \EE\sigma_0^2 + 4LH\gamma\EE\sigma_0^2}{\gamma K} + 2\gamma\left(D_1' + \frac{4B'D_2}{3\rho} + 2L\gamma D_3\right). \label{eq:main_result_2_const_loop}
	\end{eqnarray}
\end{corollary}

\begin{remark}
As we will see later when looking at particular special cases,  local gradient methods are only as good as their non-local counterparts (i.e., when $\tau=1$) in terms of the communication complexity in the fully heterogeneous setup. Furthermore, the non-local methods outperform local ones in terms of  computation complexity. While one might think that this observation is a byproduct of our analysis, our observations are supported by findings in recent literature on this topic~\cite{karimireddy2019scaffold, khaled2020tighter}. To rise to the defense of local methods, we remark that they might be preferable to their non-local cousins in the homogeneous data setup~\cite{woodworth2020local} or for personalized federated learning~\cite{hanzely2020federated}.
\end{remark}

\subsubsection{$\zeta$-Heterogeneous Data}\label{sec:cont_loop_homo}
In this section we assume that $f_1, f_2, \ldots, f_n$ are $\zeta$-heterogeneous (see Definition~\ref{def:zeta_hetero}). Moreover, we additionally assume that $\EE\left[g_i^k\mid x_i^k\right] = \nabla f_i(x_i^k)$ and that the functions $f_i$ for $i\in[n]$ are  $\mu$-strongly convex,
\begin{equation}
	f_i(x) \ge f_i(y) + \langle\nabla f_i(y), x-y\rangle + \frac{\mu}{2}\|x-y\|^2\qquad \forall x,y\in\R^d \label{eq:strong_convexity}
\end{equation}
which implies (e.g., see \cite{nesterov2018lectures})
\begin{equation}
	\langle \nabla f_i(x) - \nabla f_i(y), x-y\rangle \ge \mu\|x-y\|^2\qquad \forall x,y\in\R^d. \label{eq:coercivity}
\end{equation}
\begin{lemma}\label{lem:V_k_lemma_homo}
	Let Assumption~\ref{ass:L_smoothness} be satisfied, inequalities \eqref{eq:unbiasedness}-\eqref{eq:sigma_k+1_bound} hold and\footnote{When $\rho = 1$ one can always set the parameters in such a way that $B = C = G = 0$, $D_2 = 0$. In this case we assume that $\frac{2BC}{\rho(1-\rho)} = \frac{2BG}{\rho(1-\rho)} = 0$.}
	\begin{equation*}
		\gamma \le \min\left\{\frac{1}{4(\tau-1)\mu}, \frac{1}{2\sqrt{(\tau-1)\left(F+\frac{2BG}{\rho(1-\rho)}\right)}}, \frac{1}{4\sqrt{2L(\tau-1)\left(A + \frac{2BC}{\rho(1-\rho)}\right)}}\right\}.
	\end{equation*}
	Moreover, assume that $f_1, f_2, \ldots, f_n$ are $\zeta$-heterogeneous and $\mu$-strongly convex, and $\EE\left[g_i^k\mid x_i^k\right] = \nabla f_i	(x_i^k)$ for all $i\in [n]$. Then \eqref{eq:sum_V_k_bounds} holds with
	\begin{equation}
		H = \frac{4B(\tau-1)\gamma^2(2+\rho)}{\rho},\quad D_3 = 2(\tau-1)\left(D_1 + \frac{\zeta^2}{\gamma\mu} + \frac{2BD_2}{\rho}\right).\label{eq:V_k_bound_homo}
	\end{equation}
\end{lemma}
\begin{proof}
	First of all, if $k \mod \tau = 0$, then $V_k = 0$ by definition. Otherwise, we have
	\begin{eqnarray*}
		V_k &\overset{\eqref{eq:local_sgd_def},\eqref{eq:x^k_recurrsion}}{=}& \frac{1}{n}\sum\limits_{i=1}^n \left\|x_i^{k-1} - x^{k-1} -\gamma g_i^{k-1} + \gamma g^{k-1}\right\|^2\\
		&=& \frac{1}{n}\sum\limits_{i=1}^n\|x_i^{k-1} - x^{k-1}\|^2 + \frac{2\gamma}{n}\sum\limits_{i=1}^n\left\langle x_i^{k-1} - x^{k-1}, g^{k-1} - g_i^{k-1} \right\rangle  + \frac{\gamma^2}{n}\sum\limits_{i=1}^n\|g_i^{k-1} - g^{k-1}\|^2\\
		&=& V_{k-1} + 2\gamma\left\langle\frac{1}{n}\sum\limits_{i=1}^nx_i^{k-1} - x^{k-1}, g^{k-1}\right\rangle + \frac{2\gamma}{n}\sum\limits_{i=1}^n\left\langle x^{k-1}-x_i^{k-1}, g_i^{k-1} \right\rangle\\
		&&\quad  + \frac{\gamma^2}{n}\sum\limits_{i=1}^n\|g_i^{k-1} - g^{k-1}\|^2\\
		&=& V_{k-1} + \frac{2\gamma}{n}\sum\limits_{i=1}^n\left\langle x^{k-1}-x_i^{k-1}, g_i^{k-1} \right\rangle  + \frac{\gamma^2}{n}\sum\limits_{i=1}^n\|g_i^{k-1} - g^{k-1}\|^2.
	\end{eqnarray*}
	Next, we take the conditional expectation $\EE\left[\cdot\mid x^{k-1}\right] \eqdef \EE\left[\cdot\mid x_1^{k-1},\ldots, x_n^{k-1}\right]$ on both sides of the obtained inequality and get
	\begin{eqnarray*}
		\EE\left[V_k\mid x^{k-1}\right] &=& V_{k-1} + \frac{2\gamma}{n}\sum\limits_{i=1}^n\left\langle x^{k-1} - x_i^{k-1}, \nabla f_i(x_i^{k-1}) \right\rangle + \frac{\gamma^2}{n}\sum\limits_{i=1}^n\EE\left[\|g_i^{k-1} - g^{k-1}\|^2\mid x^{k-1}\right]\\
		&\overset{\eqref{eq:variance_decomposition}}{\le}& V_{k-1} + \frac{2\gamma}{n}\sum\limits_{i=1}^n \left\langle x^{k-1} - x_i^{k-1}, \nabla f_i(x_i^{k-1}) - \nabla f_i(x^{k-1}) \right\rangle\\
		&&\quad + \frac{2\gamma}{n}\sum\limits_{i=1}^n\left\langle x^{k-1} - x_i^{k-1}, \nabla f_i(x^{k-1}) \right\rangle + \frac{\gamma^2}{n}\sum\limits_{i=1}^n\EE\left[\|g_i^{k-1}\|^2\mid x^{k-1}\right].
	\end{eqnarray*}
	Since $\frac{1}{n}\sum_{i=1}^n\langle x^{k-1}-x_i^{k-1},\nabla f(x^{k-1}) \rangle = 0$, we can continue as follows:
	\begin{eqnarray*}
		\EE\left[V_k\mid x^{k-1}\right] &\overset{\eqref{eq:coercivity}}{\le}& V_{k-1} - \frac{2\gamma\mu}{n}\sum\limits_{i=1}^n \| x^{k-1} - x_i^{k-1}\|^2 + \frac{\gamma^2}{n}\sum\limits_{i=1}^n\EE\left[\|g_i^{k-1}\|^2\mid x^{k-1}\right]\\
		&&\quad + \frac{2\gamma}{n}\sum\limits_{i=1}^n\left\langle x^{k-1} - x_i^{k-1}, \nabla f_i(x^{k-1}) - \nabla f(x^{k-1}) \right\rangle\\
		&\overset{\eqref{eq:fenchel_young}}{\le}& (1-2\gamma\mu)V_{k-1} + \frac{\gamma^2}{n}\sum\limits_{i=1}^n\EE\left[\|g_i^{k-1}\|^2\mid x^{k-1}\right]\\
		&&\quad +\frac{2\gamma}{n}\sum\limits_{i=1}^n\left(\frac{\mu}{2}\|x^{k-1}-x_i^{k-1}\|^2 + \frac{1}{2\mu}\|\nabla f_i(x^{k-1})-\nabla f(x^{k-1})\|^2\right)\\
		&\overset{\eqref{eq:bounded_data_dissimilarity}}{\le}& (1-\gamma\mu)V_{k-1} + \frac{\gamma^2}{n}\sum\limits_{i=1}^n\EE\left[\|g_i^{k-1}\|^2\mid x^{k-1}\right] + \frac{\gamma\zeta^2}{\mu}.
	\end{eqnarray*}
	Taking full expectation on both sides of previous inequality, we obtain
	\begin{eqnarray*}
		\EE V_k &\overset{\eqref{eq:tower_property}}{\le}& \EE\left[V_{k-1}\right] + \frac{\gamma^2}{n}\sum\limits_{i=1}^n\EE\left[\|g_i^{k-1}\|^2\right] + \frac{\gamma\zeta^2}{\mu}.
	\end{eqnarray*}
	Let $t$ be a non-negative integer for which $\tau t \le k < \tau(t+1)$. Using this and $V_{\tau t} = 0$, we unroll the recurrence and derive
	\begin{eqnarray*}
		\EE[V_k] &\le& \frac{\gamma^2}{n}\sum\limits_{l = \tau t}^{k-1}\sum\limits_{i=1}^n\EE\left[\|g_i^{l}\|^2\right] + \frac{\gamma\zeta^2(k-\tau t)}{\mu}\\
		&\overset{\eqref{eq:second_moment_bound}}{\le}& \gamma^2\sum\limits_{l=\tau t}^{k-1}\left(2A\EE\left[f(x^l) - f(x^*)\right] + B\EE[\sigma_l^2] + F\EE[V_l] + D_1\right) + \frac{\gamma\zeta^2(k-\tau t)}{\mu},
	\end{eqnarray*}
	whence
	\begin{eqnarray}
		\sum\limits_{j=\tau t}^k w_j \EE V_j &\le& 2A\gamma^2\sum\limits_{j=\tau t}^k\sum\limits_{l=\tau t}^{j-1}w_j\EE\left[f(x^l) - f(x^*)\right] + B\gamma^2\sum\limits_{j=\tau t}^k\sum\limits_{l=\tau t}^{j-1}w_j\EE\sigma_l^2 \notag \\
		&&\quad + F\gamma^2\sum\limits_{j=\tau t}^k\sum\limits_{l=\tau t}^{j-1} w_j\EE V_l + (\tau-1)\left(\gamma^2 D_1+\frac{\gamma\zeta^2}{\mu}\right)\sum\limits_{j=\tau t}^kw_j. \notag
	\end{eqnarray}
	If we substitute $A$ with $e(\tA(\tau-1)+\hA)$, $B$ with $e(\tB(\tau-1)+\hB)$, $F$ with $e(\tF(\tau-1)+\hF)$, and $\left(\gamma^2 D_1+\frac{\gamma\zeta^2}{\mu}\right)$ with $e\gamma^2(\tD_1(\tau-1) + \hD_1)$ in the inequality above, we will get inequality \eqref{eq:V_k_lemma_technical_1}.
%	We notice that this inequality is essentially \eqref{eq:V_k_lemma_technical_1} with $\tau-1$ times smaller right-hand side and the second factor in the last term is $\gamma^2 D_1+\frac{\gamma\zeta^2}{\mu}$ instead of $\gamma^2 D_1$. 
	Following the same steps as in the proof of Lemma~\ref{lem:V_k_lemma}, we get
	\begin{eqnarray*}
		\sum\limits_{k=0}^K w_k \EE V_k &\le& 4(\tau-1)\gamma^2\left(A +\frac{2BC}{\rho(1-\rho)}\right)\sum\limits_{k=0}^K w_k r_k + \frac{2B\EE\sigma_0^2(2+\rho)(\tau-1)\gamma^2}{\rho}\\
		&&\quad + 2(\tau-1)\gamma^2\left(F + \frac{2BG}{\rho(1-\rho)}\right)\sum\limits_{k=0}^K w_k\EE V_k + (\tau-1)\gamma^2\left(D_1 + \frac{\zeta^2}{\gamma\mu} + \frac{2BD_2}{\rho}\right)W_K.
	\end{eqnarray*}
	Our choice of $\gamma$ implies that
	\begin{equation*}
		4(\tau-1)\gamma^2\left(A + \frac{2BC}{\rho(1-\rho)}\right) \le \frac{1}{8L}\quad\text{and}\quad 2(\tau-1)\gamma^2\left(F + \frac{2BG}{\rho(1-\rho)}\right) \le \frac{1}{2}.
	\end{equation*}
	Using these inequalities we continue our derivations
	\begin{eqnarray*}
		\frac{1}{2}\sum\limits_{k=0}^K w_k\EE V_k &\le& \frac{1}{8L}\sum\limits_{k=0}^K w_k r_k + \frac{2B\EE\sigma_0^2(2+\rho)(\tau-1)\gamma^2}{\rho}\\
		&&\quad + (\tau-1)\gamma^2\left(D_1 + \frac{\zeta^2}{\gamma\mu} + \frac{2BD_2}{\rho}\right)W_K.
	\end{eqnarray*}
	Multiplying both sides by $4L$ we get the result.
\end{proof}

Clearly, this lemma and Theorem~\ref{thm:main_result} imply the following result.

\begin{corollary}\label{cor:const_loop_homo}
	Let the assumptions of Lemma~\ref{lem:V_k_lemma_homo} be satisfied. Then Assumption~\ref{ass:key_assumption} holds and, in particular, if
	\begin{eqnarray*}
		\gamma &\le& \min\left\{\frac{1}{2(A'+CM)}, \frac{L}{F'+GM}\right\},\quad M = \frac{4B'}{3\rho},\\
		\gamma &\le& \min\left\{\frac{1}{4(\tau-1)\mu}, \frac{1}{2\sqrt{(\tau-1)\left(F+\frac{2BG}{\rho(1-\rho)}\right)}}, \frac{1}{4\sqrt{2L(\tau-1)\left(A + \frac{2BC}{\rho(1-\rho)}\right)}}\right\},
	\end{eqnarray*}
	then for all $K\ge 0$ we have
	\begin{eqnarray}
		\EE\left[f(\overline{x}^K) - f(x^*)\right] &\le& \frac{2T^0 + 4LH\gamma\EE\sigma_0^2}{\gamma W_K} + 2\gamma\left(D_1' + MD_2 + 2L\gamma D_3\right), \label{eq:main_result_const_loop_homo}
	\end{eqnarray}
	where $\overline{x}^K \eqdef \frac{1}{W_K}\sum_{k=0}^K w_k x^k$ and
	\begin{equation*}
		H = \frac{4B(\tau-1)\gamma^2(2+\rho)}{\rho},\quad D_3 = 2(\tau-1)\left(D_1 + \frac{\zeta^2}{\gamma\mu} + \frac{2BD_2}{\rho}\right).
	\end{equation*}
	 Moreover, if $\mu > 0$, then
	\begin{eqnarray}
		\EE\left[f(\overline{x}^K) - f(x^*)\right] &\le& \left(1 - \min\left\{\gamma\mu,\frac{\rho}{4}\right\}\right)^K\frac{2T^0 + 4LH\gamma\EE\sigma_0^2}{\gamma} + 2\gamma\left(D_1' + MD_2 + 2L\gamma D_3\right), \label{eq:main_result_1_const_loop_homo}
	\end{eqnarray}
	and in the case when $\mu = 0$, we have
	\begin{eqnarray}
		\EE\left[f(\overline{x}^K) - f(x^*)\right] &\le& \frac{2T^0 + 4LH\gamma\EE\sigma_0^2}{\gamma K} + 2\gamma\left(D_1' + MD_2 + 2L\gamma D_3\right). \label{eq:main_result_2_const_loop_homo}
	\end{eqnarray}
\end{corollary}

\subsection{Random Local Loop}\label{sec:random_local_loop}
In this section we show how our results can be applied to analyze \eqref{eq:local_sgd_def} in the case when
\begin{equation*}
	c_{k} = \begin{cases}1,& \text{with probability } p,\\ 0,& \text{with probability } 1-p, \end{cases}
\end{equation*}
where $p$ encodes the probability of initiating communication. This choice in effect leads to a method using a random-length local loop on all devices.

\subsubsection{Heterogeneous Data}\label{sec:random_local_loop_hetero}
As in Section~\ref{sec:const_loop_hetero}, our analysis of \eqref{eq:local_sgd_def} with random length of the local loop relies on Assumption~\ref{ass:hetero_second_moment}. Next lemma provides an upper bound for the weighted sum of $\EE\left[ V_k \right]$ in this case.
\begin{lemma}\label{lem:V_k_lemma_random}
	Let Assumptions~\ref{ass:quasi_strong_convexity},~\ref{ass:L_smoothness}~and~\ref{ass:hetero_second_moment} be satisfied and\footnote{When $\rho = 1$ one can always set the parameters in such a way that $\tB = \hB = C = G = 0$, $D_2 = 0$. In this case we assume that $\frac{2\tB C}{\rho(1-\rho)} = \frac{2\hB C}{\rho(1-\rho)} = \frac{2\tB G}{\rho(1-\rho)} = \frac{2\hB G}{\rho(1-\rho)} = 0$.}
	\begin{eqnarray*}
		\gamma &\le& \min\left\{\frac{p}{16\mu}, \frac{p}{2\sqrt{(1-p)((2+p)\tF+p\hF)}}\right\},\\
		\gamma &\le& \min\left\{\frac{p\sqrt{3\rho(1-\rho)}}{8\sqrt{2G(1-p)\left((p+2)\tB + p\hB\right)}}, \frac{p\sqrt{3}}{16\sqrt{2L(1-p)\left((2+p)\tA + p\hA + \frac{2C\left((p+2)\tB + p\hB\right)}{\rho(1-\rho)}\right)}}\right\}.
	\end{eqnarray*}
	Then \eqref{eq:sum_V_k_bounds} holds with
	\begin{equation}
		H = \frac{64(1-p)\left((p+2)\tB + p\hB\right)(2+\rho)\gamma^2}{3p^2\rho},\quad D_3 = \frac{8(1-p)}{p^2}\left((p+2)\tD_{1} + p \hD_{1} + \frac{8D_2\left((p+2)\tB+p\hB\right)}{3\rho}\right).\label{eq:V_k_bound_random}
	\end{equation}
\end{lemma}
\begin{proof}
	First of all, we introduce new notation: $\EE[\cdot\mid x^{k},g^{k}]\eqdef \EE[\cdot\mid x_1^{k},\ldots,x_n^{k},g_1^{k},\ldots,g_n^{k}]$, $\EE[\cdot\mid x^{k}]\eqdef \EE[\cdot\mid x_1^{k},\ldots,x_n^{k}]$. By definition of $V_k$, we have
	\begin{eqnarray*}
		\EE\left[V_{k+1}\mid x^k\right] &\overset{\eqref{eq:tower_property}}{=}& \frac{1}{n}\sum\limits_{i=1}^n\EE\left[\EE\left[\|x_i^{k+1} - x^{k+1}\|^2\mid x^k, g^k\right]\mid x^k\right]\\
		&=& \frac{1-p}{n}\sum\limits_{i=1}^n\EE\left[\|x_i^k - x^k - \gamma g_i^k + \gamma g^k\|^2\mid x^k\right]\\
		&\overset{\eqref{eq:variance_decomposition}}{=}& \frac{1-p}{n}\sum\limits_{i=1}^n\|x_i^k - x^k - \gamma \bar{g}_i^k + \gamma \bar{g}^k\|^2 + \frac{(1-p)\gamma^2}{n}\sum\limits_{i=1}^n\EE\left[\|g_i^k - \bar{g}_i^k - (g^k - \bar{g}^k)\|^2\mid x^k\right]\\
		&\overset{\eqref{eq:a+b_norm_beta},\eqref{eq:variance_decomposition}}{\le}& \frac{(1-p)\left(1+\frac{p}{2}\right)}{n}\sum\limits_{i=1}^n\|x_i^k - x^k\|^2 + \frac{(1-p)\left(1+\frac{2}{p}\right)\gamma^2}{n}\sum\limits_{i=1}^n\|\bar{g}_i^k - \bar{g}^k\|^2\\
		&&\quad + \frac{(1-p)\gamma^2}{n}\sum\limits_{i=1}^n\EE\left[\|g_i^k - \bar{g}_i^k\|^2\mid x^k\right]\\
		&\overset{\eqref{eq:1+p/2_inequality},\eqref{eq:variance_decomposition}}{\le}& \left(1 - \frac{p}{2}\right)V_k + \frac{(1-p)(2+p)\gamma^2}{pn}\sum\limits_{i=1}^n\left\|\bar{g}_i^k\right\|^2 + \frac{(1-p)\gamma^2}{n}\sum\limits_{i=1}^n\EE\left[\|g_i^k - \bar{g}_i^k\|^2\mid x^k\right],
	\end{eqnarray*}
	where $\bar{g}^k = \EE[g^k\mid x^k]$. Taking the full expectation we derive
	\begin{eqnarray*}
		\EE\left[V_{k+1}\right] &\le& \left(1 - \frac{p}{2}\right)\EE\left[V_k\right] + \frac{(1-p)(2+p)\gamma^2}{pn}\sum\limits_{i=1}^n\EE\left[\left\|\bar{g}_i^k\right\|^2\right] + \frac{(1-p)\gamma^2}{n}\sum\limits_{i=1}^n\EE\left[\|g_i^k - \bar{g}_i^k\|^2\right]\\
		&\overset{\eqref{eq:hetero_second_moment_bound},\eqref{eq:hetero_var_bound}}{\le}& \left(1 - \frac{p}{2}\right)\EE\left[V_k\right] + 2(1-p)\gamma^2\left(\frac{2+p}{p}\tA + \hA\right)\EE\left[f(x^k)-f(x^*)\right]\\
		&&\quad + (1-p)\gamma^2\left(\left(\frac{2+p}{p}\tB + \hB\right)\EE\sigma_k^2+\left(\frac{2+p}{p}\tF + \hF\right)\EE V_k\right)\\
		&&\quad + (1-p)\gamma^2\left(\frac{2+p}{p}\tD_{1} + \hD_{1}\right).
	\end{eqnarray*}
	This inequality together with $\gamma \le \frac{p}{2\sqrt{(1-p)((2+p)\tF+p\hF)}}$  imply
	\begin{eqnarray*}
		\EE\left[V_{k+1}\right] &\le& \left(1 - \frac{p}{4}\right)\EE\left[V_k\right] + 2(1-p)\gamma^2\left(\frac{2+p}{p}\tA + \hA\right)\EE\left[f(x^k)-f(x^*)\right]\\
		&&\quad + (1-p)\gamma^2\left(\frac{2+p}{p}\tB + \hB\right)\EE\sigma_k^2 + (1-p)\gamma^2\left(\frac{2+p}{p}\tD_{1} + \hD_{1}\right).
	\end{eqnarray*}
	Unrolling the recurrence, we obtain
	\begin{eqnarray*}
		\EE\left[V_{k+1}\right] &\le& 2(1-p)\gamma^2\left(\frac{2+p}{p}\tA + \hA\right)\sum\limits_{l=0}^k \left(1 - \frac{p}{4}\right)^{k-l}\EE\left[f(x^l) - f(x^*)\right]\\
		&&\quad + (1-p)\gamma^2\left(\frac{2+p}{p}\tB + \hB\right)\sum\limits_{l=0}^k \left(1 - \frac{p}{4}\right)^{k-l}\EE\sigma_l^2\\
		&&\quad + (1-p)\gamma^2\left(\frac{2+p}{p}\tD_{1} + \hD_{1}\right)\sum\limits_{l=0}^k\left(1-\frac{p}{4}\right)^{k-l}.
	\end{eqnarray*}
	As a consequence, we derive
	\begin{eqnarray}
		\sum\limits_{k=0}^K w_k\EE\left[V_k\right] &\le& \frac{2(1-p)\left((2+p)\tA + p\hA\right)\gamma^2}{p\left(1-\frac{p}{4}\right)}\sum\limits_{k=0}^K\sum\limits_{l=0}^k\left(1 - \frac{p}{4}\right)^{k-l}w_kr_l\notag\\
		&&\quad + \frac{(1-p)\left((2+p)\tB + p\hB\right)\gamma^2}{p\left(1-\frac{p}{4}\right)}\sum\limits_{k=0}^K\sum\limits_{l=0}^k\left(1 - \frac{p}{4}\right)^{k-l}w_k\EE\left[\sigma_l^2\right]\notag\\
		&&\quad + \frac{(1-p)\left((2+p)\tD_{1} + p\hD_{1}\right)\gamma^2}{p}\sum\limits_{k=0}^K\sum\limits_{l=0}^{k-1}\left(1 - \frac{p}{4}\right)^{k-1-l}w_k,\label{eq:V_k_bound_rand_tech_1}
	\end{eqnarray}
	where we use new notation: $r_l = \EE\left[f(x^l) - f(x^*)\right]$. Recall that $w_k = (1 - \eta)^{-(k+1)}$ and $\eta = \min\left\{\gamma\mu, \frac{\rho}{4}\right\}$. Together with our assumption on $\gamma$ it implies that for all $0 \le i < k$ we have
	\begin{eqnarray}
		w_k &=& (1 - \eta)^{-(k-i+1)}\left(1 - \eta\right)^{-i} \overset{\eqref{eq:1-p/2_inequality}}{\le} w_{k-i}\left(1 + 2\eta\right)^{i} \notag\\
		&\le& w_{k-i}\left(1 + 2\gamma\mu\right)^{i} \le w_{k-i}\left(1+\frac{p}{8}\right)^i, \label{eq:V_k_bound_rand_tech_2}\\
		w_k &=& \left(1 - \eta\right)^{-(k-i+1)}\left(1 - \eta\right)^{-i} \overset{\eqref{eq:1-p/2_inequality}}{\le} w_{k-i}\left(1 + 2\eta\right)^i \le w_{k-i}\left(1 + \frac{\rho}{2}\right)^i , \label{eq:V_k_bound_rand_tech_3}\\
		w_k &\overset{\eqref{eq:1-p/2_inequality}}{\le}& \left(1 + 2\eta\right)^{k+1} \le \left(1 + \frac{\rho}{2}\right)^{k+1}. \label{eq:V_k_bound_rand_tech_4}
	\end{eqnarray}
	Having these inequalities in hand we obtain
	\begin{eqnarray*}
		\sum\limits_{k=0}^K\sum\limits_{l=0}^k\left(1 - \frac{p}{4}\right)^{k-l}w_kr_l &\overset{\eqref{eq:V_k_bound_rand_tech_2}}{\le}& \sum\limits_{k=0}^K\sum\limits_{l=0}^k\left(1 - \frac{p}{4}\right)^{k-l}\left(1 + \frac{p}{8}\right)^{k-l}w_lr_l\\
		&\overset{\eqref{eq:1+p/2_inequality}}{\le}& \sum\limits_{k=0}^K\sum\limits_{l=0}^k\left(1 - \frac{p}{8}\right)^{k-l}w_lr_l \le \left(\sum\limits_{k=0}^K w_k r_k\right)\left(\sum\limits_{k=0}^{\infty}\left(1 - \frac{p}{8}\right)^{k}\right)\\
		&=&\frac{8}{p}\sum\limits_{k=0}^K w_k r_k,
	\end{eqnarray*}
	\begin{eqnarray*}
		\sum\limits_{k=0}^K\sum\limits_{l=0}^k\left(1 - \frac{p}{4}\right)^{k-l}w_k\EE\left[\sigma_l^2\right] &\overset{\eqref{eq:V_k_bound_rand_tech_2}}{\le}& \sum\limits_{k=0}^K\sum\limits_{l=0}^k\left(1 - \frac{p}{4}\right)^{k-l}\left(1 + \frac{p}{8}\right)^{k-l}w_l\EE\left[\sigma_l^2\right]\\
		&\overset{\eqref{eq:1+p/2_inequality}}{\le}& \sum\limits_{k=0}^K\sum\limits_{l=0}^k\left(1 - \frac{p}{8}\right)^{k-l}w_l\EE\left[\sigma_l^2\right] \le \left(\sum\limits_{k=0}^K w_k \EE\left[\sigma_k^2\right]\right)\left(\sum\limits_{k=0}^{\infty}\left(1 - \frac{p}{8}\right)^{k}\right)\\
		&=&\frac{8}{p}\sum\limits_{k=0}^K w_k \EE\left[\sigma_k^2\right],
	\end{eqnarray*}
	and
	\begin{eqnarray*}
		\sum\limits_{k=0}^K\sum\limits_{l=0}^{k-1}\left(1 - \frac{p}{4}\right)^{k-1-l}w_k &\le& \left(\sum\limits_{k=0}^K w_k \right)\left(\sum\limits_{k=0}^{\infty}\left(1 - \frac{p}{4}\right)^{k}\right) = \frac{4W_K}{p}.
	\end{eqnarray*}
	Plugging these inequalities together with $1-\frac{p}{4}\ge \frac{3}{4}$ in \eqref{eq:V_k_bound_rand_tech_1}, we derive
	\begin{eqnarray}
		\sum\limits_{k=0}^K w_k\EE\left[V_k\right] &\le& \frac{64(1-p)\left((2+p)\tA + p\hA\right)\gamma^2}{3p^2}\sum\limits_{k=0}^Kw_kr_k + \frac{32(1-p)\left((2+p)\tB + p\hB\right)\gamma^2}{3p^2}\sum\limits_{k=0}^Kw_k\EE\left[\sigma_k^2\right]\notag\\
		&&\quad + \frac{4(1-p)\left((2+p)\tD_{1} + p\hD_{1}\right)\gamma^2}{p^2}W_K.\label{eq:V_k_bound_rand_tech_5}
	\end{eqnarray}
	It remains to estimate the second term on the right-hand side of this inequality. We notice that an analogous term appears in the proof of Lemma~\ref{lem:V_k_lemma}. In particular, in that proof inequality \eqref{eq:sigma_k_technical_bound} was shown via inequalities \eqref{eq:sigma_k+1_bound}, \eqref{eq:V_k_bound_rand_tech_3}, \eqref{eq:V_k_bound_rand_tech_4} and \eqref{eq:1+p/2_inequality} which hold in this case too. Therefore, we get that
	\begin{eqnarray}
	\sum\limits_{k=0}^Kw_k\EE\left[\sigma_k^2\right] &\overset{\eqref{eq:sigma_k_technical_bound}}{\le}& \frac{\EE\sigma_0^2(2+\rho)}{\rho} + \frac{4C}{\rho(1-\rho)}\sum\limits_{k=0}^Kw_kr_k + \frac{2G}{\rho(1-\rho)}\sum\limits_{k=0}^Kw_k\EE V_k +  \frac{D_2 W_K}{\rho},\notag
	\end{eqnarray}
	whence
	\begin{eqnarray}
		\sum\limits_{k=0}^K w_k\EE\left[V_k\right] &\overset{\eqref{eq:V_k_bound_rand_tech_5}}{\le}& \frac{64(1-p)\gamma^2\left((2+p)\tA + p\hA + \frac{2C\left((p+2)\tB + p\hB\right)}{\rho(1-\rho)}\right)}{3p^2}\sum\limits_{k=0}^Kw_kr_k \notag\\
		&&\quad + \frac{32(1-p)\left((p+2)\tB + p\hB\right)(2+\rho)\gamma^2\EE\sigma_0^2}{3p^2\rho}\notag\\
		&&\quad + \frac{64G(1-p)\left((p+2)\tB + p\hB\right)\gamma^2}{3p^2\rho(1-\rho)}\sum\limits_{k=0}^K w_k\EE\left[V_k\right]\notag\\
		&&\quad + \frac{4(1-p)\gamma^2}{p^2}\left((p+2)\tD_{1} + p \hD_{1} + \frac{8D_2\left((p+2)\tB+p\hB\right)}{3\rho}\right)W_K.\notag
	\end{eqnarray}
	Our assumptions on $\gamma$ imply
	\begin{eqnarray*}
		\frac{64(1-p)\gamma^2\left((2+p)\tA + p\hA + \frac{2C\left((p+2)\tB + p\hB\right)}{\rho(1-\rho)}\right)}{3p^2} \le \frac{1}{8L},\quad \frac{64G(1-p)\left((p+2)\tB + p\hB\right)\gamma^2}{3p^2\rho(1-\rho)} \le \frac{1}{2}.
	\end{eqnarray*}
	Next, we introduce new notation as follows:
	\begin{equation*}
		H = \frac{64(1-p)\left((p+2)\tB + p\hB\right)(2+\rho)\gamma^2}{3p^2\rho},\quad D_3 = \frac{8(1-p)}{p^2}\left((p+2)\tD_{1} + p \hD_{1} + \frac{8D_2\left((p+2)\tB+p\hB\right)}{3\rho}\right).
	\end{equation*}
	Putting all together, we get
	\begin{equation*}
		\frac{1}{2}\sum\limits_{k=0}^Kw_k\EE\left[V_k\right] \le \frac{1}{8L}\sum\limits_{k=0}^K w_k r_k + \frac{H}{2}\EE\sigma_0^2 + \frac{D_3}{2}\gamma^2 W_K,
	\end{equation*}
	which concludes the proof.
\end{proof}

This lemma and Theorem~\ref{thm:main_result} imply the following result.
\begin{corollary}\label{cor:rand_loop}
	Let the assumptions of Lemma~\ref{lem:V_k_lemma_random} be satisfied. Then Assumption~\ref{ass:key_assumption} holds and, in particular, if
	\begin{eqnarray*}
		\gamma &\le& \min\left\{\frac{1}{2\left(A'+\frac{4B'C}{3\rho}\right)}, \frac{L}{F'+\frac{4B'G}{3\rho}}, \frac{p}{16\mu}, \frac{p}{2\sqrt{(1-p)((2+p)\tF+p\hF)}}\right\},\\
		\gamma &\le& \min\left\{\frac{p\sqrt{3\rho(1-\rho)}}{8\sqrt{2G(1-p)\left((p+2)\tB + p\hB\right)}}, \frac{p\sqrt{3}}{16\sqrt{2L(1-p)\left((2+p)\tA + p\hA + \frac{2C\left((p+2)\tB + p\hB\right)}{\rho(1-\rho)}\right)}}\right\},
	\end{eqnarray*}
	then for all $K\ge 0$ we have
	\begin{eqnarray}
		\EE\left[f(\overline{x}^K) - f(x^*)\right] &\le& \frac{2\|x^0 - x^*\|^2 +   \frac{8B'}{3\rho}\gamma^2 \EE\sigma_0^2 + 4LH\gamma\EE\sigma_0^2}{\gamma W_K} + 2\gamma\left(D_1' + \frac{4B'D_2}{3\rho} + 2L\gamma D_3\right), \label{eq:main_result_rand_loop}
	\end{eqnarray}
	where $\overline{x}^K \eqdef \frac{1}{W_K}\sum_{k=0}^K w_k x^k$ and
	\begin{equation*}
		H = \frac{64(1-p)\left((p+2)\tB + p\hB\right)(2+\rho)\gamma^2}{3p^2\rho},\quad D_3 = \frac{8(1-p)}{p^2}\left((p+2)\tD_{1} + p \hD_{1} + \frac{8D_2\left((p+2)\tB+p\hB\right)}{3\rho}\right).
	\end{equation*}
	 Moreover, if $\mu > 0$, then
	\begin{eqnarray}
		\EE\left[f(\overline{x}^K) - f(x^*)\right] &\le& \left(1 - \min\left\{\gamma\mu,\frac{\rho}{4}\right\}\right)^K\frac{2\|x^0 - x^*\|^2 +   \frac{8B'}{3\rho}\gamma^2 \EE\sigma_0^2 + 4LH\gamma\EE\sigma_0^2}{\gamma}\notag\\
		&&\quad + 2\gamma\left(D_1' + \frac{4B'D_2}{3\rho} + 2L\gamma D_3\right), \label{eq:main_result_1_rand_loop}
	\end{eqnarray}
	and in the case when $\mu = 0$, we have
	\begin{eqnarray}
		\EE\left[f(\overline{x}^K) - f(x^*)\right] &\le& \frac{2\|x^0 - x^*\|^2 +   \frac{8B'}{3\rho}\gamma^2 \EE\sigma_0^2 + 4LH\gamma\EE\sigma_0^2}{\gamma K} + 2\gamma\left(D_1' + \frac{4B'D_2}{3\rho} + 2L\gamma D_3\right). \label{eq:main_result_2_rand_loop}
	\end{eqnarray}
\end{corollary}

\subsubsection{$\zeta$-Heterogeneous Data}\label{sec:random_loop_homo}
In this section we assume that $f_1, f_2, \ldots, f_n$ are $\zeta$-heterogeneous (see Definition~\ref{def:zeta_hetero}). Moreover, we additionally assume that $\EE\left[g_i^k\mid x_i^k\right] = \nabla f_i(x_i^k)$ and we also assume $\mu$-strong convexity of the functions $f_i$ for $i\in[n]$.

\begin{lemma}\label{lem:V_k_lemma_random_homo}
	Let Assumption \ref{ass:L_smoothness} be satisfied, inequalities \eqref{eq:unbiasedness}-\eqref{eq:sigma_k+1_bound} hold and\footnote{When $\rho = 1$ one can always set the parameters in such a way that $B = C = G = 0$, $D_2 = 0$. In this case we assume that $\frac{2BC}{\rho(1-\rho)} = \frac{2BG}{\rho(1-\rho)} = 0$.}
	\begin{eqnarray*}
		\gamma &\le& \min\left\{\frac{p}{8\mu},\sqrt{\frac{p}{2F(1-p)}}, \sqrt{\frac{p\rho(1-\rho)}{32BG(1-p)}}, \sqrt{\frac{p}{128L(1-p)\left(A + \frac{2BC}{\rho(1-\rho)}\right)}}\right\}.
	\end{eqnarray*}
	Moreover, assume that $f_1, f_2, \ldots, f_n$ are $\zeta$-heterogeneous and $\mu$-strongly convex, and $\EE\left[g_i^k\mid x_i^k\right] = \nabla f_i	(x_i^k)$ for all $i\in [n]$. Then \eqref{eq:sum_V_k_bounds} holds with
	\begin{equation}
		H = \frac{16B(1-p)(2+\rho)\gamma^2}{p\rho},\quad D_3 = \frac{4(1-p)}{p}\left(D_1 + \frac{\zeta^2}{\gamma\mu} + \frac{4BD_2}{\rho}\right).\label{eq:V_k_bound_random_homo}
	\end{equation}
\end{lemma}
\begin{proof}
	First of all, we introduce new notation: $\EE[\cdot\mid x^{k},g^{k}]\eqdef \EE[\cdot\mid x_1^{k},\ldots,x_n^{k},g_1^{k},\ldots,g_n^{k}]$. By definition of $V_k$ for all $k\ge 1$ we have
	\begin{eqnarray*}
		\EE[V_k\mid x^{k-1},g^{k-1}] &\overset{\eqref{eq:local_sgd_def},\eqref{eq:x^k_recurrsion}}{=}& \frac{1-p}{n}\sum\limits_{i=1}^n \left\|x_i^{k-1} - x^{k-1} -\gamma g_i^{k-1} + \gamma g^{k-1}\right\|^2\\
		&=& \frac{1-p}{n}\sum\limits_{i=1}^n\|x_i^{k-1} - x^{k-1}\|^2 + \frac{2\gamma(1-p)}{n}\sum\limits_{i=1}^n\left\langle x_i^{k-1} - x^{k-1}, g^{k-1} - g_i^{k-1} \right\rangle\\
		&&\quad  + \frac{\gamma^2(1-p)}{n}\sum\limits_{i=1}^n\|g_i^{k-1} - g^{k-1}\|^2\\
		&=& (1-p)V_{k-1} + 2\gamma(1-p)\left\langle\frac{1}{n}\sum\limits_{i=1}^nx_i^{k-1} - x^{k-1}, g^{k-1}\right\rangle\\
		&&\quad + \frac{2\gamma(1-p)}{n}\sum\limits_{i=1}^n\left\langle x^{k-1}-x_i^{k-1}, g_i^{k-1} \right\rangle + \frac{\gamma^2(1-p)}{n}\sum\limits_{i=1}^n\|g_i^{k-1} - g^{k-1}\|^2\\
		&=& (1-p)V_{k-1} + \frac{2\gamma(1-p)}{n}\sum\limits_{i=1}^n\left\langle x^{k-1}-x_i^{k-1}, g_i^{k-1} \right\rangle \\
		&&\quad + \frac{\gamma^2(1-p)}{n}\sum\limits_{i=1}^n\|g_i^{k-1} - g^{k-1}\|^2.
	\end{eqnarray*}
	Next, we take the conditional expectation $\EE\left[\cdot\mid x^{k-1}\right] \eqdef \EE\left[\cdot\mid x_1^{k-1},\ldots, x_n^{k-1}\right]$ on both sides of the obtained inequality and get
	\begin{eqnarray*}
		\EE\left[V_k\mid x^{k-1}\right] &=& (1-p)V_{k-1} + \frac{2\gamma(1-p)}{n}\sum\limits_{i=1}^n\left\langle x^{k-1} - x_i^{k-1}, \nabla f_i(x_i^{k-1}) \right\rangle\\
		&&\quad + \frac{\gamma^2(1-p)}{n}\sum\limits_{i=1}^n\EE\left[\|g_i^{k-1} - g^{k-1}\|^2\mid x^{k-1}\right]\\
		&\overset{\eqref{eq:variance_decomposition}}{\le}& (1-p)V_{k-1} + \frac{2\gamma(1-p)}{n}\sum\limits_{i=1}^n\left\langle x^{k-1} - x_i^{k-1}, \nabla f_i(x_i^{k-1}) - \nabla f_i(x^{k-1}) \right\rangle\\
		&&\quad + \frac{2\gamma(1-p)}{n}\sum\limits_{i=1}^n\left\langle x^{k-1} - x_i^{k-1}, \nabla f_i(x^{k-1}) \right\rangle\\
		&&\quad + \frac{\gamma^2(1-p)}{n}\sum\limits_{i=1}^n\EE\left[\|g_i^{k-1}\|^2\mid x^{k-1}\right].
	\end{eqnarray*}
	Since $\frac{1}{n}\sum_{i=1}^n\langle x^{k-1}-x_i^{k-1},\nabla f(x^{k-1}) \rangle = 0$, we can continue as follows:
	\begin{eqnarray*}
		\EE\left[V_k\mid x^{k-1}\right] &\overset{\eqref{eq:coercivity}}{\le}& (1-p)V_{k-1} - \frac{2\gamma\mu(1-p)}{n}\sum\limits_{i=1}^n\|x^{k-1} - x_i^{k-1}\|^2\\
		&&\quad + \frac{2\gamma(1-p)}{n}\sum\limits_{i=1}^n\left\langle x^{k-1} - x_i^{k-1}, \nabla f_i(x^{k-1}) - \nabla f(x^{k-1}) \right\rangle\\
		&&\quad + \frac{\gamma^2(1-p)}{n}\sum\limits_{i=1}^n\EE\left[\|g_i^{k-1}\|^2\mid x^{k-1}\right]\\
		&\overset{\eqref{eq:fenchel_young}}{\le}& (1-p)(1-2\gamma\mu)V_{k-1} + \frac{\gamma^2(1-p)}{n}\sum\limits_{i=1}^n\EE\left[\|g_i^{k-1}\|^2\mid x^{k-1}\right]\\
		&&\quad + \frac{2\gamma(1-p)}{n}\sum\limits_{i=1}^n\left(\frac{\mu}{2}\|x^{k-1}-x_i^{k-1}\|^2 + \frac{1}{2\mu}\|\nabla f_i(x^{k-1}) - \nabla f(x^{k-1})\|^2\right)\\
		&\overset{\eqref{eq:bounded_data_dissimilarity}}{\le}& (1-p)(1-\gamma\mu)V_{k-1} + \frac{\gamma^2(1-p)}{n}\sum\limits_{i=1}^n\EE\left[\|g_i^{k-1}\|^2\mid x^{k-1}\right] + \frac{(1-p)\gamma\zeta^2}{\mu}.
	\end{eqnarray*}
	Taking full mathematical expectation on both sides of previous inequality and using $1-\gamma\mu \le 1$ we obtain
	\begin{eqnarray*}
		\EE V_k &\overset{\eqref{eq:tower_property}}{\le}& (1-p)\EE\left[V_{k-1}\right] + \frac{\gamma^2(1-p)}{n}\sum\limits_{i=1}^n\EE\left[\|g_i^{k-1}\|^2\right] + \frac{(1-p)\gamma\zeta^2}{\mu}\\
		&\overset{\eqref{eq:second_moment_bound}}{\le}&(1-p)\EE[V_{k-1}] + (1-p)\gamma^2\left(2A\EE[f(x^{k-1})-f(x^*)] + B\EE[\sigma_k^2] + F\EE[V_{k-1}] + D_1\right)\\
		&&\quad + \frac{(1-p)\gamma\zeta^2}{\mu}.
	\end{eqnarray*}
	Since $\gamma \le \sqrt{\frac{p}{2F(1-p)}}$ we have $(1-p)\gamma^2 F \le \frac{p}{2}$ and
	\begin{eqnarray*}
		\EE V_k &\le& \left(1-\frac{p}{2}\right)\EE[V_{k-1}] + (1-p)\gamma^2\left(2A\EE[f(x^{k-1})-f(x^*)] + B\EE[\sigma_k^2] + D_1 + \frac{\zeta^2}{\gamma\mu}\right).
	\end{eqnarray*}
	Unrolling the recurrence we obtain
	\begin{eqnarray*}
		\EE\left[V_{k}\right] &\le& (1-p)\gamma^2\sum\limits_{l=0}^{k-1} \left(1 - \frac{p}{2}\right)^{k-1-l}\left(2A\EE\left[f(x^l) - f(x^*)\right] + B\EE\left[\sigma_l^2\right] + D_1 + \frac{\zeta^2}{\gamma\mu}\right).
	\end{eqnarray*}
	As a consequence, we derive
	\begin{eqnarray}
		\sum\limits_{k=0}^K w_k\EE\left[V_k\right] &\le& \frac{2A(1-p)\gamma^2}{1-\frac{p}{2}}\sum\limits_{k=0}^K\sum\limits_{l=0}^k\left(1 - \frac{p}{2}\right)^{k-l}w_kr_l\notag\\
		&&\quad + \frac{B(1-p)\gamma^2}{1-\frac{p}{2}}\sum\limits_{k=0}^K\sum\limits_{l=0}^k\left(1 - \frac{p}{2}\right)^{k-l}w_k\EE\left[\sigma_l^2\right]\notag\\
		&&\quad + \left(D_1 + \frac{\zeta^2}{\gamma\mu}\right)(1-p)\gamma^2\sum\limits_{k=0}^{K}\sum\limits_{l=0}^{k-1}\left(1 - \frac{p}{2}\right)^{k-1-l}w_k,\label{eq:V_k_bound_rand_tech_1_homo}
	\end{eqnarray}
	where we use new notation: $r_l = \EE\left[f(x^l) - f(x^*)\right]$. Recall that $w_k = (1 - \eta)^{-(k+1)}$ and $\eta = \min\left\{\gamma\mu, \frac{\rho}{4}\right\}$. Together with our assumption on $\gamma$ it implies that for all $0 \le i < k$ we have
	\begin{eqnarray}
		w_k &=& (1 - \eta)^{-(k-i+1)}\left(1 - \eta\right)^{-i} \overset{\eqref{eq:1-p/2_inequality}}{\le} w_{k-i}\left(1 + 2\eta\right)^{i} \notag\\
		&\le& w_{k-i}\left(1 + 2\gamma\mu\right)^{i} \le w_{k-i}\left(1+\frac{p}{4}\right)^i, \label{eq:V_k_bound_rand_tech_2_homo}\\
		w_k &=& \left(1 - \eta\right)^{-(k-i+1)}\left(1 - \eta\right)^{-i} \overset{\eqref{eq:1-p/2_inequality}}{\le} w_{k-i}\left(1 + 2\eta\right)^i \le w_{k-i}\left(1 + \frac{\rho}{2}\right)^i , \label{eq:V_k_bound_rand_tech_3_homo}\\
		w_k &\overset{\eqref{eq:1-p/2_inequality}}{\le}& \left(1 + 2\eta\right)^{k+1} \le \left(1 + \frac{\rho}{2}\right)^{k+1}. \label{eq:V_k_bound_rand_tech_4_homo}
	\end{eqnarray}
	Having these inequalities in hand we obtain
	\begin{eqnarray*}
		\sum\limits_{k=0}^K\sum\limits_{l=0}^k\left(1 - \frac{p}{2}\right)^{k-l}w_kr_l &\overset{\eqref{eq:V_k_bound_rand_tech_2_homo}}{\le}& \sum\limits_{k=0}^K\sum\limits_{l=0}^k\left(1 - \frac{p}{2}\right)^{k-l}\left(1 + \frac{p}{4}\right)^{k-l}w_lr_l\\
		&\overset{\eqref{eq:1+p/2_inequality}}{\le}& \sum\limits_{k=0}^K\sum\limits_{l=0}^k\left(1 - \frac{p}{4}\right)^{k-l}w_lr_l \le \left(\sum\limits_{k=0}^K w_k r_k\right)\left(\sum\limits_{k=0}^{\infty}\left(1 - \frac{p}{4}\right)^{k}\right)\\
		&=&\frac{4}{p}\sum\limits_{k=0}^K w_k r_k,
	\end{eqnarray*}
	\begin{eqnarray*}
		\sum\limits_{k=0}^K\sum\limits_{l=0}^k\left(1 - \frac{p}{2}\right)^{k-l}w_k\EE\left[\sigma_l^2\right] &\overset{\eqref{eq:V_k_bound_rand_tech_2_homo}}{\le}& \sum\limits_{k=0}^K\sum\limits_{l=0}^k\left(1 - \frac{p}{2}\right)^{k-l}\left(1 + \frac{p}{4}\right)^{k-l}w_l\EE\left[\sigma_l^2\right]\\
		&\overset{\eqref{eq:1+p/2_inequality}}{\le}& \sum\limits_{k=0}^K\sum\limits_{l=0}^k\left(1 - \frac{p}{4}\right)^{k-l}w_l\EE\left[\sigma_l^2\right] \le \left(\sum\limits_{k=0}^K w_k \EE\left[\sigma_k^2\right]\right)\left(\sum\limits_{k=0}^{\infty}\left(1 - \frac{p}{4}\right)^{k}\right)\\
		&=&\frac{4}{p}\sum\limits_{k=0}^K w_k \EE\left[\sigma_k^2\right],
	\end{eqnarray*}
	and
	\begin{eqnarray*}
		\sum\limits_{k=0}^K\sum\limits_{l=0}^{k-1}\left(1 - \frac{p}{2}\right)^{k-1-l}w_k &\le& \left(\sum\limits_{k=0}^K w_k \right)\left(\sum\limits_{k=0}^{\infty}\left(1 - \frac{p}{2}\right)^{k}\right) = \frac{2W_K}{p}.
	\end{eqnarray*}
	Plugging these inequalities together with $1-\frac{p}{2}\ge \frac{1}{2}$ in \eqref{eq:V_k_bound_rand_tech_1_homo} we derive
	\begin{eqnarray}
		\sum\limits_{k=0}^K w_k\EE\left[V_k\right] &\le& \frac{16A(1-p)\gamma^2}{p}\sum\limits_{k=0}^Kw_kr_k + \frac{8B(1-p)\gamma^2}{p}\sum\limits_{k=0}^Kw_k\EE\left[\sigma_k^2\right]\notag\\
		&&\quad + \frac{2\left(D_1 + \frac{\zeta^2}{\gamma\mu}\right)(1-p)\gamma^2}{p}W_K.\label{eq:V_k_bound_rand_tech_5_homo}
	\end{eqnarray}
	It remains to estimate the second term in the right-hand side of this inequality. We notice that an analogous term appear in the proof of Lemma~\ref{lem:V_k_lemma}. In particular, in that proof inequality \eqref{eq:sigma_k_technical_bound} was shown via inequalities \eqref{eq:sigma_k+1_bound}, \eqref{eq:V_k_bound_rand_tech_3}, \eqref{eq:V_k_bound_rand_tech_4} and \eqref{eq:1+p/2_inequality} which hold in this case too. Therefore, we get that
	\begin{eqnarray}
	\sum\limits_{k=0}^Kw_k\EE\left[\sigma_k^2\right] &\overset{\eqref{eq:sigma_k_technical_bound}}{\le}& \frac{\EE\sigma_0^2(2+\rho)}{\rho} + \frac{4C}{\rho(1-\rho)}\sum\limits_{k=0}^Kw_kr_k + \frac{2G}{\rho(1-\rho)}\sum\limits_{k=0}^Kw_k\EE V_k +  \frac{D_2 W_K}{\rho},\notag
	\end{eqnarray}
	hence
	\begin{eqnarray}
		\sum\limits_{k=0}^K w_k\EE\left[V_k\right] &\overset{\eqref{eq:V_k_bound_rand_tech_5}}{\le}& \frac{16(1-p)\gamma^2\left(A + \frac{2BC}{\rho(1-\rho)}\right)}{p}\sum\limits_{k=0}^Kw_kr_k \notag\\
		&&\quad + \frac{8B(1-p)(2+\rho)\gamma^2\EE\sigma_0^2}{p\rho}+ \frac{16BG(1-p)\gamma^2}{p\rho(1-\rho)}\sum\limits_{k=0}^K w_k\EE\left[V_k\right]\notag\\
		&&\quad + \frac{2(1-p)\gamma^2}{p}\left(D_1 + \frac{\zeta^2}{\gamma\mu} + \frac{4BD_2}{\rho}\right)W_K.\notag
	\end{eqnarray}
	Our assumption on $\gamma$ imply
	\begin{eqnarray*}
		\frac{16(1-p)\gamma^2\left(A + \frac{2BC}{\rho(1-\rho)}\right)}{p} \le \frac{1}{8L},\quad \frac{16BG(1-p)\gamma^2}{p\rho(1-\rho)} \le \frac{1}{2}.
	\end{eqnarray*}
	Next, we introduce new notation as follows:
	\begin{equation*}
		H = \frac{16B(1-p)(2+\rho)\gamma^2}{p\rho},\quad D_3 = \frac{4(1-p)}{p}\left(D_1 + \frac{\zeta^2}{\gamma\mu} + \frac{4BD_2}{\rho}\right).
	\end{equation*}
	Putting all together we get
	\begin{equation*}
		\frac{1}{2}\sum\limits_{k=0}^Kw_k\EE\left[V_k\right] \le \frac{1}{8L}\sum\limits_{k=0}^K w_k r_k + \frac{H}{2}\EE\sigma_0^2 + \frac{D_3}{2}\gamma^2 W_K
	\end{equation*}
	which concludes the proof.
\end{proof}

This lemma and Theorem~\ref{thm:main_result} imply the following result.
\begin{corollary}\label{cor:rand_loop_homo}
	Let the assumptions of Lemma~\ref{lem:V_k_lemma_random_homo} are satisfied. Then Assumption~\ref{ass:key_assumption} holds and, in particular, if
	\begin{eqnarray*}
		\gamma &\le& \min\left\{\frac{1}{2(A'+CM)}, \frac{L}{F'+GM}, \frac{p}{8\mu}\right\},\quad M = \frac{4B'}{3\rho},\\
		\gamma &\le& \min\left\{\sqrt{\frac{p}{2F(1-p)}}, \sqrt{\frac{p\rho(1-\rho)}{32BG(1-p)}}, \sqrt{\frac{p}{128L(1-p)\left(A + \frac{2BC}{\rho(1-\rho)}\right)}}\right\},
	\end{eqnarray*}
	then for all $K\ge 0$ we have
	\begin{eqnarray}
		\EE\left[f(\overline{x}^K) - f(x^*)\right] &\le& \frac{2T^0 + 4LH\gamma\EE\sigma_0^2}{\gamma W_K} + 2\gamma\left(D_1' + MD_2 + 2L\gamma D_3\right), \label{eq:main_result_rand_loop_homo}
	\end{eqnarray}
	where $\overline{x}^K \eqdef \frac{1}{W_K}\sum_{k=0}^K w_k x^k$ and
	\begin{equation*}
		H = \frac{16B(1-p)(2+\rho)\gamma^2}{p\rho},\quad D_3 = \frac{4(1-p)}{p}\left(D_1 + \frac{\zeta^2}{\gamma\mu} + \frac{4BD_2}{\rho}\right).
	\end{equation*}
	 Moreover, if $\mu > 0$, then
	\begin{eqnarray}
		\EE\left[f(\overline{x}^K) - f(x^*)\right] &\le& \left(1 - \min\left\{\gamma\mu,\frac{\rho}{4}\right\}\right)^K\frac{2T^0 + 4LH\gamma\EE\sigma_0^2}{\gamma}\notag\\
		&&\quad + 2\gamma\left(D_1' + MD_2 + 2L\gamma D_3\right), \label{eq:main_result_1_rand_loop_homo}
	\end{eqnarray}
	and in the case when $\mu = 0$ we have
	\begin{eqnarray}
		\EE\left[f(\overline{x}^K) - f(x^*)\right] &\le& \frac{2T^0 + 4LH\gamma\EE\sigma_0^2}{\gamma K} + 2\gamma\left(D_1' + MD_2 + 2L\gamma D_3\right). \label{eq:main_result_2_rand_loop_homo}
	\end{eqnarray}
\end{corollary}

\section{Missing Parts from Section~\ref{sec:local_solver}}
Let us start with an useful Lemma that bounds the Bregman distance between the local iterate $x_i^k$  and the optimum $x^*$ by the Bregman distance between the virtual iterate $x^k$ and the optimum.

\begin{lemma}
Assume $f_i$ is $L$-smooth for all $i\in [n]$. Then
\begin{equation} \label{eq:poiouhnkj}
D_{f_i} (x^k_i, x^*) \leq
2D_{f_i} (x^k, x^*) +
L \| x_i^k-x^k\|^2\quad \forall i\in[n].
\end{equation}
\end{lemma}
\begin{proof}
Using corollaries of $L$-smoothness and Young's inequality, we derive
\begin{eqnarray*}
 D_{f_i} (x^k_i, x^*)
&\overset{\eqref{eq:L_smoothness_cor_1}}{\leq} &
D_{f_i} (x^k, x^*) + \langle \nabla f_i(x^k)- \nabla f_i(x^*) , x_i^k-x^k\rangle + \frac{L}{2}\| x_i^k-x^k \|^2
 \\
&\overset{\eqref{eq:fenchel_young}}{\leq}&
D_{f_i} (x^k, x^*) +
\frac{1}{2L} \|  \nabla f_i(x^k)- \nabla f_i(x^*) \|^2  + L\| x_i^k-x^k\|^2
\\
&\overset{\eqref{eq:L_smoothness}}{\leq}&
2D_{f_i} (x^k, x^*) +
L \| x_i^k-x^k\|^2.
\end{eqnarray*}
\end{proof}

\subsection{Proof of Lemma~\ref{lem:local_solver}}

 Let us bound $\frac{1}{n}\sum\limits_{i=1}^n \EE_k\left[\| g_i^k \|^2\right] $ first:
\begin{eqnarray*}
	\frac{1}{n}\sum\limits_{i=1}^n \EE_k\left[\| g_i^k \|^2\right]
	&=& 	\frac{1}{n}\sum\limits_{i=1}^n \EE_k\left[\| a_i^k - b_i^k \|^2\right] \\
	&=& 	\frac{1}{n}\sum\limits_{i=1}^n \EE_k\left[\| a_i^k - \nabla f_i(x^*)- (b_i^k - \nabla f_i(x^*))\|^2\right] \\
		&\leq& 	\frac{2}{n}\sum\limits_{i=1}^n \EE_k\left[\| a_i^k - \nabla f_i(x^*)\|^2 + \|b_i^k - \nabla f_i(x^*)\|^2\right] \\
						&\leq& 	\frac{2}{n}\sum\limits_{i=1}^n \left(2A_iD_{f_i} (x^k_i, x^*) + B_i\sigma_{i,k}^2 + D_{1,i} +\EE_k\left[\|b_i^k - \nabla f_i(x^*)\|^2\right]\right)
						\\
								&\stackrel{\eqref{eq:poiouhnkj}}{\leq}&
									\frac{2}{n}\sum\limits_{i=1}^n \left(4A_iD_{f_i} (x^k, x^*)
				+ 2A_iL \| x_i^k-x^k\|^2
		+ B_i\sigma_{i,k}^2 + D_{1,i} + \EE_k\left[\|b_i^k - \nabla f_i(x^*)\|^2\right]\right)
		 \\
		&\leq &
		8\max_{i} \{A_i\} (f(x^k)-f(x^*))
				+ 4\max_{i} \{A_i\}  L V_k +
		\frac{2}{n}\sum\limits_{i=1}^n  \left( B_i\sigma_{i,k}^2 + D_{1,i} + \EE_k\left[\|b_i^k - \nabla f_i(x^*)\|^2\right] \right).
\end{eqnarray*}

Taking the full expectation, we arrive at
\begin{equation}\label{eq:ndjkabhdavgjda}
\frac{1}{n}\sum\limits_{i=1}^n \EE\left[\| g_i^k \|^2\right]
\leq
		8\max_{i} \{A_i\} \EE(f(x^k)-f(x^*))
				+ 4\max_{i} \{A_i\}  L \EE V_k +
		\frac{2}{n}\sum\limits_{i=1}^n  \left( B_i\EE\sigma_{i,k}^2 + D_{1,i} +\EE \|b_i^k - \nabla f_i(x^*)\|^2 \right).
\end{equation}

Next, we have

\begin{eqnarray*}
	\EE_k\left[ \left \|  \frac{1}{n}\sum\limits_{i=1}^n  g_i^k \right\|^2\right]
	&=&
	\EE_k\left[\left\| \frac{1}{n}\sum\limits_{i=1}^n  a_i^k - b_i^k \right\|^2\right]
	\\
	&=&
	\EE_k\left[\left\| \frac{1}{n}\sum\limits_{i=1}^n  a_i^k - \nabla f_i(x^*)\right\|^2\right]
		\\
	&=&
	\Var\left[\frac{1}{n}\sum\limits_{i=1}^n  a_i^k - \nabla f_i(x^*)\right] + 
	\left\|\frac{1}{n}\sum\limits_{i=1}^n \nabla f_i(x_i^k)- \nabla f_i(x^*)\right\|^2
	\\
	&\leq &
		\Var\left[\frac{1}{n}\sum\limits_{i=1}^n  a_i^k - \nabla f_i(x^*)\right] + 
 \frac{1}{n}\sum\limits_{i=1}^n	\left\| \nabla f_i(x_i^k)- \nabla f_i(x^*)\right\|^2
 	\\
	&\leq &
		\Var\left[\frac{1}{n}\sum\limits_{i=1}^n  a_i^k - \nabla f_i(x^*)\right] + 
\frac{2L}{n}\sum\limits_{i=1}^nD_{f_i}(x_i^k,x^*)
\\
	&= &
	\frac{1}{n^2}	\sum\limits_{i=1}^n \Var\left[  a_i^k - \nabla f_i(x^*)\right] + \frac{2L}{n}\sum\limits_{i=1}^nD_{f_i}(x_i^k,x^*)
		\\
	&\leq &
	\frac{1}{n^2} \sum\limits_{i=1}^n   \EE_k\left[\left\| a_i^k - \nabla f_i(x^*)\right\|^2\right]
	+ 
\frac{2L}{n}\sum\limits_{i=1}^nD_{f_i}(x_i^k,x^*)
	\\
		&\le &
	\frac{1}{n^2} \sum\limits_{i=1}^n  \left(
	 2A_iD_{f_i} (x^k_i, x^*) + B_i\sigma_{i,k}^2 + D_{1,i}\right)
	 + 
\frac{2L}{n}\sum\limits_{i=1}^nD_{f_i}(x_i^k,x^*)
	 \\
	 	&\le &
\frac{1}{n^2} \sum\limits_{i=1}^n  \left(  	2\left(\max_{i} \{A_i\}+nL\right)D_{f_i} (x^k_i, x^*) +   B_i\sigma_{i,k}^2 + D_{1,i}\right)
		 \\
	 	&\stackrel{\eqref{eq:poiouhnkj}}{\leq} &
  	\left( \frac{4\max_{i} \{A_i\}}{n} + 2L\right) D_{f} (x^k, x^*) + \frac{1}{n^2} \sum\limits_{i=1}^n  \left( 2(\max_{i} \{A_i\} L+nL^2) \| x_i^k - x^*\|^2+   B_i\sigma_{i,k}^2 + D_{1,i}\right)
	 \\
	 	&=&
  	\left( \frac{4\max_{i} \{A_i\}}{n} + 2L\right)\left(f(x^k)- f(x^*)\right) +2\left(\frac{\max_{i}\{A_i\} L  }{n} + L^2\right) V_k + \frac{1}{n^2}\sum\limits_{i=1}^n  \left(   B_i\sigma_{i,k}^2 + D_{1,i}\right).
\end{eqnarray*}

 Further, we define
\begin{equation}\label{eq:omegak_eq_dnakjdjsgajvgdhasvgjlds}
\omega_k^2 \eqdef \frac{2}{n}\sum\limits_{i=1}^n   B_i\sigma_{i,k}^2
\end{equation}
and consequently, we get
\begin{eqnarray}
\EE\left[\omega_{k+1}^2  \right]
&=&
\frac{2}{n}\sum\limits_{i=1}^n B_i \EE\left[   \sigma_{i,k+1}^2  \right]
\nonumber
\\
&\leq&
 (1-\rho) \omega_{k}^2+ \frac{2}{n} \sum\limits_{i=1}^n B_iC_i D_{f_i}(x^k_i, x^*) +  \frac{2}{n}\sum\limits_{i=1}^n B_iD_{2,i}
\nonumber
\\
&\stackrel{\eqref{eq:poiouhnkj}}{\leq}&
 (1-\rho) \omega_{k}^2+ \frac{4}{n} \sum\limits_{i=1}^n B_iC_iD_{f_i}(x^k, x^*) + \frac{2}{n} \sum\limits_{i=1}^n B_iC_i L\|x^k_i - x^k\|^2 +\frac{2}{n} \sum\limits_{i=1}^n B_iD_{2,i}
 \nonumber
\\
&\leq &
 (1-\rho) \omega_{k}^2+ 4 \max_{i}\{B_iC_i \} D_{f}(x^k, x^*) + 2 \max_{i}\{B_iC_i \} LV_k + \frac2n\sum\limits_{i=1}^n B_iD_{2,i}.
\notag
\end{eqnarray}

We will provide a bound on $\EE\|b_i^k - \nabla f_i(x^*)\|^2 $ based on the choices of $b_i^k$:

\begin{itemize}
\item[Case I.] The choice $b_i^k = 0$ yields
$
 \EE\|b_i^k - \nabla f_i(x^*)\|^2 =  \|\nabla f_i(x^*)\|^2.
$
\item[Case II.] The choice $b_i^k = \nabla f_i(x^*)$ yields $ \EE\|b_i^k - \nabla f_i(x^*)\|^2 =  0$.
 Overall, for both Case I and II we have
\[
\EE\sigma^2_{k+1} \leq
(1-\rho)\EE\sigma^2_k
+ 4 \max_{i}\{B_iC_i \}  D_{f}(x^k, x^*) + 2 \max_{i}\{B_iC_i \}L V_k + \frac2n\sum\limits_{i=1}^n B_iD_{2,i}
\]
as desired, where $\sigma_k = \omega_k$.
\item[Case III.] The choice $b_i^k = h_i^k - \frac1n  \sum_{i=1}^n h_i^k $ yields
\[
\frac1n \sum_{i=1}^n  \|b_i^k - \nabla f_i(x^*)\|^2 = \frac1n \sum_{i=1}^n  \left\| h_i^k  - \frac1n \sum_{i=1}^n  h_i^k  - \nabla f_i(x^*)\right \|^2 \leq  \frac1n \sum_{i=1}^n  \| h_i^k   - \nabla f_i(x^*)\|^2
\]
where
\begin{eqnarray*}
\EE_k \left[ \|  h_i^{k+1}   - \nabla f_i(x^*)\|^2\right]
 &=&
 (1-\rho_i')\|  h_i^k   - \nabla f_i(x^*)\|^2  + \rho_i' \EE_k\|  l_i^k   - \nabla f_i(x^*)\|^2
 \\
  &\stackrel{\eqref{eq:bdef}}{\leq } &
   (1-\rho_i')\|h_i^k   - \nabla f_i(x^*)\|^2  +   2\rho_i' A'_iD_{f_i} (x^k_i, x^*)+ \rho_i' D_{3,i}.
\end{eqnarray*}
Next, set $\sigma^2_k \eqdef \omega^2_k +  \| h_i^k  - \nabla f_i(x^*)\|^2$ for this case.  Consequently, we have
\begin{eqnarray*}
	\EE_k\sigma^2_{k+1} &\leq& 
(1-\rho)\sigma^2_k
+ 4( \max_{i}\{B_iC_i \} + \max_i\{\rho_i' A'_i\})  D_{f}(x^k, x^*) +  2(\max_{i}\{B_iC_i \}+\max_{i}\{\rho_i'A_i' \}) LV_k \\
&&\quad+ \frac1n\sum\limits_{i=1}^n \left( 2 B_iD_{2,i} + \rho_i' D_{3,i} \right),
\end{eqnarray*}
where $\rho = \min_i\min\{\rho_i,\rho_i'\}$.
\end{itemize}

 It remains to plug everything back to~\eqref{eq:second_moment_bound}, \eqref{eq:second_moment_bound_2} and \eqref{eq:sigma_k+1_bound}.

\clearpage
\section{Special Cases: Technical details~\label{sec:special_appendix}}

%\subsection{Summary of parameters from Assumption~\ref{ass:key_assumption}}

\subsection{{\tt Local-SGD}} \label{sec:local_sgd_app}
We start with the analysis of {\tt Local-SGD} (see Algorithm~\ref{alg:local_sgd}) under different assumptions of stochastic gradients and data similarity.
\begin{algorithm}[h]
   \caption{{\tt Local-SGD}}\label{alg:local_sgd}
\begin{algorithmic}[1]
   \Require learning rate $\gamma>0$, initial vector $x^0 \in \R^d$, communication period $\tau \ge 1$
	\For{$k=0,1,\dotsc$}
       	\For{$i=1,\dotsc,n$ in parallel}
            \State Sample $g^{k}_i = \nabla f_{\xi_i^k}(x_i^k)$ independently from other nodes
            \If{$k+1 \mod \tau = 0$}
            \State $x_i^{k+1} = x^{k+1} = \frac{1}{n}\sum\limits_{i=1}^n\left(x_i^k - \gamma g_i^k\right)$ \Comment{averaging}
            \Else
            \State $x_i^{k+1} = x_i^k - \gamma g_i^k$ \Comment{local update}
            \EndIf
        \EndFor
   \EndFor
\end{algorithmic}
\end{algorithm}

\subsubsection{Uniformly Bounded Variance}\label{sec:sgd_bounded_var}

In this section we assume that $f_i$ has a form of expectation (see \eqref{eq:f_i_expectation}) and stochastic gradients $\nabla f_{\xi_i}(x)$ satisfy
\begin{equation}
	\EE_{\xi_i}\left[\|\nabla f_{\xi_i}(x) - \nabla f_i(x)\|^2\right] \le D_{1,i},\quad \forall\; x\in\R^d,\;\forall\; i\in [n].\label{eq:bounded_variance}
\end{equation}
We also introduce the average variance $\sigma^2$ and the parameter of heterogeneity at the solution $\zeta_*^2$ in the following way:
\begin{equation*}
	\sigma^2 = \frac{1}{n}\sum\limits_{i=1}^nD_{1,i},\quad \zeta_*^2 = \frac{1}{n}\sum\limits_{i=1}^n\|\nabla f_i(x^*)\|^2.
\end{equation*}

\begin{lemma}\label{lem:local_sgd_interesting_labels}
Assume that functions $f_i$ are convex and $L$-smooth for all $i\in[n]$. Then
\begin{equation}\label{eq:dnaossniadd}
\frac{1}{n}\sum\limits_{i=1}^n\|\nabla f_i(x_i^k)\|^2\leq
		6L\left(f(x^k) - f(x^*)\right) + 3L^2 V_k + 3\zeta_*^2
\end{equation}
and
\begin{equation}\label{eq:vdgasvgda}
	 \left\|\frac{1}{n}\sum\limits_{i=1}^n\nabla f_i(x_i^k)\right\|^2
		\leq  4L\left(f(x^k)-f(x^*)\right) + 2L^2 V_k .
\end{equation}
\end{lemma}
\begin{proof}
First, to show~\eqref{eq:dnaossniadd} we shall have
	\begin{eqnarray*}
 \frac{1}{n}\sum\limits_{i=1}^n\|\nabla f_i(x_i^k)\|^2
		&\overset{\eqref{eq:a_b_norm_squared}}{\le}&
		\frac{3}{n}\sum\limits_{i=1}^n\|\nabla f_i(x_i^k) - \nabla f_i(x^k)\|^2 + \frac{3}{n}\sum\limits_{i=1}^n\|\nabla f_i(x^k) - \nabla f_i(x^*)\|^2 \\
		&& \qquad
		+ \frac{3}{n}\sum\limits_{i=1}^n\|\nabla f_i(x^*)\|^2
		\\
		&\overset{\eqref{eq:L_smoothness},\eqref{eq:L_smoothness_cor}}{\le}& \frac{3L^2}{n}\sum\limits_{i=1}^n\|x_i^k - x^k\|^2+ \frac{6L}{n}\sum\limits_{i=1}^nD_{f_i}(x^k,x^*) + 3\zeta_*^2
		\\
		&=&
		6L\left(f(x^k) - f(x^*)\right) + 3L^2 V_k + 3\zeta_*^2.
	\end{eqnarray*}
Next, to establish~\eqref{eq:vdgasvgda}, we have
	\begin{eqnarray*}
	 \left\|\frac{1}{n}\sum\limits_{i=1}^n\nabla f_i(x_i^k)\right\|^2
		&=&  \left\|\frac{1}{n}\sum\limits_{i=1}^n\left(\nabla f_i(x_i^k)-\nabla f_i(x^*)\right)\right\|^2
		\\
		&\overset{\eqref{eq:a_b_norm_squared}}{\le}&  \frac{2}{n}\sum\limits_{i=1}^n\|\nabla f_i(x_i^k) - \nabla f(x^k)\|^2 + \frac{2}{n}\sum\limits_{i=1}^n\|\nabla f_i(x^k) - \nabla f(x^*)\|^2
		\\
		&\overset{\eqref{eq:L_smoothness},\eqref{eq:L_smoothness_cor}}{\le}& \frac{2L^2}{n}\sum\limits_{i=1}^n\|x_i^k - x^k\|^2 + \frac{4L}{n}\sum\limits_{i=1}^nD_{f_i}(x^k,x^*)\\
		&=& 4L\left(f(x^k)-f(x^*)\right) + 2L^2 V_k .
	\end{eqnarray*}
\end{proof}

\begin{lemma}\label{lem:local_sgd_second_moment}
	Let $f_i$ be convex and $L$-smooth for all $i\in[n]$. Then for all $k\ge 0$
	\begin{eqnarray}
		\frac{1}{n}\sum\limits_{i=1}^n \EE\left[\|g_i^k\|^2\mid x^k\right] &\le& 6L\left(f(x^k) - f(x^*)\right) + 3L^2 V_k + \sigma^2 + 3\zeta_*^2,\label{eq:second_moment_local_sgd}\\
		\frac{1}{n}\sum\limits_{i=1}^n \EE\left[\|g_i^k-\bar{g}_i^k\|^2\mid x^k\right] &\le& \sigma^2,\label{eq:variance_local_sgd}\\
		\EE\left[\left\|\frac{1}{n}\sum\limits_{i=1}^ng_i^k\right\|^2\mid x^k\right] &\le& 4L\left(f(x^k) - f(x^*)\right) + 2L^2 V_k + \frac{\sigma^2}{n},\label{eq:second_moment_local_sgd_2}
	\end{eqnarray}
	where $\EE[\cdot\mid x^k]\eqdef \EE[\cdot\mid x_1^k,\ldots,x_n^k]$.
\end{lemma}
\begin{proof}
	First of all, we notice that $\bar{g}_i^k = \EE\left[g_i^k\mid x^k\right] = \nabla f_i(x_i^k)$. Using this we get
	\begin{eqnarray*}
		\frac{1}{n}\sum\limits_{i=1}^n \EE\left[\|g_i^k-\bar g_i^k\|^2\mid x_i^k\right] &=& \frac{1}{n}\sum\limits_{i=1}^n \EE_{\xi_i^k}\|\nabla f_{\xi_i^k}(x_i^k) - \nabla f_i(x_i^k)\|^2 \overset{\eqref{eq:bounded_variance}}{\le} \frac{1}{n}\sum\limits_{i=1}^n D_{1,i},\\
		\frac{1}{n}\sum\limits_{i=1}^n \EE\left[\|g_i^k\|^2\mid x_i^k\right] &\overset{\eqref{eq:variance_decomposition}}{=}& \frac{1}{n}\sum\limits_{i=1}^n \EE_{\xi_i^k}\|\nabla f_{\xi_i^k}(x_i^k) - \nabla f_i(x_i^k)\|^2 + \frac{1}{n}\sum\limits_{i=1}^n\|\nabla f_i(x_i^k)\|^2\\
		&\overset{\eqref{eq:bounded_variance},\eqref{eq:dnaossniadd}}{\le}&
		 6L\left(f(x^k) - f(x^*)\right) + 3L^2 V_k + \frac{1}{n}\sum\limits_{i=1}^n\left(D_{1,i} + 3\|\nabla f_i(x^*)\|^2\right).
	\end{eqnarray*}
	Finally, using independence of $g_1^k,g_2^k,\ldots,g_n^k$ we obtain
	\begin{eqnarray*}
		\EE\left[\left\|\frac{1}{n}\sum\limits_{i=1}^ng_i^k\right\|^2\mid x^k\right] &\overset{\eqref{eq:variance_decomposition}}{\le}& \EE\left[\left\|\frac{1}{n}\sum\limits_{i=1}^n\left(g_i^k - \nabla f_i(x_i^k)\right)\right\|^2\mid x^k\right] + \left\|\frac{1}{n}\sum\limits_{i=1}^n\nabla f_i(x_i^k)\right\|^2\\
		&=& \frac{1}{n^2}\sum\limits_{i=1}^n\EE\left[\|g_i^k - \nabla f_i(x_i^k)\|^2\mid x_i^k\right] +  \left\|\frac{1}{n}\sum\limits_{i=1}^n\nabla f_i(x_i^k)\right\|^2\\
		&\overset{\eqref{eq:bounded_variance},\eqref{eq:vdgasvgda}}{\le}&
		4L\left(f(x^k)-f(x^*)\right) + 2L^2 V_k + \frac{1}{n^2}\sum\limits_{i=1}^nD_{1,i}.
	\end{eqnarray*}
\end{proof}

\subsubsection*{Heterogeneous Data}
Applying Corollary~\ref{cor:const_loop} and Lemmas~\ref{lem:local_sgd_interesting_labels}~and~\ref{lem:local_sgd_second_moment} we get the following result.
\begin{theorem}\label{thm:local_sgd}
	Assume that $f_i(x)$ is $\mu$-strongly convex and $L$-smooth for every $i\in[n]$. Then {\tt Local-SGD} satisfies Assumption~\ref{ass:hetero_second_moment} with
	\begin{gather*}
		\tA = 3L,\quad \hA= 0,\quad \tB = \hB = 0,\quad \tF = 3L^2,\quad \hF = 0, \quad \tD_1 = 3\zeta_*^2,\quad \hD = \sigma^2,\\
		A' = 2L,\quad B' = 0,\quad F' = 2L^2, \quad D_1' = \frac{\sigma^2}{n},\quad \sigma_k^2 \equiv 0,\quad \rho = 1,\quad C = 0,\quad G = 0,\quad D_2 = 0,\\
		H = 0,\quad D_3 = 2e(\tau-1)\left(3(\tau-1)\zeta_*^2 + \sigma^2\right)
	\end{gather*}
	with $\gamma$ satisfying
	\begin{eqnarray*}
		\gamma &\le& \min\left\{\frac{1}{4L}, \frac{1}{4\sqrt{6e}(\tau-1)L}\right\}.
	\end{eqnarray*}
	and for all $K \ge 0$
	\begin{eqnarray}
		\EE\left[f(\overline{x}^K) - f(x^*)\right] &\le& \frac{2\|x^0-x^*\|^2}{\gamma W_K} + 2\gamma\left(\nicefrac{\sigma^2}{n} + 4eL(\tau-1)\gamma \left(\sigma^2 + 3(\tau-1)\zeta_*^2\right)\right). \notag
	\end{eqnarray}
	In particular, if $\mu > 0$ then
	\begin{eqnarray}
		\EE\left[f(\overline{x}^K) - f(x^*)\right] &\le& \left(1 - \gamma\mu\right)^K\frac{2\|x^0-x^*\|^2}{\gamma} + 2\gamma\left(\nicefrac{\sigma^2}{n} + 4eL(\tau-1)\gamma \left(\sigma^2 + 3(\tau-1)\zeta_*^2\right)\right) \label{eq:local_sgd_str_cvx}
	\end{eqnarray}
	and when $\mu = 0$ we have
	\begin{eqnarray}
		\EE\left[f(\overline{x}^K) - f(x^*)\right] &\le& \frac{2\|x^0-x^*\|^2}{\gamma K} + 2\gamma\left(\nicefrac{\sigma^2}{n} + 4eL(\tau-1)\gamma \left(\sigma^2 + 3(\tau-1)\zeta_*^2\right)\right). \label{eq:local_sgd_cvx}
	\end{eqnarray}
\end{theorem}

The theorem above together with Lemma~\ref{lem:lemma2_stich} implies the following result.
\begin{corollary}\label{cor:local_sgd_str_cvx}
	Let assumptions of Theorem~\ref{thm:local_sgd} hold with $\mu > 0$. Then for
	\begin{equation*}
		\gamma = \min\left\{\frac{1}{4L}, \frac{1}{4\sqrt{6e}(\tau-1)L},  \frac{\ln\left(\max\left\{2,\min\left\{\nicefrac{\|x^0 - x^*\|^2n\mu^2K^2}{\sigma^2},\nicefrac{\|x^0 - x^*\|^2\mu^3K^3}{4eL(\tau-1)\left(\sigma^2 + 3(\tau-1)\zeta_*^2\right)}\right\}\right\}\right)}{\mu K}\right\}
	\end{equation*}
	for all $K$ such that 
	\begin{eqnarray*}
		\text{either}&&\frac{\ln\left(\max\left\{2,\min\left\{\nicefrac{\|x^0 - x^*\|^2n\mu^2K^2}{\sigma^2},\nicefrac{\|x^0 - x^*\|^2\mu^3K^3}{4eL(\tau-1)\left(\sigma^2 + 3(\tau-1)\zeta_*^2\right)}\right\}\right\}\right)}{ K} \le 1\\
		\text{or}&&\min\left\{\frac{1}{4L}, \frac{1}{4\sqrt{6e}(\tau-1)L}\right\} \le \frac{\ln\left(\max\left\{2,\min\left\{\nicefrac{\|x^0 - x^*\|^2n\mu^2K^2}{\sigma^2},\nicefrac{\|x^0 - x^*\|^2\mu^3K^3}{4eL(\tau-1)\left(\sigma^2 + 3(\tau-1)\zeta_*^2\right)}\right\}\right\}\right)}{\mu K}
	\end{eqnarray*}
	we have that
	\begin{equation}
		\EE\left[f(\overline{x}^K)-f(x^*)\right] = \widetilde\cO\left(\tau L\|x^0 - x^*\|^2\exp\left(- \frac{\mu}{\tau L} K\right) + \frac{\sigma^2}{n\mu K} + \frac{L(\tau-1)\left(\sigma^2+(\tau-1)\zeta_*^2\right)}{\mu^2K^2}\right).\label{eq:local_sgd_str_cvx_1}
	\end{equation}
	That is, to achieve $\EE\left[f(\overline{x}^K)-f(x^*)\right] \le \varepsilon$ in this case {\tt Local-SGD} requires
	\begin{equation*}
		\widetilde\cO\left(\frac{\tau L}{\mu} + \frac{\sigma^2}{n\mu\varepsilon} + \sqrt{\frac{L(\tau-1)\left(\sigma^2+(\tau-1)\zeta_*^2\right)}{\mu^2\varepsilon}}\right)
	\end{equation*}
	iterations/oracle calls per node and $\tau$ times less communication rounds.
\end{corollary}
Now we consider some special cases. First of all, if $D_{1,i} = 0$ for all $i\in [n]$, i.e.\ $g_i^k = \nabla f_i(x_i^k)$ almost surely, then our result implies that for {\tt Local-SGD} it is enough to perform
\begin{equation*}
	\widetilde\cO\left(\frac{\tau L}{\mu} + \sqrt{\frac{L(\tau-1)^2\zeta_*^2}{\mu^2\varepsilon}}\right)
\end{equation*}
iterations in order to achieve $\EE\left[f(\overline{x}^K)-f(x^*)\right] \le \varepsilon$. It is clear that for this scenario the optimal choice for $\tau$ is $\tau = 1$ which recovers\footnote{We notice that for this particular case our analysis doesn't give extra logarithmical factors if we apply \eqref{eq:local_sgd_str_cvx} instead of \eqref{eq:local_sgd_str_cvx_1}.} the rate of Gradient Descent.

Secondly, if $\tau = 1$ then we recover the rate of parallel {\tt SGD}:
\begin{eqnarray*}
	\widetilde\cO\left(\frac{L}{\mu} + \frac{\sigma^2}{n\mu\varepsilon}\right)&& \text{communication rounds/oracle calls per node}
\end{eqnarray*}
in order to achieve $\EE\left[f(\overline{x}^K)-f(x^*)\right] \le \varepsilon$.

Finally, our result gives a negative answer to the following question: is {\tt Local-SGD} always worse then Parallel Minibatch {\tt SGD} ({\tt PMSGD}) for heterogeneous data? To achieve $\EE\left[f(\overline{x}^K)-f(x^*)\right] \le \varepsilon$ {\tt Local-SGD} requires
\begin{equation*}
	\widetilde\cO\left(\frac{\tau L}{\mu}+ \frac{\sigma^2}{n\mu\varepsilon}+ \sqrt{\frac{L(\tau-1)\left(\sigma^2+(\tau-1)\zeta_*^2\right)}{\mu^2\varepsilon}}\right)\quad \text{oracle calls per node.}
\end{equation*}
It means that if $\frac{\sigma^2}{n\sqrt{L(\tau-1)\left(\sigma^2+(\tau-1)\zeta_*^2\right)\varepsilon}} \ge 1$ for given $\tau > 1$ and $\varepsilon$ and $\sigma^2$ are such that the first term in the complexity bound is dominated by other terms, then the second term corresponding to the complexity of {\tt PMSGD} dominates the third term. Informally speaking, if the variance is large or $\varepsilon$ is small then {\tt Local-SGD} with $\tau > 1$ has the same complexity bounds as {\tt PMSGD}.

Combining Theorem~\ref{thm:local_sgd} and Lemma~\ref{lem:lemma_technical_cvx} we derive the following result for the convergence of {\tt Local-SGD} in the case when $\mu = 0$.
\begin{corollary}
	\label{cor:local_sgd_cvx}
	Let assumptions of Theorem~\ref{thm:local_sgd} hold with $\mu = 0$. Then for
	\begin{equation*}
		\gamma = \min\left\{\frac{1}{4L}, \frac{1}{4\sqrt{6e}(\tau-1)L}, \sqrt{\frac{nR_0^2}{\sigma^2 K}}, \sqrt[3]{\frac{R_0^2}{4eL(\tau-1)\left(\sigma^2+(\tau-1)\zeta_*^2\right)K}}\right\},
	\end{equation*}
	where $R_0 = \|x^0 - x^*\|$, we have that
	\begin{equation}
		\EE\left[f(\overline{x}^K)-f(x^*)\right] = \cO\left(\frac{\tau LR_0^2}{K} + \sqrt{\frac{R_0^2\sigma^2}{nK}} + \frac{\sqrt[3]{LR_0^4(\tau-1)\left(\sigma^2+(\tau-1)\zeta_*^2\right)}}{K^{\nicefrac{2}{3}}} \right).\label{eq:local_sgd_cvx_1}
	\end{equation}
	That is, to achieve $\EE\left[f(\overline{x}^K)-f(x^*)\right] \le \varepsilon$ in this case {\tt Local-SGD} requires
	\begin{equation*}
		\cO\left(\frac{\tau LR_0^2}{\varepsilon} + \frac{R_0^2\sigma^2}{n\varepsilon^2} + \frac{R_0^2\sqrt{L(\tau-1)\left(\sigma^2+(\tau-1)\zeta_*^2\right)}}{\varepsilon^{\nicefrac{3}{2}}}\right)
	\end{equation*}
	iterations/oracle calls per node and $\tau$ times less communication rounds.
\end{corollary}

\subsubsection*{Homogeneous Data}
In this case we modify the approach a little bit and apply the following result.
\begin{lemma}[Lemma~1 from \cite{khaled2020tighter}]
	Under the homogeneous data assumption for {\tt Local-SGD} we have
	\begin{equation}
		\EE\left[V_k\right] \le (\tau-1)\gamma^2\sigma^2 \label{eq:homocase_V_k_ubv}
	\end{equation}
	for all $k \ge 0$.
\end{lemma}
Using this we derive the following inequality for the weighted sum of $V_k$:
\begin{equation*}
	2L\sum\limits_{k=0}^Kw_k\EE[V_k] \le 2L(\tau-1)\gamma^2\sigma^2\sum\limits_{k=0}^Kw_k = 2L(\tau-1)\gamma^2\sigma^2 W_K.
\end{equation*}
Together with Lemmas~\ref{lem:local_sgd_interesting_labels}~and~\ref{lem:local_sgd_second_moment} and Theorem~\ref{thm:main_result} it gives the following result.
\begin{theorem}\label{thm:local_sgd_homo}
	Assume that $f(x)$ is $\mu$-strongly convex and $L$-smooth and $f_1 = \ldots = f_n = f$. Then {\tt Local-SGD} satisfies Assumption~\ref{ass:key_assumption} with
	\begin{gather*}
		A = 3L,\quad B = 0,\quad F = 3L^2, \quad D_1 = \sigma^2,\quad A' = 2L,\quad B' = 0,\quad F' = 2L^2, \quad D_1' = \frac{\sigma^2}{n},\\
		\sigma_k^2 \equiv 0,\quad \rho = 1,\quad C = 0,\quad G = 0,\quad D_2 = 0,\quad H = 0,\quad D_3 = (\tau-1)\sigma^2
	\end{gather*}
	with $\gamma$ satisfying
	\begin{eqnarray*}
		\gamma &\le& \frac{1}{4L}.
	\end{eqnarray*}
	and for all $K \ge 0$
	\begin{eqnarray}
		\EE\left[f(\overline{x}^K) - f(x^*)\right] &\le& \frac{2\|x^0-x^*\|^2}{\gamma W_K} + 2\gamma\left(\nicefrac{\sigma^2}{n} + 2L(\tau-1)\gamma \sigma^2\right). \notag
	\end{eqnarray}
	In particular, if $\mu > 0$ then
	\begin{eqnarray}
		\EE\left[f(\overline{x}^K) - f(x^*)\right] &\le& \left(1 - \gamma\mu\right)^K\frac{2\|x^0-x^*\|^2}{\gamma} + 2\gamma\left(\nicefrac{\sigma^2}{n} + 2L(\tau-1)\gamma \sigma^2\right) \label{eq:local_sgd_str_cvx_homo}
	\end{eqnarray}
	and when $\mu = 0$ we have
	\begin{eqnarray}
		\EE\left[f(\overline{x}^K) - f(x^*)\right] &\le& \frac{2\|x^0-x^*\|^2}{\gamma K} + 2\gamma\left(\nicefrac{\sigma^2}{n} + 2L(\tau-1)\gamma \sigma^2\right). \label{eq:local_sgd_cvx_homo}
	\end{eqnarray}
\end{theorem}

The theorem above together with Lemma~\ref{lem:lemma2_stich} implies the following result.
\begin{corollary}\label{cor:local_sgd_str_cvx_homo}
	Let assumptions of Theorem~\ref{thm:local_sgd_homo} hold with $\mu > 0$. Then for
	\begin{equation*}
		\gamma = \min\left\{\frac{1}{4L}, \frac{\ln\left(\max\left\{2,\min\left\{\nicefrac{\|x^0 - x^*\|^2n\mu^2K^2}{\sigma^2},\nicefrac{\|x^0 - x^*\|^2\mu^3K^3}{2L(\tau-1)\sigma^2}\right\}\right\}\right)}{\mu K}\right\}
	\end{equation*}
	for all $K$ such that 
	\begin{eqnarray*}
		\text{either} && \frac{\ln\left(\max\left\{2,\min\left\{\nicefrac{\|x^0 - x^*\|^2n\mu^2K^2}{\sigma^2},\nicefrac{\|x^0 - x^*\|^2\mu^3K^3}{2L(\tau-1)\sigma^2}\right\}\right\}\right)}{ K} \le 1\\
		\text{or} && \frac{1}{4L} \le \frac{\ln\left(\max\left\{2,\min\left\{\nicefrac{\|x^0 - x^*\|^2n\mu^2K^2}{\sigma^2},\nicefrac{\|x^0 - x^*\|^2\mu^3K^3}{2L(\tau-1)\sigma^2}\right\}\right\}\right)}{\mu K}
	\end{eqnarray*}
	we have that
	\begin{equation}
		\EE\left[f(\overline{x}^K)-f(x^*)\right] = \widetilde\cO\left(L\|x^0 - x^*\|^2\exp\left(- \frac{\mu}{L} K\right) + \frac{\sigma^2}{n\mu K} + \frac{L(\tau-1)\sigma^2}{\mu^2K^2}\right).\label{eq:local_sgd_str_cvx_1_homo}
	\end{equation}
	That is, to achieve $\EE\left[f(\overline{x}^K)-f(x^*)\right] \le \varepsilon$ in this case {\tt Local-SGD} requires
	\begin{equation*}
		\widetilde\cO\left(\frac{L}{\mu}\ln\left(\frac{L\|x^0 - x^*\|^2}{\varepsilon}\right) + \frac{\sigma^2}{n\mu\varepsilon} + \sqrt{\frac{L(\tau-1)\sigma^2}{\mu^2\varepsilon}}\right)
	\end{equation*}
	iterations/oracle calls per node and $\tau$ times less communication rounds.
\end{corollary}
It means that if $\frac{\sigma^2}{n^2L\varepsilon} \ge 1$, $\tau \le 1 + \frac{\sigma^2}{n^2L\varepsilon}$ and $\varepsilon$ and $\sigma^2$ are such that the first term in the complexity bound is dominated by other terms, then the second term corresponding to the complexity of {\tt PMSGD} dominates the third term. Informally speaking, if the variance is large or $\varepsilon$ is small then {\tt Local-SGD} with $\tau > 1$ has the same complexity bounds as {\tt PMSGD}.

Combining Theorem~\ref{thm:local_sgd_homo} and Lemma~\ref{lem:lemma_technical_cvx} we derive the following result for the convergence of {\tt Local-SGD} in the case when $\mu = 0$.
\begin{corollary}
	\label{cor:local_sgd_cvx_homo}
	Let assumptions of Theorem~\ref{thm:local_sgd_homo} hold with $\mu = 0$. Then for
	\begin{equation*}
		\gamma = \min\left\{\frac{1}{4L}, \sqrt{\frac{nR_0^2}{\sigma^2 K}}, \sqrt[3]{\frac{R_0^2}{2L(\tau-1)\sigma^2 K}}\right\},
	\end{equation*}
	where $R_0 = \|x^0 - x^*\|$, we have that
	\begin{equation}
		\EE\left[f(\overline{x}^K)-f(x^*)\right] = \cO\left(\frac{LR_0^2}{K} + \sqrt{\frac{R_0^2\sigma^2}{nK}} + \frac{\sqrt[3]{LR_0^4(\tau-1)\sigma^2}}{K^{\nicefrac{2}{3}}} \right).\label{eq:local_sgd_cvx_1_homo}
	\end{equation}
	That is, to achieve $\EE\left[f(\overline{x}^K)-f(x^*)\right] \le \varepsilon$ in this case {\tt Local-SGD} requires
	\begin{equation*}
		\cO\left(\frac{LR_0^2}{\varepsilon} + \frac{R_0^2\sigma^2}{n\varepsilon^2} + \frac{R_0^2\sqrt{L(\tau-1)\sigma^2}}{\varepsilon^{\nicefrac{3}{2}}}\right)
	\end{equation*}
	iterations/oracle calls per node and $\tau$ times less communication rounds.
\end{corollary}

\subsubsection*{$\zeta$-Heterogeneous Data}
In this setup we also use an external result to bound $\EE[V_k]$.
\begin{lemma}[Lemma~8 from \cite{woodworth2020minibatch}]
	If $f_1,f_2,\ldots,f_n$ are $\zeta$-heterogeneous then for {\tt Local-SGD} we have
	\begin{equation}
		\EE\left[V_k\right] \le 3\tau\gamma^2\sigma^2 + 6\tau^2\gamma^2\zeta^2 \label{eq:zeta_heterocase_V_k_ubv}
	\end{equation}
	for all $k \ge 0$.
\end{lemma}
Using this we derive the following inequality for the weighted sum of $V_k$:
\begin{equation*}
	2L\sum\limits_{k=0}^Kw_k\EE[V_k] \le 6\tau L\gamma^2\left(\sigma^2 + 2\tau\zeta^2\right)\sum\limits_{k=0}^Kw_k = 6\tau L\gamma^2\left(\sigma^2 + 2\tau\zeta^2\right) W_K.
\end{equation*}
Together with Lemmas~\ref{lem:local_sgd_interesting_labels}~and~\ref{lem:local_sgd_second_moment} and Theorem~\ref{thm:main_result} it gives the following result.
\begin{theorem}\label{thm:local_sgd_zeta_hetero}
	Assume that $f_1,\ldots,f_n$ are $\zeta$-heterogeneous, $\mu$-strongly convex and $L$-smooth functions. Then {\tt Local-SGD} satisfies Assumption~\ref{ass:key_assumption} with
	\begin{gather*}
		A = 3L,\quad B = 0,\quad F = 3L^2, \quad D_1 = \sigma^2 + 3\zeta_*^2,\quad A' = 2L,\quad B' = 0,\quad F' = 2L^2, \quad D_1' = \frac{\sigma^2}{n},\\
		\sigma_k^2 \equiv 0,\quad \rho = 1,\quad C = 0,\quad G = 0,\quad D_2 = 0,\quad H = 0,\quad D_3 = 3\tau\left(\sigma^2 + 2\tau\zeta^2\right)
	\end{gather*}
	with $\gamma$ satisfying
	\begin{eqnarray*}
		\gamma &\le& \frac{1}{4L}.
	\end{eqnarray*}
	and for all $K \ge 0$
	\begin{eqnarray}
		\EE\left[f(\overline{x}^K) - f(x^*)\right] &\le& \frac{2\|x^0-x^*\|^2}{\gamma W_K} + 2\gamma\left(\nicefrac{\sigma^2}{n} + 6L\tau\gamma \left(\sigma^2 + 2\tau\zeta^2\right)\right). \notag
	\end{eqnarray}
	In particular, if $\mu > 0$ then
	\begin{eqnarray}
		\EE\left[f(\overline{x}^K) - f(x^*)\right] &\le& \left(1 - \gamma\mu\right)^K\frac{2\|x^0-x^*\|^2}{\gamma} + 2\gamma\left(\nicefrac{\sigma^2}{n} + 6L\tau\gamma \left(\sigma^2 + 2\tau\zeta^2\right)\right) \label{eq:local_sgd_str_cvx_zeta_hetero}
	\end{eqnarray}
	and when $\mu = 0$ we have
	\begin{eqnarray}
		\EE\left[f(\overline{x}^K) - f(x^*)\right] &\le& \frac{2\|x^0-x^*\|^2}{\gamma K} + 2\gamma\left(\nicefrac{\sigma^2}{n} + 6L\tau\gamma \left(\sigma^2 + 2\tau\zeta^2\right)\right). \label{eq:local_sgd_cvx_zeta_hetero}
	\end{eqnarray}
\end{theorem}
The theorem above together with Lemma~\ref{lem:lemma2_stich} implies the following result.
\begin{corollary}\label{cor:local_sgd_str_cvx_zeta_hetero}
	Let assumptions of Theorem~\ref{thm:local_sgd_zeta_hetero} hold with $\mu > 0$. Then for
	\begin{equation*}
		\gamma = \min\left\{\frac{1}{4L}, \frac{\ln\left(\max\left\{2,\min\left\{\nicefrac{\|x^0 - x^*\|^2n\mu^2K^2}{\sigma^2},\nicefrac{\|x^0 - x^*\|^2\mu^3K^3}{6L\tau(\sigma^2+2\tau\zeta^2)}\right\}\right\}\right)}{\mu K}\right\}
	\end{equation*}
	for all $K$ such that 
	\begin{eqnarray*}
		\text{either} && \frac{\ln\left(\max\left\{2,\min\left\{\nicefrac{\|x^0 - x^*\|^2n\mu^2K^2}{\sigma^2},\nicefrac{\|x^0 - x^*\|^2\mu^3K^3}{6L\tau(\sigma^2+2\tau\zeta^2)}\right\}\right\}\right)}{K} \le 1\\
		\text{or} && \frac{1}{4L} \le \frac{\ln\left(\max\left\{2,\min\left\{\nicefrac{\|x^0 - x^*\|^2n\mu^2K^2}{\sigma^2},\nicefrac{\|x^0 - x^*\|^2\mu^3K^3}{6L\tau(\sigma^2+2\tau\zeta^2)}\right\}\right\}\right)}{\mu K}
	\end{eqnarray*}
	we have that
	\begin{equation}
		\EE\left[f(\overline{x}^K)-f(x^*)\right] = \widetilde\cO\left(L\|x^0 - x^*\|^2\exp\left(- \frac{\mu}{L} K\right) + \frac{\sigma^2}{n\mu K} + \frac{L\tau(\sigma^2+\tau\zeta^2)}{\mu^2K^2}\right).\label{eq:local_sgd_str_cvx_1_zeta_hetero}
	\end{equation}
	That is, to achieve $\EE\left[f(\overline{x}^K)-f(x^*)\right] \le \varepsilon$ in this case {\tt Local-SGD} requires
	\begin{equation*}
		\widetilde\cO\left(\frac{L}{\mu}\ln\left(\frac{L\|x^0 - x^*\|^2}{\varepsilon}\right) + \frac{\sigma^2}{n\mu\varepsilon} + \sqrt{\frac{L\tau (\sigma^2+\tau\zeta^2)}{\mu^2\varepsilon}}\right)
	\end{equation*}
	iterations/oracle calls per node and $\tau$ times less communication rounds.
\end{corollary}

Combining Theorem~\ref{thm:local_sgd_zeta_hetero} and Lemma~\ref{lem:lemma_technical_cvx} we derive the following result for the convergence of {\tt Local-SGD} in the case when $\mu = 0$.
\begin{corollary}
	\label{cor:local_sgd_cvx_zeta_hetero}
	Let assumptions of Theorem~\ref{thm:local_sgd_zeta_hetero} hold with $\mu = 0$. Then for
	\begin{equation*}
		\gamma = \min\left\{\frac{1}{4L}, \sqrt{\frac{nR_0^2}{\sigma^2 K}}, \sqrt[3]{\frac{R_0^2}{6L\tau(\sigma^2+2\tau\zeta^2) K}}\right\},
	\end{equation*}
	where $R_0 = \|x^0 - x^*\|$, we have that
	\begin{equation}
		\EE\left[f(\overline{x}^K)-f(x^*)\right] = \cO\left(\frac{LR_0^2}{K} + \sqrt{\frac{R_0^2\sigma^2}{nK}} + \frac{\sqrt[3]{LR_0^4\tau(\sigma^2+\tau\zeta^2)}}{K^{\nicefrac{2}{3}}} \right).\label{eq:local_sgd_cvx_1_zeta_hetero}
	\end{equation}
	That is, to achieve $\EE\left[f(\overline{x}^K)-f(x^*)\right] \le \varepsilon$ in this case {\tt Local-SGD} requires
	\begin{equation*}
		\cO\left(\frac{LR_0^2}{\varepsilon} + \frac{R_0^2\sigma^2}{n\varepsilon^2} + \frac{R_0^2\sqrt{L\tau(\sigma^2+\tau\zeta^2)}}{\varepsilon^{\nicefrac{3}{2}}}\right)
	\end{equation*}
	iterations/oracle calls per node and $\tau$ times less communication rounds.
\end{corollary}

\subsubsection{Expected Smoothness and Arbitrary Sampling}\label{sec:sgd_es}
In this section we continue our consideration of {\tt Local-SGD} but now we make another assumption on stochastic gradients $\nabla f_{\xi_i}(x)$.
\begin{assumption}[Expected Smoothness]\label{ass:expected_smoothness}
	We assume that for all $i\in[n]$ stochastic gradients $\nabla f_{\xi_i}(x)$ are unbiased estimators of $\nabla f_i(x)$ and there exists such constant $\cL > 0$ that $\forall x,y\in\R^d$
	\begin{equation}
		\EE_{\xi_i\sim \cD_i}\left[\left\|\nabla f_{\xi_i}(x) - \nabla f_{\xi_i}(x^*)\right\|^2\right] \le 2\cL D_{f_i}(x,x^*) \label{eq:expected_smoothness_1}
	\end{equation}
%	\begin{equation}
%		\EE_{\xi_i\sim \cD_i}\left[\left\|\nabla f_{\xi_i}(x) - \nabla f_{\xi_i}(y)\right\|^2\right] \le \cL^2 \|x-y\|^2 \label{eq:expected_smoothness_2}
%	\end{equation}
	where $D_{f_i}(x,y) \eqdef f_i(x) - f_i(y) - \langle\nabla f_i(y), x-y\rangle$.
\end{assumption}

In particular, let us consider the following special case. Assume that $f_i(x)$ has a form of finite sum (see \eqref{eq:f_i_sum}) and consider the following stochastic reformulation:
\begin{equation}
	f_i(x) = \EE_{\xi_i}\left[f_{\xi_i}(x)\right],\quad f_{\xi_i}(x) = \frac{1}{m}\sum\limits_{j=1}^m \xi_{i,j}f_{i,j}(x), \label{eq:stochastic_reformulation}
\end{equation}
where $\EE[\xi_{i,j}] = 1$ and $\EE[\xi_{i,j}^2] < \infty$. In this case, $\EE_{\xi_i}[\nabla f_{\xi_i}] = \nabla f_i(x)$. If each $f_{i,j}(x)$ is $L_{i,j}$-smooth then there exists such $\cL \le \max_{j\in[m]}L_{i,j}$ that Assumption~\ref{ass:expected_smoothness} holds. Clearly, $\cL$ depends on the sampling strategy and in some cases one can make $\cL$ much smaller than $\max_{j\in[m]}L_{i,j}$ via good choice of this strategy. Our analysis works for an arbitrary sampling strategy that satisfies Assumption~\ref{ass:expected_smoothness}.

\begin{lemma}\label{lem:local_sgd_es_second_moment}
	Let $f_i$ be convex and $L$-smooth for all $i\in[n]$. Then for all $k\ge 0$
	\begin{eqnarray}
		\frac{1}{n}\sum\limits_{i=1}^n \EE\left[\|g_i^k\|^2\mid x^k\right] &\le& 8\cL\left(f(x^k) - f(x^*)\right) + 4\cL L V_k + 2\sigma_*^2 + 2\zeta_*^2,\label{eq:second_moment_local_sgd_es}\\
		\frac{1}{n}\sum\limits_{i=1}^n \EE\left[\|g_i^k-\bar{g}_i^k\|^2\mid x^k\right] &\le& 8\cL\left(f(x^k) - f(x^*)\right) + 4\cL L V_k + 2\sigma_*^2,\label{eq:var_local_sgd_es}\\
		\EE\left[\left\|\frac{1}{n}\sum\limits_{i=1}^ng_i^k\right\|^2\mid x^k\right] &\le& 4\left(\nicefrac{2\cL}{n} + L\right)(f(x^k) - f(x^*)) + 2L\left(\nicefrac{2\cL}{n} + L\right) V_k + \frac{2\sigma_*^2}{n},\label{eq:second_moment_local_sgd_es_2}
	\end{eqnarray}
	where $\sigma_*^2 = \frac{1}{n}\sum_{i=1}^n\EE\|\nabla f_{\xi_i}(x^*) - \nabla f_i(x^*)\|^2$, $\zeta_*^2 = \frac{1}{n}\sum_{i=1}^n\|\nabla f_i(x^*)\|^2$ and $\EE[\cdot\mid x^k]\eqdef \EE[\cdot\mid x_1^k,\ldots,x_n^k]$.
\end{lemma}
\begin{proof}
	First of all, we notice that $\bar{g}_i^k = \EE\left[g_i^k\mid x^k\right] = \nabla f_i(x_i^k)$. Using this we get
	\begin{eqnarray*}
		\frac{1}{n}\sum\limits_{i=1}^n \EE\left[\|g_i^k\|^2\mid x^k\right] &\overset{\eqref{eq:a_b_norm_squared}}{\le}& \frac{2}{n}\sum\limits_{i=1}^n \EE_{\xi_i^k}\|\nabla f_{\xi_i^k}(x_i^k) - \nabla f_{\xi_i^k}(x^*)\|^2 + \frac{2}{n}\sum\limits_{i=1}^n \EE_{\xi_i^k}\|\nabla f_{\xi_i^k}(x^*)\|^2\\
		&\overset{\eqref{eq:expected_smoothness_1},\eqref{eq:variance_decomposition}}{\le}&
		\frac{4\cL}{n}\sum\limits_{i=1}^nD_{f_i}(x_i^k,x^*) + \frac{2}{n}\sum\limits_{i=1}^n \EE_{\xi_i}\left[\|\nabla f_{\xi_i}(x^*)-\nabla f_i(x^*)\|^2\right] + \frac{2}{n}\sum\limits_{i=1}^n \|\nabla f_i(x^*)\|^2\\
		&\overset{\eqref{eq:poiouhnkj}}{\le}& 8\cL\left(f(x^k) - f(x^*)\right) + 4\cL L V_k + 2\sigma_*^2 + 2\zeta_*^2
	\end{eqnarray*}
	and
	\begin{eqnarray}
		\frac{1}{n}\sum\limits_{i=1}^n \EE\left[\|g_i^k-\bar{g}_i^k\|^2\mid x^k\right] &=& \frac{1}{n}\sum\limits_{i=1}^n\EE_{\xi_i^k}\|\nabla f_{\xi_i^k}(x_i^k)-\nabla f_i(x_i^k)\|^2\notag\\
		&\overset{\eqref{eq:variance_decomposition}}{\le}& \frac{1}{n}\sum\limits_{i=1}^n\EE_{\xi_i^k}\|\nabla f_{\xi_i^k}(x_i^k)-\nabla f_i(x^*)\|^2 \notag \\
		&\overset{\eqref{eq:a_b_norm_squared}}{\le}& \frac{2}{n}\sum\limits_{i=1}^n\EE_{\xi_i^k}\|\nabla f_{\xi_i^k}(x_i^k)-\nabla f_{\xi_i^k}(x^*)\|^2 + \frac{2}{n}\sum\limits_{i=1}^n\EE_{\xi_i^k}\|\nabla f_{\xi_i^k}(x^*)-\nabla f_{i}(x^*)\|^2\notag\\
		&\overset{\eqref{eq:expected_smoothness_1}}{\le}& \frac{4\cL}{n}\sum\limits_{i=1}^nD_{f_i}(x_i^k,x^*) + 2\sigma_*^2\notag\\
		&\overset{\eqref{eq:poiouhnkj}}{\le}& 8\cL\left(f(x^k) - f(x^*)\right) + 4\cL L V_k + 2\sigma_*^2. \label{eq:hbdsfbhdbfvfdh}
	\end{eqnarray}		
	Finally, using independence of $\xi_1^k,\xi_2^k,\ldots,\xi_n^k$ we obtain
	\begin{eqnarray*}
		\EE\left[\left\|\frac{1}{n}\sum\limits_{i=1}^ng_i^k\right\|^2\mid x^k\right] &\overset{\eqref{eq:variance_decomposition}}{=}& \EE_{\xi_i^k}\left[\left\|\frac{1}{n}\sum\limits_{i=1}^n(\nabla f_{\xi_i^k}(x_i^k) - \nabla f_{i}(x_i^k))\right\|^2\right] + \left\|\frac{1}{n}\sum\limits_{i=1}^n\nabla f_{i}(x_i^k)\right\|^2\\
		&=& \frac{1}{n^2}\sum\limits_{i=1}^n\EE_{\xi_i^k}\left[\|\nabla f_{\xi_i^k}(x_i^k) - \nabla f_{i}(x_i^k)\|^2\right] + \left\|\frac{1}{n}\sum\limits_{i=1}^n\nabla f_{i}(x_i^k)\right\|^2\\
		&\overset{\eqref{eq:hbdsfbhdbfvfdh},\eqref{eq:vdgasvgda}}{\le}& 4\left(\nicefrac{2\cL}{n} + L\right)(f(x^k) - f(x^*)) + 2L\left(\nicefrac{2\cL}{n} + L\right) V_k + \frac{2\sigma_*^2}{n}.
	\end{eqnarray*}
\end{proof}

\subsubsection*{Heterogeneous Data}
Applying Corollary~\ref{cor:const_loop} and Lemmas~\ref{lem:local_sgd_interesting_labels}~and~\ref{lem:local_sgd_es_second_moment} we get the following result.
\begin{theorem}\label{thm:local_sgd_es}
	Assume that $f_i(x)$ is $\mu$-strongly convex and $L$-smooth for $i\in[n]$. Let Assumption~\ref{ass:expected_smoothness} holds. Then {\tt Local-SGD} satisfies Assumption~\ref{ass:hetero_second_moment} with
	\begin{gather*}
		\tA = 3L,\quad \hA = 4\cL,\quad \tB = \hB = 0,\quad \tF = 3L^2,\quad \hF = 4\cL L, \quad \tD_1 = 3\zeta_*^2,\quad \hD_1 = 2\sigma_*^2\\
		A' = \frac{4\cL}{n} + 2L,\quad B' = 0,\quad F' = \frac{4\cL L}{n} + 2L^2, \quad D_1' = \frac{2\sigma_*^2}{n},\\
		\sigma_k^2 \equiv 0,\quad \rho = 1,\quad C = 0,\quad G = 0,\quad D_2 = 0,\\
		H = 0,\quad D_3 = 2e(\tau-1)\left(2\sigma_*^2+3(\tau-1)\zeta_*^2\right)
	\end{gather*}
	with $\gamma$ satisfying
	\begin{eqnarray*}
		\gamma &\le& \min\left\{\frac{1}{\nicefrac{8\cL}{n} + 4L}, \frac{1}{4\sqrt{2eL(\tau-1)\left(3L(\tau-1)+4\cL \right)}}\right\}.
	\end{eqnarray*}
	and for all $K \ge 0$
	\begin{eqnarray}
		\EE\left[f(\overline{x}^K) - f(x^*)\right] &\le& \frac{2\|x^0-x^*\|^2}{\gamma W_K} + 2\gamma\left(\nicefrac{2\sigma_*^2}{n} + 4eL(\tau-1)\gamma\left(2\sigma_*^2+3(\tau-1)\zeta_*^2\right)\right). \notag
	\end{eqnarray}
	In particular, if $\mu > 0$ then
	\begin{eqnarray}
		\EE\left[f(\overline{x}^K) - f(x^*)\right] &\le& \left(1 - \gamma\mu\right)^K\frac{2\|x^0-x^*\|^2}{\gamma} + 2\gamma\left(\nicefrac{2\sigma_*^2}{n} + 4eL(\tau-1)\gamma\left(2\sigma_*^2+3(\tau-1)\zeta_*^2\right)\right) \label{eq:local_sgd_es_str_cvx}
	\end{eqnarray}
	and when $\mu = 0$ we have
	\begin{eqnarray}
		\EE\left[f(\overline{x}^K) - f(x^*)\right] &\le& \frac{2\|x^0-x^*\|^2}{\gamma K} + 2\gamma\left(\nicefrac{2\sigma_*^2}{n} + 4eL(\tau-1)\gamma\left(2\sigma_*^2+3(\tau-1)\zeta_*^2\right)\right). \label{eq:local_sgd_es_cvx}
	\end{eqnarray}
\end{theorem}

The theorem above together with Lemma~\ref{lem:lemma2_stich} implies the following result.
\begin{corollary}\label{cor:local_sgd_es_str_cvx}
	Let assumptions of Theorem~\ref{thm:local_sgd_es} hold with $\mu > 0$. Then for 
	\begin{eqnarray*}
		\gamma_0 &=& \min\left\{\frac{1}{\nicefrac{8\cL}{n} + 4L}, \frac{1}{4\sqrt{2eL(\tau-1)\left(3L(\tau-1)+4\cL \right)}}\right\},\\
		\gamma &=& \min\left\{\gamma_0,\frac{\ln\left(\max\left\{2, \min\left\{\nicefrac{n\|x^0 - x^*\|^2\mu^2K^2}{2\sigma_*^2}, \nicefrac{\|x^0 - x^*\|^2\mu^3K^3}{4eL(\tau-1)\gamma\left(2\sigma_*^2+3(\tau-1)\zeta_*^2\right)}\right\}\right\}\right)}{\mu K}\right\}
	\end{eqnarray*}
	for all $K$ such that 
	\begin{eqnarray*}
		\text{either} && \frac{\ln\left(\max\left\{2, \min\left\{\nicefrac{n\|x^0 - x^*\|^2\mu^2K^2}{2\sigma_*^2}, \nicefrac{\|x^0 - x^*\|^2\mu^3K^3}{4eL(\tau-1)\gamma\left(2\sigma_*^2+3(\tau-1)\zeta_*^2\right)}\right\}\right\}\right)}{ K} \le 1\\
		\text{or} && \gamma_0 \le \frac{\ln\left(\max\left\{2, \min\left\{\nicefrac{n\|x^0 - x^*\|^2\mu^2K^2}{2\sigma_*^2}, \nicefrac{\|x^0 - x^*\|^2\mu^3K^3}{4eL(\tau-1)\gamma\left(2\sigma_*^2+3(\tau-1)\zeta_*^2\right)}\right\}\right\}\right)}{\mu K}
	\end{eqnarray*}
	we have that $\EE\left[f(\overline{x}^K)-f(x^*)\right]$ is of the order
	\begin{equation}
		 \widetilde\cO\left(\left(L\tau + \nicefrac{\cL}{n}+\sqrt{(\tau-1)\cL L}\right)R_0^2\exp\left(- \frac{\mu}{L\tau + \nicefrac{\cL}{n}+\sqrt{(\tau-1)\cL L}} K\right) + \frac{\sigma_*^2}{n\mu K} + \frac{L(\tau-1)\left(\sigma_*^2+(\tau-1)\zeta_*^2\right)}{\mu^2 K^2}\right),\notag
	\end{equation}
	where $R_0 = \|x^0-x^*\|$. That is, to achieve $\EE\left[f(\overline{x}^K)-f(x^*)\right] \le \varepsilon$ in this case {\tt Local-SGD} requires
	\begin{equation*}
		\widetilde{\cO}\left(\frac{L\tau}{\mu}+\frac{\cL}{n\mu}+\frac{\sqrt{(\tau-1)\cL L}}{\mu} + \frac{\sigma_*^2}{n\mu\varepsilon} + \sqrt{\frac{L(\tau-1)\left(\sigma_*^2 + (\tau-1)\zeta_*^2\right)}{\mu^2\varepsilon}}\right)
	\end{equation*}
	iterations/oracle calls per node and $\tau$ times less communication rounds.	
\end{corollary}

Combining Theorem~\ref{thm:local_sgd_es} and Lemma~\ref{lem:lemma_technical_cvx} we derive the following result for the convergence of {\tt Local-SGD} in the case when $\mu = 0$.
\begin{corollary}
	\label{cor:local_sgd_es_cvx}
	Let assumptions of Theorem~\ref{thm:local_sgd_es} hold with $\mu = 0$. Then for
	\begin{eqnarray*}
		\gamma_0 &=& \min\left\{\frac{1}{\nicefrac{8\cL}{n} + 4L}, \frac{1}{4\sqrt{2eL(\tau-1)\left(3L(\tau-1)+4\cL \right)}}\right\},\\
		\gamma &=& \min\left\{\gamma_0, \sqrt{\frac{nR_0^2}{2\sigma_*^2 K}}, \sqrt[3]{\frac{R_0^2}{4eL(\tau-1)\left(2\sigma_*^2+3(\tau-1)\zeta_*^2\right) K}}\right\},
	\end{eqnarray*}
	where $R_0 = \|x^0 - x^*\|$, we have that
	\begin{equation}
		\EE\left[f(\overline{x}^K)-f(x^*)\right] = \cO\left(\frac{\left(L\tau + \nicefrac{\cL}{n}+\sqrt{(\tau-1)\cL L}\right)R_0^2}{K} + \sqrt{\frac{R_0^2\sigma_*^2}{nK}} + \frac{\sqrt[3]{LR_0^4(\tau-1)\left(\sigma_*^2+(\tau-1)\zeta_*^2\right)}}{K^{\nicefrac{2}{3}}} \right).\label{eq:local_sgd_es_cvx_1}
	\end{equation}
	That is, to achieve $\EE\left[f(\overline{x}^K)-f(x^*)\right] \le \varepsilon$ in this case {\tt Local-SGD} requires
	\begin{equation*}
		\cO\left(\frac{\left(L\tau + \nicefrac{\cL}{n}+\sqrt{(\tau-1)\cL L}\right)R_0^2}{\varepsilon} + \frac{R_0^2\sigma_*^2}{n\varepsilon^2} + \frac{R_0^2\sqrt{L(\tau-1)\left(\sigma_*^2+(\tau-1)\zeta_*^2\right)}}{\varepsilon^{\nicefrac{3}{2}}}\right)
	\end{equation*}
	iterations/oracle calls per node and $\tau$ times less communication rounds.
\end{corollary}

\subsubsection*{$\zeta$-Heterogeneous Data}
Applying Corollary~\ref{cor:const_loop_homo} and Lemma~\ref{lem:local_sgd_es_second_moment} we get the following result.
\begin{theorem}\label{thm:local_sgd_es_homo}
	Assume that $f_i(x)$ is $L$-smooth for $i\in[n]$ and $f_1, \ldots, f_n$ are $\zeta$-heterogeneous and $\mu$-strongly convex. Let Assumption~\ref{ass:expected_smoothness} holds. Then {\tt Local-SGD} satisfies Assumption~\ref{ass:key_assumption} with
	\begin{gather*}
		A = 4\cL,\quad B = 0,\quad F = 4\cL L, \quad D_1 = 2\sigma_*^2 + 2\zeta_*^2,\\
		A' = \frac{4\cL}{n} + 2L,\quad B' = 0,\quad F' = \frac{4\cL L}{n} + 2L^2, \quad D_1' = \frac{2\sigma_*^2}{n},\\
		\sigma_k^2 \equiv 0,\quad \rho = 1,\quad C = 0,\quad G = 0,\quad D_2 = 0,\\
		H = 0,\quad D_3 = 2(\tau-1)\left(2\sigma_*^2 + 2\zeta_*^2 + \frac{\zeta^2}{\gamma\mu}\right)
	\end{gather*}
	with $\gamma$ satisfying
	\begin{eqnarray*}
		\gamma &\le& \min\left\{\frac{1}{\nicefrac{8\cL}{n} + 4L}, \frac{1}{8\sqrt{2L\cL(\tau-1)}}\right\}.
	\end{eqnarray*}
	and for all $K \ge 0$
	\begin{eqnarray}
		\EE\left[f(\overline{x}^K) - f(x^*)\right] &\le& \frac{2\|x^0-x^*\|^2}{\gamma W_K} + 2\gamma\left(\frac{2\sigma_*^2}{n} + \frac{4L\zeta^2(\tau-1)}{\mu} + 8L(\tau-1)\gamma\left(\sigma_*^2+\zeta_*^2\right) \right). \notag
	\end{eqnarray}
	In particular, if $\mu > 0$ then
	\begin{eqnarray}
		\EE\left[f(\overline{x}^K) - f(x^*)\right] &\le& \left(1 - \gamma\mu\right)^K\frac{2\|x^0-x^*\|^2}{\gamma} + 2\gamma\left(\frac{2\sigma_*^2}{n} + \frac{4L\zeta^2(\tau-1)}{\mu} + 8L(\tau-1)\gamma\left(\sigma_*^2+\zeta_*^2\right) \right) \label{eq:local_sgd_es_str_cvx_homo}
	\end{eqnarray}
	and when $\mu = 0$ we have
	\begin{eqnarray}
		\EE\left[f(\overline{x}^K) - f(x^*)\right] &\le& \frac{2\|x^0-x^*\|^2}{\gamma K} + 2\gamma\left(\frac{2\sigma_*^2}{n} + \frac{4L\zeta^2(\tau-1)}{\mu} + 8L(\tau-1)\gamma\left(\sigma_*^2+\zeta_*^2\right) \right). \label{eq:local_sgd_es_cvx_homo}
	\end{eqnarray}
\end{theorem}

The theorem above together with Lemma~\ref{lem:lemma2_stich} implies the following result.
\begin{corollary}\label{cor:local_sgd_es_str_cvx_homo}
	Let assumptions of Theorem~\ref{thm:local_sgd_es_homo} hold with $\mu > 0$. Then for 
	\begin{eqnarray*}
		\gamma_0 &=& \min\left\{\frac{1}{\nicefrac{8\cL}{n} + 4L}, \frac{1}{8\sqrt{2L\cL(\tau-1)}}\right\},\\
		\gamma &=& \min\left\{\gamma_0,\frac{\ln\left(\max\left\{2, \min\left\{\nicefrac{\|x^0 - x^*\|^2\mu^2K^2}{\left(\nicefrac{2\sigma_*^2}{n}+\nicefrac{4L\zeta^2(\tau-1)}{\mu}\right)}, \nicefrac{\|x^0 - x^*\|^2\mu^3K^3}{8L(\tau-1)\left(\sigma_*^2+\zeta_*^2\right) }\right\}\right\}\right)}{\mu K}\right\}
	\end{eqnarray*}
	for all $K$ such that 
	\begin{eqnarray*}
		\text{either} && \frac{\ln\left(\max\left\{2, \min\left\{\nicefrac{\|x^0 - x^*\|^2\mu^2K^2}{\left(\nicefrac{2\sigma_*^2}{n}+\nicefrac{4L\zeta^2(\tau-1)}{\mu}\right)}, \nicefrac{\|x^0 - x^*\|^2\mu^3K^3}{8L(\tau-1)\left(\sigma_*^2+\zeta_*^2\right) }\right\}\right\}\right)}{ K}\le 1\\
		\text{or} && \gamma_0 \le \frac{\ln\left(\max\left\{2, \min\left\{\nicefrac{\|x^0 - x^*\|^2\mu^2K^2}{\left(\nicefrac{2\sigma_*^2}{n}+\nicefrac{4L\zeta^2(\tau-1)}{\mu}\right)}, \nicefrac{\|x^0 - x^*\|^2\mu^3K^3}{8L(\tau-1)\left(\sigma_*^2+\zeta_*^2\right) }\right\}\right\}\right)}{\mu K}
	\end{eqnarray*}
	we have that $\EE\left[f(\overline{x}^K)-f(x^*)\right]$ is of the order
	\begin{equation}
		 \widetilde\cO\left(\left(L + \nicefrac{\cL}{n}+ \sqrt{(\tau-1)\cL L}\right)R_0^2\exp\left(- \frac{\mu}{L + \nicefrac{\cL}{n}+ \sqrt{(\tau-1)\cL L}} K\right) + \frac{\sigma_*^2}{n\mu K} + \frac{L\zeta^2(\tau-1)}{\mu^2 K} + \frac{L(\tau-1)\left(\sigma_*^2+\zeta_*^2\right)}{\mu^2 K^2}\right),\notag
	\end{equation}
	where $R_0 = \|x^0 - x^*\|$. That is, to achieve $\EE\left[f(\overline{x}^K)-f(x^*)\right] \le \varepsilon$ in this case {\tt Local-SGD} requires
	\begin{equation*}
		\widetilde{\cO}\left(\frac{L}{\mu}+\frac{\cL}{n\mu}+\frac{\sqrt{(\tau-1)\cL L}}{\mu} + \frac{\sigma_*^2}{n\mu\varepsilon} + \frac{L\zeta^2(\tau-1)}{\mu^2 \varepsilon} + \sqrt{\frac{L(\tau-1)\left(\sigma_*^2 + \zeta_*^2\right)}{\mu^2\varepsilon}}\right)
	\end{equation*}
	iterations/oracle calls per node and $\tau$ times less communication rounds.	
\end{corollary}

Combining Theorem~\ref{thm:local_sgd_es_homo} and Lemma~\ref{lem:lemma_technical_cvx} we derive the following result for the convergence of {\tt Local-SGD} in the case when $\mu = 0$.
\begin{corollary}
	\label{cor:local_sgd_es_cvx_homo}
	Let assumptions of Theorem~\ref{thm:local_sgd_es_homo} hold with $\mu = 0$. Then for
	\begin{eqnarray*}
	\gamma_0 &=& \min\left\{\frac{1}{\nicefrac{8\cL}{n} + 4L}, \frac{1}{8\sqrt{2L\cL(\tau-1)}}\right\},\\
		\gamma &=& \min\left\{\gamma_0, \sqrt{\frac{R_0^2}{\left(\nicefrac{2\sigma_*^2}{n} + \nicefrac{4L\zeta^2(\tau-1)}{\mu}\right) K}}, \sqrt[3]{\frac{R_0^2}{8L(\tau-1)\left(\sigma_*^2+\zeta_*^2\right) K}}\right\},
	\end{eqnarray*}
	where $R_0 = \|x^0 - x^*\|$, we have that
	\begin{equation}
		\EE\left[f(\overline{x}^K)-f(x^*)\right] = \cO\left(\frac{\left(L + \nicefrac{\cL}{n}+ \sqrt{(\tau-1)\cL L}\right)R_0^2}{K} + \sqrt{\frac{R_0^2\left(\nicefrac{\sigma_*^2}{n} + \nicefrac{L\zeta^2(\tau-1)}{\mu}\right)}{K}} + \frac{\sqrt[3]{LR_0^4(\tau-1)\left(\sigma_*^2+\zeta_*^2\right)}}{K^{\nicefrac{2}{3}}} \right).\notag
	\end{equation}
	That is, to achieve $\EE\left[f(\overline{x}^K)-f(x^*)\right] \le \varepsilon$ in this case {\tt Local-SGD} requires
	\begin{equation*}
		\cO\left(\frac{\left(L + \nicefrac{\cL}{n}+ \sqrt{(\tau-1)\cL L}\right)R_0^2}{\varepsilon} + \frac{\left(\nicefrac{\sigma_*^2}{n} + \nicefrac{L\zeta^2(\tau-1)}{\mu}\right)R_0^2}{\varepsilon^2} + \frac{R_0^2\sqrt{L(\tau-1)\left(\sigma_*^2+\zeta_*^2\right)}}{\varepsilon^{\nicefrac{3}{2}}}\right)
	\end{equation*}
	iterations/oracle calls per node and $\tau$ times less communication rounds.
\end{corollary}

\subsection{{\tt Local-SVRG} \label{sec:llsvrg}}

As an alternative to {\tt Local-SGD} when the local objective is of a finite sum structure~\eqref{eq:f_i_sum}, we propose {\tt L-SVRG}~\cite{hofmann2015variance, kovalev2019don} stochastic gradient as a local direction instead of the plain stochastic gradient. Specifically, we consider 
\[
a_i^k \eqdef  \nabla f_{i,j_i}(x_i^k) -\nabla f_{i,j_i}(w_i^k) + \nabla f_{i}(w_i^k), \qquad  b_i^k = 0,
\]
where index $1\leq j_i \leq m$ is selected uniformly at random and $w_i^k$ is a particular iterate from the local history updated as follows: 
\[
w_i^{k+1} =  \begin{cases}
            x^{k}_i & \text{w.p. } \psvrg \\
            w_i^k & \text{w.p. } 1- \psvrg.
            \end{cases}
\]

Next, we will assume that the local functions $f_{i,j}$ are $\max L_{ij}$-smooth.\footnote{It is easy to see that we must have $\max L_{ij}\geq L \geq \frac1m\max L_{ij}$.} Lastly, we will equip the mentioned method with the fixed local loop. The formal statement of the described instance of~\eqref{eq:local_sgd_def} is given as Algorithm~\ref{alg:local_svrg}.

\begin{algorithm}[h]
   \caption{{\tt Local-SVRG}}\label{alg:local_svrg}
\begin{algorithmic}[1]
   \Require learning rate $\gamma>0$, initial vector $x^0 \in \R^d$, communication period $\tau \ge 1$
	\For{$k=0,1,\dotsc$}
       	\For{$i=1,\dotsc,n$ in parallel}
            \State Choose $j_i$ uniformly at random, independently across nodes
            \State $g_i^k = \nabla f_{i,j_i}(x_i^k) -\nabla f_{i,j_i}(w_i^k) + \nabla f_{i}(w_i^k)$
            \State $w_i^{k+1} =  \begin{cases}
            x^{k}_i & \text{w.p. } \psvrg \\
            w_i^k & \text{w.p. } 1- \psvrg
            \end{cases}$
            \If{$k+1 \mod \tau = 0$}
            \State $x_i^{k+1} = x^{k+1} = \frac{1}{n}\sum\limits_{i=1}^n\left(x_i^k - \gamma g_i^k\right)$ \Comment{averaging}
            \Else
            \State $x_i^{k+1} = x_i^k - \gamma g_i^k$ \Comment{local update}
            \EndIf
        \EndFor
   \EndFor
\end{algorithmic}
\end{algorithm}

Let us next provide the details on the convergence rate. In order to do so, let us identify the parameters of Assumption~\ref{ass:sigma_k_original}.

\begin{proposition}[see \cite{gorbunov2019unified}]\label{prop:local_svrg} 
Gradient estimator $a_i^k$ satisfies Assumption~\ref{ass:sigma_k_original} with parameters $A_i = 2\max L_{ij}, B_i=2, D_{1,i} = 0, \rho_i = \psvrg, C_i = \max L_{ij} \psvrg, D_{2,i}= 0$, and $\sigma_{i,k}^2 = \frac{1}{m}\sum\limits_{j=1}^m\|\nabla f_{ij}(w_i^k)-\nabla f_{ij}(x^*)\|^2$.
\end{proposition}

\subsubsection{$\zeta$-Heterogeneous Data}
It remains to use Lemma~\ref{lem:local_solver} along with Corollary~\ref{cor:const_loop_homo} to recover all parameters of Assumption~\ref{ass:key_assumption} and obtain a convergence rate of Algorithm~\ref{alg:local_svrg} in $\zeta$-heterogeneous case.
\begin{theorem}\label{thm:local_svrg_homo}
	Assume that $f_i(x)$ is $\mu$-strongly convex and $L$-smooth for $i\in[n]$ and $f_1, \ldots, f_n$ are $\zeta$-heterogeneous, convex and $\max L_{ij}$-smooth. Then {\tt Local-SVRG} satisfies Assumption~\ref{ass:key_assumption} with
	\begin{gather*}
		A = 8\max L_{ij},\quad B = 2,\quad F = 8L\max L_{ij}, \quad D_1 = 2\zeta_*^2,\\
		A' = \frac{4 \max L_{ij}}{n} + L,\quad B' = \frac{1}{n},\quad F' = \frac{4L \max L_{ij}}{n} + 2L^2, \quad D_1' = 0,\\
		\sigma_k^2 = \frac{4}{nm}\sum\limits_{i=1}^n\sum\limits_{j=1}^m\|\nabla f_{ij}(w_i^k) - \nabla f_{ij}(x^*)\|^2,\quad \rho = q,\quad C = 8q\max L_{ij},\quad G = 4qL\max L_{ij},\quad D_2 = 0,\\
		H = \frac{8(\tau-1)(2+q)\gamma^2}{q},\quad D_3 = 2(\tau-1)\left(2\zeta_*^2 + \frac{\zeta^2}{\gamma\mu}\right)
	\end{gather*}
	with $\gamma$ satisfying
	\begin{eqnarray*}
		\gamma &\le& \min\left\{\frac{1}{2\left(\nicefrac{44\max L_{ij}}{n}+L\right)}, \frac{1}{16\sqrt{L\max L_{ij}(\tau-1)\left(1+\nicefrac{4}{(1-q)}\right)}}\right\}.
	\end{eqnarray*}
	and for all $K \ge 0$
	\begin{eqnarray}
		\EE\left[f(\overline{x}^K) - f(x^*)\right] &\le& \frac{\Phi^0}{\gamma W_K} + 8L(\tau-1)\gamma\left(\frac{\zeta^2}{\mu}+2\gamma\zeta_*^2\right), \notag
	\end{eqnarray}
	where $\Phi^0 = 2\|x^0 - x^*\|^2 +   \frac{8}{3nq}\gamma^2 \sigma_0^2 + \frac{32L(\tau-1)(2+q)\gamma^3}{q}\sigma_0^2$. In particular, if $\mu > 0$ then
	\begin{eqnarray}
		\EE\left[f(\overline{x}^K) - f(x^*)\right] &\le& \left(1 - \min\left\{\gamma\mu,\frac{q}{4}\right\}\right)^K\frac{\Phi^0}{\gamma} + 8L(\tau-1)\gamma\left(\frac{\zeta^2}{\mu}+2\gamma\zeta_*^2\right) \label{eq:local_svrg_str_cvx_homo}
	\end{eqnarray}
	and when $\mu = 0$ we have
	\begin{eqnarray}
		\EE\left[f(\overline{x}^K) - f(x^*)\right] &\le& \frac{\Phi^0}{\gamma K} + 8L(\tau-1)\gamma\left(\frac{\zeta^2}{\mu}+2\gamma\zeta_*^2\right). \label{eq:local_svrg_cvx_homo}
	\end{eqnarray}
\end{theorem}

The theorem above together with Lemma~\ref{lem:lemma2_stich} implies the following result.
\begin{corollary}\label{cor:local_svrg_str_cvx_homo}
	Let assumptions of Theorem~\ref{thm:local_svrg_homo} hold with $\mu > 0$. Then for 
	\begin{eqnarray*}
		\gamma_0 &=& \min\left\{\frac{1}{2\left(\nicefrac{44\max L_{ij}}{n}+L\right)}, \frac{1}{16\sqrt{L\max L_{ij}(\tau-1)\left(1+\nicefrac{4}{(1-q)}\right)}}\right\},\quad q = \frac{1}{m},\quad m > 1,\\
		\widetilde{\Phi}^0 &=& 2\|x^0 - x^*\|^2 +   \frac{8}{3nq}\gamma_0^2 \sigma_0^2 + \frac{32L(\tau-1)(2+q)\gamma_0^3}{q}\sigma_0^2,\\
		\gamma &=& \min\left\{\gamma_0,\frac{\ln\left(\max\left\{2, \min\left\{\nicefrac{\widetilde{\Phi}^0\mu^3K^2}{8L\zeta^2(\tau-1)}, \nicefrac{\widetilde{\Phi}^0\mu^3K^3}{16L(\tau-1)\zeta_*^2 }\right\}\right\}\right)}{\mu K}\right\},
	\end{eqnarray*}
	for all $K$ such that 
	\begin{eqnarray*}
		\text{either} && \frac{\ln\left(\max\left\{2, \min\left\{\nicefrac{\widetilde{\Phi}^0\mu^3K^2}{8L\zeta^2(\tau-1)}, \nicefrac{\widetilde{\Phi}^0\mu^3K^3}{16L(\tau-1)\zeta_*^2 }\right\}\right\}\right)}{ K} \le \frac{1}{m}\\
		\text{or} && \gamma_0 \le \frac{\ln\left(\max\left\{2, \min\left\{\nicefrac{\widetilde{\Phi}^0\mu^3K^2}{8L\zeta^2(\tau-1)}, \nicefrac{\widetilde{\Phi}^0\mu^3K^3}{16L(\tau-1)\zeta_*^2 }\right\}\right\}\right)}{\mu K}
	\end{eqnarray*}
	we have that $\EE\left[f(\overline{x}^K)-f(x^*)\right]$ is of the order
	\begin{equation}
		 \widetilde\cO\left(\frac{\widetilde{\Phi}^0}{\gamma_0}\exp\left(- \min\left\{m^{-1}, \gamma_0\mu\right\} K\right) + \frac{\zeta^2L(\tau-1)}{\mu^2 K} + \frac{L(\tau-1)\zeta_*^2}{\mu^2 K^2}\right).\notag
	\end{equation}
	That is, to achieve $\EE\left[f(\overline{x}^K)-f(x^*)\right] \le \varepsilon$ in this case {\tt Local-SVRG} requires
	\begin{equation*}
		\widetilde{\cO}\left(m + \frac{L}{\mu}+\frac{\max L_{ij}}{n\mu}+\frac{\sqrt{(\tau-1)L\max L_{ij}}}{\mu} + \frac{L\zeta^2(\tau-1)}{\mu^2 \varepsilon} + \sqrt{\frac{L(\tau-1)\zeta_*^2}{\mu^2\varepsilon}}\right)
	\end{equation*}
	iterations/oracle calls per node and $\tau$ times less communication rounds.	
\end{corollary}

Combining Theorem~\ref{thm:local_svrg_homo} and Lemma~\ref{lem:lemma_technical_cvx} we derive the following result for the convergence of {\tt Local-SVRG} in the case when $\mu = 0$.
\begin{corollary}
	\label{cor:local_svrg_cvx_homo}
	Let assumptions of Theorem~\ref{thm:local_svrg_homo} hold with $\mu = 0$. Then for
	\begin{eqnarray*}
		\gamma_0 &=& \min\left\{\frac{1}{2\left(\nicefrac{44\max L_{ij}}{n}+L\right)}, \frac{1}{16\sqrt{L\max L_{ij}(\tau-1)\left(1+\nicefrac{4}{(1-q)}\right)}}\right\},\quad q = \frac{1}{m},\quad m > 1,\\	
		\gamma &=& \min\left\{\gamma_0, \sqrt{\frac{3nR_0^2}{4m\sigma_0^2}}, \sqrt[3]{\frac{R_0^2}{16Lm(\tau-1)(2+\nicefrac{1}{m})\sigma_0^2}}, \sqrt{\frac{\mu R_0^2}{4L\zeta^2(\tau-1) K}}, \sqrt[3]{\frac{R_0^2}{8L(\tau-1)\zeta_*^2 K}}\right\},
	\end{eqnarray*}
	where $R_0 = \|x^0 - x^*\|$, we have that $\EE\left[f(\overline{x}^K)-f(x^*)\right]$ is of the order
	\begin{eqnarray*}
		\cO\Bigg(\frac{(L +\nicefrac{\max L_{ij}}{n} +\sqrt{(\tau-1)L\max L_{ij}})R_0^2 + \sqrt{\nicefrac{m\sigma_0^2R_0^2}{n}} + \sqrt[3]{Lm(\tau-1)\sigma_0^2 R_0^4}}{K}& \\
		&\hspace{-2cm}+ \sqrt{\frac{LR_0^2\zeta^2(\tau-1)}{\mu K}} + \frac{\sqrt[3]{LR_0^4(\tau-1)\zeta_*^2}}{K^{\nicefrac{2}{3}}} \Bigg).
	\end{eqnarray*}
	That is, to achieve $\EE\left[f(\overline{x}^K)-f(x^*)\right] \le \varepsilon$ in this case {\tt Local-SVRG} requires
	\begin{eqnarray*}
		\cO\Bigg(\frac{(L +\nicefrac{\max L_{ij}}{n} +\sqrt{(\tau-1)L\max L_{ij}})R_0^2 + \sqrt{\nicefrac{m\sigma_0^2R_0^2}{n}} + \sqrt[3]{Lm(\tau-1)\sigma_0^2 R_0^4}}{\varepsilon}&\\
		&\hspace{-2cm} + \frac{L\zeta^2(\tau-1)R_0^2}{\mu\varepsilon^2} + \frac{R_0^2\sqrt{L(\tau-1)\zeta_*^2}}{\varepsilon^{\nicefrac{3}{2}}}\Bigg)
	\end{eqnarray*}
	iterations/oracle calls per node and $\tau$ times less communication rounds.
\end{corollary}

\begin{remark}
	To get the rate from Tbl.~\ref{tbl:special_cases_weakly_convex} it remains to apply the following inequality:
	\begin{eqnarray*}
		\sigma_0^2 = \frac{4}{nm}\sum\limits_{i=1}^n\sum\limits_{j=1}^m\|\nabla f_{ij}(x^0)-\nabla f_{ij}(x^*)\|^2 \overset{\eqref{eq:L_smoothness}}{\le} 4\max L_{ij}^2 \|x^0-x^*\|^2.
	\end{eqnarray*}
\end{remark}

\subsubsection{Heterogeneous Data}
First of all, we need the following lemma.
\begin{lemma}\label{lem:small_lemma_local_svrg}
	Assume that $f_i(x)$ is $L$-smooth for $i\in[n]$ and $f_{ij}$ is convex and $\max L_{ij}$-smooth for $i\in[n], j\in [m]$. Then for {\tt Local-SVRG} we have
	\begin{eqnarray}
		\frac{1}{n}\sum\limits_{i=1}^n\EE\left[\|\bar g_i^k\|^2\right] &\le& 6L\EE\left[f(x^k) - f(x^*)\right] + 3L^2 \EE[V_k] + 3\zeta_*^2,\label{eq:hetero_ineq_1_local_svrg}\\
		\frac{1}{n}\sum\limits_{i=1}^n\EE\left[\left\|g_i^k-\bar g_i^k\right\|^2\right] &\le& 8\max L_{ij}\EE\left[f(x^k)-f(x^*)\right] + \frac{1}{2}\EE[\sigma_k^2] + 4L\max L_{ij}\EE[V_k], \label{eq:hetero_ineq_2_local_svrg}
	\end{eqnarray}
	where $\sigma_k^2 = \frac{4}{nm}\sum\limits_{i=1}^n\sum\limits_{j=1}^m\|\nabla f_{ij}(w_i^k) - \nabla f_{ij}(x^*)\|^2$.
\end{lemma}
\begin{proof}
	Inequality \eqref{eq:hetero_ineq_1_local_svrg} follows from $\bar g_i^k = \EE\left[g_i^k\mid x^k\right] = \nabla f_i(x_i^k)$ and inequality \eqref{eq:dnaossniadd}. Next, using Young's inequality we derive
	\begin{eqnarray*}
		\frac{1}{n}\sum\limits_{i=1}^n\EE\left[\left\|g_i^k-\bar g_i^k\right\|^2\right] &\overset{\eqref{eq:variance_decomposition}}{\le}& \frac{1}{n}\sum\limits_{i=1}^n\EE\left[\left\|g_i^k- \nabla f_i(x^*)\right\|^2\right]\\
		&\overset{\eqref{eq:a_b_norm_squared}}{\le}& \frac{2}{n}\sum\limits_{i=1}^n\EE\left[\|\nabla f_{ij_i}(x_i^k)-\nabla f_{ij_i}(x^*)\|^2\right]\\
		&&\quad + \frac{2}{n}\sum\limits_{i=1}^n\EE\left[\|\nabla f_{ij_i}(w_i^k)-\nabla f_{ij_i}(x^*) - (\nabla f_i(w_i^k)-\nabla f_i(x^*))\|^2\right]\\
		&\overset{\eqref{eq:tower_property}}{=}& \frac{2}{nm}\sum\limits_{i=1}^n\sum\limits_{j=1}^m\EE\left[\|\nabla f_{ij}(x_i^k)-\nabla f_{ij}(x^*)\|^2\right]\\
		&&\quad + \frac{2}{nm}\sum\limits_{i=1}^n\sum\limits_{j=1}^m\EE\left[\|\nabla f_{ij}(w_i^k)-\nabla f_{ij}(x^*) - (\nabla f_i(w_i^k)-\nabla f_i(x^*))\|^2\right]\\
		&\overset{\eqref{eq:L_smoothness_cor},\eqref{eq:variance_decomposition}}{\le}& \frac{4\max L_{ij}}{n}\sum\limits_{i=1}^n\EE\left[D_{f_i}(x_i^k,x^*)\right] + \frac{2}{nm}\sum\limits_{i=1}^n\sum\limits_{j=1}^m\EE\left[\|\nabla f_{ij}(w_i^k)-\nabla f_{ij}(x^*)\|^2\right]\\
		&\overset{\eqref{eq:poiouhnkj}}{\le}& 8\max L_{ij}\EE\left[f(x^k)-f(x^*)\right] + \frac{1}{2}\EE[\sigma_k^2] + 4L\max L_{ij}\EE[V_k].	
	\end{eqnarray*}
\end{proof}

Applying Corollary~\ref{cor:const_loop}, Lemma~\ref{lem:small_lemma_local_svrg}, Proposition~\ref{prop:local_svrg} and Lemma~\ref{lem:local_solver} we get the following result.
\begin{theorem}\label{thm:local_svrg}
	Assume that $f_i(x)$ is $\mu$-strongly convex and $L$-smooth for $i\in[n]$ and $f_{ij}$ is convex and $\max L_{ij}$-smooth for $i\in[n], j\in [m]$. Then {\tt Local-SVRG} satisfies Assumption~\ref{ass:hetero_second_moment} with
	\begin{gather*}
		\tA = 3L,\quad \hA = 4\max L_{ij},\quad \tB = 0,\quad \hB = \frac{1}{2},\quad \tF = 3L^2,\quad \hF = 4L\max L_{ij}, \quad \tD_1 = 3\zeta_*^2,\quad \hD_1 = 0\\
		A' = \frac{4 \max L_{ij}}{n} + L,\quad B' = \frac{1}{n},\quad F' = \frac{4L \max L_{ij}}{n} + 2L^2, \quad D_1' = 0,\\
		\sigma_k^2 = \frac{4}{nm}\sum\limits_{i=1}^n\sum\limits_{j=1}^m\|\nabla f_{ij}(w_i^k) - \nabla f_{ij}(x^*)\|^2,\quad \rho = q,\quad C = 8q\max L_{ij},\quad G = 4qL\max L_{ij},\quad D_2 = 0,\\
		H = \frac{2e(\tau-1)(2+q)\gamma^2}{q},\quad D_3 = 6e(\tau-1)^2\zeta_*^2
	\end{gather*}
	with $\gamma$ satisfying
	\begin{eqnarray*}
		\gamma &\le& \min\left\{\frac{1}{2\left(\nicefrac{44\max L_{ij}}{n}+L\right)}, \frac{1}{4\sqrt{2eL(\tau-1)\left(3L(\tau-1)+4\max L_{ij} + \nicefrac{8\max L_{ij}}{(1-q)}\right)}}\right\}.
	\end{eqnarray*}
	and for all $K \ge 0$
	\begin{eqnarray}
		\EE\left[f(\overline{x}^K) - f(x^*)\right] &\le& \frac{\Phi^0}{\gamma W_K} + 24e L(\tau-1)^2\zeta_*^2\gamma^2, \notag
	\end{eqnarray}
	where $\Phi^0 = 2\|x^0 - x^*\|^2 +   \frac{8}{3nq}\gamma^2 \sigma_0^2 + \frac{8eL(\tau-1)(2+q)\gamma^3}{q}\sigma_0^2$
	In particular, if $\mu > 0$ then
	\begin{eqnarray}
		\EE\left[f(\overline{x}^K) - f(x^*)\right] &\le& \left(1 - \min\left\{\gamma\mu,\frac{q}{4}\right\}\right)^K\frac{\Phi^0}{\gamma} + 24e L(\tau-1)^2\zeta_*^2\gamma^2 \label{eq:local_svrg_str_cvx}
	\end{eqnarray}
	and when $\mu = 0$ we have
	\begin{eqnarray}
		\EE\left[f(\overline{x}^K) - f(x^*)\right] &\le& \frac{\Phi^0}{\gamma K} + 24e L(\tau-1)^2\zeta_*^2\gamma^2. \label{eq:local_svrg_cvx}
	\end{eqnarray}
\end{theorem}

The theorem above together with Lemma~\ref{lem:lemma2_stich} implies the following result.
\begin{corollary}\label{cor:local_svrg_str_cvx}
	Let assumptions of Theorem~\ref{thm:local_svrg} hold with $\mu > 0$. Then for 
	\begin{eqnarray*}
		\gamma_0 &=& \min\left\{\frac{1}{2\left(\nicefrac{44\max L_{ij}}{n}+L\right)}, \frac{1}{4\sqrt{2eL(\tau-1)\left(3L(\tau-1)+4\max L_{ij} + \nicefrac{8\max L_{ij}}{(1-q)}\right)}}\right\},\\
		\widetilde{\Phi}^0 &=& 2\|x^0 - x^*\|^2 +   \frac{8}{3nq}\gamma_0^2 \sigma_0^2 + \frac{8eL(\tau-1)(2+q)\gamma_0^3}{q}\sigma_0^2,\quad q = \frac{1}{m},\quad m > 1,\\
		\gamma &=& \min\left\{\gamma_0,\frac{\ln\left(\max\left\{2, \nicefrac{\widetilde{\Phi}^0\mu^3K^3}{24eL(\tau-1)^2\zeta_*^2 }\right\}\right)}{\mu K}\right\},
	\end{eqnarray*}
	for all $K$ such that 
	\begin{eqnarray*}
		\text{either} \qquad \frac{\ln\left(\max\left\{2, \nicefrac{\widetilde{\Phi}^0\mu^3K^3}{24eL(\tau-1)^2\zeta_*^2 }\right\}\right)}{ K} \le \frac{1}{m}\qquad \text{or} \qquad \gamma_0 \le \frac{\ln\left(\max\left\{2, \nicefrac{\widetilde{\Phi}^0\mu^3K^3}{24eL(\tau-1)^2\zeta_*^2 }\right\}\right)}{\mu K}
	\end{eqnarray*}
	we have that $\EE\left[f(\overline{x}^K)-f(x^*)\right]$ is of the order
	\begin{equation}
		 \widetilde\cO\left(\frac{\widetilde{\Phi}^0}{\gamma_0}\exp\left(- \min\left\{m^{-1}, \gamma_0\mu\right\} K\right)  + \frac{L(\tau-1)^2\zeta_*^2}{\mu^2 K^2}\right).\notag
	\end{equation}
	That is, to achieve $\EE\left[f(\overline{x}^K)-f(x^*)\right] \le \varepsilon$ in this case {\tt Local-SVRG} requires
	\begin{equation*}
		\widetilde{\cO}\left(m + \frac{L\tau}{\mu}+\frac{\max L_{ij}}{n\mu}+\frac{\sqrt{(\tau-1)L\max L_{ij}}}{\mu} + \sqrt{\frac{L(\tau-1)^2\zeta_*^2}{\mu^2\varepsilon}}\right)
	\end{equation*}
	iterations/oracle calls per node and $\tau$ times less communication rounds.	
\end{corollary}

Combining Theorem~\ref{thm:local_svrg} and Lemma~\ref{lem:lemma_technical_cvx} we derive the following result for the convergence of {\tt Local-SVRG} in the case when $\mu = 0$.
\begin{corollary}
	\label{cor:local_svrg_cvx}
	Let assumptions of Theorem~\ref{thm:local_svrg} hold with $\mu = 0$. Then for $q = \frac{1}{m},$ $m > 1$ and
	\begin{eqnarray*}
		\gamma_0 &=& \min\left\{\frac{1}{2\left(\nicefrac{44\max L_{ij}}{n}+L\right)}, \frac{1}{4\sqrt{2eL(\tau-1)\left(3L(\tau-1)+4\max L_{ij} + \nicefrac{8\max L_{ij}}{(1-q)}\right)}}\right\},\\	
		\gamma &=& \min\left\{\gamma_0, \sqrt{\frac{3nR_0^2}{4m\sigma_0^2}}, \sqrt[3]{\frac{R_0^2}{4eLm(\tau-1)(2+\nicefrac{1}{m})\sigma_0^2}}, \sqrt[3]{\frac{R_0^2}{12eL(\tau-1)^2\zeta_*^2 K}}\right\},
	\end{eqnarray*}
	where $R_0 = \|x^0 - x^*\|$, we have that $\EE\left[f(\overline{x}^K)-f(x^*)\right]$ is of the order
	\begin{eqnarray*}
		\cO\left(\frac{(L\tau +\nicefrac{\max L_{ij}}{n} +\sqrt{(\tau-1)L\max L_{ij}})R_0^2 + \sqrt{\nicefrac{m\sigma_0^2R_0^2}{n}} + \sqrt[3]{Lm(\tau-1)\sigma_0^2 R_0^4}}{K}+ \frac{\sqrt[3]{LR_0^4(\tau-1)^2\zeta_*^2}}{K^{\nicefrac{2}{3}}} \right).
	\end{eqnarray*}
	That is, to achieve $\EE\left[f(\overline{x}^K)-f(x^*)\right] \le \varepsilon$ in this case {\tt Local-SVRG} requires
	\begin{eqnarray*}
		\cO\left(\frac{(L\tau +\nicefrac{\max L_{ij}}{n} +\sqrt{(\tau-1)L\max L_{ij}})R_0^2 + \sqrt{\nicefrac{m\sigma_0^2R_0^2}{n}} + \sqrt[3]{Lm(\tau-1)\sigma_0^2 R_0^4}}{\varepsilon} + \frac{R_0^2\sqrt{L(\tau-1)^2\zeta_*^2}}{\varepsilon^{\nicefrac{3}{2}}}\right)
	\end{eqnarray*}
	iterations/oracle calls per node and $\tau$ times less communication rounds.
\end{corollary}

\begin{remark}
	To get the rate from Tbl.~\ref{tbl:special_cases_weakly_convex} it remains to apply the following inequality:
	\begin{eqnarray*}
		\sigma_0^2 = \frac{4}{nm}\sum\limits_{i=1}^n\sum\limits_{j=1}^m\|\nabla f_{ij}(x^0)-\nabla f_{ij}(x^*)\|^2 \overset{\eqref{eq:L_smoothness}}{\le} 4\max L_{ij}^2 \|x^0-x^*\|^2.
	\end{eqnarray*}
\end{remark}

\subsection{{\tt S*-Local-SGD}}\label{sec:sgd_star_bounded_var}
In this section we consider the same settings as in Section~\ref{sec:sgd_bounded_var} and our goal is to remove one of the main drawbacks of {\tt Local-SGD} in heterogeneous case which in the case of $\mu$-strongly convex $f_i$ with $\mu > 0$ converges with linear rate only to the neighbourhood of the solution even in the full-gradients case, i.e.\ when $D_{1,i} = 0$ for all $i\in[n]$. However, we start with unrealistic assumption that $i$-th node has an access to $\nabla f_i(x^*)$ for all $i\in[n]$. Under this assumption we present a new method called Star-Shifted Local-SGD ({\tt S*-Local-SGD}, see Algorithm~\ref{alg:local_sgd_star}).

\begin{algorithm}[h]
   \caption{{\tt S*-Local-SGD}}\label{alg:local_sgd_star}
\begin{algorithmic}[1]
   \Require learning rate $\gamma>0$, initial vector $x^0 \in \R^d$, communication period $\tau \ge 1$
	\For{$k=0,1,\dotsc$}
       	\For{$i=1,\dotsc,n$ in parallel}
            \State Sample $\hat g^{k}_i = \nabla f_{\xi_i^k}(x_i^k)$ independently from other nodes
            \State $g_i^k = \hat g_i^k - \nabla f_i(x^*)$
            \If{$k+1 \mod \tau = 0$}
            \State $x_i^{k+1} = x^{k+1} = \frac{1}{n}\sum\limits_{i=1}^n\left(x_i^k - \gamma g_i^k\right)$ \Comment{averaging}
            \Else
            \State $x_i^{k+1} = x_i^k - \gamma g_i^k$ \Comment{local update}
            \EndIf
        \EndFor
   \EndFor
\end{algorithmic}
\end{algorithm}

\begin{lemma}\label{lem:local_sgd_star_second_moment}
	Let $f_i$ be convex and $L$-smooth for all $i\in[n]$. Then for all $k\ge 0$
	\begin{eqnarray}
		\frac{1}{n}\sum\limits_{i=1}^n \EE\left[g_i^k\mid x_i^k\right] &=& \frac{1}{n}\sum\limits_{i=1}^n\nabla f_i(x_i^k), \label{eq:unbiasedness_local_sgd_star}\\
		\frac{1}{n}\sum\limits_{i=1}^n \|\bar{g}_i^k\|^2 &\le& 4L\left(f(x^k) - f(x^*)\right) + 2L^2 V_k,\label{eq:second_moment_2_local_sgd_star}\\
		\frac{1}{n}\sum\limits_{i=1}^n \EE\left[\|g_i^k - \bar{g}_i^k\|^2\mid x_i^k\right] &\le& \sigma^2,\label{eq:variance_local_sgd_star}\\
%		\frac{1}{n}\sum\limits_{i=1}^n \EE\left[\|g_i^k\|^2\mid x_i^k\right] &\le& 4L\left(f(x^k) - f(x^*)\right) + 2L^2 V_k + \sigma^2,\label{eq:second_moment_local_sgd_star}\\
		\EE\left[\left\|\frac{1}{n}\sum\limits_{i=1}^ng_i^k\right\|^2\mid x^k\right] &\le& 4L\left(f(x^k) - f(x^*)\right) + 2L^2 V_k + \frac{\sigma^2}{n},\label{eq:second_moment_local_sgd_star_2}
	\end{eqnarray}
	where $\sigma^2 \eqdef \frac{1}{n}\sum_{i=1}^nD_{1,i}$ and $\EE[\cdot\mid x^k]\eqdef \EE[\cdot\mid x_1^k,\ldots,x_n^k]$.
\end{lemma}
\begin{proof}
	First of all, we notice that $\EE\left[g_i^k\mid x_i^k\right] = \nabla f_i(x_i^k)-\nabla f_i(x^*)$ and
	\begin{equation*}
		\frac{1}{n}\sum\limits_{i=1}^n\EE\left[g_i^k\mid x_i^k\right] = \frac{1}{n}\sum\limits_{i=1}^n\left(\nabla f_i(x_i^k)-\nabla f_i(x^*)\right) = \frac{1}{n}\sum\limits_{i=1}^n\nabla f_i(x_i^k).
	\end{equation*}
	Using this we get
	\begin{eqnarray*}
		\frac{1}{n}\sum\limits_{i=1}^n \|\bar{g}_i^k\|^2 &=& \frac{1}{n}\sum\limits_{i=1}^n \|\nabla f_i(x_i^k) - \nabla f_i(x^*)\|^2 \overset{\eqref{eq:L_smoothness_cor}}{\le} \frac{2L}{n}\sum\limits_{i=1}^nD_{f_i}(x_i^k,x^*)\\
		&\overset{\eqref{eq:poiouhnkj}}{\le}& 4L\left(f(x^k) - f(x^*)\right) + 2L^2 V_k
	\end{eqnarray*}
	and
	\begin{equation*}
		\frac{1}{n}\sum\limits_{i=1}^n \EE\left[\|g_i^k - \bar{g}_i^k\|^2\mid x_i^k\right] = \frac{1}{n}\sum\limits_{i=1}^n \EE\left[\|\nabla f_{\xi_i^k}(x_i^k) - \nabla f_i(x_i^k)\|^2\right] \overset{\eqref{eq:bounded_variance}}{\le} \frac{1}{n}\sum\limits_{i=1}^nD_{1,i} =: \sigma^2.
	\end{equation*}
	Finally, using independence of $g_1^k,g_2^k,\ldots,g_n^k$ and $\frac{1}{n}\sum_{i=1}^n\nabla f_i(x^*) = \nabla f(x^*) = 0$ we obtain
	\begin{eqnarray*}
		\EE\left[\left\|\frac{1}{n}\sum\limits_{i=1}^ng_i^k\right\|^2\mid x^k\right] &\overset{\eqref{eq:variance_decomposition},\eqref{eq:unbiasedness_local_sgd_star}}{=}& \EE\left[\left\|\frac{1}{n}\sum\limits_{i=1}^n\left(g_i^k - \nabla f_i(x_i^k)\right)\right\|^2\mid x^k\right] + \left\|\frac{1}{n}\sum\limits_{i=1}^n\nabla f_i(x_i^k)\right\|^2\\
		&=& \EE\left[\left\|\frac{1}{n}\sum\limits_{i=1}^n\left(\nabla f_{\xi_i^k}(x_i^k) - \nabla f_i(x_i^k)\right)\right\|^2\mid x^k\right] + \left\|\frac{1}{n}\sum\limits_{i=1}^n\nabla f_i(x_i^k)\right\|^2\\
		&=& \frac{1}{n^2}\sum\limits_{i=1}^n\EE_{\xi_i^k}\left[\|\nabla f_{\xi_i^k}(x_i^k) - \nabla f_i(x_i^k)\|^2\right] + \left\|\frac{1}{n}\sum\limits_{i=1}^n\nabla f_i(x_i^k)\right\|^2\\
		&\overset{\eqref{eq:bounded_variance},\eqref{eq:vdgasvgda}}{\le}& 4L\left(f(x^k) - f(x^*)\right) + 2L^2 V_k + \frac{\sigma^2}{n}.
	\end{eqnarray*}
\end{proof}

Applying Corollary~\ref{cor:const_loop} and Lemma~\ref{lem:local_sgd_star_second_moment} we get the following result.
\begin{theorem}\label{thm:local_sgd_star}
	Assume that $f_i(x)$ is $\mu$-strongly convex and $L$-smooth for every $i\in[n]$. Then {\tt S*-Local-SGD} satisfies Assumption~\ref{ass:hetero_second_moment} with
	\begin{gather*}
		\tA = 2L,\quad \hA = 0,\quad \tB = \hB = 0,\quad \tF = 2L^2,\quad \hF = 0,\quad \tD_1 = 0,\quad \hD_1 = \sigma^2 := \frac{1}{n}\sum\limits_{i=1}^nD_{1,i}\\
		A' = 2L,\quad B' = 0,\quad F' = 2L^2, \quad D_1' = \frac{\sigma^2}{n},\quad \sigma_k^2 \equiv 0,\quad \rho = 1,\quad C = 0,\quad G = 0,\quad D_2 = 0,\\
		H = 0,\quad D_3 = 2e(\tau-1)\sigma^2.
	\end{gather*}
	Consequently, if
	\begin{eqnarray*}
		\gamma &\le& \min\left\{\frac{1}{4L}, \frac{1}{8\sqrt{e}(\tau-1)L}\right\}.
	\end{eqnarray*}
we have for $\mu > 0$
	\begin{eqnarray}
		\EE\left[f(\overline{x}^K) - f(x^*)\right] &\le& \left(1 - \gamma\mu\right)^K\frac{2\|x^0-x^*\|^2}{\gamma} + 2\gamma\left(\frac{\sigma^2}{n} + 4eL(\tau-1)\gamma \sigma^2\right) \notag
	\end{eqnarray}
	and when $\mu = 0$ we have
	\begin{eqnarray}
		\EE\left[f(\overline{x}^K) - f(x^*)\right] &\le& \frac{2\|x^0-x^*\|^2}{\gamma K} + 2\gamma\left(\frac{\sigma^2}{n} + 4eL(\tau-1)\gamma \sigma^2\right). \notag
	\end{eqnarray}
\end{theorem}

In the special case when $\nabla f_{\xi_i^k}(x_i^k) = \nabla f_i(x_i^k)$ for all $i\in[n]$ and $k\ge 0$ we obtain {\tt S*-Local-GD} which converges with $\cO\left(\tau\kappa\ln \frac{1}{\varepsilon}\right)$ rate when $\mu > 0$ and with $\cO\left(\frac{L\tau\|x^0 - x^*\|^2}{\varepsilon}\right)$ rate when $\mu = 0$ to the exact solution asymptotically.

The theorem above together with Lemma~\ref{lem:lemma2_stich} implies the following result.
\begin{corollary}\label{cor:sstarsgd}
	Let assumptions of Theorem~\ref{thm:local_sgd_star} hold with $\mu > 0$. Then for
	\begin{equation*}
		\gamma = \min\left\{\frac{1}{4L}, \frac{1}{8\sqrt{e}(\tau-1)L},  \frac{\ln\left(\max\left\{2,\min\left\{\nicefrac{\|x^0 - x^*\|^2n\mu^2K^2}{\sigma^2},\nicefrac{\|x^0 - x^*\|^2\mu^3K^3}{4eL(\tau-1)\sigma^2}\right\}\right\}\right)}{\mu K}\right\}
	\end{equation*}
	for all $K$ such that 
	\begin{eqnarray*}
		\text{either}&&\frac{\ln\left(\max\left\{2,\min\left\{\nicefrac{\|x^0 - x^*\|^2n\mu^2K^2}{\sigma^2},\nicefrac{\|x^0 - x^*\|^2\mu^3K^3}{4eL(\tau-1)\sigma^2}\right\}\right\}\right)}{ K} \le 1\\
		\text{or}&&\min\left\{\frac{1}{4L}, \frac{1}{8\sqrt{e}(\tau-1)L}\right\} \le \frac{\ln\left(\max\left\{2,\min\left\{\nicefrac{\|x^0 - x^*\|^2n\mu^2K^2}{\sigma^2},\nicefrac{\|x^0 - x^*\|^2\mu^3K^3}{4eL(\tau-1)\sigma^2}\right\}\right\}\right)}{\mu K}
	\end{eqnarray*}
	we have that
	\begin{equation}
		\EE\left[f(\overline{x}^K)-f(x^*)\right] = \widetilde\cO\left(\tau L\|x^0 - x^*\|^2\exp\left(- \frac{\mu}{\tau L} K\right) + \frac{\sigma^2}{n\mu K} + \frac{L(\tau-1)\sigma^2}{\mu^2K^2}\right).\notag
	\end{equation}
	That is, to achieve $\EE\left[f(\overline{x}^K)-f(x^*)\right] \le \varepsilon$ in this case {\tt S*-Local-SGD} requires
	\begin{equation*}
		\widetilde\cO\left(\frac{\tau L}{\mu} + \frac{\sigma^2}{n\mu\varepsilon} + \sqrt{\frac{L(\tau-1)\sigma^2}{\mu^2\varepsilon}}\right)
	\end{equation*}
	iterations/oracle calls per node and $\tau$ times less communication rounds.
\end{corollary}

Combining Theorem~\ref{thm:local_sgd_star} and Lemma~\ref{lem:lemma_technical_cvx} we derive the following result for the convergence of {\tt S*-Local-SGD} in the case when $\mu = 0$.
\begin{corollary}
	\label{cor:local_sgd_star_cvx}
	Let assumptions of Theorem~\ref{thm:local_sgd_star} hold with $\mu = 0$. Then for
	\begin{equation*}
		\gamma = \min\left\{\frac{1}{4L}, \frac{1}{8\sqrt{e}(\tau-1)L}, \sqrt{\frac{nR_0^2}{\sigma^2 K}}, \sqrt[3]{\frac{R_0^2}{4eL(\tau-1)\sigma^2K}}\right\},
	\end{equation*}
	where $R_0 = \|x^0 - x^*\|$, we have that
	\begin{equation}
		\EE\left[f(\overline{x}^K)-f(x^*)\right] = \cO\left(\frac{\tau LR_0^2}{K} + \sqrt{\frac{R_0^2\sigma^2}{nK}} + \frac{\sqrt[3]{LR_0^4(\tau-1)\sigma^2}}{K^{\nicefrac{2}{3}}} \right).\notag
	\end{equation}
	That is, to achieve $\EE\left[f(\overline{x}^K)-f(x^*)\right] \le \varepsilon$ in this case {\tt S*-Local-SGD} requires
	\begin{equation*}
		\cO\left(\frac{\tau LR_0^2}{\varepsilon} + \frac{R_0^2\sigma^2}{n\varepsilon^2} + \frac{R_0^2\sqrt{L(\tau-1)\sigma^2}}{\varepsilon^{\nicefrac{3}{2}}}\right)
	\end{equation*}
	iterations/oracle calls per node and $\tau$ times less communication rounds.
\end{corollary}

\subsection{{\tt SS-Local-SGD}}

\subsubsection{Uniformly Bounded Variance}\label{sec:loopless_local_svrg}
In this section we consider the same settings as in Section~\ref{sec:sgd_bounded_var}

\begin{algorithm}[h]
   \caption{Stochastically Shifted {\tt Local-SGD} ({\tt {\tt SS-Local-SGD}})}\label{alg:l_local_svrg}
\begin{algorithmic}[1]
   \Require learning rate $\gamma>0$, initial vector $x^0 \in \R^d$, probability of communication $p\in(0,1]$, probability of the shift's update $q\in(0,1]$, batchsize $r$ for computing shifts
   \State $y^0 = x^0$
   \State For $i\in[n]$ compute $r$ independent samples $\nabla f_{\oxi_{i,1}^0}(y^0), \nabla f_{\oxi_{i,2}^0}(y^0), \ldots, \nabla f_{\oxi_{i,r}^0}(y^0)$, set $\nabla f_{\oxi_i^0}(y^0) = \frac{1}{r}\sum_{j=1}^r\nabla f_{\oxi_{i,j}^0}(y^0)$ and $\nabla f_{\oxi^0}(y^0) = \frac{1}{n}\sum_{i=1}^n\nabla f_{\oxi_i^0}(y^0)$
	\For{$k=0,1,\dotsc$}
       	\For{$i=1,\dotsc,n$ in parallel}
       		\State Sample $\nabla f_{\xi_i^k}(x_i^k)$ independently from other nodes
            \State $g_i^k = \nabla f_{\xi_i^k}(x_i^k) - \nabla f_{\txi_i^k}(y^k) + \nabla f_{\txi^k}(y^k)$, where $\nabla f_{\oxi_i^k}(y^k) = \frac{1}{r}\sum_{j=1}^r\nabla f_{\oxi_{i,j}^k}(y^k)$ and $\nabla f_{\oxi^k}(y^k) = \frac{1}{n}\sum_{i=1}^n\nabla f_{\oxi_i^k}(y^k)$
            \State $x_i^{k+1} = \begin{cases}x^{k+1},&\text{w.p. } p,\\
            x_i^k - \gamma g_i^k,& \text{w.p. } 1 - p, \end{cases}$ where $x^{k+1} = \frac{1}{n}\sum\limits_{i=1}^n(x_i^k - \gamma g_i^k)$
            \State $y^{k+1} = \begin{cases}x^k,&\text{w.p. } q,\\
            y^k,& \text{w.p. } 1 - q, \end{cases}$ and for all $i\in[n]$, $j\in[r]\;$ $\oxi_{i,j}^{k+1}$ is $\begin{cases}\text{a fresh sample},&\text{if } y^{k+1}\neq y^k,\\
            \text{equal to } \oxi_{i,j}^k,& \text{otherwise}. \end{cases}$
        \EndFor
   \EndFor
\end{algorithmic}
\end{algorithm}

The main algorithm in this section is Stochastically Shifted {\tt Local-SGD} ({\tt SS-Local-SVRG}, see Algorithm~\ref{alg:l_local_svrg}). We notice that the updates for $x_i^{k+1}$ and $y^{k+1}$ can be dependent, e.g., one can take $p = q$ and update $y^{k+1}$ as $x^k$ every time $x_i^{k+1}$ is updated by $x^{k+1}$. Moreover, with probability $q$ line $8$ implies a round of communication and computation of new stochastic gradient by each worker.

We emphasize that in expectation $y^k$ is updated only once per $\left\lceil\nicefrac{1}{q}\right\rceil$ iterations. Therefore, if $r = O\left(\nicefrac{1}{q}\right)$ and $q \le p$, then up to a constant numerical factor the overall expected number of oracle calls and communication rounds are the same as for {\tt Local-SGD} with either the same probability $p$ of communication or with constant local loop length $\tau = \left\lceil\nicefrac{1}{p}\right\rceil$.

Finally, we notice that due to independence of $\oxi_{i,1}^k,\oxi_{i,2}^k,\ldots,\oxi_{i,r}^k$ we have
\begin{equation}
	\EE\|\nabla f_{\oxi_i^k}(y^k)-\nabla f_i(y^k)\|^2 \overset{\eqref{eq:bounded_variance}}{\le} \frac{D_{1,i}}{r}. \label{eq:stoch_shifts_variance}
\end{equation}

\begin{lemma}\label{lem:loopless_local_svrg_second_moment}
	Let $f_i$ be convex and $L$-smooth for all $i\in[n]$. Then for all $k\ge 0$
	\begin{eqnarray}
		\frac{1}{n}\sum\limits_{i=1}^n \EE_k\left[g_i^k\right] &=& \frac{1}{n}\sum\limits_{i=1}^n\nabla f_i(x_i^k), \label{eq:unbiasedness_loopless_local_svrg}\\
		\frac{1}{n}\sum\limits_{i=1}^n \EE\left[\|\bar{g}_i^k\|^2\right] &\le& 8L\EE\left[f(x^k)-f(x^*)\right] + 2\EE[\sigma_k^2] + 4L^2\EE[V_k] + \frac{2\sigma^2}{r},\label{eq:second_moment_loopless_local_svrg}\\
		\frac{1}{n}\sum\limits_{i=1}^n \EE\left[\|g_i^k-\bar{g}_i^k\|^2\right] &\le& \sigma^2,\label{eq:variance_loopless_local_svrg}\\
		\EE\left[\left\|\frac{1}{n}\sum\limits_{i=1}^ng_i^k\right\|^2\right] &\le& 4L\EE\left[f(x^k) - f(x^*)\right] + 2L^2\EE\left[V_k\right] + \frac{\sigma^2}{n},\label{eq:second_moment_loopless_local_svrg_2}
	\end{eqnarray}
	where $\sigma_k^2 \eqdef \frac{1}{n}\sum\limits_{i=1}^n\left\|\nabla f_i(y^k) - \nabla f_i(x^*)\right\|^2$ and $\sigma^2 \eqdef \frac{1}{n}\sum_{i=1}^nD_{1,i}$.
\end{lemma}
\begin{proof}
	We start with unbiasedness:
	\begin{eqnarray*}
		\frac{1}{n}\sum\limits_{i=1}^n \EE_k\left[g_i^k\right] &=& \frac{1}{n}\sum\limits_{i=1}^n \EE_k\left[\nabla f_{\xi_i^k}(x_i^k) - \nabla f_{\oxi_i^k}(y^k) + \nabla f_{\oxi^k}(y^k)\right]\\
		&=& \frac{1}{n}\sum\limits_{i=1}^n\EE_k\left[\nabla f_{\xi_i^k}(x_i^k)\right] + \EE_k\left[\nabla f_{\oxi^k}(y^k) - \frac{1}{n}\sum\limits_{i=1}^n \nabla f_{\oxi_i^k}(y^k)\right] = \frac{1}{n}\sum\limits_{i=1}^n\nabla f_i(x_i^k).
	\end{eqnarray*}
	Using this we get
	\begin{eqnarray*}
		\frac{1}{n}\sum\limits_{i=1}^n\EE\left[\|\bar{g}_i^k\|^2\right] &\overset{\eqref{eq:a_b_norm_squared}}{\le}& \frac{2}{n}\sum\limits_{i=1}^n\EE\left[\|\nabla f_i(x_i^k) - \nabla f_i(x^*)\|^2\right]\\
		&&\quad + \frac{2}{n}\sum\limits_{i=1}^n\EE\left[\left\|\nabla f_{\oxi_i^k}(y^k) - \nabla f_i(x^*) - \left(\nabla f_{\oxi^k}(y^k) - \nabla f(x^*)\right)\right\|^2\right]\\
		&\overset{\eqref{eq:L_smoothness_cor},\eqref{eq:variance_decomposition}}{\le}& \frac{4L}{n}\sum\limits_{i=1}^n\EE\left[D_{f_i}(x_i^k,x^*)\right] + \frac{2}{n}\sum\limits_{i=1}^n\EE\left[\left\|\nabla f_{\oxi_i^k}(y^k) - \nabla f_i(x^*)\right\|^2\right]\\
		&\overset{\eqref{eq:poiouhnkj},\eqref{eq:variance_decomposition}}{\le}& 8L\EE\left[f(x^k)-f(x^*)\right] + 4L^2\EE[V_k] + \frac{2}{n}\sum\limits_{i=1}^n\EE\left[\left\|\nabla f_{i}(y^k) - \nabla f_i(x^*)\right\|^2\right]\\
		&&\quad + \frac{2}{n}\sum\limits_{i=1}^n\EE\left[\left\|\nabla f_{\oxi_i^k}(y^k) - \nabla f_i(y^k)\right\|^2\right]\\
		&\overset{\eqref{eq:stoch_shifts_variance}}{\le}& 8L\EE\left[f(x^k)-f(x^*)\right] + 2\EE[\sigma_k^2] + 4L^2\EE[V_k] + \frac{2\sigma^2}{r}
	\end{eqnarray*}
	and
	\begin{equation*}
		\frac{1}{n}\sum\limits_{i=1}^n \EE\left[\|g_i^k-\bar{g}_i^k\|^2\right] = \frac{1}{n}\sum\limits_{i=1}^n \EE\left[\|\nabla f_{\xi_i^k}(x_i^k)-\nabla f_i(x_i^k)\|^2\right] \overset{\eqref{eq:bounded_variance}}{\le} \sigma^2.
	\end{equation*}
	Finally, we use independence of $\nabla f_{\xi_1^k}(x_1^k),\ldots,\nabla f_{\xi_n^k}(x_n^k)$ and derive
	\begin{eqnarray*}
		\EE\left[\left\|\frac{1}{n}\sum\limits_{i=1}^n g_i^k\right\|^2\right] &=& \EE\left[\left\|\frac{1}{n}\sum\limits_{i=1}^n\nabla f_{\xi_i^k}(x_i^k)\right\|^2\right]\\
		&\overset{\eqref{eq:variance_decomposition}}{=}& \EE\left[\left\|\frac{1}{n}\sum\limits_{i=1}^n \nabla f_i(x_i^k)\right\|^2\right] + \EE\left[\left\|\frac{1}{n}\sum\limits_{i=1}^n\left(\nabla f_{\xi_i^k}(x_i^k) - \nabla f_i(x_i^k)\right)\right\|^2\right]\\
		&\overset{\eqref{eq:vdgasvgda}}{\le}& 4L\EE\left[f(x^k)-f(x^*)\right] + 2L^2\EE[V_k] + \frac{1}{n^2}\sum\limits_{i=1}^n\EE\left[\|\nabla f_{\xi_i^k}(x_i^k) - \nabla f_i(x_i^k)\|^2\right]\\
		&\overset{\eqref{eq:bounded_variance}}{\le}& 4L\EE\left[f(x^k) - f(x^*)\right] + 2L^2\EE\left[V_k\right] + \frac{\sigma^2}{n}
	\end{eqnarray*}
	which finishes the proof.
\end{proof}

\begin{lemma}\label{lem:loopless_local_svrg_sigma_k_bound}
	Let $f_i$ be convex and $L$-smooth for all $i\in[n]$. Then for all $k\ge 0$
	\begin{eqnarray}
		\EE\left[\sigma_{k+1}^2\right] &\le& (1-q)\EE\left[\sigma_k^2\right] + 2Lq\EE\left[f(x^k) - f(x^*)\right]\label{eq:loopless_local_svrg_sigma_k_bound}
	\end{eqnarray}
	where $\sigma_k^2 \eqdef \frac{1}{n}\sum\limits_{i=1}^n\left\|\nabla f_i(y^k) - \nabla f_i(x^*)\right\|^2$.
\end{lemma}
\begin{proof}
	By definition of $y^{k+1}$ we have
	\begin{eqnarray*}
		\EE\left[\sigma_{k+1}^2\mid x_1^k,\ldots, x_n^k\right] &=& \frac{1-q}{n}\sum\limits_{i=1}^n\|\nabla f_i(y^k) - \nabla f_i(x^*)\|^2 + \frac{q}{n}\sum\limits_{i=1}^n\|\nabla f_i(x^k) - \nabla f_i(x^*)\|^2\\
		&\overset{\eqref{eq:L_smoothness_cor}}{\le}& (1-q)\sigma_k^2 + 2Lq(f(x^k) - f(x^*)).
	\end{eqnarray*}
	Taking the full mathematical expectation on both sides of previous inequality and using the tower property \eqref{eq:tower_property} we get the result.
\end{proof}

Using Corollary~\ref{cor:rand_loop} we obtain the following theorem.
\begin{theorem}\label{thm:ss_local_sgd}
	Assume that $f_i(x)$ is $\mu$-strongly convex and $L$-smooth for every $i\in[n]$. Then {\tt SS-Local-SGD} satisfies Assumption~\ref{ass:hetero_second_moment} with
	\begin{gather*}
		\tA = 4L,\quad \hA = 0,\quad \tB = 2,\quad \hB = 0,\quad \tF = 4L^2,\quad \hF = 0, \quad \tD_1 = \frac{2\sigma^2}{r},\quad \hD_1 = \sigma^2,\quad \sigma^2 = \frac{1}{n}\sum\limits_{i=1}^nD_{1,i},\\
		A' = 2L,\quad B' = 0,\quad F' = 2L^2, \quad D_1' = \frac{\sigma^2}{n},\\
		\sigma_k^2 = \frac{1}{n}\sum\limits_{i=1}^n\left\|\nabla f_i(y^k) - \nabla f_i(x^*)\right\|^2,\quad \rho = q,\quad C = Lq,\quad G = 0,\quad D_2 = 0,\\
		H = \frac{128(1-p)(2+p)(2+q)\gamma^2}{3p^2q},\quad D_3 = \frac{8(1-p)}{p^2}\left(\frac{2(p+2)\sigma^2}{r} + p\sigma^2\right)
	\end{gather*}
	under assumption that
	\begin{eqnarray*}
		\gamma &\le& \min\left\{\frac{1}{4L}, \frac{p\sqrt{3}}{32L\sqrt{2(1-p)(2+p)\left(1+\nicefrac{1}{(1-q)}\right)}}\right\}.
	\end{eqnarray*}
	Moreover, for $\mu > 0$ we have
	\begin{eqnarray}
		\EE\left[f(\overline{x}^K) - f(x^*)\right] &\le& \left(1 - \min\left\{\gamma\mu,\frac{q}{4}\right\}\right)^K\frac{\Phi^0}{\gamma} +2\gamma\left(\frac{\sigma^2}{n} + \gamma \frac{16L(1-p)}{p^2}\left(\frac{2(p+2)\sigma^2}{r} + p\sigma^2\right)\right) \notag
	\end{eqnarray}
	and when $\mu = 0$ we have
	\begin{eqnarray}
		\EE\left[f(\overline{x}^K) - f(x^*)\right] &\le& \frac{\Phi^0}{\gamma K} +2\gamma\left(\frac{\sigma^2}{n} + \gamma \frac{16L(1-p)}{p^2}\left(\frac{2(p+2)\sigma^2}{r} + p\sigma^2\right)\right) \notag
	\end{eqnarray}
	where $\Phi^0 = 2\|x^0-x^*\|^2+ \frac{512L(1-p)(2+p)(2+q)\gamma^3\sigma_0^2}{3p^2q}$.
\end{theorem}

The theorem above together with Lemma~\ref{lem:lemma2_stich} implies the following result.
\begin{corollary}\label{cor:ss_local_sgd_str_cvx}
	Let assumptions of Theorem~\ref{thm:ss_local_sgd} hold with $\mu > 0$. Then for 
	\begin{eqnarray*}
		\gamma_0 &=& \min\left\{\frac{1}{4L}, \frac{p\sqrt{3}}{32L\sqrt{2(1-p)(2+p)\left(1+\nicefrac{1}{(1-q)}\right)}}\right\},\\
		\widetilde{\Phi}^0 &=& 2\|x^0-x^*\|^2+ \frac{512L(1-p)(2+p)(2+q)\gamma_0^3\sigma_0^2}{3p^2q},\quad q = p,\\
		\gamma &=& \min\left\{\gamma_0,\frac{\ln\left(\max\left\{2, \min\left\{\nicefrac{n\widetilde{\Phi}^0\mu^2K^2}{2\sigma^2 },\nicefrac{p\widetilde{\Phi}^0\mu^3K^3}{32L(1-p)(3p+4)\sigma^2}\right\}\right\}\right)}{\mu K}\right\},\quad r = \left\lceil\frac{1}{p}\right\rceil,
	\end{eqnarray*}
	for all $K$ such that 
	\begin{eqnarray*}
		\text{either} && \frac{\ln\left(\max\left\{2, \min\left\{\nicefrac{n\widetilde{\Phi}^0\mu^2K^2}{2\sigma^2 },\nicefrac{p\widetilde{\Phi}^0\mu^3K^3}{32L(1-p)(3p+4)\sigma^2}\right\}\right\}\right)}{ K} \le p\\
		\text{or} && \gamma_0 \le \frac{\ln\left(\max\left\{2, \min\left\{\nicefrac{n\widetilde{\Phi}^0\mu^2K^2}{2\sigma^2 },\nicefrac{p\widetilde{\Phi}^0\mu^3K^3}{32L(1-p)(3p+4)\sigma^2}\right\}\right\}\right)}{\mu K}
	\end{eqnarray*}
	we have that $\EE\left[f(\overline{x}^K)-f(x^*)\right]$ is of the order
	\begin{equation}
		 \widetilde\cO\left(\frac{\widetilde{\Phi}^0}{\gamma_0}\exp\left(- \min\left\{\frac{1}{p}, \gamma_0\mu\right\} K\right) + \frac{\sigma^2}{n\mu K}  + \frac{L(1-p)\sigma^2}{p\mu^2 K^2}\right).\notag
	\end{equation}
	That is, to achieve $\EE\left[f(\overline{x}^K)-f(x^*)\right] \le \varepsilon$ in this case {\tt SS-Local-SGD} requires
	\begin{equation*}
		\widetilde{\cO}\left(\frac{L}{p\mu} + \frac{\sigma^2}{n\mu\varepsilon} + \sqrt{\frac{L(1-p)\sigma^2}{p\mu^2\varepsilon}}\right)
	\end{equation*}
	iterations/oracle calls per node (in expectation) and $\nicefrac{1}{p}$ times less communication rounds.	
\end{corollary}

Combining Theorem~\ref{thm:ss_local_sgd} and Lemma~\ref{lem:lemma_technical_cvx} we derive the following result for the convergence of {\tt SS-Local-SGD} in the case when $\mu = 0$.
\begin{corollary}
	\label{cor:ss_local_sgd_cvx}
	Let assumptions of Theorem~\ref{thm:ss_local_sgd} hold with $\mu = 0$. Then for $q = p,$ $r = \lceil\nicefrac{1}{p}\rceil$ and
	\begin{eqnarray*}
		\gamma_0 &=& \min\left\{\frac{1}{4L}, \frac{p\sqrt{3}}{32L\sqrt{2(1-p)(2+p)\left(1+\nicefrac{1}{(1-q)}\right)}}\right\},\\	
		\gamma &=& \min\left\{\gamma_0, \sqrt[3]{\frac{3p^3R_0^2}{256L(1-p)(2+p)^2\sigma_0^2}}, \sqrt{\frac{nR_0^2}{\sigma^2 K}}, \sqrt[3]{\frac{pR_0^2}{16L(1-p)(3p+4)\sigma^2 K}}\right\},
	\end{eqnarray*}
	where $R_0 = \|x^0 - x^*\|$, we have that $\EE\left[f(\overline{x}^K)-f(x^*)\right]$ is of the order
	\begin{eqnarray*}
		\cO\left(\frac{LR_0^2 + \sqrt[3]{L(1-p)\sigma_0^2R_0^4}}{pK} + \sqrt{\frac{\sigma^2R_0^2}{n K}} + \frac{\sqrt[3]{LR_0^4(1-p)\sigma^2}}{p^{\nicefrac{1}{3}}K^{\nicefrac{2}{3}}} \right).
	\end{eqnarray*}
	That is, to achieve $\EE\left[f(\overline{x}^K)-f(x^*)\right] \le \varepsilon$ in this case {\tt SS-Local-SGD} requires
	\begin{eqnarray*}
		\cO\left(\frac{LR_0^2 + \sqrt[3]{L(1-p)\sigma_0^2R_0^4}}{p\varepsilon} + \frac{\sigma^2 R_0^2}{n\varepsilon^2} + \frac{R_0^2\sqrt{L(1-p)\sigma^2}}{p^{\nicefrac{1}{2}}\varepsilon^{\nicefrac{3}{2}}}\right)
	\end{eqnarray*}
	iterations/oracle calls per node (in expectation) and $\nicefrac{1}{p}$ times less communication rounds.
\end{corollary}

\begin{remark}
	To get the rate from Tbl.~\ref{tbl:special_cases_weakly_convex} it remains to apply the following inequality:
	\begin{eqnarray*}
		\sigma_0^2 = \frac{1}{n}\sum\limits_{i=1}^n\|\nabla f_{i}(x^0)-\nabla f_{i}(x^*)\|^2 \overset{\eqref{eq:L_smoothness}}{\le} L^2 \|x^0-x^*\|^2.
	\end{eqnarray*}
\end{remark}

\subsubsection{Expected Smoothness and Arbitrary Sampling}\label{sec:loopless_local_svrg_es}
In this section we consider the same method {\tt SS-Local-SGD}, but without assumption that the stochastic gradient has a uniformly bounded variance. Instead of this we consider the same setup as in Section~\ref{sec:sgd_es}, i.e.\ we assume that each worker $i\in [n]$ at any point $x\in\R^d$ has an access to the unbiased estimator $\nabla f_{\xi_i}(x)$ of $\nabla f_i(x)$ satisfying Assumption~\ref{ass:expected_smoothness}.

\begin{lemma}\label{lem:loopless_local_svrg_es_second_moment}
	Let $f_i$ be convex and $L$-smooth for all $i\in[n]$. Let Assumption~\ref{ass:expected_smoothness} holds. Then for all $k\ge 0$
	\begin{eqnarray}
		\frac{1}{n}\sum\limits_{i=1}^n \EE_k\left[g_i^k\right] &=& \frac{1}{n}\sum\limits_{i=1}^n\nabla f_i(x_i^k), \label{eq:unbiasedness_loopless_local_svrg_es}\\
		\frac{1}{n}\sum\limits_{i=1}^n \EE\left[\|\bar{g}_i^k\|^2\right] &\le& 8L\EE\left[f(x^k)-f(x^*)\right] + 2\EE[\sigma_k^2] + 4L^2\EE[V_k],\label{eq:second_moment_loopless_local_svrg_es}\\
		\frac{1}{n}\sum\limits_{i=1}^n \EE\left[\|g_i^k-\bar{g}_i^k\|^2\right] &\le& 8\cL\EE\left[f(x^k) - f(x^*)\right] + 4\cL L\EE[V_k] + 2\sigma_*^2,\label{eq:variance_loopless_local_svrg_es}\\
		\EE\left[\left\|\frac{1}{n}\sum\limits_{i=1}^ng_i^k\right\|^2\right] &\le& 4\left(\frac{2\cL}{n}+L\right)\EE\left[f(x^k)-f(x^*)\right] + 2L\left(\frac{2\cL}{n} + L\right)\EE[V_k] + \frac{2\sigma_*^2}{n},\label{eq:second_moment_loopless_local_svrg_es_2}
	\end{eqnarray}
	where $\sigma_k^2 \eqdef \frac{1}{n}\sum\limits_{i=1}^n\left\|\nabla f_{\oxi_i^k}(y^k) - \nabla f_i(x^*)\right\|^2$ and $\sigma_*^2 \eqdef \frac{1}{n}\sum_{i=1}^n\EE_{\xi_i}\|\nabla f_{\xi_i}(x^*) - \nabla f_i(x^*)\|^2$.
\end{lemma}
\begin{proof}
	First of all, \eqref{eq:unbiasedness_loopless_local_svrg_es} follows from \eqref{eq:unbiasedness_loopless_local_svrg}. Next, using $\bar{g}_i^k = \nabla f_i(x_i^k) - \nabla f_{\oxi_i^k}(y^k) + \nabla f_{\oxi^k}(y^k)$ we get
	\begin{eqnarray*}
		\frac{1}{n}\sum\limits_{i=1}^n\EE\left[\|\bar{g}_i^k\|^2\right] &\overset{\eqref{eq:a_b_norm_squared}}{\le}& \frac{2}{n}\sum\limits_{i=1}^n\EE\left[\|\nabla f_{i}(x_i^k) - \nabla f_{i}(x^*)\|^2\right]\\
		&&\quad + \frac{2}{n}\sum\limits_{i=1}^n\EE\left[\left\|\nabla f_{\oxi_i^k}(y^k) - \nabla f_i(x^*) - (\nabla f_{\oxi^k}(y^k) - \nabla f(x^*))\right\|^2\right]\\
		&\overset{\eqref{eq:L_smoothness_cor},\eqref{eq:variance_decomposition}}{\le}& \frac{4L}{n}\sum\limits_{i=1}^n\EE\left[D_{f_i}(x_i^k,x^*)\right] + \frac{2}{n}\sum\limits_{i=1}^n\EE\left[\|\nabla f_{\oxi_i^k}(y^k)- \nabla f_i(x^*)\|^2\right]\\
		&\overset{\eqref{eq:poiouhnkj}}{\le}& 8L\EE\left[f(x^k)-f(x^*)\right] + 2\EE[\sigma_k^2] + 4L^2\EE[V_k]
	\end{eqnarray*}
	and	
	\begin{eqnarray}
		\frac{1}{n}\sum\limits_{i=1}^n\EE\left[\|g_i^k-\bar{g}_i^k\|^2\right] &=& \frac{1}{n}\sum\limits_{i=1}^n\EE\left[\|\nabla f_{\xi_i^k}(x_i^k) - \nabla f_i(x_i^k)\|^2\right]\notag\\
		&\overset{\eqref{eq:variance_decomposition}}{\le}& \frac{1}{n}\sum\limits_{i=1}^n\EE\left[\|\nabla f_{\xi_i^k}(x_i^k) - \nabla f_i(x^*)\|^2\right]\notag\\
		&\overset{\eqref{eq:a_b_norm_squared}}{\le}& \frac{2}{n}\sum\limits_{i=1}^n\EE\left[\|\nabla f_{\xi_i^k}(x_i^k) - \nabla f_{\xi_i^k}(x^*)\|^2\right] + \frac{2}{n}\sum\limits_{i=1}^n\EE\left[\|\nabla f_{\xi_i^k}(x^*) - \nabla f_{i}(x^*)\|^2\right]\notag \\
		&\overset{\eqref{eq:expected_smoothness_1}}{\le}& \frac{4\cL}{n}\sum\limits_{i=1}^n\EE\left[D_{f_i}(x_i^k,x^*)\right] + 2\sigma_*^2\notag\\
		&\overset{\eqref{eq:poiouhnkj}}{\le}& 8\cL\EE\left[f(x^k) - f(x^*)\right] + 4\cL L\EE[V_k] + 2\sigma_*^2. \label{eq:bshjbdhsbdhbucsb}
	\end{eqnarray}
	Finally, we use independence of $\xi_1^k,\ldots,\xi_n^k$ and derive
	\begin{eqnarray*}
		\EE\left[\left\|\frac{1}{n}\sum\limits_{i=1}^n g_i^k\right\|^2\right] &=& \EE\left[\left\|\frac{1}{n}\sum\limits_{i=1}^n \nabla f_{\xi_i^k}(x_i^k)\right\|^2\right]\\
		&\overset{\eqref{eq:tower_property},\eqref{eq:variance_decomposition}}{=}& \EE\left[\left\|\frac{1}{n}\sum\limits_{i=1}^n (\nabla f_{\xi_i^k}(x_i^k)-\nabla f_i(x_i^k))\right\|^2\right] + \EE\left[\left\|\frac{1}{n}\sum\limits_{i=1}^n \nabla f_{i}(x_i^k)\right\|^2\right]\\
		&=& \frac{1}{n^2}\sum\limits_{i=1}^n\EE\left[\|\nabla f_{\xi_i^k}(x_i^k) - \nabla f_i(x_i^k)\|^2\right] + \EE\left[\left\|\frac{1}{n}\sum\limits_{i=1}^n \nabla f_{i}(x_i^k)\right\|^2\right]\\
		&\overset{\eqref{eq:bshjbdhsbdhbucsb},\eqref{eq:vdgasvgda}}{\le}& 4\left(\frac{2\cL}{n}+L\right)\EE\left[f(x^k)-f(x^*)\right] + 2L\left(\frac{2\cL}{n} + L\right)\EE[V_k] + \frac{2\sigma_*^2}{n}
	\end{eqnarray*}
	which finishes the proof.
\end{proof}

\begin{lemma}\label{lem:loopless_local_svrg_es_sigma_k_bound}
	Let $f_i$ be convex and $L$-smooth for all $i\in[n]$ and Assumption~\ref{ass:expected_smoothness} holds. Then for all $k\ge 0$
	\begin{eqnarray}
		\EE\left[\sigma_{k+1}^2\right] &\le& (1-q)\EE\left[\sigma_k^2\right] + 2q\left(\frac{2\cL}{r} + L\right)\EE\left[f(x^k) - f(x^*)\right] + \frac{2q\sigma_*^2}{r}\label{eq:loopless_local_svrg_es_sigma_k_bound}
	\end{eqnarray}
	where $\sigma_k^2 \eqdef \frac{1}{n}\sum\limits_{i=1}^n\left\|\nabla f_{\oxi_i^k}(y^k) - \nabla f_i(x^*)\right\|^2$ and $\sigma_*^2 \eqdef \frac{1}{n}\sum_{i=1}^n\EE_{\xi_i}\|\nabla f_{\xi_i}(x^*) - \nabla f_i(x^*)\|^2$.
\end{lemma}
\begin{proof}
	By definition of $y^{k+1}$ we have
	\begin{eqnarray*}
		\EE\left[\sigma_{k+1}^2\mid x_1^k,\ldots, x_n^k\right] &=& \frac{1-q}{n}\sum\limits_{i=1}^n\|\nabla f_{\oxi_i^k}(y^k) - \nabla f_i(x^*)\|^2\\
		&&\quad + \frac{q}{n}\sum\limits_{i=1}^n\EE_{\oxi_i^{k+1}}\left[\|\nabla f_{\oxi_i^{k+1}}(x^k) - \nabla f_i(x^*)\|^2\right]\\
		&\overset{\eqref{eq:variance_decomposition}}{=}& (1-q)\sigma_k^2 + \frac{q}{n}\sum\limits_{i=1}^n\|\nabla f_{i}(x^k) - \nabla f_i(x^*)\|^2\\
		&&\quad + \frac{q}{n}\sum\limits_{i=1}^n\EE_{\oxi_i^{k+1}}\left[\|\nabla f_{\oxi_i^{k+1}}(x^k) - \nabla f_i(x^k)\|^2\right].
	\end{eqnarray*}
	Next, we use independence of $\oxi_{i,1}^{k+1}, \oxi_{i,2}^{k+1},\ldots, \oxi_{i,r}^{k+1}$ for all $i\in [n]$ and derive
	\begin{eqnarray*}
		\EE\left[\sigma_{k+1}^2\mid x_1^k,\ldots, x_n^k\right] &=& (1-q)\sigma_k^2 + \frac{q}{n}\sum\limits_{i=1}^n\|\nabla f_{i}(x^k) - \nabla f_i(x^*)\|^2\\
		&&\quad + \frac{q}{nr^2}\sum\limits_{i=1}^n\sum\limits_{j=1}^r\EE_{\oxi_{i,j}^{k+1}}\left[\|\nabla f_{\oxi_{i,j}^{k+1}}(x^k) - \nabla f_i(x^k)\|^2\right]\\
		&\overset{\eqref{eq:L_smoothness_cor},\eqref{eq:variance_decomposition}}{\le}& (1-q)\sigma_k^2 + 2Lq\left(f(x^k)-f(x^*)\right)\\
		&&\quad + \frac{q}{nr^2}\sum\limits_{i=1}^n\sum\limits_{j=1}^r\EE_{\oxi_{i,j}^{k+1}}\left[\|\nabla f_{\oxi_{i,j}^{k+1}}(x^k) - \nabla f_i(x^*)\|^2\right]\\
		&\overset{\eqref{eq:a_b_norm_squared}}{\le}& (1-q)\sigma_k^2 + 2Lq\left(f(x^k)-f(x^*)\right)\\
		&&\quad + \frac{2q}{nr^2}\sum\limits_{i=1}^n\sum\limits_{j=1}^r\EE_{\oxi_{i,j}^{k+1}}\left[\|\nabla f_{\oxi_{i,j}^{k+1}}(x^k) - \nabla f_{\oxi_{i,j}^{k+1}}(x^*)\|^2\right]\\
		&&\quad + \frac{2q}{nr^2}\sum\limits_{i=1}^n\sum\limits_{j=1}^r\EE_{\oxi_{i,j}^{k+1}}\left[\|\nabla f_{\oxi_{i,j}^{k+1}}(x^*) - \nabla f_{i}(x^*)\|^2\right]\\
		&\overset{\eqref{eq:expected_smoothness_1}}{\le}& (1-q)\sigma_k^2 + 2q\left(\frac{2\cL}{r} + L\right)\left(f(x^k) - f(x^*)\right) + \frac{2q\sigma_*^2}{r}.
	\end{eqnarray*}
	Taking the full mathematical expectation on both sides of previous inequality and using the tower property \eqref{eq:tower_property} we get the result.
\end{proof}

Using Corollary~\ref{cor:rand_loop} we obtain the following theorem.
\begin{theorem}\label{thm:ss_local_sgd_es}
	Assume that $f_i(x)$ is $\mu$-strongly convex and $L$-smooth for every $i\in[n]$. Let Assumption~\ref{ass:expected_smoothness} holds. Then {\tt SS-Local-SGD} satisfies Assumption~\ref{ass:hetero_second_moment} with
	\begin{gather*}
		\tA = 4L,\quad \hA = 4\cL,\quad \tB = 2,\quad \hB = 0,\quad \tF = 4L^2,\quad \hF = 4\cL L, \quad \tD_1 = 0,\\
		\hD_1 = 2\sigma_*^2,\quad \sigma_*^2 = \frac{1}{n}\sum\limits_{i=1}^n\EE_{\xi_i}\|\nabla f_{\xi_i}(x^*) - \nabla f_i(x^*)\|^2,\quad A' = 2\left(\frac{2\cL}{n} + L\right),\quad B' = 0,\quad F' = 2L\left(\frac{2\cL}{n} + L\right), \\
		D_1' = \frac{2\sigma_*^2}{n},\quad \sigma_k^2 = \frac{1}{n}\sum\limits_{i=1}^n\left\|\nabla f_{\oxi_i^k}(y^k) - \nabla f_i(x^*)\right\|^2,\quad \rho = q,\quad C = q\left(\frac{2\cL}{r} + L\right),\quad G = 0,\quad D_2 = \frac{2q\sigma_*^2}{r},\\
		H = \frac{128(1-p)(2+p)(2+q)\gamma^2}{3p^2q},\quad D_3 = \frac{8(1-p)}{p^2}\left(2p\sigma_*^2+\frac{32(2+p)\sigma_*^2}{3r}\right)
	\end{gather*}
	under assumption that
	\begin{eqnarray*}
		\gamma &\le& \min\left\{\frac{1}{4\left(\frac{2\cL}{n}+L\right)}, \frac{p\sqrt{3}}{32\sqrt{2L(1-p)\left((2+p)L + p\cL + \frac{(2+p)\left(\nicefrac{2\cL}{r} + L\right)}{(1-q)}\right)}}\right\}.
	\end{eqnarray*}
	Moreover, for $\mu > 0$ we have
	\begin{eqnarray}
		\EE\left[f(\overline{x}^K) - f(x^*)\right] &\le& \left(1 - \min\left\{\gamma\mu,\frac{q}{4}\right\}\right)^K\frac{\Phi^0}{\gamma} +2\gamma\left(\frac{2\sigma_*^2}{n} + \gamma \frac{16L(1-p)}{p^2}\left(2p\sigma_*^2+\frac{32(2+p)\sigma_*^2}{3r}\right)\right) \notag
	\end{eqnarray}
	and when $\mu = 0$ we have
	\begin{eqnarray}
		\EE\left[f(\overline{x}^K) - f(x^*)\right] &\le& \frac{\Phi^0}{\gamma K} +2\gamma\left(\frac{2\sigma_*^2}{n} + \gamma \frac{16L(1-p)}{p^2}\left(2p\sigma_*^2+\frac{32(2+p)\sigma_*^2}{3r}\right)\right) \notag
	\end{eqnarray}
	where $\Phi^0 = 2\|x^0-x^*\|^2+ \frac{512L(1-p)(2+p)(2+q)\gamma^3\EE[\sigma_0^2]}{3p^2q}$.
\end{theorem}

The theorem above together with Lemma~\ref{lem:lemma2_stich} implies the following result.
\begin{corollary}\label{cor:ss_local_sgd_str_cvx_es}
	Let assumptions of Theorem~\ref{thm:ss_local_sgd_es} hold with $\mu > 0$. Then for 
	\begin{eqnarray*}
		\gamma_0 &=& \min\left\{\frac{1}{4\left(\frac{2\cL}{n}+L\right)}, \frac{p\sqrt{3}}{32\sqrt{2L(1-p)\left((2+p)L + p\cL + \frac{(2+p)\left(\nicefrac{2\cL}{r} + L\right)}{(1-q)}\right)}}\right\},\\
		\widetilde{\Phi}^0 &=& 2\|x^0-x^*\|^2+ \frac{512L(1-p)(2+p)(2+q)\gamma_0^3\EE[\sigma_0^2]}{p^2q},\quad q = p,\\
		\gamma &=& \min\left\{\gamma_0,\frac{\ln\left(\max\left\{2, \min\left\{\nicefrac{n\widetilde{\Phi}^0\mu^2K^2}{4\sigma_*^2 },\nicefrac{p\widetilde{\Phi}^0\mu^3K^3}{64L(1-p)(1+\nicefrac{32(2+p)}{3})\sigma_*^2}\right\}\right\}\right)}{\mu K}\right\},\quad r = \left\lceil\frac{1}{p}\right\rceil,
	\end{eqnarray*}
	for all $K$ such that 
	\begin{eqnarray*}
		\text{either} && \frac{\ln\left(\max\left\{2, \min\left\{\nicefrac{n\widetilde{\Phi}^0\mu^2K^2}{4\sigma_*^2 },\nicefrac{p\widetilde{\Phi}^0\mu^3K^3}{64L(1-p)(1+\nicefrac{32(2+p)}{3})\sigma_*^2}\right\}\right\}\right)}{ K} \le p\\
		\text{or} && \gamma_0 \le \frac{\ln\left(\max\left\{2, \min\left\{\nicefrac{n\widetilde{\Phi}^0\mu^2K^2}{4\sigma_*^2 },\nicefrac{p\widetilde{\Phi}^0\mu^3K^3}{64L(1-p)(1+\nicefrac{32(2+p)}{3})\sigma_*^2}\right\}\right\}\right)}{\mu K}
	\end{eqnarray*}
	we have that $\EE\left[f(\overline{x}^K)-f(x^*)\right]$ is of the order
	\begin{equation}
		 \widetilde\cO\left(\frac{\widetilde{\Phi}^0}{\gamma_0}\exp\left(- \min\left\{\frac{1}{p}, \gamma_0\mu\right\} K\right) + \frac{\sigma_*^2}{n\mu K}  + \frac{L(1-p)\sigma_*^2}{p\mu^2 K^2}\right).\notag
	\end{equation}
	That is, to achieve $\EE\left[f(\overline{x}^K)-f(x^*)\right] \le \varepsilon$ in this case {\tt SS-Local-SGD} requires
	\begin{equation*}
		\widetilde{\cO}\left(\frac{L}{p\mu} + \frac{\cL}{n\mu} + \frac{\sqrt{\cL L(1-p)}}{\sqrt{p}\mu} + \frac{\sigma_*^2}{n\mu\varepsilon} + \sqrt{\frac{L(1-p)\sigma_*^2}{p\mu^2\varepsilon}}\right)
	\end{equation*}
	iterations/oracle calls per node (in expectation) and $\nicefrac{1}{p}$ times less communication rounds.	
\end{corollary}

Combining Theorem~\ref{thm:ss_local_sgd_es} and Lemma~\ref{lem:lemma_technical_cvx} we derive the following result for the convergence of {\tt SS-Local-SGD} in the case when $\mu = 0$.
\begin{corollary}
	\label{cor:ss_local_sgd_cvx_es}
	Let assumptions of Theorem~\ref{thm:ss_local_sgd_es} hold with $\mu = 0$. Then for $q = p,$ $r = \lceil\nicefrac{1}{p}\rceil$ and
	\begin{eqnarray*}
		\gamma_0 &=& \min\left\{\frac{1}{4\left(\frac{2\cL}{n}+L\right)}, \frac{p\sqrt{3}}{32\sqrt{2L(1-p)\left((2+p)L + p\cL + \frac{(2+p)\left(\nicefrac{2\cL}{r} + L\right)}{(1-q)}\right)}}\right\},\\	
		\gamma &=& \min\left\{\gamma_0, \sqrt[3]{\frac{p^3R_0^2}{256L(1-p)(2+p)^2\EE[\sigma_0^2]}}, \sqrt{\frac{nR_0^2}{2\sigma_*^2 K}}, \sqrt[3]{\frac{pR_0^2}{32L(1-p)\left(1+\nicefrac{32(2+p)}{3}\right)\sigma_*^2 K}}\right\},
	\end{eqnarray*}
	where $R_0 = \|x^0 - x^*\|$, we have that $\EE\left[f(\overline{x}^K)-f(x^*)\right]$ is of the order
	\begin{eqnarray*}
		\cO\left(\frac{\left(L+\nicefrac{p\cL}{n} + \sqrt{p(1-p)\cL L}\right)R_0^2 + \sqrt[3]{L(1-p)\EE[\sigma_0^2]R_0^4}}{pK} + \sqrt{\frac{\sigma_*^2R_0^2}{n K}} + \frac{\sqrt[3]{LR_0^4(1-p)\sigma_*^2}}{p^{\nicefrac{1}{3}}K^{\nicefrac{2}{3}}} \right).
	\end{eqnarray*}
	That is, to achieve $\EE\left[f(\overline{x}^K)-f(x^*)\right] \le \varepsilon$ in this case {\tt SS-Local-SGD} requires
	\begin{eqnarray*}
		\cO\left(\frac{\left(L+\nicefrac{p\cL}{n} + \sqrt{p(1-p)\cL L}\right)R_0^2 + \sqrt[3]{L(1-p)\EE[\sigma_0^2]R_0^4}}{p\varepsilon} + \frac{\sigma_*^2 R_0^2}{n\varepsilon^2} + \frac{R_0^2\sqrt{L(1-p)\sigma_*^2}}{p^{\nicefrac{1}{2}}\varepsilon^{\nicefrac{3}{2}}}\right)
	\end{eqnarray*}
	iterations/oracle calls per node (in expectation) and $\nicefrac{1}{p}$ times less communication rounds.
\end{corollary}

\begin{remark}
	To get the rate from Tbl.~\ref{tbl:special_cases_weakly_convex} it remains to apply the following inequality:
	\begin{eqnarray*}
		\EE[\sigma_0^2] &=& \frac{1}{n}\sum\limits_{i=1}^n\EE_{\oxi_i^0}\left[\|\nabla f_{\oxi_i^0}(x^0)-\nabla f_{i}(x^*)\|^2\right]\\
		&\overset{\eqref{eq:variance_decomposition}}{=}& \frac{1}{n}\sum\limits_{i=1}^n\|\nabla f_{i}(x^0) - \nabla f_i(x^*)\|^2 + \frac{1}{n}\sum\limits_{i=1}^n\EE_{\oxi_i^0}\left[\|\nabla f_{\oxi_i^0}(x^0) - \nabla f_i(x^0)\|^2\right]\\
		&\overset{\eqref{eq:L_smoothness_cor}}{\le}& 2L(f(x^0)-f(x^*)) + \frac{1}{nr^2}\sum\limits_{i=1}^n\sum\limits_{j=1}^r\EE_{\oxi_{i,j}^0}\left[\|\nabla f_{\oxi_{i,j}^0}(x^0) - \nabla f_i(x^0)\|^2\right]\\
		&\overset{\eqref{eq:variance_decomposition}}{\le}& 2L(f(x^0)-f(x^*)) + \frac{1}{nr}\sum\limits_{i=1}^n\EE_{\xi_i}\left[\|\nabla f_{\xi_i}(x^0)-\nabla f_i(x^*)\|^2\right]\\
		&\overset{\eqref{eq:a_b_norm_squared}}{\le}& 2L(f(x^0)-f(x^*)) + \frac{2}{nr}\sum\limits_{i=1}^n\EE_{\xi_i}\left[\|\nabla f_{\xi_i}(x^0)-\nabla f_{\xi_i}(x^*)\|^2\right] \\
		&&\quad + \frac{2}{nr}\sum\limits_{i=1}^n\EE_{\xi_i}\left[\|\nabla f_{\xi_i}(x^*)-\nabla f_{i}(x^*)\|^2\right]\\
		&\overset{r = \lceil\nicefrac{1}{p}\rceil,\eqref{eq:expected_smoothness_1}}{\le}& 2\left(L + 2p\cL\right)(f(x^0) - f(x^*)) + 2p\sigma_*^2.
	\end{eqnarray*}
\end{remark}

\subsection{{\tt S*-Local-SGD*}} \label{sec:S*-Local-SGD*}
 In this section we present doubly idealized algorithm for solving problem \eqref{eq:main_problem}+\eqref{eq:f_i_sum}. Specifically, we choose $b_i^k$ to the optimal shift $\nabla f_i(x^*)$ as per Case II, while $a_i^k$ is selected as { \tt SGD-star} gradient estimator~\cite{gorbunov2019unified}, i.e., 
\[
a^{k}_i = \nabla f_{i,j_i}(x_i^k) - \nabla f_{i,j_i}(x^*) +  \nabla f_{i}(x^*), \qquad b^{k}_i = \nabla f_{i}(x^*).
\]

 Note that now $a_i^k$ serves as an ambitious target for the local variance reduced estimators, while $b_i^k$ serves as an ambitious goal for the local shift. The resulting instance of~\eqref{eq:local_sgd_def} is presented as Algorithm~\ref{alg:local_sgd_star_star} and called Star-Shifted {\tt Local-SGD-star} ({\tt S*-Local-SGD*}).

\begin{algorithm}[h]
   \caption{{\tt S*-Local-SGD*}}\label{alg:local_sgd_star_star}
\begin{algorithmic}[1]
   \Require learning rate $\gamma>0$, initial vector $x^0 \in \R^d$, communication period $\tau \ge 1$
	\For{$k=0,1,\dotsc$}
       	\For{$i=1,\dotsc,n$ in parallel}
            \State Set $ g^{k}_i = \nabla f_{i,j_i}(x_i^k) - \nabla f_{i,j_i}(x_*)$ where $1\leq j_i\leq m$ is sampled independently from all nodes
            \If{$k+1 \mod \tau = 0$}
            \State $x_i^{k+1} = x^{k+1} = \frac{1}{n}\sum\limits_{i=1}^n\left(x_i^k - \gamma g_i^k\right)$ \Comment{averaging}
            \Else
            \State $x_i^{k+1} = x_i^k - \gamma g_i^k$ \Comment{local update}
            \EndIf
        \EndFor
   \EndFor
\end{algorithmic}
\end{algorithm}

Let us next provide the details on the convergence rate. In order to do so, let us identify the parameters of Assumption~\ref{ass:sigma_k_original}.

\begin{lemma}\label{lem:s*_local_sgd*_lemma}
	Let $f_i$ be convex and $L$-smooth and $f_{i,j}$ be convex and $\max L_{ij}$-smooth for all $i\in[n]$, $j\in[m]$. Then for all $k\ge 0$
	\begin{eqnarray}
		\frac{1}{n}\sum\limits_{i=1}^n \EE_k\left[g_i^k\right] &=& \frac{1}{n}\sum\limits_{i=1}^n\nabla f_i(x_i^k), \label{eq:unbiasedness_s*_local_sgd*}\\
		\frac{1}{n}\sum\limits_{i=1}^n \EE\left[\|\bar{g}_i^k\|^2\right] &\le& 4L\EE\left[f(x^k)-f(x^*)\right] + 2L^2\EE[V_k],\label{eq:second_moment_s*_local_sgd*}\\
		\frac{1}{n}\sum\limits_{i=1}^n \EE\left[\|g_i^k-\bar{g}_i^k\|^2\right] &\le& 4\max L_{ij}\EE\left[f(x^k) - f(x^*)\right] + 2L\max L_{ij}\EE[V_k],\label{eq:variance_s*_local_sgd*}\\
		\EE\left[\left\|\frac{1}{n}\sum\limits_{i=1}^ng_i^k\right\|^2\right] &\le& 4\left(\frac{\max L_{ij}}{n}+L\right)\EE\left[f(x^k)-f(x^*)\right] + 2L\left(\frac{\max L_{ij}}{n} + L\right)\EE[V_k].\label{eq:second_moment_s*_local_sgd*_2}
	\end{eqnarray}
\end{lemma}
\begin{proof}
	First of all,
	\begin{eqnarray*}
		\frac{1}{n}\sum\limits_{i=1}^n \EE_k\left[g_i^k\right] &=& \frac{1}{nm}\sum\limits_{i=1}^n\sum\limits_{j=1}^m \left(\nabla f_{i,j}(x_i^k) - \nabla f_{i,j}(x^*)\right) = \frac{1}{n}\sum\limits_{i=1}^n\nabla f_i(x_i^k)
	\end{eqnarray*}
	and, in particular, $\bar{g}_i^k = \EE_k\left[g_i^k\right] = \nabla f_i(x_i^k) - \nabla f_i(x^*)$. Using this we derive
	\begin{eqnarray*}
		\frac{1}{n}\sum\limits_{i=1}^n \EE\left[\|\bar{g}_i^k\|^2\right] &=& \frac{1}{n}\sum\limits_{i=1}^n \EE\left[\|\nabla f_i(x_i^k) - \nabla f_i(x^*)\|^2\right]\\
		&\overset{\eqref{eq:L_smoothness_cor}}{\le}& \frac{2L}{n}\sum\limits_{i=1}^n\EE\left[D_{f_i}(x_i^k,x^*)\right] \overset{\eqref{eq:poiouhnkj}}{\le} 4L\EE\left[f(x^k)-f(x^*)\right] + 2L^2\EE[V_k]
	\end{eqnarray*}
	and
	\begin{eqnarray}
		\frac{1}{n}\sum\limits_{i=1}^n \EE\left[\|g_i^k-\bar{g}_i^k\|^2\right] &\overset{\eqref{eq:variance_decomposition}}{\le}& \frac{1}{n}\sum\limits_{i=1}^n \EE\left[\|g_i^k\|^2\right]\notag\\
		&=& \frac{1}{nm}\sum\limits_{i=1}^n\sum\limits_{j=1}^m \|\nabla f_{i,j}(x_i^k) - \nabla f_{i,j}(x^*)\|^2\notag \\
		&\overset{\eqref{eq:L_smoothness_cor}}{\le}& \frac{2\max L_{ij}}{n}\sum\limits_{i=1}^n\EE\left[D_{f_i}(x_i^k,x^*)\right]\notag\\
		&\overset{\eqref{eq:poiouhnkj}}{\le}& 4\max L_{ij}\EE\left[f(x^k)-f(x^*)\right] + 2L\max L_{ij}\EE[V_k].\label{eq:hbdsujbcvshvbcbu}
	\end{eqnarray}
	Finally, due to the independence of $j_1, j_2, \ldots, j_n$ we have
	\begin{eqnarray*}
		\EE\left[\left\|\frac{1}{n}\sum\limits_{i=1}^ng_i^k\right\|^2\right] &\overset{\eqref{eq:variance_decomposition},\eqref{eq:tower_property}}{=}& \EE\left[\left\|\frac{1}{n}\sum\limits_{i=1}^n\left(\nabla f_{i,j_i}(x_i^k) - \nabla f_{i,j_i}(x_*) - (\nabla f_i(x_i^k) - \nabla f_i(x^*))\right)\right\|^2\right]\\
		&&\quad + \EE\left[\left\|\frac{1}{n}\sum\limits_{i=1}^n\left(\nabla f_i(x_i^k) - \nabla f_i(x^*)\right)\right\|^2\right]\\
		&=& \frac{1}{n^2}\sum\limits_{i=1}^n\EE\left[\|\nabla f_{i,j_i}(x_i^k) - \nabla f_{i,j_i}(x_*) - (\nabla f_i(x_i^k) - \nabla f_i(x^*))\|^2\right]\\
		&&\quad + \EE\left[\left\|\frac{1}{n}\sum\limits_{i=1}^n\nabla f_i(x_i^k)\right\|^2\right]\\
		&\overset{\eqref{eq:variance_decomposition}}{\le}& \frac{1}{n^2m}\sum\limits_{i=1}^n\sum\limits_{j=1}^m \|\nabla f_{i,j}(x_i^k) - \nabla f_{i,j}(x^*)\|^2 + \EE\left[\left\|\frac{1}{n}\sum\limits_{i=1}^n\nabla f_i(x_i^k)\right\|^2\right]\\
		&\overset{\eqref{eq:hbdsujbcvshvbcbu},\eqref{eq:poiouhnkj}}{\le}& 4\left(\frac{\max L_{ij}}{n}+L\right)\EE\left[f(x^k)-f(x^*)\right] + 2L\left(\frac{\max L_{ij}}{n} + L\right)\EE[V_k].
	\end{eqnarray*}
\end{proof}

Using Corollary~\ref{cor:const_loop} we obtain the following theorem.
\begin{theorem}\label{thm:s*_local_sgd*}
	Assume that $f_i(x)$ is $\mu$-strongly convex and $L$-smooth and $f_{i,j}$ is convex and $\max L_{ij}$-smooth for every $i\in[n]$, $j\in[m]$. Then {\tt S*-Local-SGD*} satisfies Assumption~\ref{ass:hetero_second_moment} with
	\begin{gather*}
		\tA = 2L,\quad \hA = 2\max L_{ij},\quad \tB = \hB = 0,\quad \tF = 2L^2,\quad \hF = 2L\max L_{ij}, \quad \tD_1 = \hD_1 = 0,\\
		A' = 2\left(\frac{\max L_{ij}}{n} + L\right),\quad B' = 0,\quad F' = 2L\left(\frac{\max L_{ij}}{n} + L\right), \\
		D_1' = 0,\quad \sigma_k^2 \equiv 0,\quad \rho = 1,\quad C = 0,\quad G = 0,\quad D_2 = 0,\quad H = 0,\quad D_3 = 0
	\end{gather*}
	under assumption that
	\begin{eqnarray*}
		\gamma &\le& \min\left\{\frac{1}{4\left(\frac{\max L_{ij}}{n}+L\right)}, \frac{1}{8\sqrt{eL(\tau-1)\left(L(\tau-1)+\max L_{ij}\right)}}\right\}.
	\end{eqnarray*}
	Moreover, for $\mu > 0$ we have
	\begin{eqnarray}
		\EE\left[f(\overline{x}^K) - f(x^*)\right] &\le& \left(1 - \gamma\mu\right)^K\frac{2\|x^0-x^*\|^2}{\gamma} \notag
	\end{eqnarray}
	and when $\mu = 0$ we have
	\begin{eqnarray}
		\EE\left[f(\overline{x}^K) - f(x^*)\right] &\le& \frac{2\|x^0-x^*\|^2}{\gamma K}. \notag
	\end{eqnarray}
\end{theorem}

The theorem above together with Lemma~\ref{lem:lemma2_stich} implies the following result.
\begin{corollary}\label{cor:s*_local_sgd*_str_cvx}
	Let assumptions of Theorem~\ref{thm:s*_local_sgd*} hold with $\mu > 0$. Then for 
	\begin{eqnarray*}
		\gamma &=& \min\left\{\frac{1}{4\left(\frac{\max L_{ij}}{n}+L\right)}, \frac{1}{8\sqrt{eL(\tau-1)\left(L(\tau-1)+\max L_{ij}\right)}}\right\}
	\end{eqnarray*}
	and for all $K\ge 1$ we have $\EE\left[f(\overline{x}^K)-f(x^*)\right]$ of order
	\begin{equation}
		\cO\left(\left(L\tau + \nicefrac{\max L_{ij}}{n} + \sqrt{(\tau-1)L\max L_{ij}}\right)\|x^0-x^*\|^2\exp\left(- \frac{\mu}{L\tau + \nicefrac{\max L_{ij}}{n} + \sqrt{(\tau-1)L\max L_{ij}}}K\right)\right).\notag
	\end{equation}
	That is, to achieve $\EE\left[f(\overline{x}^K)-f(x^*)\right] \le \varepsilon$ in this case {\tt S*-Local-SGD*} requires
	\begin{equation*}
		\cO\left(\left(\frac{L\tau}{\mu} + \frac{\max L_{ij}}{n\mu} + \frac{\sqrt{(\tau-1) L\max L_{ij}}}{\mu}\right)\log\frac{\left(L\tau + \nicefrac{\max L_{ij}}{n} + \sqrt{(\tau-1)L\max L_{ij}}\right)\|x^0-x^*\|^2}{\varepsilon}\right)
	\end{equation*}
	iterations/oracle calls per node and $\tau$ times less communication rounds.	
\end{corollary}

Next, we derive the following result for the convergence of {\tt S*-Local-SGD*} in the case when $\mu = 0$.
\begin{corollary}
	\label{cor:s*_local_sgd*_cvx}
	Let assumptions of Theorem~\ref{thm:s*_local_sgd*} hold with $\mu = 0$. Then for
	\begin{eqnarray*}
		\gamma &=& \min\left\{\frac{1}{4\left(\frac{\max L_{ij}}{n}+L\right)}, \frac{1}{8\sqrt{eL(\tau-1)\left(L(\tau-1)+\max L_{ij}\right)}}\right\},
	\end{eqnarray*}
	we have that $\EE\left[f(\overline{x}^K)-f(x^*)\right]$ is of the order
	\begin{eqnarray*}
		\cO\left(\frac{\left(L\tau + \nicefrac{\max L_{ij}}{n} + \sqrt{(\tau-1)L\max L_{ij}}\right)R_0^2}{K}\right),
	\end{eqnarray*}
	where $R_0 = \|x^0 - x^*\|$. That is, to achieve $\EE\left[f(\overline{x}^K)-f(x^*)\right] \le \varepsilon$ in this case {\tt S*-Local-SGD*} requires
	\begin{eqnarray*}
		\cO\left(\frac{\left(L\tau + \nicefrac{\max L_{ij}}{n} + \sqrt{(\tau-1)L\max L_{ij}}\right)R_0^2}{\varepsilon}\right)
	\end{eqnarray*}
	iterations/oracle calls per node and $\tau$ times less communication rounds.
\end{corollary}

%\begin{proposition}\cite{gorbunov2019unified} 
%Gradient estimator $a_i^k$ satisfies Assumption~\ref{ass:sigma_k_original} with parameters $A_i = \max L_{ij}, B_i=0, D_{1,i} = 0, \rho_i = 1, C_i = 0, D_{2,i}= 0$. 
%\end{proposition}
%
%It remains to use Lemma~\ref{lem:local_solver} along with the results from Section~\ref{sec:a_data_and_loop} to recover all parameters of Assumption~\ref{ass:key_assumption}. Plugging these in Theorem~\ref{thm:main_result} and Corollaries~\ref{cor:app_complexity_cor_cvx},  Corollaries~\ref{cor:app_complexity_cor_str_cvx}, we obtain the overall complexity of Algorithm~\ref{alg:local_sgd_star_star}.
%
%
%\filip{TODO: finish corollary}
%
%\begin{corollary}\label{cor:s*localsgd*}
%
%\end{corollary}

\subsection{{\tt S-Local-SVRG}}\label{sec:loopless_local_svrg_fs}

\begin{algorithm}[h]
   \caption{Shifted Local {\tt SVRG} ({\tt S-Local-SVRG}) for minimizing local finite sums}\label{alg:l_local_svrg_fs}
\begin{algorithmic}[1]
   \Require learning rate $\gamma>0$, initial vector $x^0 \in \R^d$, probability of communication $p\in(0,1]$, probability of local full gradient computation $q\in(0,1]$, initialization $y^0 = x^0$
	\For{$k=0,1,\dotsc$}
       	\For{$i=1,\dotsc,n$ in parallel}
       		\State Choose $j_i$ uniformly at random from $[m]$
            \State $g_i^k = \nabla f_{i,j_i}(x_i^k) - \nabla f_{i,j_i}(y^k) + \nabla f(y^k)$
            \State $x_i^{k+1} = \begin{cases}x^{k+1},&\text{w.p. } p,\\
            x_i^k - \gamma g_i^k,& \text{w.p. } 1 - p, \end{cases}$ where $x^{k+1} = \frac{1}{n}\sum\limits_{i=1}^n(x_i^k - \gamma g_i^k)$
            \State $y^{k+1} = \begin{cases}x^k,&\text{w.p. } q,\\
            y^k,& \text{w.p. } 1 - q \end{cases}$
        \EndFor
   \EndFor
\end{algorithmic}
\end{algorithm}

In this section we are interested in problem \eqref{eq:main_problem}+\eqref{eq:f_i_sum}. To solve this problem we propose a new method called Shifted {\tt Local-SVRG} ({\tt S-Local-SVRG}, see Algorithm~\ref{alg:l_local_svrg_fs}).

We note that our analysis works even when updates in lines $5$,$6$ are not independent. Moreover, in order for {\tt S-Local-SVRG} to be efficient, we shall require $q\leq p$.

\begin{remark}
Unlike all other special cases, the rate of {\tt S-Local-SVRG} can not be directly obtained from the theory of the local stochastic solver described in Section~\ref{sec:local_solver}. Specifically, we construct the sequence $l_i^k$ using $y^k$ in contrast to $x_i^k$ used in Section~\ref{sec:local_solver}. While we could construct $l_i^k$ from the local iterate sequences, setting it as the virtual iterates yields a tighter rate. We remark that such a choice is rather poor in general; we can implement it efficiently thanks to the specific structure of {\tt S-Local-SVRG}. 
\end{remark}

\begin{lemma}\label{lem:loopless_local_svrg_fs_second_moment}
	Let $f_{i}$ be convex and $L$-smooth and $f_{i,j}$ be convex and $\max L_{ij}$-smooth for all $i\in[n]$, $j\in[m]$. Then for all $k\ge 0$
	\begin{eqnarray}
		\frac{1}{n}\sum\limits_{i=1}^n \EE_k\left[g_i^k\right] &=& \frac{1}{n}\sum\limits_{i=1}^n\nabla f_i(x_i^k), \label{eq:unbiasedness_loopless_local_svrg_fs}\\
		\frac{1}{n}\sum\limits_{i=1}^n \EE\left[\|\bar{g}_i^k\|^2\right] &\le&8L\EE\left[f(x^k)-f(x^*)\right] + 2\EE[\sigma_k^2] + 4L^2\EE[V_k],\label{eq:second_moment_loopless_local_svrg_fs}\\
		\frac{1}{n}\sum\limits_{i=1}^n \EE\left[\|g_i^k-\bar{g}_i^k\|^2\right] &\le& 8\max L_{ij}\EE\left[f(x^k)-f(x^*)\right] + 2\EE[\sigma_k^2] + 4L\max L_{ij}\EE[V_k],\label{eq:variance_loopless_local_svrg_fs}\\
		\EE\left[\left\|\frac{1}{n}\sum\limits_{i=1}^ng_i^k\right\|^2\right] &\le& 4\left(\frac{2\max L_{ij}}{n} + L\right)\EE\left[f(x^k)-f(x^*)\right] + \frac{2}{n}\EE[\sigma_k^2]\notag \\
		&&\quad + 2L\left(\frac{2\max L_{ij}}{n} + L\right)\EE[V_k],\label{eq:second_moment_loopless_local_svrg_fs_2}
	\end{eqnarray}
	where $\sigma_k^2 \eqdef \frac{1}{nm}\sum\limits_{i=1}^n\sum\limits_{j=1}^m\left\|\nabla f_{i,j}(y^k) - \nabla f_{i,j}(x^*)\right\|^2 + \frac{1}{n}\sum\limits_{i=1}^n\left\|\nabla f_{i}(y^k) - \nabla f_{i}(x^*)\right\|^2$.
\end{lemma}
\begin{proof}
	First of all, we have
	\begin{eqnarray*}
		\frac{1}{n}\sum\limits_{i=1}^n \EE_k\left[g_i^k\right] &=& \frac{1}{n}\sum\limits_{i=1}^n \EE_k\left[\nabla f_{i,j^k}(x_i^k) - \nabla f_{i,j_i}(y^k)  + \nabla f(y^k)\right]\\
		&=& \frac{1}{nm}\sum\limits_{i=1}^n\sum\limits_{j=1}^m\left(\nabla f_{i,j}(x_i^k) - \nabla f_{i,j}(y^k) + \nabla f(y^k)\right)\\
		&=& \frac{1}{n}\sum\limits_{i=1}^n\nabla f_i(x_i^k)
	\end{eqnarray*}
	and, in particular, $\bar{g}_i^k = \EE_k[g_i^k] = \nabla f_i(x_i^k) - \nabla f_i(y^k) + \nabla f(y^k)$.
	Using this we get
	\begin{eqnarray*}
		\frac{1}{n}\sum\limits_{i=1}^n\EE\left[\|\bar{g}_i^k\|^2\right] &\overset{\eqref{eq:a_b_norm_squared}}{\le}& \frac{2}{n}\sum\limits_{i=1}^n\EE\left[\|\nabla f_i(x_i^k)-\nabla f_i(x^*)\|^2\right]\\
		&&\quad + \frac{2}{n}\sum\limits_{i=1}^n\EE\left[\|\nabla f_i(y^k)-\nabla f_i(x^*) - (\nabla f(y^k)-\nabla f(x^*))\|^2\right]\\
		&\overset{\eqref{eq:L_smoothness_cor},\eqref{eq:variance_decomposition}}{\le}& \frac{4L}{n}\sum\limits_{i=1}^n\EE\left[D_{f_i}(x_i^k,x^*)\right] + \frac{2}{n}\sum\limits_{i=1}^n\EE\left[\|\nabla f_i(y^k)-\nabla f_i(x^*)\|^2\right]\\
		&\overset{\eqref{eq:poiouhnkj}}{\le}& 8L\EE\left[f(x^k)-f(x^*)\right] + 2\EE[\sigma_k^2] + 4L^2\EE[V_k]
	\end{eqnarray*}
	and
	\begin{eqnarray}
		\frac{1}{n}\sum\limits_{i=1}^n \EE\left[\|g_i^k-\bar{g}_i^k\|^2\right] &=& \frac{1}{n}\sum\limits_{i=1}^n\EE\left[\|\nabla f_{i,j_i}(x_i^k) - \nabla f_{i,j_i}(y^k)-(\nabla f_i(x_i^k) - \nabla f_i(y^k))\|^2\right]\notag\\
		&\overset{\eqref{eq:variance_decomposition}}{\le}&\frac{1}{n}\sum\limits_{i=1}^n\EE\left[\|\nabla f_{i,j_i}(x_i^k) - \nabla f_{i,j_i}(y^k)\|^2\right]\notag\\
		&\overset{\eqref{eq:a_b_norm_squared}}{\le}& \frac{2}{nm}\sum\limits_{i=1}^n\sum\limits_{j=1}^m\EE\left[\|\nabla f_{i,j}(x_i^k) - \nabla f_{i,j}(x^*)\|^2\right]\notag\\
		&&\quad +\frac{2}{nm}\sum\limits_{i=1}^n\sum\limits_{j=1}^m\EE\left[\|\nabla f_{i,j}(y^k) - \nabla f_{i,j}(x^*)\|^2\right]\notag \\
		&\overset{\eqref{eq:L_smoothness_cor}}{\le}& \frac{4\max L_{ij}}{n}\sum\limits_{i=1}^n\EE\left[D_{f_i}(x_i^k,x^*)\right] + 2\EE[\sigma_k^2]\notag\\
		&\overset{\eqref{eq:poiouhnkj}}{\le}& 8\max L_{ij}\EE\left[f(x^k)-f(x^*)\right] + 2\EE[\sigma_k^2] + 4L\max L_{ij}\EE[V_k].\label{eq:hscdvgdvauaicna}
	\end{eqnarray}
	Finally, using independence of $j_1,j_2,\ldots,j_n$ we derive
	\begin{eqnarray*}
		\EE\left[\left\|\frac{1}{n}\sum\limits_{i=1}^n g_i^k\right\|^2\right] &\overset{\eqref{eq:variance_decomposition},\eqref{eq:unbiasedness_loopless_local_svrg_fs}}{=}&  \EE\left[\left\|\frac{1}{n}\sum\limits_{i=1}^n \nabla f_i(x_i^k)\right\|^2\right] \\
		&&\quad + \EE\left[\left\|\frac{1}{n}\sum\limits_{i=1}^n (\nabla f_{i,j_i}(x_i^k)-\nabla f_{i,j_i}(y^k) - (\nabla f_i(x_i^k) - \nabla f_i(y^k)))\right\|^2\right]\\
		&=& \EE\left[\left\|\frac{1}{n}\sum\limits_{i=1}^n \nabla f_i(x_i^k)\right\|^2\right] \\
		&&\quad + \frac{1}{n^2}\sum\limits_{i=1}^n \EE\left[\|(\nabla f_{i,j_i}(x_i^k)-\nabla f_{i,j_i}(y^k) - (\nabla f_i(x_i^k) - \nabla f_i(y^k)))\|^2\right]\\
		&\overset{\eqref{eq:vdgasvgda},\eqref{eq:hscdvgdvauaicna}}{\le}& 4\left(\frac{2\max L_{ij}}{n} + L\right)\EE\left[f(x^k)-f(x^*)\right] + \frac{2}{n}\EE[\sigma_k^2] + 2L\left(\frac{2\max L_{ij}}{n} + L\right)\EE[V_k].
	\end{eqnarray*}
\end{proof}

\begin{lemma}\label{lem:loopless_local_svrg_fs_sigma_k_bound}
	Let $f_{i}$ be convex and $L$-smooth and $f_{i,j}$ be convex and $\max L_{ij}$-smooth for all $i\in[n]$, $j\in[m]$. Then for all $k\ge 0$
	\begin{eqnarray}
		\EE\left[\sigma_{k+1}^2\right] &\le& \left(1-q\right)\EE\left[\sigma_k^2\right] + 2(L+\max L_{ij})q\EE\left[f(x^k) - f(x^*)\right] \label{eq:loopless_local_svrg_fs_sigma_k_bound}
	\end{eqnarray}
	where  $\sigma_k^2 \eqdef \frac{1}{nm}\sum\limits_{i=1}^n\sum\limits_{j=1}^m\left\|\nabla f_{i,j}(y^k) - \nabla f_{i,j}(x^*)\right\|^2 + \frac{1}{n}\sum\limits_{i=1}^n\left\|\nabla f_{i}(y^k) - \nabla f_{i}(x^*)\right\|^2$.
\end{lemma}
\begin{proof}
	First of all, we introduce new notations:
	\begin{eqnarray*}
		\sigma_{k,1}^2 \eqdef \frac{1}{nm}\sum\limits_{i=1}^n\sum\limits_{j=1}^m\left\|\nabla f_{i,j}(y^k) - \nabla f_{i,j}(x^*)\right\|^2,\quad \sigma_{k,2}^2 = \frac{1}{n}\sum\limits_{i=1}^n\left\|\nabla f_{i}(y^k) - \nabla f_{i}(x^*)\right\|^2.
	\end{eqnarray*}
	Secondly, by definition of $y^{k+1}$ we have
	\begin{eqnarray*}
		\EE\left[\sigma_{k+1,1}^2\mid x_1^k,\ldots, x_n^k\right] &=& \frac{1-q}{nm}\sum\limits_{i=1}^n\sum\limits_{j=1}^m\left\|\nabla f_{i,j}(y^k) - \nabla f_{i,j}(x^*)\right\|^2 + \frac{q}{nm}\sum\limits_{i=1}^n\sum\limits_{j=1}^m\left\|\nabla f_{i,j}(x^k) - \nabla f_{i,j}(x^*)\right\|^2\\
		&\overset{\eqref{eq:L_smoothness_cor}}{\le}& (1-q)\sigma_{k,1}^2 + 2q\max L_{ij}(f(x^k) - f(x^*)),
	\end{eqnarray*}
	hence
	\begin{equation}
		\EE\left[\sigma_{k+1,1}^2\right] \le (1-q)\EE\left[\sigma_{k,1}^2\right] + 2q\max L_{ij}\EE\left[f(x^k) - f(x^*)\right]. \label{eq:loopless_local_svrg_fs_sigma_k_bound_tech1}
	\end{equation}
	Next, the definition of $y^{k+1}$ implies
	\begin{eqnarray*}
		\EE\left[\sigma_{k+1,2}^2\mid x_1^k,\ldots, x_n^k\right] &=& \frac{1-q}{n}\sum\limits_{i=1}^n\|\nabla f_i(y^k) - \nabla f_i(x^*)\|^2 + \frac{q}{n}\sum\limits_{i=1}^n\|\nabla f_i(x^k) - \nabla f_i(x^*)\|^2\\
		&\overset{\eqref{eq:L_smoothness_cor}}{\le}& (1-q)\sigma_k^2 + 2Lq(f(x^k) - f(x^*)),
	\end{eqnarray*}
	hence
	\begin{equation}
		\EE\left[\sigma_{k+1,2}^2\right] \le (1-q)\EE\left[\sigma_{k,2}^2\right] + 2Lq\EE\left[f(x^k) - f(x^*)\right]. \label{eq:loopless_local_svrg_fs_sigma_k_bound_tech2}
	\end{equation}
	Finally, we combine obtained inequalities and get
	\begin{eqnarray*}
		\EE\left[\sigma_{k+1}\right] &=& \EE\left[\sigma_{k+1,1}^2\right] + \EE\left[\sigma_{k+1,2}^2\right]\\
		&\overset{\eqref{eq:loopless_local_svrg_fs_sigma_k_bound_tech1},\eqref{eq:loopless_local_svrg_fs_sigma_k_bound_tech2}}{\le}& (1-q)\left(\EE\left[\sigma_{k,1}^2\right] + \EE\left[\sigma_{k,2}^2\right]\right) + 2(L+\max L_{ij})q\EE\left[f(x^k) - f(x^*)\right] \\
		&=& \left(1-q\right)\EE\left[\sigma_k^2\right] + 2(L+\max L_{ij})q\EE\left[f(x^k) - f(x^*)\right],
	\end{eqnarray*}
	which concludes the proof.
\end{proof}

Using Corollary~\ref{cor:rand_loop} we obtain the following theorem.
\begin{theorem}\label{thm:loopless_local_svrg_fs}
	Assume that $f_{i}$ is $\mu$-strongly convex and $L$-smooth and $f_{i,j}$ is convex and $\max L_{ij}$-smooth for all $i\in[n]$, $j\in[m]$. Then {\tt S-Local-SVRG} satisfies Assumption~\ref{ass:hetero_second_moment} with
	\begin{gather*}
		\tA = 4L,\quad \hA = 4\max L_{ij},\quad \tB = \hB = 2,\quad \tF = 4L^2,\quad \hF = 4L\max L_{ij} \quad \tD_1 = \hD_1 = 0,\\
		A' = \frac{4\max L_{ij}}{n} + 2L,\quad B' = \frac{2}{n},\quad F' = 2L\left(\frac{2\max L_{ij}}{n} + L\right), \quad D_1' = 0,\\
		\sigma_k^2 = \frac{1}{nm}\sum\limits_{i=1}^n\sum\limits_{j=1}^m\left\|\nabla f_{i,j}(y^k) - \nabla f_{i,j}(x^*)\right\|^2 + \frac{1}{n}\sum\limits_{i=1}^n\left\|\nabla f_{i}(y^k) - \nabla f_{i}(x^*)\right\|^2,\\
		 \rho = q,\quad C = (L+\max L_{ij})q,\quad G = 0,\quad D_2 = 0,\quad H = \frac{256(1-p^2)(2+q)\gamma^2}{3p^2q},\quad D_3 = 0
	\end{gather*}
	under assumption that
	\begin{eqnarray*}
		\gamma &\le& \min\left\{\frac{1}{\nicefrac{56\max L_{ij}}{3n} + 4L + \nicefrac{32L}{3n}}, \frac{p\sqrt{3}}{32\sqrt{2L(1-p)\left(L(2+p)+p\max L_{ij} + \nicefrac{4(L+\max L_{ij})(1+p)}{(1-q)}\right)}}\right\}.
	\end{eqnarray*}
	Moreover, for $\mu > 0$ we have
	\begin{eqnarray}
		\EE\left[f(\overline{x}^K) - f(x^*)\right] &\le& \left(1 - \min\left\{\gamma\mu,\frac{q}{4}\right\}\right)^K\frac{2\|x^0-x^*\|^2 + \frac{16\gamma^2\sigma_0^2}{nq} + \frac{1024L(1-p^2)(2+q)\gamma^3\sigma_0^2}{3p^2q}}{\gamma}\notag
	\end{eqnarray}
	and when $\mu = 0$ we have
	\begin{eqnarray}
		\EE\left[f(\overline{x}^K) - f(x^*)\right] &\le& \frac{2\|x^0-x^*\|^2 + \frac{16\gamma^2\sigma_0^2}{nq} + \frac{1024L(1-p^2)(2+q)\gamma^3\sigma_0^2}{3p^2q}}{\gamma K}. \notag
	\end{eqnarray}
\end{theorem}

The theorem above together with Lemma~\ref{lem:lemma2_stich} implies the following result.
\begin{corollary}\label{cor:s_local_svrg_str_cvx}
	Let assumptions of Theorem~\ref{thm:loopless_local_svrg_fs} hold with $\mu > 0$. Then for $q = \nicefrac{1}{m}$, $m \ge \nicefrac{1}{p}$,
	\begin{eqnarray*}
		\gamma &=& \min\left\{\frac{1}{\nicefrac{56\max L_{ij}}{3n} + 4L + \nicefrac{32L}{3n}}, \frac{p\sqrt{3}}{32\sqrt{2L(1-p)\left(L(2+p)+p\max L_{ij} + \nicefrac{4(L+\max L_{ij})(1+p)}{(1-q)}\right)}}\right\}
	\end{eqnarray*}
	and for all $K\ge 1$ we have $\EE\left[f(\overline{x}^K)-f(x^*)\right]$ of order
	\begin{equation}
		\cO\left(\left(\frac{L}{p} + \frac{\max L_{ij}}{n} + \frac{\sqrt{(1-p)L\max L_{ij}}}{p}\right)\Phi^0\exp\left(- \min\left\{\frac{\mu}{\frac{L}{p} + \frac{\max L_{ij}}{n} + \frac{\sqrt{(1-p)L\max L_{ij}}}{p}},\frac{1}{m}\right\}K\right)\right),\notag
	\end{equation}
	where $\Phi^0 = 2\|x^0-x^*\|^2 + \frac{16\gamma^2\sigma_0^2}{nq} + \frac{1024L(1-p^2)(2+q)\gamma^3\sigma_0^2}{3p^2q}$. That is, to achieve $\EE\left[f(\overline{x}^K)-f(x^*)\right] \le \varepsilon$ in this case {\tt S-Local-SVRG} requires
	\begin{equation*}
		K = \cO\left(\left(m+ \frac{L}{p\mu} + \frac{\max L_{ij}}{n\mu} + \frac{\sqrt{(1-p) L\max L_{ij}}}{p\mu}\right)\log\frac{\left(\frac{L}{p} + \frac{\max L_{ij}}{n} + \frac{\sqrt{(1-p)L\max L_{ij}}}{p}\right)\Phi^0}{\varepsilon}\right)
	\end{equation*}
	iterations/oracle calls per node (in expectation) and $\nicefrac{1}{p}$ times less communication rounds.	
\end{corollary}

That is, {\tt S-Local-SVRG} is the first implementable linearly converging stochastic method with local updates with a convergence guarantee in terms of the number of communications that is not worse than that of {\tt GD} even in the arbitrary heterogeneous data regime.

Next, we derive the following result for the convergence of {\tt S-Local-SVRG} in the case when $\mu = 0$.
\begin{corollary}
	\label{cor:s_local_svrg_cvx}
	Let assumptions of Theorem~\ref{thm:loopless_local_svrg_fs} hold with $\mu = 0$. Then for $q = \nicefrac{1}{m}$, $m \ge \nicefrac{1}{p}$ and
	\begin{eqnarray*}
		\gamma_0 &=& \min\left\{\frac{1}{\nicefrac{56\max L_{ij}}{3n} + 4L + \nicefrac{32L}{3n}}, \frac{p\sqrt{3}}{32\sqrt{2L(1-p)\left(L(2+p)+p\max L_{ij} + \nicefrac{4(L+\max L_{ij})(1+p)}{(1-q)}\right)}}\right\},\\
		\gamma &=& \min\left\{\gamma_0,\sqrt{\frac{nR_0^2}{8m\sigma_0^2}},\sqrt[3]{\frac{3p^2R_0^2}{512L(1-p^2)(2m+1)\sigma_0^2}}\right\}
	\end{eqnarray*}
	we have 
	\begin{eqnarray*}
\EE\left[f(\overline{x}^K)-f(x^*)\right]=		\cO\left(\frac{\left(L + \nicefrac{p\max L_{ij}}{n} + \sqrt{(1-p)L\max L_{ij}}\right)R_0^2}{pK} + \frac{\sqrt{m\sigma_0^2 R_0^2}}{\sqrt{n}K} + \frac{\sqrt[3]{Lm\sigma_0^2 R_0^4}}{p^{\nicefrac{2}{3}}K}\right),
	\end{eqnarray*}
	where $R_0 = \|x^0 - x^*\|$. That is, to achieve $\EE\left[f(\overline{x}^K)-f(x^*)\right] \le \varepsilon$ in this case {\tt S-Local-SVRG} requires
	\begin{eqnarray*}
	K=	\cO\left(\frac{\left(L + \nicefrac{p\max L_{ij}}{n} + \sqrt{(1-p)L\max L_{ij}}\right)R_0^2}{p\varepsilon} + \frac{\sqrt{m\sigma_0^2 R_0^2}}{\sqrt{n}\varepsilon} + \frac{\sqrt[3]{Lm\sigma_0^2 R_0^4}}{p^{\nicefrac{2}{3}}\varepsilon}\right)
	\end{eqnarray*}
	iterations/oracle calls per node (in expectation) and $\nicefrac{1}{p}$ times less communication rounds.
\end{corollary}

\begin{remark}
	To get the rate from Tbl.~\ref{tbl:special_cases_weakly_convex} it remains to apply the following inequality:
	\begin{eqnarray*}
		\sigma_0^2 &=& \frac{1}{nm}\sum\limits_{i=1}^n\sum\limits_{j=1}^m\left\|\nabla f_{i,j}(x^0) - \nabla f_{i,j}(x^*)\right\|^2 + \frac{1}{n}\sum\limits_{i=1}^n\left\|\nabla f_{i}(x^0) - \nabla f_{i}(x^*)\right\|^2\\
		&\overset{\eqref{eq:L_smoothness}}{\le}& 2\left(\max L_{ij}^2 + L^2\right)\|x^0 - x^*\|^2.
	\end{eqnarray*}
\end{remark}

\clearpage

\section{Basic Facts}\label{sec:basic_facts}
For all $a,b,x_1,\ldots,x_n\in\R^d$, $\beta > 0$ and $p\in(0,1]$ the following inequalities hold
\begin{equation}\label{eq:fenchel_young}
	\langle a,b\rangle \le \frac{\|a\|^2}{2\beta} + \frac{\beta\|b\|^2}{2},
\end{equation}
\begin{equation}\label{eq:a-b_a+b}
	\langle a-b,a+b\rangle = \|a\|^2 - \|b\|^2,
\end{equation}
\begin{equation}\label{eq:1/2a_minus_b}
    \frac{1}{2}\|a\|^2 - \|b\|^2 \le \|a+b\|^2,
\end{equation}
\begin{equation}\label{eq:a+b_norm_beta}
    \|a+b\|^2 \le (1+\beta)\|a\|^2 + (1+\nicefrac{1}{\beta})\|b\|^2,
\end{equation}
\begin{equation}\label{eq:a_b_norm_squared}
	\left\|\sum\limits_{i=1}^n x_n\right\|^2 \le n\sum\limits_{i=1}^n\|x_i\|^2,
\end{equation}
\begin{equation}
	\left(1 - \frac{p}{2}\right)^{-1} \le 1 + p, \label{eq:1-p/2_inequality}
\end{equation}
\begin{equation}
	\left(1 + \frac{p}{2}\right)(1 - p) \le 1 - \frac{p}{2}. \label{eq:1+p/2_inequality}
\end{equation}

\textbf{Variance decomposition.} For a random vector $\xi \in \R^d$ and any deterministic vector $x \in \R^d$, the variance of $\xi$ can be decomposed as
\begin{equation}\label{eq:variance_decomposition}
	\EE\left[\left\|\xi - \EE[\xi] \right\|^2\right] = \EE\left[\|\xi-x\|^2\right] - \left\|\EE[\xi] - x\right\|^2
\end{equation}

\textbf{Tower property of mathematical expectation.} For random variables $\xi,\eta\in \R^d$ we have
\begin{equation}
	\EE\left[\xi\right] = \EE\left[\EE\left[\xi\mid \eta\right]\right]\label{eq:tower_property}
\end{equation}
under assumption that all expectations in the expression above are well-defined.

\clearpage

\section{Technical Lemmas}\label{sec:tech_lemmas}

We now present a key technical lemma enabling our analysis. This is a refined version of Lemma 14 from \cite{stich2019error}.

\begin{lemma}[see also Lemma 14 from \cite{stich2019error}]\label{lem:lemma14_stich}
	For any $\tau$ random vectors $\xi_1,\ldots,\xi_\tau\in\R^d$ such that for all $t=2,\ldots,\tau$ random vector $\xi_t$ depends on $\xi_{1},\ldots,\xi_{t-1}$ and does not depend on $\xi_{t+1},\ldots,\xi_{\tau}$ the following inequality holds
	\begin{equation}
		\EE\left[\left\|\sum\limits_{t=1}^\tau\xi_t\right\|^2\right] \le e\tau\sum\limits_{t=1}^\tau\EE\left[\left\|\EE_t[\xi_{t}]\right\|^2\right] + e\sum\limits_{t=1}^\tau\EE\left[\left\|\xi_t-\EE_t[\xi_{t}]\right\|^2\right], \label{eq:lemma14_stich}
	\end{equation}
	where $\EE_t[\cdot]$ denotes the conditional expectation $\EE[\cdot\mid \xi_{t-1},\ldots,\xi_1]$.
\end{lemma}
\begin{proof}
	First of all, if $\tau = 1$ then \eqref{eq:lemma14_stich} immediately follows from variance decompostion \eqref{eq:variance_decomposition}. Otherwise ($\tau > 1$) for all $l=1,\ldots,\tau$ we have
	\begin{eqnarray*}
		\EE_l\left[\left\|\sum\limits_{t=1}^l\xi_t\right\|^2\right] &\overset{\eqref{eq:variance_decomposition}}{=}& \left\|\EE_l[\xi_l] + \sum\limits_{t=1}^{l-1}\xi_t\right\|^2 + \EE_l\left[\|\xi_l - \EE_l[\xi_l]\|^2\right]\\
		&\overset{\eqref{eq:a+b_norm_beta}}{\le}& \left(1 + \frac{1}{\tau-1}\right)\left\|\sum\limits_{t=1}^{l-1}\xi_t\right\|^2 + \tau\left\|\EE_l[\xi_l]\right\|^2 + \EE_l\left[\|\xi_l - \EE_l[\xi_l]\|^2\right].
	\end{eqnarray*}
	Taking full mathematical expectation and using tower property \eqref{eq:tower_property} we derive
	\begin{equation*}
		\EE\left[\left\|\sum\limits_{t=1}^l\xi_t\right\|^2\right] \le \left(1 + \frac{1}{\tau-1}\right)\EE\left[\left\|\sum\limits_{t=1}^{l-1}\xi_t\right\|^2\right] + \tau\EE\left[\left\|\EE_l[\xi_l]\right\|^2\right] + \EE\left[\|\xi_l - \EE_l[\xi_l]\|^2\right]
	\end{equation*}
	for all $l=1,\ldots,\tau$. Unrolling the recurrence for $\EE\left[\left\|\sum\limits_{t=1}^l\xi_t\right\|^2\right]$ we obtain
	\begin{eqnarray*}
		\EE\left[\left\|\sum\limits_{t=1}^\tau\xi_t\right\|^2\right] &\le& \tau\sum\limits_{t=1}^\tau \left(1 + \frac{1}{\tau-1}\right)^{\tau-t}\EE\left[\left\|\EE_t[\xi_{t}]\right\|^2\right] + \sum\limits_{t=1}^\tau \left(1 + \frac{1}{\tau-1}\right)^{\tau-t}\EE\left[\left\|\xi_t-\EE_t[\xi_{t}]\right\|^2\right].
	\end{eqnarray*}
	Since $\left(1 + \frac{1}{\tau-1}\right)^{\tau-t} \le \left(1 + \frac{1}{\tau-1}\right)^{\tau-1} \le e$ for all $t=1,\ldots,\tau$ we get \eqref{eq:lemma14_stich}.
\end{proof}

\begin{lemma}[see also Lemma 2 from \cite{stich2019unified}]\label{lem:lemma2_stich}
	Let $\{r_k\}_{k\ge 0}$ satisfy
	\begin{equation}
		r_K \le \frac{a}{\gamma W_K} + c_1\gamma + c_2\gamma^2 \label{eq:lemma2_stich_tech_1}
	\end{equation}
	for all $K\ge 0$ with some constants $a, c_2\ge 0$, $c_1 \ge 0$ where $\{w_k\}_{k\ge 0}$ and $\{W_K\}_{K\ge 0}$ are defined in \eqref{eq:w_k_definition}, $\gamma \le \frac{1}{h}$. Then for all $K$ such that
	\begin{eqnarray*}
		\text{either} && \frac{\ln\left(\max\{2,\min\{\nicefrac{a\mu^2K^2}{c_1},\nicefrac{a\mu^3K^3}{c_2}\}\}\right)}{K}\le \rho\\
		\text{or} && \frac{1}{h}\le \frac{\ln\left(\max\{2,\min\{\nicefrac{a\mu^2K^2}{c_1},\nicefrac{a\mu^3K^3}{c_2}\}\}\right)}{\mu K}
	\end{eqnarray*}		
	 and
	\begin{equation}
		\gamma = \min\left\{\frac{1}{h}, \frac{\ln\left(\max\{2,\min\{\nicefrac{a\mu^2K^2}{c_1},\nicefrac{a\mu^3K^3}{c_2}\}\}\right)}{\mu K}\right\} \label{eq:lemma2_stich_gamma}
	\end{equation}
	we have that
	\begin{equation}
		r_K = \widetilde\cO\left(ha\exp\left(-\min\left\{\frac{\mu}{h}, \rho\right\}K\right) + \frac{c_1}{\mu K} + \frac{c_2}{\mu^2 K^2}\right). \label{eq:lemma2_stich}
	\end{equation}
\end{lemma}
\begin{proof}
	Since $W_K \ge w_K = (1-\eta)^{-(K+1)}$ we have
	\begin{eqnarray}
		r_K &\le& (1-\eta)^{K+1}\frac{a}{\gamma} + c_1\gamma + c_2\gamma^2 \le \frac{a}{\gamma}\exp\left(-\eta(K+1)\right) + c_1\gamma + c_2\gamma^2.\label{eq:lemma2_stich_tech_2}
	\end{eqnarray}
	Next we consider two possible situations.
	\begin{enumerate}
		\item If $\frac{1}{h} \ge \frac{\ln\left(\max\{2,\min\{\nicefrac{a\mu^2K^2}{c_1},\nicefrac{a\mu^3K^3}{c_2}\}\}\right)}{\mu K}$ then we choose $\gamma = \frac{\ln\left(\max\{2,\min\{\nicefrac{a\mu^2K^2}{c_1},\nicefrac{a\mu^3K^3}{c_2}\}\}\right)}{\mu K}$ and get that
		\begin{eqnarray*}
			r_K &\overset{\eqref{eq:lemma2_stich_tech_2}}{\le}& \frac{a}{\gamma}\exp\left(-\eta(K+1)\right) + c_1\gamma + c_2\gamma^2 \\
			&=& \widetilde\cO\left(a\mu K\exp\left(-\min\left\{\rho, \frac{\ln\left(\max\{2,\min\{\nicefrac{a\mu^2K^2}{c_1},\nicefrac{a\mu^3K^3}{c_2}\}\}\right)}{K}\right\}K\right)\right) \\
			&&\quad + \widetilde\cO\left(\frac{c_1}{\mu K} + \frac{c_2}{\mu^2 K^2}\right).
		\end{eqnarray*}
		Since $\frac{\ln\left(\max\{2,\min\{\nicefrac{a\mu^2K^2}{c_1},\nicefrac{a\mu^3K^3}{c_2}\}\}\right)}{K}\le \rho$ we have
		\begin{eqnarray*}
			r_K &=& \widetilde\cO\left(a\mu K\exp\left(-\ln\left(\max\left\{2,\min\left\{\frac{a\mu^2K^2}{c_1},\frac{a\mu^3K^3}{c_2}\right\}\right\}\right)\right)\right)\\
			&&\quad + \widetilde\cO\left(\frac{c_1}{\mu K} + \frac{c_2}{\mu^2 K^2}\right)\\
			&=& \widetilde\cO\left(\frac{c_1}{\mu K} + \frac{c_2}{\mu^2 K^2}\right).
		\end{eqnarray*}
		\item If $\frac{1}{h} \le \frac{\ln\left(\max\{2,\min\{\nicefrac{a\mu^2K^2}{c_1},\nicefrac{a\mu^3K^3}{c_2}\}\}\right)}{\mu K}$ then we choose $\gamma = \frac{1}{h}$ which implies that
		\begin{eqnarray*}
			r_K &\overset{\eqref{eq:lemma2_stich_tech_2}}{\le}& ha\exp\left(-\min\left\{\frac{\mu}{h},\frac{\rho}{4}\right\}(K+1)\right) + \frac{c_1}{h} + \frac{c_2}{h^2} \\
			&=& \widetilde\cO\left(ha\exp\left(-\min\left\{\frac{\mu}{h}, \rho\right\}K\right) + \frac{c_1}{\mu K} + \frac{c_2}{\mu^2K^2}\right). 
		\end{eqnarray*}
	\end{enumerate}
	Combining the obtained bounds we get the result. 
\end{proof}

\begin{lemma}\label{lem:lemma_technical_cvx}
	Let $\{r_k\}_{k\ge 0}$ satisfy
	\begin{equation}
		r_K \le \frac{a}{\gamma K} + \frac{b_1\gamma}{K} + \frac{b_2\gamma^2}{K} + c_1\gamma + c_2\gamma^2 \label{eq:lemma_technical_cvx_1}
	\end{equation}
	for all $K\ge 0$ with some constants $a> 0$, $b_1, b_2, c_1, c_2 \ge 0$ where $\gamma \le \gamma_0$. Then for all $K$ and
	\begin{equation*}
		\gamma = \min\left\{\gamma_0, \sqrt{\frac{a}{b_1}}, \sqrt[3]{\frac{a}{b_2}}, \sqrt{\frac{a}{c_1 K}}, \sqrt[3]{\frac{a}{c_2 K}}\right\}
	\end{equation*}
	we have that
	\begin{equation}
		r_K = \cO\left(\frac{a}{\gamma_0 K} + \frac{\sqrt{ab_1}}{K} + \frac{\sqrt[3]{a^2b_2}}{K} + \sqrt{\frac{ac_1}{K}} + \frac{\sqrt[3]{a^2c_2}}{K^{\nicefrac{2}{3}}} \right). \label{eq:lemma_technical_cvx_2}
	\end{equation}
\end{lemma}
\begin{proof}
	We have
	\begin{eqnarray*}
		r_K &\le& \frac{a}{\gamma K} + \frac{b_1\gamma}{K} + \frac{b_2\gamma^2}{K} + c_1\gamma + c_2\gamma^2\\
		&\le& \frac{a}{\min\left\{\gamma_0, \sqrt{\frac{a}{b_1}}, \sqrt[3]{\frac{a}{b_2}}, \sqrt{\frac{a}{c_1 K}}, \sqrt[3]{\frac{a}{c_2 K}}\right\}K} + \frac{b_1}{K}\cdot\sqrt{\frac{a}{b_1}} + \frac{b_2}{K}\cdot\sqrt[3]{\frac{a^2}{b_2^2}} + c_1\cdot\sqrt{\frac{a}{c_1 K}} + c_2 \left(\sqrt[3]{\frac{a}{c_2 K}}\right)^2\\
		 &=& \cO\left(\frac{a}{\gamma_0 K} + \frac{\sqrt{ab_1}}{K} + \frac{\sqrt[3]{a^2b_2}}{K} + \sqrt{\frac{ac_1}{K}} + \frac{\sqrt[3]{a^2c_2}}{K^{\nicefrac{2}{3}}} \right).
	\end{eqnarray*}
\end{proof}

\newpage
\begin{table*}[!t]
\caption{The parameters for which the methods from Table~\ref{tbl:special_cases} satisfy Assumption~\ref{ass:key_assumption}/\ref{ass:hetero_second_moment}. Absolute  constants were omitted. The meaning of the expressions appearing in the table, as well as their justification, is detailed in Section~\ref{sec:special_cases}. UBV stands for the ``Uniform Bound on the Variance'' of local stochastic gradient, which is often assumed when $f_i$ is of the form~\eqref{eq:f_i_expectation}. ES stands for the ``Expected Smoothness'' inequality~\cite{gower2019sgd}, which does not impose any extra assumption on the objective/noise, but rather can be derived given the sampling strategy and the smoothness structure of $f_i$. Consequently, such a setup allows us to obtain local methods with importance sampling. Next, the simple setting is a special case of ES when we uniformly sample a single index on each node each iteration.}
\label{tbl:special_cases-parameters}
\begin{center}
\footnotesize
\begin{adjustbox}{angle=90}
\begin{tabular}{|c|c|c|c|c|c|c|c|}
\hline
 Method, Setting &   $A$, $\tA$, $\hA$, $A'$ & $B$, $\tB$, $\hB$, $B'$ & $\rho$ & $C$ & $F$, $\tF$, $\hF$, $F'$ & $G$ & $D_1'$, $D_1$, $\tD_1$, $\hD_1$, $D_2$, $D_3$ \\
\hline
 \begin{tabular}{c}
 	{\tt Local-SGD}\\ 
 	UBV, $\zeta$-Het.
 \end{tabular}   &  $L$, $-$, $-$, $L$ & $0$, $-$, $-$, $0$ & $1$ & $0$ & $L^2$, $-$, $-$, $L^2$ & $0$ & \makecell{$\frac{\sigma^2}{n}$, $\sigma^2+\zeta_*^2$, $-$, $-$, $0$,\\ $\tau\sigma^2+\tau^2\zeta^2$}\\
 %%%%%%%%%%%%%%%%%%%%
 %%%%%%%%%%%%%%%%%%%%
 \hline
 \begin{tabular}{c}
 	{\tt Local-SGD}\\ 
 	UBV, Het.
 \end{tabular}   &  $-$, $L$, $0$, $L$ & $-$, $0$, $0$, $0$ & $1$ & $0$ & $-$, $L^2$, $0$, $L^2$ & $0$ & \makecell{$\frac{\sigma^2}{n}$, $-$, $\zeta_*^2$, $\sigma^2$, $0$,\\ $(\tau-1)\sigma^2+(\tau-1)^2\zeta_*^2$}\\
 %%%%%%%%%%%%%%%%%%%%
 %%%%%%%%%%%%%%%%%%%%
 \hline
 \begin{tabular}{c}
 	{\tt Local-SGD}\\ 
 	ES, $\zeta$-Het.
 \end{tabular}   &  $\cL$, $-$, $-$, $\frac{\cL}{n}+L$ & $0$, $-$, $-$, $0$ & $1$ & $0$ & $\cL L$, $-$, $-$, $\frac{\cL L}{n}+L^2$ & $0$ & \makecell{$\frac{\sigma_*^2}{n}$, $\sigma_*^2 + \zeta_*^2$, $-$, $-$, $0$,\\ $(\tau-1)\left(\sigma_*^2+\zeta_*^2+\frac{\zeta^2}{\gamma\mu}\right)$}\\
 %%%%%%%%%%%%%%%%%%%%
 %%%%%%%%%%%%%%%%%%%%
 \hline
 \begin{tabular}{c}
 	{\tt Local-SGD}\\ 
 	ES, Het.
 \end{tabular}   &  $-$, $L$, $\cL$, $\frac{\cL}{n}+L$ & $-$, $0$, $0$, $0$ & $1$ & $0$ & $-$, $L^2$, $\cL L$, $\frac{\cL L}{n}+L^2$ & $0$ & \makecell{$\frac{\sigma_*^2}{n}$, $-$, $\zeta_*^2$, $\sigma_*^2$, $0$,\\ $(\tau-1)\sigma_*^2+(\tau-1)^2\zeta_*^2$}\\
 %%%%%%%%%%%%%%%%%%%%
 %%%%%%%%%%%%%%%%%%%%
 \hline
 \begin{tabular}{c}
	{\tt Local-SVRG}\\
	simple, $\zeta$-Het.  
\end{tabular}   &  \makecell{$\max L_{ij}$, $-$, $-$,\\ $\frac{\max L_{ij}}{n}+L$} & $1$, $-$, $-$, $\frac1n$ & $\psvrg$ & $\max L_{ij}\psvrg$ & \makecell{$\max L_{ij} L$, $-$, $-$,\\ $\frac{\max L_{ij} L}{n} + L^2$} & $\max L_{ij} L q$ & \makecell{$0$, $\zeta_*^2$, $-$, $-$, $0$,\\ $(\tau-1)\left(\zeta_*^2 + \frac{\zeta^2}{\gamma\mu}\right)$}\\
 %%%%%%%%%%%%%%%%%%%%
 %%%%%%%%%%%%%%%%%%%%
 \hline
 \begin{tabular}{c}
	{\tt Local-SVRG}\\
	simple, Het.  
\end{tabular}   &  \makecell{$-$, $L$, $\max L_{ij}$,\\ $\frac{\max L_{ij}}{n}+L$} & $-$, $0$, $1$, $\frac1n$ & $\psvrg$ & $\max L_{ij}\psvrg$ & \makecell{$-$, $L^2$, $\max L_{ij} L$,\\ $\frac{\max L_{ij} L}{n} + L^2$} & $\max L_{ij} L q$ & \makecell{$0$, $-$, $\zeta_*^2$, $0$, $0$, $(\tau-1)^2\zeta_*^2$}\\
 %%%%%%%%%%%%%%%%%%%%
 %%%%%%%%%%%%%%%%%%%%
 \hline
  \begin{tabular}{c}
  	{\tt S*-Local-SGD}\\
  	UBV, Het.
	\end{tabular}    &  $-$, $L$, $0$, $L$ & $-$, $0$, $0$, $0$ & $1$ & $0$ & $-$, $L^2$, $0$, $l^2$ & $0$ & $\frac{\sigma^2}{n}$, $-$, $0$, $\sigma^2$, $(\tau-1)\sigma^2$\\
 %%%%%%%%%%%%%%%%%%%%
 %%%%%%%%%%%%%%%%%%%%
 \hline
 \begin{tabular}{c}
 	{\tt SS-Local-SGD}\\
 	UBV, Het.,\\
 	$p=q$, $r=\lceil\nicefrac{1}{p} \rceil$
\end{tabular}    & $-$, $L$, $0$, $L$ & $-$, $1$, $0$, $0$ & $p$ & $Lp$ & $-$, $L^2$, $0$, $L^2$ & $0$ & \makecell{$\frac{\sigma^2}{n}$, $-$, $p\sigma^2$, $\sigma^2$, $0$, $\frac{(1-p)\sigma^2}{p}$ }  \\
 %%%%%%%%%%%%%%%%%%%%
 %%%%%%%%%%%%%%%%%%%%
 \hline
 \begin{tabular}{c}
 	{\tt SS-Local-SGD}\\
 	ES, Het.,\\
 	$p=q$, $r=\lceil\nicefrac{1}{p} \rceil$
\end{tabular}    & $-$, $L$, $\cL$, $\frac{\cL}{n}+L$ & $-$, $1$, $0$, $0$ & $p$ & $Lp+\cL p^2$ & $-$, $L^2$, $\cL L$, $\frac{\cL L}{n}+L^2$ & $0$ & \makecell{$\frac{\sigma_*^2}{n}$, $-$, $0$, $\sigma_*^2$, $p^2\sigma_*^2$, $\frac{(1-p)\sigma_*^2}{p}$ }  \\
 %%%%%%%%%%%%%%%%%%%%
 %%%%%%%%%%%%%%%%%%%%
 \hline
 \begin{tabular}{c}
	{\tt S*-Local-SGD*}\\
	simple, Het. 
\end{tabular}    &  \makecell{$-$, $L$, $\max L_{ij}$,\\ $\frac{\max L_{ij}}{n}+L$ }& $-$, $0$, $0$, $0$ & $p$ & $0$ & \makecell{$-$, $L^2$, $\max L_{ij} L$,\\ $\frac{L\max L_{ij}}{n}+L^2$} & $0$ & $0$, $-$, $0$, $0$, $0$, $0$ \\
 %%%%%%%%%%%%%%%%%%%%
 %%%%%%%%%%%%%%%%%%%%
 \hline
 \begin{tabular}{c}
	{\tt S-Local-SVRG}\\
	simple, Het.,\\
	$q = \frac{1}{m}$, $m \ge \frac{1}{p}$
\end{tabular}    &  \makecell{$-$, $L$, $\max L_{ij}$,\\ $\frac{\max L_{ij}}{n}+L$ }& $-$, $1$, $1$, $\frac{1}{n}$ & $\frac{1}{m}$ & $\frac{L+\max L_{ij}}{m}$ & \makecell{$-$, $L^2$, $\max L_{ij} L$,\\ $\frac{L\max L_{ij}}{n}+L^2$} & $0$ & $0$, $-$, $0$, $0$, $0$, $0$ \\
 %%%%%%%%%%%%%%%%%%%%
 %%%%%%%%%%%%%%%%%%%%
 \hline
\end{tabular}
\end{adjustbox}
\end{center}
\end{table*}

\end{document}